\documentclass[10pt,journal,compsoc]{IEEEtran}
  \usepackage{graphicx}
  \usepackage{epsfig}
  \usepackage{algorithm}
  \usepackage{algorithmicx}
  \usepackage{algpseudocode}
  \usepackage{url}
  \usepackage{amsmath}
  \usepackage{amssymb}
  \usepackage{caption}
  \usepackage{amssymb}
  \usepackage{amsthm}
  \usepackage{subfig}
  \usepackage{enumerate}
  \usepackage{cleveref}
  \usepackage{lineno}
  \usepackage{color}
  \usepackage{multirow}
  \usepackage{footnote}
  \usepackage{epstopdf}

  \newtheorem{myPro}{Property}
  \newtheorem{myLem}{Lemma}
  \newtheorem{myDef}{Definition}
  \newtheorem{myTheo}{Theorem}

  \DeclareMathOperator*{\argmin}{argmin}
  
  \crefname{equation}{Eq.}{Eqs.}
  \crefname{figure}{Fig.}{Figs.}

\ifCLASSINFOpdf
\else
\fi
\begin{document}
%
\title{Global Optimal Path-Based Clustering Algorithm}
%
%
%
%

\author{Qidong~Liu and Ruisheng~Zhang,~\IEEEmembership{Member,~IEEE}
\IEEEcompsocitemizethanks{\IEEEcompsocthanksitem Q. Liu is with School of Electrical and Electronic Engineering, Nanyang Technological University, Singapore, and with School of Information Engineering, Zhengzhou University, China E-mail: qidong.liu@ntu.edu.sg; ieqdliu@zzu.edu.cn.
\IEEEcompsocthanksitem R. Zhang is with School of Information Science and Engineering, Lanzhou University, China E-mail: zhangrs@lzu.edu.cn.}
}

\IEEEtitleabstractindextext{%
\begin{abstract}
Combinatorial optimization problems for clustering are known to be NP-hard. Most optimization methods are not able to find the global optimum solution for all datasets. To solve this problem, we propose a global optimal path-based clustering (GOPC) algorithm in this paper. The GOPC algorithm is based on two facts: (1) medoids have the minimum degree in their clusters; (2) the minimax distance between two objects in one cluster is smaller than the minimax distance between objects in different clusters. Extensive experiments are conducted on synthetic and real-world datasets to evaluate the performance of the GOPC algorithm. The results on synthetic datasets show that the GOPC algorithm can recognize all kinds of clusters regardless of their shapes, sizes, or densities. Experimental results on real-world datasets demonstrate the effectiveness and efficiency of the GOPC algorithm. In addition, the GOPC algorithm needs only one parameter, i.e., the number of clusters, which can be estimated by the decision graph. The advantages mentioned above make GOPC a good candidate as a general clustering algorithm. Codes are available at https://github.com/Qidong-Liu/Clustering.
\end{abstract}

}

\maketitle

\IEEEdisplaynontitleabstractindextext

%
\IEEEpeerreviewmaketitle

\IEEEraisesectionheading{\section{Introduction}\label{sec:introduction}}

%
%
%
%
\IEEEPARstart{C}{lustering} algorithms classify elements into categories, or clusters, on the basis of their similarity or distance~\cite{frey2007clustering}. As an important topic in exploratory data analysis and pattern recognition, many clustering algorithms have been proposed, such as $k$-means~\cite{macqueen1967some}, spectral clustering~\cite{ng2002spectral}, density based spatial clustering of applications with noise (DBSCAN)~\cite{ester1996density}, non-negative matrix factorization-based methods~\cite{Ding2010Convex}, etc.

Well-separated clusters mean that objects in the same group are more similar to each other than to those in different groups~\cite{steinbach2004challenges,jain1988algorithms,Wang2016Automatic}. In clustering algorithms, measuring the dissimilarity between any pair of points is very important. The most commonly used dissimilarity method is Euclidean distance. However, in many real-world applications of pattern classification and data mining, we are often confronted with high-dimensional features of the investigated data, which adversely affects clustering performance due to the curse of dimensionality~\cite{Zhao2017Robust,Hou2017Discriminative}. It is widely acknowledged that many real-world datasets stringently obey low-rank rules, which means that they are distributed on a manifold of a dimensionality that is often much lower than that of ambient space~\cite{Tu2014A,Roweis2000Nonlinear,Yang2017Discrete}. In low dimensional manifold data space, dissimilarity between two objects is established not only by direct comparison, but can be induced by the mediating objects between them. Based on the minimax distance~\cite{zahn1971graph,Chehreghani2017Efficient}, we propose a method that captures the manifold structure of low dimensional data in this study. The minimax distance can be described as follows:

Let $P_{i,j}$ denote the set of all paths from vertex $i$ to vertex $j$. For each path $p\in P_{i,j}$, if the edge weight is calculated by a dissimilarity measure (e.g., Euclidean distance), then the effective distance $d_{i,j}^p$ is the maximum edge weight along the path, i.e., 
\begin{equation}
d_{i,j}^p=\max \limits_{1\leq h<|p|}{e_{p[h],p[h+1]}}
\end{equation}
where $p[h]$ denotes the $h$th vertex along path $p$ and $e_{x,y}$ denotes the edge weight between $x$ and $y$. The minimax distance $d_{i,j}$ is the minimum $d_{i,j}^p$ between all paths, i.e.,
\begin{equation} \label{mms}
d_{i,j}=\min \limits_{p\in P_{i,j}}{d_{i,j}^p}.
\end{equation}
On the other hand, if the edge weight is calculated using a similarity measure (e.g., cosine), then
\[
d_{i,j}^p=\min \limits_{1\leq h<|p|}{e_{p[h],p[h+1]}}
\]
and
\[
d_{i,j}=\max \limits_{p\in P_{i,j}}{d_{i,j}^p}.
\]
Without loss of generality, we assume the edge weight is calculated using a dissimilarity measure for the remainder of this paper.


The unique path in the minimum spanning tree (MST) for the whole dataset from vertex $i$ to vertex $j$ is a minimax path from $i$ to $j$~\cite{zahn1971graph}. Thus, the minimax distance between any pair of objects can be computed by the MST. These kinds of algorithms are also known as MST-based clustering algorithms.

MST-based clustering algorithms begin by constructing an MST over a given weighted graph, and then an edge inconsistency measure partitions the tree into clusters~\cite{wang2013enhancing}. There are two well-known problems with MST-based clustering algorithms. First, two connected clusters may be regarded as two parts of one cluster. Second, a few objects that are far away from all other objects define a separate cluster. For the first problem, Chang and Yeung proposed a robust minimax distance that was also based on the MST but took the local density into account. The new dissimilarity was robust to noises and reflected the genuine distance. For the second problem, Fischer and Buhmann proposed path-based clustering by describing an agglomerative optimization method for the objective function, defined in Refs.~\cite{fischer2003path,fischer2003bagging}. Their method solves the second problem because the objective function sums the average dissimilarity per cluster weighted with the number of objects in that cluster. However, their method may fall into a local optimum. In this paper, we propose a global optimal path-based clustering (GOPC) algorithm that gives the global optimal solution for the objective function defined in \cref{criterion}.  

The remaining topics in this paper are organized as follows. Some related work is briefly reviewed in Section \ref{rw}. In Section \ref{sec2}, our algorithm is described in detail. Experimental results on synthetic datasets as well as real-world datasets are presented in Section \ref{sec3}. In Section \ref{is} we apply our algorithm to the domain of image segmentation. Finally, some concluding remarks are given in the last section.

\section{Related work}
\label{rw}
MST-based clustering was first proposed by Zahn~\cite{zahn1971graph}. In his works, the method first constructed an MST over the entire dataset and then proceeded to remove inconsistent edges. In most cases, the inconsistent edges were simply the longest edges (i.e., edges with large weight were the priority for removal). However, not all edges with large weight were inter-cluster edges. Outliers in the dataset made the longest edges an unreliable indication of cluster separation. For this problem, Chowdbury and Murthy~\cite{chowdhury1997minimal} proposed a new inconsistency measure based on finding valley regions. In Ref.~\cite{laszlo2005minimum}, Laszlo and Mukherjee put a constraint on the minimum cluster size. They considered an edge to be an inter-cluster edge only when edge removal resulted in two clusters with sizes both larger than the minimum cluster size. To be less sensitive to the outliers, Zhong et al. proposed a two-round MST-based clustering algorithm~\cite{zhong2010graph}. However, this algorithm was complicated because it needed too many parameters. After that, Zhong et al. proposed a split-and-merger hierarchical clustering method~\cite{zhong2011minimum}. In their work, MST was employed to guide the splitting and merging process. The drawback of this algorithm was its high time complexity. Computational efficiency is a major issue for large databases. Thus, Wang et al. proposed a fast MST-inspired clustering algorithm~\cite{wang2009divide} that performed much better than $O(n^2)$.

In addition to the algorithms mentioned above that classify the dataset into groups by removing inconsistent edges from MST, Fischer and Buhmann proposed path-based clustering by describing an agglomerative optimization method for the objective function defined in Refs.~\cite{fischer2003path,fischer2003bagging}. However, their algorithm could not guarantee the global optimal solution for all datasets~\cite{LIU201969}. Chang and Yeung proposed an algorithm to combine spectral clustering and path-based clustering (here denoted as RPSC)~\cite{chang2008robust}. Spectral clustering is a recently popular clustering algorithm that has demonstrated excellent performance on some clustering tasks involving highly non-spherical clusters~\cite{ng2002spectral}. The accuracy of such methods depends on the affinity matrix~\cite{jordan2004learning}. Most spectral clustering methods adopt the Gaussian kernel function as the similarity method~\cite{zhang2011local}. Nevertheless, computing the eigenvectors of the affinity matrix generally has a computational complexity of $O(n^3)$ that makes spectral clustering infeasible for large datasets, and the same is true for RPSC. 

Clustering is an unsupervised learning problem, as it classifies a dataset without using any prior knowledge~\cite{zhong2011minimum}. Thus, the number of parameters of a good algorithm should be as few as possible. The GOPC algorithm requires only one parameter (i.e., the number of clusters). The number of clusters is either given as an input parameter or computed by the algorithms themselves. Therefore, we also propose a method to estimate the number of clusters, like the method proposed by Alex and Alessandro~\cite{rodriguez2014clustering}. In their algorithm, the number of clusters arises intuitively through the decision graph.

\section{Proposed method}
\label{sec2}
Table~\ref{tab} lists the mathematical notation used in this paper.
\newcommand{\tabincell}[2]{\begin{tabular}{@{}#1@{}}#2\end{tabular}}
\begin{table}
	\centering
	\caption{Mathematical Notation}
	\label{tab}
	\begin{tabular}{c c}
		\hline
		Notation & Description \\
		\hline
		$k$ & The number of clusters.\\
		$C_t$ & The $t$th cluster.\\
		$m_t$ & The medoid of $C_t$.\\
		$U$ & $U=\{x_1,x_2,\cdots,x_n\}$. The whole dataset.\\
		$M$ & $M=\{m_1,m_2,\cdots,m_k\}$. All medoids.\\
		$x$ & $x\in U$. One of the objects in $U$.\\
		$\tau_t(x)$ & The nearest medoid to object $x$ in the $t$th epoch. \\
		\hline
	\end{tabular}
\end{table}
\subsection{Problem formulation}
\label{subsec2.1} 

Given a set of objects $U=\{x_1,x_2,\cdots,x_n\}$ that is characterized by the minimax distance, the goal of clustering is to find a mapping $\Phi: U\rightarrow \{1,2,\cdots,k\}$ that assigns each object to one of the $k$ groups. To measure the quality of the clustering solutions, the objective function in \cref{criterion} is used:

\begin{equation} \label{criterion}
E=\min \limits_{(m_1,m_2,\ldots,m_k)}\sum\limits_{t=1}\limits^k \sum\limits_{x\in C_t}d_{x,m_t}
\end{equation}
where $C_t$ denotes the $t$th cluster, $m_t$ is one of the objects in $C_t$, and $d_{x,m_t}$ denotes the minimax distance between object $x$ and $m_t$. Our method aims to find a set of objects $M=(m_1,m_2,\cdots,m_k)$ by which we can obtain the minimum value of $E$. Also, these objects are known as medoids.

\begin{myDef} \label{def2}
	A medoid is the object in a cluster whose average distance to all the objects in the cluster is minimal.
\end{myDef}

After obtaining the medoids, each object is assigned to its nearest medoid. If there is an object $o_i$ that has the same minimum distance to more than one medoid, assigning $o_i$ to any one of the medoids does not affect the value of $E$ in \cref{criterion}. Therefore, in the rest of the paper, we assume that an object $x$ belongs to a cluster $C_t$ only when it has the unique minimum distance to the corresponding medoid $m_t$. As for objects like $o_i$, they are temporarily considered as noise, and we will discuss how to deal with them in detail in Subsection~\ref{subsec3}. The entire clustering procedure is actually a $k$-medoids problem with respect to the minimax distance. Our contribution in this paper is to propose a method to obtain the global optimal solution.

\subsection{Properties of medoids}
\label{subsec2.2} 
For the clustering problem, any two objects have a pairwise distance. Thus, for $\forall z\in U$, the paths $x\to z$ and $z\to y$ can form a path from $x$ to $y$. According to the definition of minimax distance defined in \cref{mms}, the minimax distance is the minimum effective distance among all paths. Therefore, the minimax distance is an $ultrametric\; space$ because it satisfies the ultrametric property (\Cref{pom}) introduced by Hartigan \cite{Hartigan1967Representation} and Johnson \cite{Johnson1967Hierarchical}.

\begin{myPro} \label{pom}
	For $\forall(x,y,z)\in U$, we have $d_{x,y}\leq \max(d_{x,z},d_{y,z})$
\end{myPro}

The medoid is the most centrally located object in a cluster. \Cref{pom} makes the medoid special. Next we introduce an important lemma based on \Cref{pom}.

\begin{figure}
	\centering
	\includegraphics[width=0.4\textwidth]{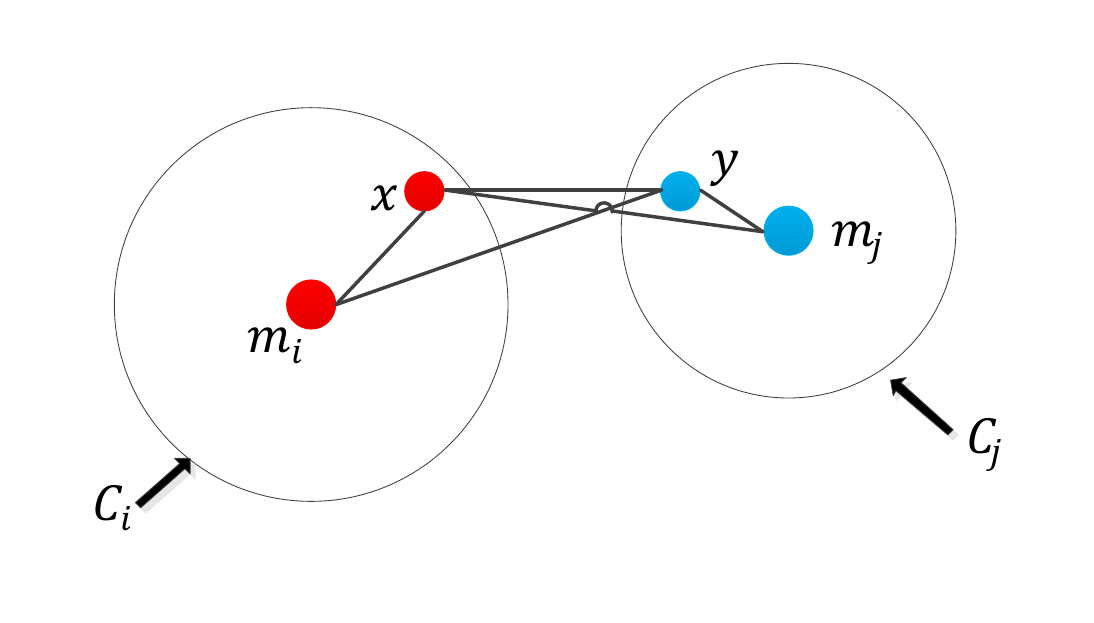}
	\caption{An example for \Cref{lem1}: $m_i$ and $m_j$ are the medoids of $C_i$ and $C_j$, respectively.}
	\label{fig1}
\end{figure}

\begin{myLem} \label{lem1}
	Given a cluster $C_i$ and its corresponding medoid $m_i$, $\forall x\in C_i$, $\forall y\notin C_i$, then
	\begin{equation} \label{eq3}
	d_{y, m_i}=d_{y, x}
	\end{equation}
	
	\begin{proof}
		Suppose $y\in C_j$, and $m_j$ is the medoid (see \cref{fig1}). All objects are assigned to the nearest medoid. Thus,
		\begin{equation} \label{eq4}
		d_{x,m_i}<d_{x,m_j},
		\end{equation}
		and
		\begin{equation} \label{eq5}
		d_{y,m_j}<d_{y,m_i}.
		\end{equation}
		According to \Cref{pom}, we have
		\begin{equation} \label{eq6}
		d_{y,m_i}\leq \max (d_{x,m_i}, d_{y,x}), 
		\end{equation}
		\begin{equation} \label{eq7}
		d_{x,m_j}\leq \max (d_{y,x},d_{y,m_j}),
		\end{equation}
		\begin{equation} \label{eq8}
		d_{y,x}\leq \max (d_{y,m_i},d_{x,m_i}).
		\end{equation}
		The relationship between $d_{y,x}$ and $d_{x,m_i}$ has two possibilities: $d_{y,x}\leq d_{x,m_i}$ and $d_{y,x}>d_{x,m_i}$.
		\begin{enumerate}
			\item If $d_{y,x}\leq d_{x,m_i}$:
			
			According to \cref{eq4} and \cref{eq7}, we have
			\[
			d_{x,m_i} < d_{x,m_j}\leq \max(d_{y,x},d_{y,m_j}).
			\]
			The above equation denotes that at least one of $d_{y,x}$ and $d_{y,m_j}$ is larger than $d_{x,m_i}$. Because $d_{y,x}\leq d_{x,m_i}$,
			\begin{equation}
			\label{eq9}
			d_{x,m_i}< d_{y,m_j}.
			\end{equation}
			According to \cref{eq5} and \cref{eq6}, we have
			\[
			d_{y,m_j} < d_{y,m_i}\leq \max(d_{x,m_i},d_{y,x}).			
			\]
			Because $\max(d_{x,m_i},d_{y,x})=d_{x,m_i}$, we get $d_{y,m_j}<d_{x,m_i}$, which contradicts with \cref{eq9}.
			
			\item If $d_{y,x}>d_{x,m_i}$, according to \cref{eq6} and \cref{eq8}, we get $d_{y,m_i}\leq d_{y,x}$ and $d_{y,x}\leq d_{y,m_i}$, respectively, which implies that $d_{y,m_i}=d_{y,x}$.
		\end{enumerate}
		Therefore, for $\forall x\in C_i$, $\forall y\notin C_i$, $d_{y,m_i}=d_{y,x}$.
	\end{proof}
\end{myLem}

\begin{figure}
	\centering
	\includegraphics[width=0.45\textwidth]{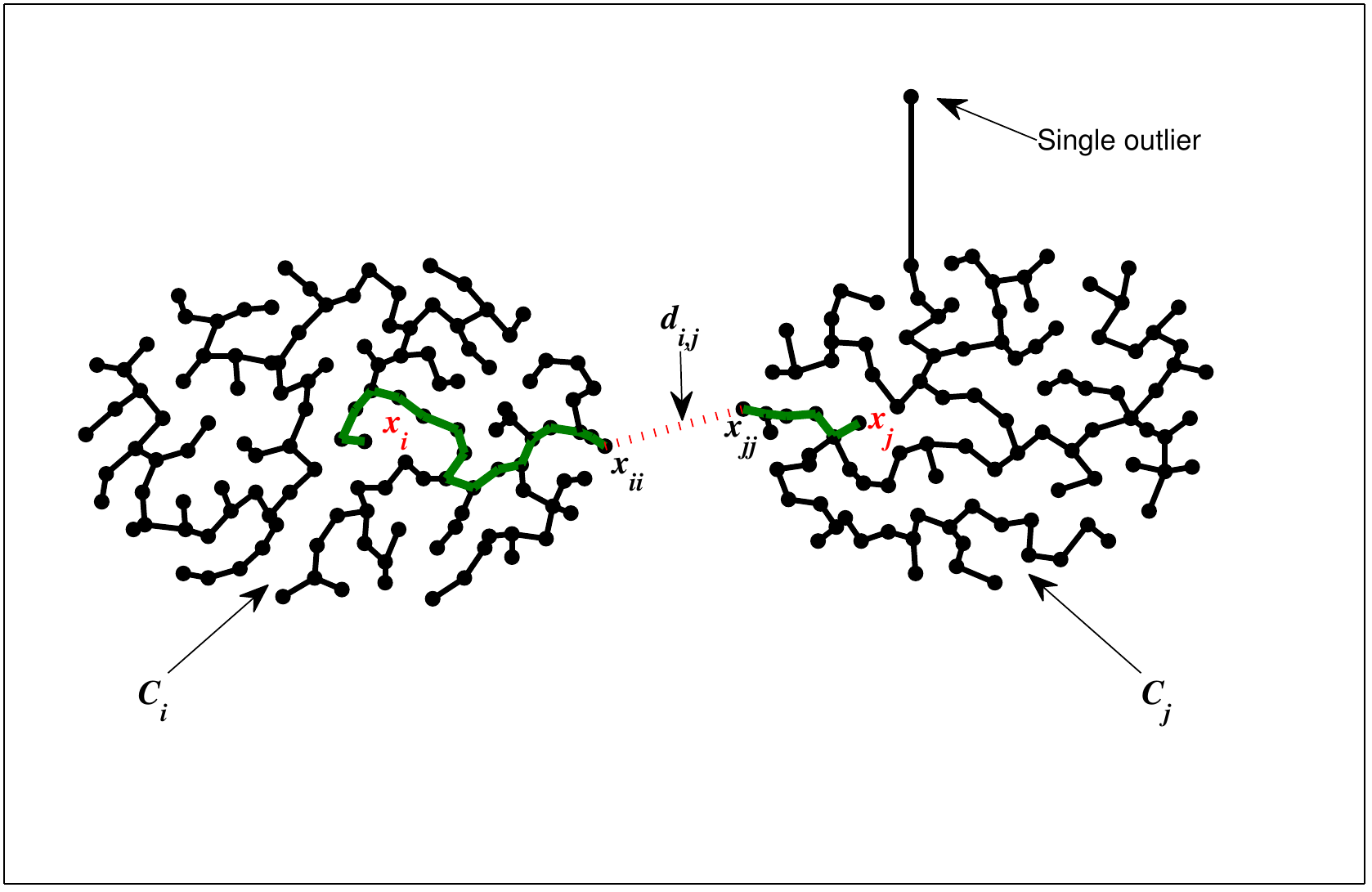}
	\caption{$C_i$ and $C_j$ are two separate clusters. The minimax similarity between $x_i$ and $x_j$ is the edge weight of $(x_{ii},x_{jj})$.}
	\label{mst}
\end{figure}

In the process of proving \Cref{lem1}, we find that $d_{y,x}$ is not only equal to $d_{y,m_i}$ but also larger than $d_{x,m_i}$. Next, we give a physical explanation for the above proof. Figure~\ref{mst} shows an example of a dataset in which the lines denote the MST graph. According to the theorem proposed by Zahn, the unique path in the MST from vertex $x_i$ to $x_j$ is a minimax path from $x_i$ to $x_j$ \cite{zahn1971graph}. Thus, the minimax path between $x_i$ and $x_j$ is the green path shown in \cref{mst}, and the minimax distance between $x_i$ and $x_j$ is the maximum distance along the minimax path, i.e., $d_{i,j}=d_{ii,jj}$. This indicates that, if the dissimilarity between $x_{ii}$ and $x_{jj}$ is large enough, there is a shorter minimax distance between two objects in one cluster than between either of the objects and any object in a different cluster.

Letting
\begin{equation}
\label{degree}
degree(x)=\sum \limits_{y\in U}d_{y,x}, 
\end{equation}
we get \Cref{P1}.
\begin{myTheo} \label{P1}
	Given a cluster $C_i$ and its corresponding medoid $m_i$, $\forall x\in C_i$, $x\neq m_i$, then $degree(m_i)\leq degree(x)$.
\end{myTheo}
\begin{proof}
	According to \Cref{def2}, a medoid is the object in a cluster whose average distance to all objects in the cluster is minimal. Thus, we have $p(m_i)\leq p(x)$, where
	\[
	p(x)=\sum_{y\in C_i}{d_{y,x}}.
	\]
	We have proved in \Cref{lem1} that for $\forall x\in C_i$, $\forall y\notin C_i$, then $ d_{y,m_i}=d_{y,x}$. In other words, letting 
	$
	q(x)=\sum_{y\notin C_i}▒d_{y,x},
	$
	$\forall x\in C_i$, $x\neq m_i$, we have $q(m_i)=q(x)$.
	
	According to \cref{degree},
	\[
	\begin{split}
	degree(x)& =\sum\limits_{y\in U}d_{y,x} \\
	& =\sum \limits_{y\in C_i}▒d_{y,x}+ \sum \limits_{y\notin C_i} d_{y,x} \\
	& =p(x)+q(x) \\
	\end{split}
	\]
	Because $p(m_i)\leq p(x)$ and $q(m_i)=q(x)$, we have
	\[
	p(m_i)+q(m_i)\leq p(x)+q(x),
	\]
	which implies that
	\[
	degree(m_i)\leq degree(x).
	\]
\end{proof}

\Cref{P1} indicates that medoids have the minimum degree in their clusters. If more than one object in a cluster has the minimum degree, any of them can be taken as the medoid of the cluster. Next, we discuss the relationship between medoids.
Letting
\begin{equation}
\label{nn}
nn(x)=\mathop{\argmin} \limits_{y\in S(x)}d_{y,x},
\end{equation}
where $S(x)=\{y|degree(y)<degree(x)\}$. \textbf{If there is more than one object in $S(x)$ having the minimum distance to $x$, $nn(x)$ will be the one with the minimum degree}. There is one theorem for $nn(x)$, shown as follows:
\begin{myTheo} \label{P2}
	If $x \in M$, then $nn(x) \in M$.
\end{myTheo}

\begin{proof}
	According to the definition of $nn(x)$, we obtain $degree(nn(x))<degree(x)$. Because $x\in M$, $x$ has the minimum degree in its cluster. As a result, $nn(x)$ and $x$ don't belong to one cluster. Suppose $nn(x)\notin M$ and $m$ is the medoid of the cluster that $nn(x)$ belongs to. We obtain 
	\begin{equation}
	\label{eq12}
	degree(m)\leq degree(nn(x)).
	\end{equation}
	\Cref{lem1} shows that any object outside a cluster has the same distance to objects in the cluster, so $d_{x,m}=d_{x,nn(x)}$. Object $nn(x)$ has the minimum degree among all the candidates that have the minimum distance to $x$, so $degree(nn(x))<degree(m)$, which contradicts \cref{eq12}. Thus, $nn(x)\in M$.
\end{proof}

\subsection{GOPC algorithm}
\label{subsec2.3}

The GOPC algorithm aims to find a set of medoids by which we can obtain the minimum value of $E$ for the objective function defined in \cref{criterion}. 
According to \Cref{P1}, the object with the minimum degree in the entire dataset must be one of the medoids. Therefore, we take this object to be the first medoid $m_1$. Next, we propose an iterative method to find the remaining $k-1$ medoids in turn. The GOPC algorithm is depicted in \Cref{gop}.
\floatname{algorithm}{Algorithm}
\renewcommand{\algorithmicrequire}{\textbf{Input:}}
\renewcommand{\algorithmicensure}{\textbf{Output:}}
\begin{algorithm}
	\caption{Global optimal path-based clustering (GOPC) algorithm}
	\begin{algorithmic}[1]
		\Require $U$ is the dataset; $k$ is the number of clusters.
		\Ensure $M$ denotes the set of all medoids.
		\State Calculate the minimax distance by Prim's algorithm.
		\For {each $x\in U$}
		\State Calculate $degree(x)$; // According to \cref{degree}.
		\EndFor
		\State $S\gets sort(degree)$ // sorted by degree in ascending order.
		\For {each $x\in U$}
		\State Calculate $nn(x)$. // According to \cref{nn}.
		\EndFor
		\State $m_1\gets s_1$; $M.add(m_1)$ // $s_i$ denotes the $i$th object in $S$.   
		\For {each $y\in U$}
		\State  $\tau_1(y)\gets m_1$; 
		\EndFor  
		\State $t\gets 2$;         
		\While {$t\leq k$}
		\For {each $x\in S$}
		\If {$nn(x)\notin M$}
		\State continue;
		\EndIf
		\State $r_t(x)\gets 0$; 
		\For {each $y\in U$}
		\State $r_t(x)\gets r_t(x)+\max (d_{y,\tau_{t-1}(y)}-d_{y,x}, 0)$; // According to \cref{eq18}
		\EndFor
		\EndFor
		\State Take the object with the maximum value $r_t(x)$ as the $t$th medoid $m_t$.
		\State $M.add(m_t)$;
		\For {each $y\in U$}
		\If {$d_{y,m_t}< d_{y,\tau_{t-1}(y)}$}
		\State $\tau_{t}(y)\gets m_t$;
		\Else 
		\State $\tau_{t}(y)\gets \tau_{t-1}(y)$
		\EndIf
		\EndFor
		\State $t\gets t+1$; 
		\EndWhile
	\end{algorithmic}
	\label{gop}
\end{algorithm}

For each epoch, the objects are visited according to their degree in ascending order. \Cref{P2} indicates that if $nn(x)\notin M$, then $x\notin M$. Thus, we check the status of $nn(x)$ for each object $x$, and only the objects that satisfy $nn(x)\in M$ are taken as candidates for the medoids. The above step can filter out a lot of unqualified objects and greatly reduce the number of candidates. Next, we introduce a method to mine the $t$th medoid $m_t$ from the candidates. Let
\begin{equation} \label{eq17}
\Omega _t(x)={\{y|d_{y,x}< d_{y,\tau_{t-1}(y)}\} },
\end{equation}
and
\begin{equation} \label{eq18}
r_t(x)=\sum\limits_{y\in \Omega _t(x)}{(d_{y,\tau_{t-1}(y)}-d_{y,x})}, 
\end{equation}
where $\tau_{t-1}(y)$ denotes the nearest medoid to object $y$ in the last epoch. $\Omega_t(x)$ and $r_t(x)$ respectively denote the set of objects that are assigned to $x$ and the changes in $E$ if $x$ is selected as the new medoid in the $t$th epoch. Medoids should be the influential nodes in the dataset, so the object with the maximum $r_t(x)$ will be taken to be $m_t$ in the $t$th epoch. We will prove the correctness of the above hypothesis by \Cref{P3}.

\textbf{Note:} $M$ is a set of medoids that can achieve the global optimum for the objective function in \cref{criterion}. If there are two or more objects that have the maximum $r_t(x)$ in each epoch, the one selected will not change the objective value. Therefore, any combination of objects belong to $M$ as long as it minimizes the value of the objective function, which implies that if $x\notin M$, letting $M'=\{M\backslash {m_t}\}\cup{x}$ and $\Phi(M)$ denotes the objective value of $M$, we obtain $\Phi(M)-\Phi(M')<0$.

\begin{figure}
	\centering
	\includegraphics[width=0.4\textwidth]{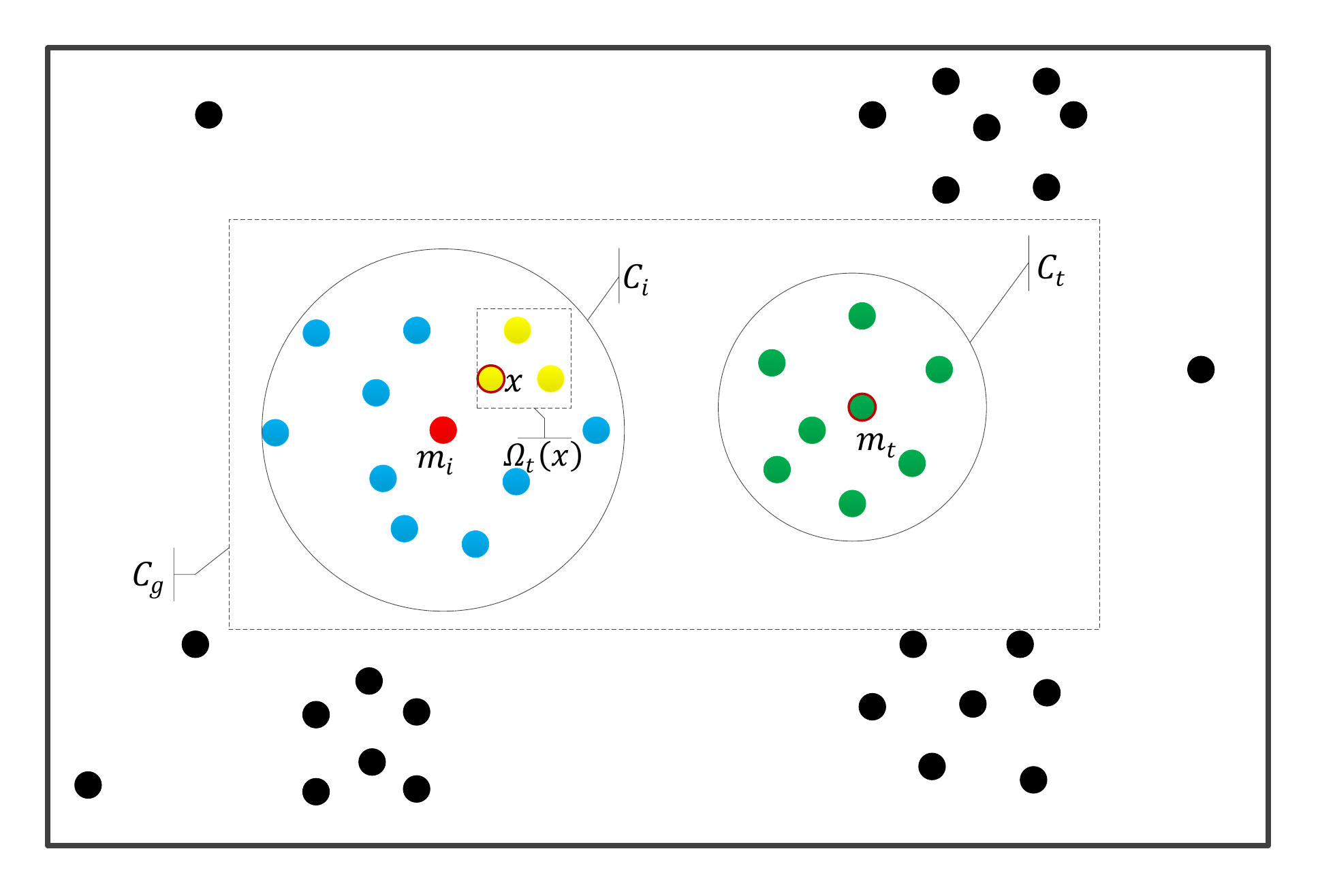}
	\caption{Example for \Cref{P3}. $m_i=nn(x)$ denotes the medoid that was discovered before the $t$th epoch; $m_t$ denotes the $t$th medoid that waits to be discovered.}
	\label{example}
\end{figure}

\begin{myTheo} \label{P3}
	Given an object $x$ where $nn(x)\in M$, if for $\forall y$ where $nn(y)\in M \wedge y\neq x$, $r_t(x)\geq r_t(y)$, then $x\in M$.
\end{myTheo}

\begin{proof}
	Suppose $x\notin M$ and let $m_i=nn(x)$. According to the definition of $nn(x)$ and \Cref{P2}, if $x\notin M$, then $x\in C_i$. See \cref{example} for an example. \Cref{lem1} shows that for $\forall y\notin C_i$, we have $d_{y,m_i}=d_{y,x}$. Note that $\tau_{t-1}(y)$ denotes the nearest medoid to object $y$ in the last epoch. Thus for $\forall y\notin C_i$, we have $d_{y,\tau_{t-1}(y)}\leq d_{y,m_i}=d_{y,x}$ which implies that $\Omega _{t}(x)\subset C_i$.
	
	Letting $M'=\{M\backslash {m_t}\}\cup{x}$, and $\Phi(M)$ denotes the objective value of M, we have
   \begin{alignat*}{1}
	\Phi(M) &=\sum\limits_{y\in C_i}{d_{y,m_i}}+\sum\limits_{y\in C_t}{d_{y,m_t}}+\sum\limits_{\ell \neq i, \ell \neq t}\limits^k \sum\limits_{y\in C_\ell}d_{y,m_\ell} \\
	\Phi(M') &= \sum\limits_{y\in \Omega _{t}(x)}▒{d_{y,x}} +\sum \limits_{y\in C_i-\Omega _{t}(x)}▒{d_{y,m_i}} +\sum\limits_{y\in C_t}{d_{y,\tau_{t-1}(y)}} \\	&+\sum\limits_{\ell \neq i, \ell \neq t}\limits^k \sum\limits_{y\in C_\ell}d_{y,m_\ell} 
	\end{alignat*}
	Therefore,	
	\begin{align*}
	D&=\Phi(M)-\Phi(M') \\
	&= \sum\limits_{y\in \Omega _{t}(x)}▒(d_{y,m_i}-d_{y,x})-	\sum\limits_{y\in C_t}▒(d_{y,\tau_{t-1}(y)}-d_{y,m_t}).
	\end{align*}
	According to the definition of $\Omega _{t}(x)$, $C_t\subset \Omega _{t}(m_t)$. In addition, for $\forall y\in C_i$, $\tau_{t-1}(y)=m_i$. Thus,
	\begin{align*}
	  D	&\geq \sum\limits_{y\in \Omega _{t}(x)}▒(d_{y,m_i}-d_{y,x})-\sum\limits_{y\in \Omega _{t}(m_t)}▒(d_{y,\tau_{t-1}(y)}-d_{y,m_t}) \\
	&= r_{t}(x)-r_{t}(m_t). 
	\end{align*}
	Because $M$ is optimal, we have $D<0$, which implies that $r_{t}(m_t)>r_{t}(x)$. This contradicts the assumption that $r_{t}(x)$ is one of the largest of all the objects that satisfy $nn(x)\in M$. Thus, $x\in M$.
\end{proof}

We repeat the above steps until we obtain the $k$th medoid. The final clustering result is obtained by assigning the rest of objects to their nearest medoid. 

There are two famous MST algorithms: Prim~\cite{prim1957shortest} and Kruskal~\cite{kruskal1956shortest}. The time complexity of Prim's algorithm and Kruskal's algorithm are $O(n^2)$ and $O(l\log{n})$, respectively. Here, $l$ is the number of links, and $n$ is the number of nodes in the graph. As the MST in a clustering problem is generally constructed from a complete graph, $l=n^2$. Thus, in this study, we use the Prim's algorithm to construct the MST. 

\subsection{Outlier or noise problem}
\label{subsec3}

\begin{figure} 
	\centering
	\includegraphics[width=0.4\textwidth]{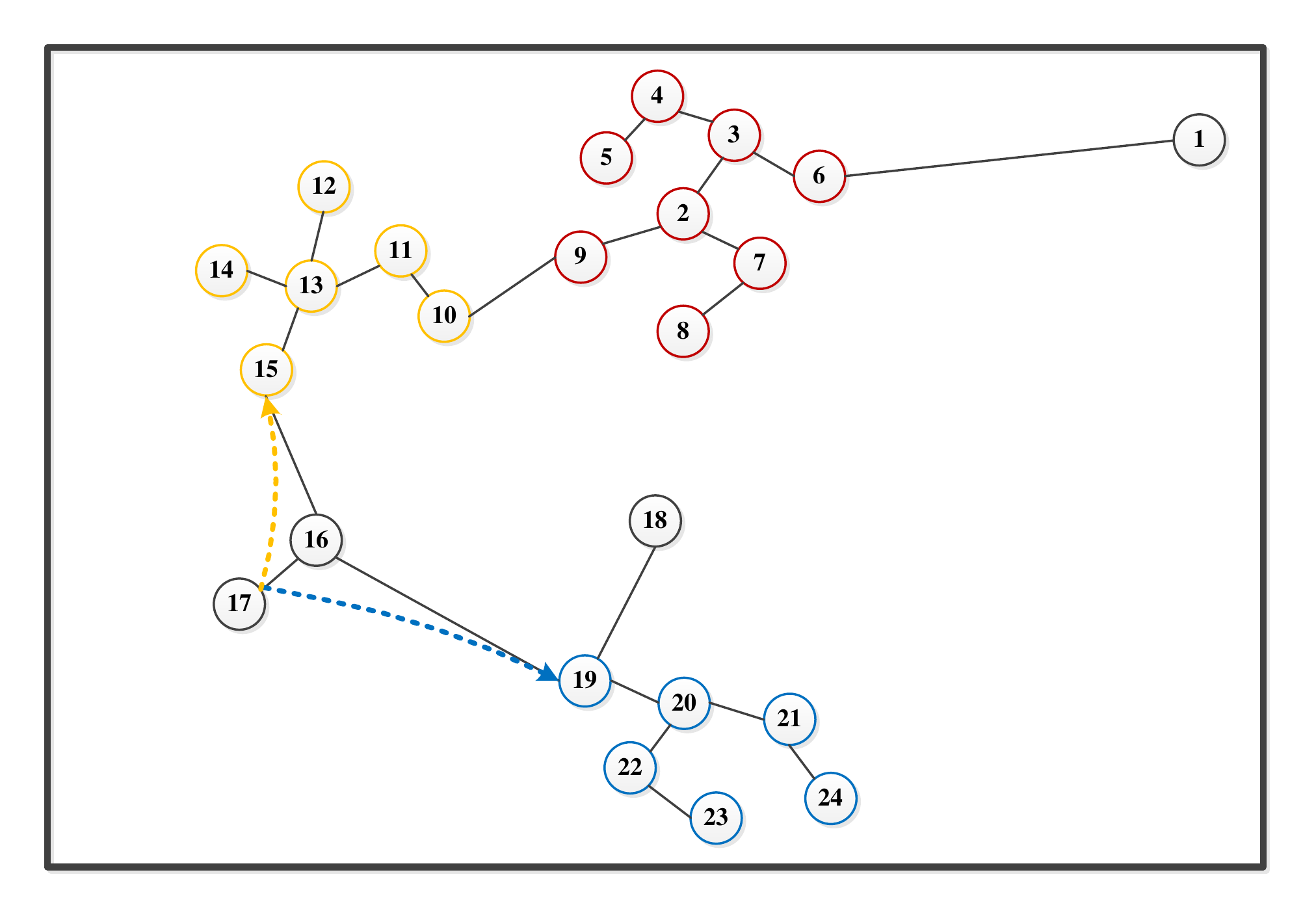}
	\caption{MST of an example dataset. Points are colored according to the cluster to which they are assigned. The black points are noise.}
	\label{fig4}
\end{figure}

Each object is assigned to the nearest medoid. What if there is an object that has the minimum distance to more than one medoid? An example is shown in \cref{fig4}. The black point 1 has a long distance to the normal clusters. From the perspective of minimax distance: $d_{1,2}=d_{1,13}=d_{1,6}$ (2 and 13 are medoids). We take these points (such as black point 1) to be outliers or noise. Outliers and noise are generated by objects that do not belong to any cluster and are a long distance from the normal clusters. Often, these objects have a negative influence on clustering. For example, the algorithm may take a few outliers or noise to be a cluster. However, the GOPC algorithm is robust against this noise. The reason is that the outliers are always alone or have few neighbors. Thus, the $r_{k}(x)$ of these outliers is often smaller than those of medoids. For these outliers or noise, we have two solutions:
\begin{enumerate}
	\item Take them to be a separate cluster.
	\item Assign these outliers to normal clusters based on the MST.
	This method is similar to Kruskal's algorithm. First, create a set $\Lambda$ containing all the edges in the MST where at least one of the endpoints is an outlier. Second, remove the edge with minimum weight from $\Lambda$. The two endpoints of this edge will be assigned to the same cluster. Repeat the second step until $\Lambda= \emptyset$.
	
\end{enumerate}

\begin{figure*} 
	\centering
	\subfloat[Taking noise to be a separate cluster.]{\includegraphics[width=0.4\textwidth]{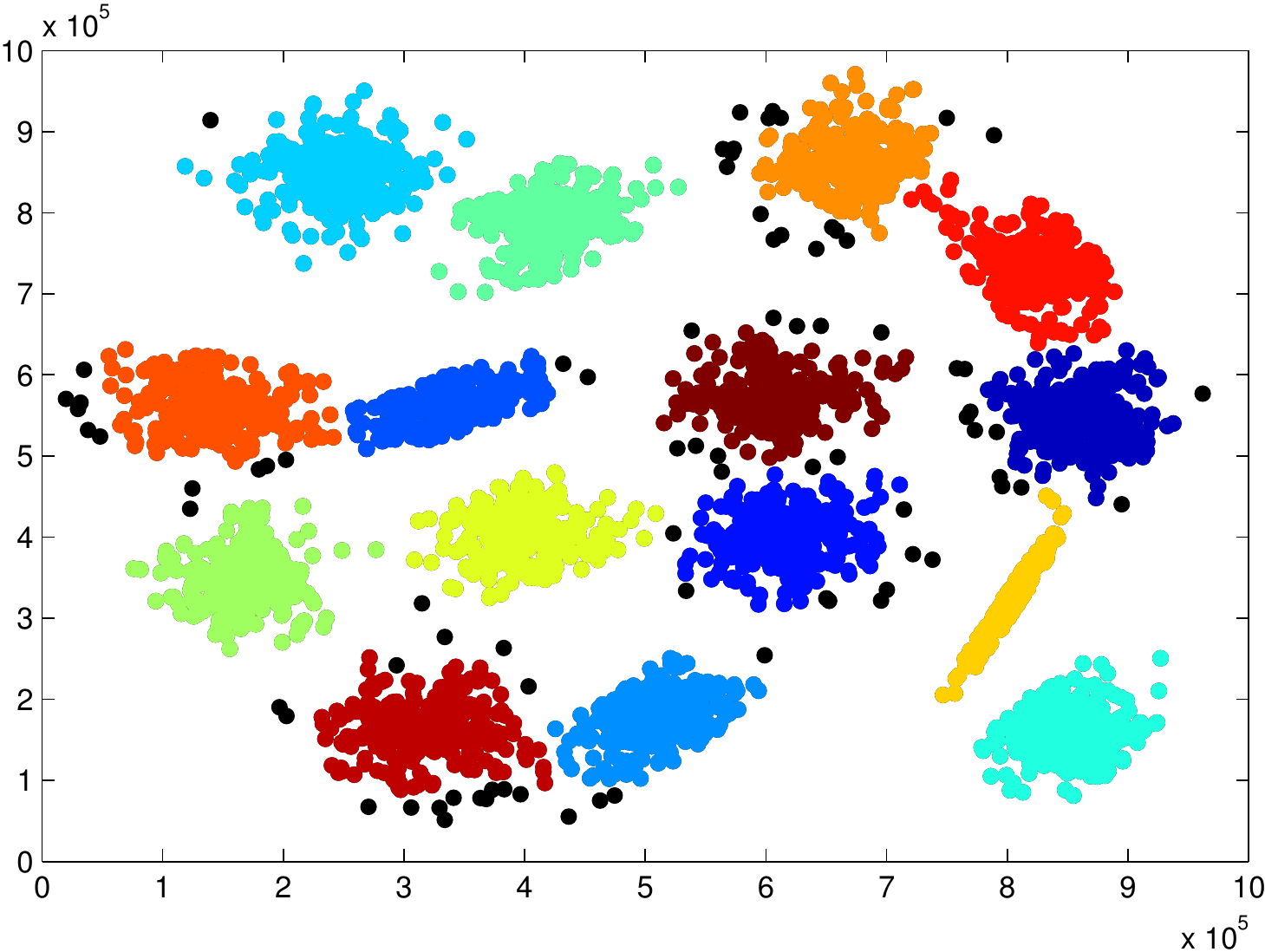}}
	\subfloat[Assigning noise to normal clusters based on the MST.]{\includegraphics[width=0.4\textwidth]{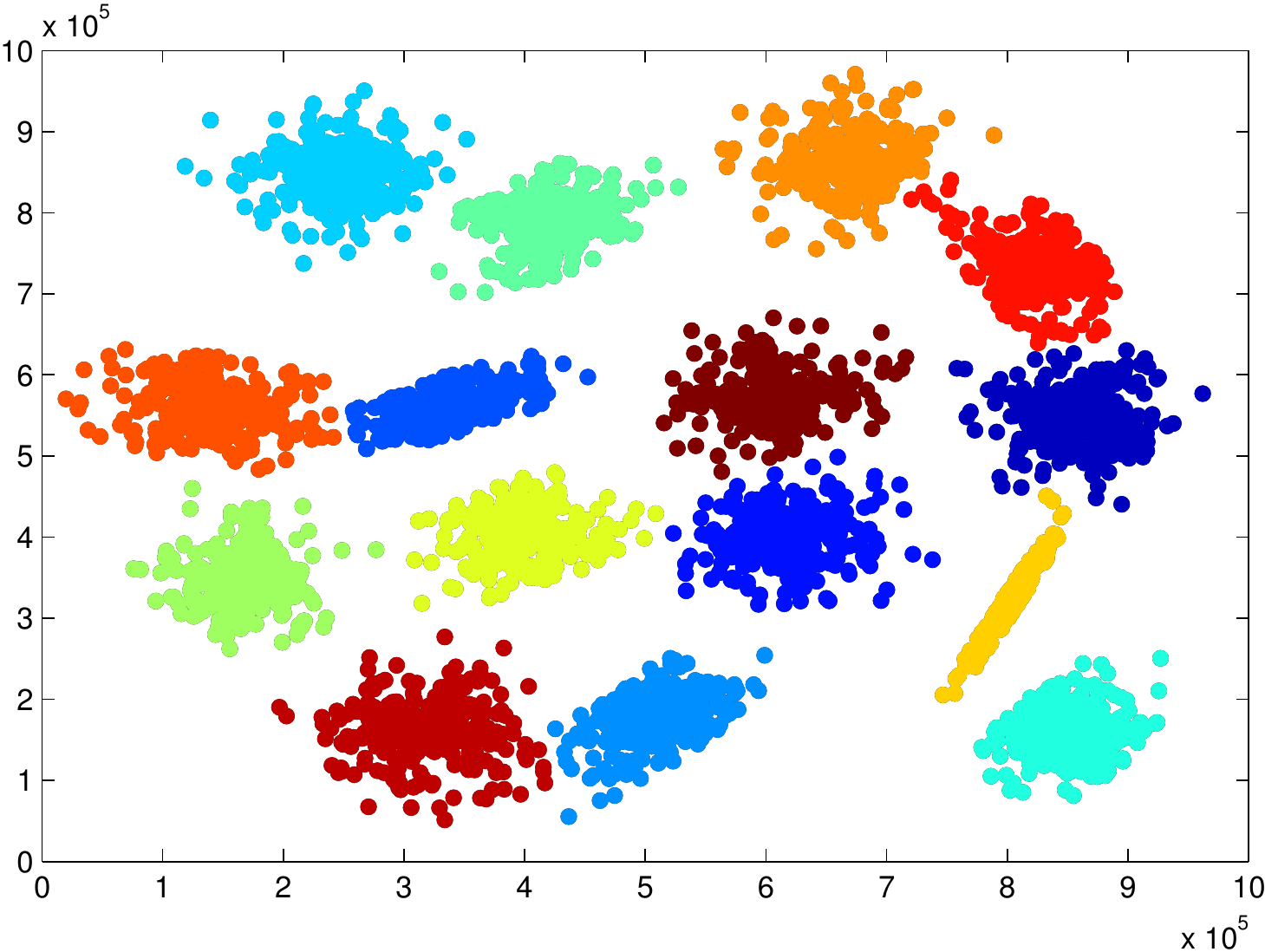}} 
	\caption{Clustering results on $S1$; black points are noise.}
	\label{fig5}
\end{figure*}

Figure~\ref{fig5} shows the clustering results of the two possibilities performed on dataset $s1$~\cite{franti2006iterative}, respectively.

\subsection{Estimating the number of clusters}
\label{subsec4}
The GOPC algorithm only needs to input one parameter, $k$ (i.e., the number of clusters), which is usually set manually. Sometimes $k$ is given as prior knowledge, but in most cases, we expect that the clustering algorithm can determine the correct number of clusters automatically.

In this subsection, we propose a simple but effective method to help find the number of clusters. This method is based on the fact that the medoid has a large value of $r_t(x)$. So, we record the maximum value of $r_t(x)$ for each epoch, as shown in \cref{fig6}. 

The vertical coordinates represent the maximum value of $r_t(x)$ for the $t$th epoch (denoted as $\max(R)$). \Cref{fig6} shows that with the increase in $t$, $\max(R)$ decreases gradually, and there is a cliff-like drop after the 15th epoch. The number of clusters is the threshold of the cliff. Note that the values in the last five epochs are not zero (they are too small to be observed when compared with the value of the medoids). It is very simple for a user to observe this cliff in a graphical representation. Therefore, \cref{fig6} can be used as a decision graph that helps in estimating the number of clusters.

\begin{figure}
	\centering	\includegraphics[width=0.4\textwidth]{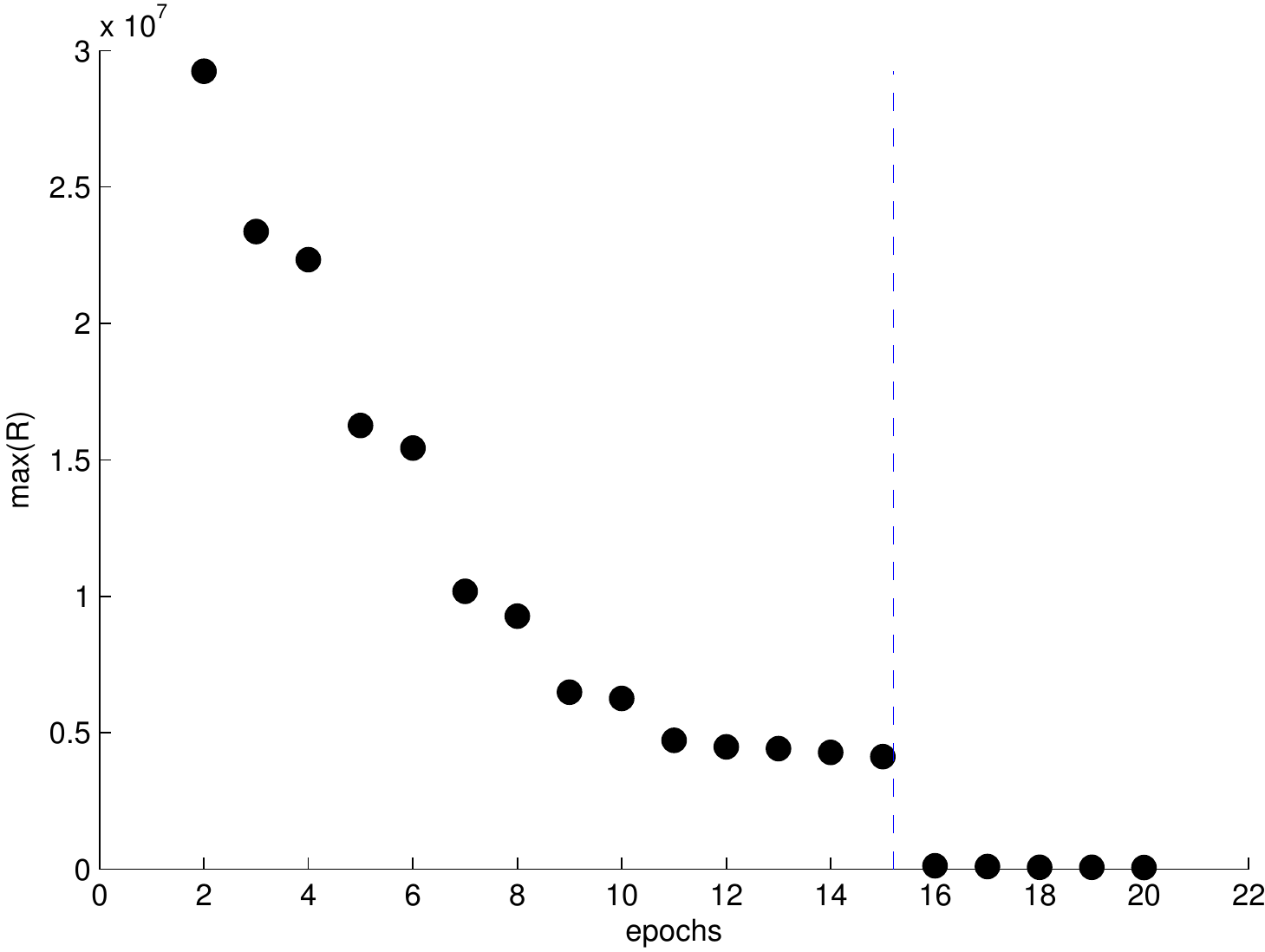}
	\caption{Experimental results on $S1$: the maximum value of $R$ ($R=r_t(x_1), r_t(x_2), \cdots, r_t(x_n)$) for each epoch.}
	\label{fig6}
\end{figure}

\section{Experiments}
\label{sec3}

In this section, we present the tests of the proposed algorithm on various synthetic and real-world datasets. State-of-the-art algorithms---single-linkage clustering (SLC)~\cite{Sneath1957The}, kernel $k$-means (Kernel KM)~\cite{Sch1998Nonlinear}, clustering by fast search and find of density peaks (FDPC)~\cite{rodriguez2014clustering}, and normalized cut (NCUT)~\cite{Shi2000Normalized}---were used for baseline comparison. SLC is a method for hierarchical clustering. The naive algorithm for SLC is essentially the same as Kruskal's algorithm for MST~\cite{Gower1969Minimum}. The kernel KM algorithm is a generalization of the standard $k$-means algorithm. It can detect clusters that are nonlinearly separable in input space by using a kernel function (here we use only the Gaussian kernel) to map data points to a higher dimensional feature space~\cite{Dhillon2007Weighted}. NCUT utilizes the spectrum (i.e., eigenvalues) of the similarity matrix of the data to map the data into a space in which the objects can be clustered by traditional clustering techniques~\cite{Wang2016Automatic}. This algorithm performs very well at non-convex boundaries. FDPC is based on the idea that cluster centers are characterized by a higher density than their neighbors and by a relatively large distance from points with higher densities~\cite{rodriguez2014clustering}. This method can detect non-spherical clusters. The source code for the kernel KM\footnote{http://www.dcs.gla.ac.uk/$\sim$srogers/firstcourseml/matlab/}, FDPC\footnote{http://people.sissa.it/$\sim$laio/Research/Research.php}, and NCUT\footnote{http://www.cis.upenn.edu/$\sim$jshi/software/} algorithms are taken from the author's website. SLC is implemented using a Python-based ecosystem of open-source software SciPy\footnote{https://docs.scipy.org/doc/scipy-0.14.0/reference/cluster.html}. The FDPC algorithm needs a parameter $d_c$. As the authors suggested, one can choose $d_c$ so that the average number of neighbors is approximately $1\%$ to $2\%$ of the total number of points in the dataset. Both the GOPC and SLC algorithm need the same parameter---the number of clusters $k$---which is given as prior knowledge. The remaining two algorithms (Kernel KM and NCUT) need not only the parameter $k$, but also another parameter $\sigma$ to control the width of the radial basis function (RBF) kernel. Parameter $\sigma$ is selected by trial and error. Because the Kernel KM algorithm may fall into a local optimum, we run it ten times for the proper $\sigma$ and reported the best results. 

\subsection{Experiments on synthetic datasets}
\label{subsec3.1}
To access the efficacy of the GOPC algorithm for clustering tasks, we first performed some experiments on synthetic datasets. As two randomly distributed datasets generated by us, $DS1$ consisted of three linear shaped clusters, while $DS2$ consisted of three concentric circles. $Spiral$, $Aggregation$, $Flame$, $Jain$, $Unbalance$, $R15$, $D31$, $A3$, and $S1-S2$ were obtained from University of Eastern Finland website\footnote{http://cs.joensuu.fi/sipu/datasets/}. $PanelB$ was taken from Ref.~\cite{rodriguez2014clustering}. The others were taken from Ref.~\cite{Liu2016K}. Some of these datasets had different shapes (e.g., $Spiral$), some had different sizes (e.g., $Unbalance$), some had different densities (e.g., $Jain$), some were randomly distributed (e.g., $DS1$ and $DS2$), and some were normally distributed (e.g., $R15$ and $D31$). More importantly, there were also some connected clusters (e.g., $Aggregation$). Through experiments on these datasets, we tested the robustness of the GOPC algorithm. \Cref{tab1} gives a brief description of the datasets used. 

\begin{table}
	\centering
	\caption{Description of datasets used}
	\label{tab1}
	\begin{tabular}{c c c c}
		\hline
		Datasets & Data size & Dimensionality & Number of clusters \\
		\hline
		DS1 & 695 & 2 & 3 \\
		DS2 & 385 & 2 & 3 \\
		Spiral & 312 & 2 & 3 \\
		Aggregation & 788 & 2 & 7 \\
		Jain & 373 & 2 & 2 \\
		Flame & 240 & 2 & 2 \\
		Unbalance & 6500 & 2 & 8 \\
		R15 & 600 & 2 & 15 \\
		D31 & 3100 & 2 & 31 \\
		S1 & 5000 & 2 & 15 \\
		S2 & 5000 & 2 & 15 \\
		A3 & 7500 & 2 & 50 \\
		PanelB & 4000 & 2 & 5 \\
		Lyga & 1500 & 2 & 3 \\
		Lygd & 546 & 2 & 2 \\
		\hline
	\end{tabular}
\end{table}

\begin{figure*}
	\begin{minipage}{0\linewidth}
		\rightline{A}
	\end{minipage}
	\hfill
	\begin{minipage}{0.19\linewidth}
		\centerline{\includegraphics[width=1\textwidth]{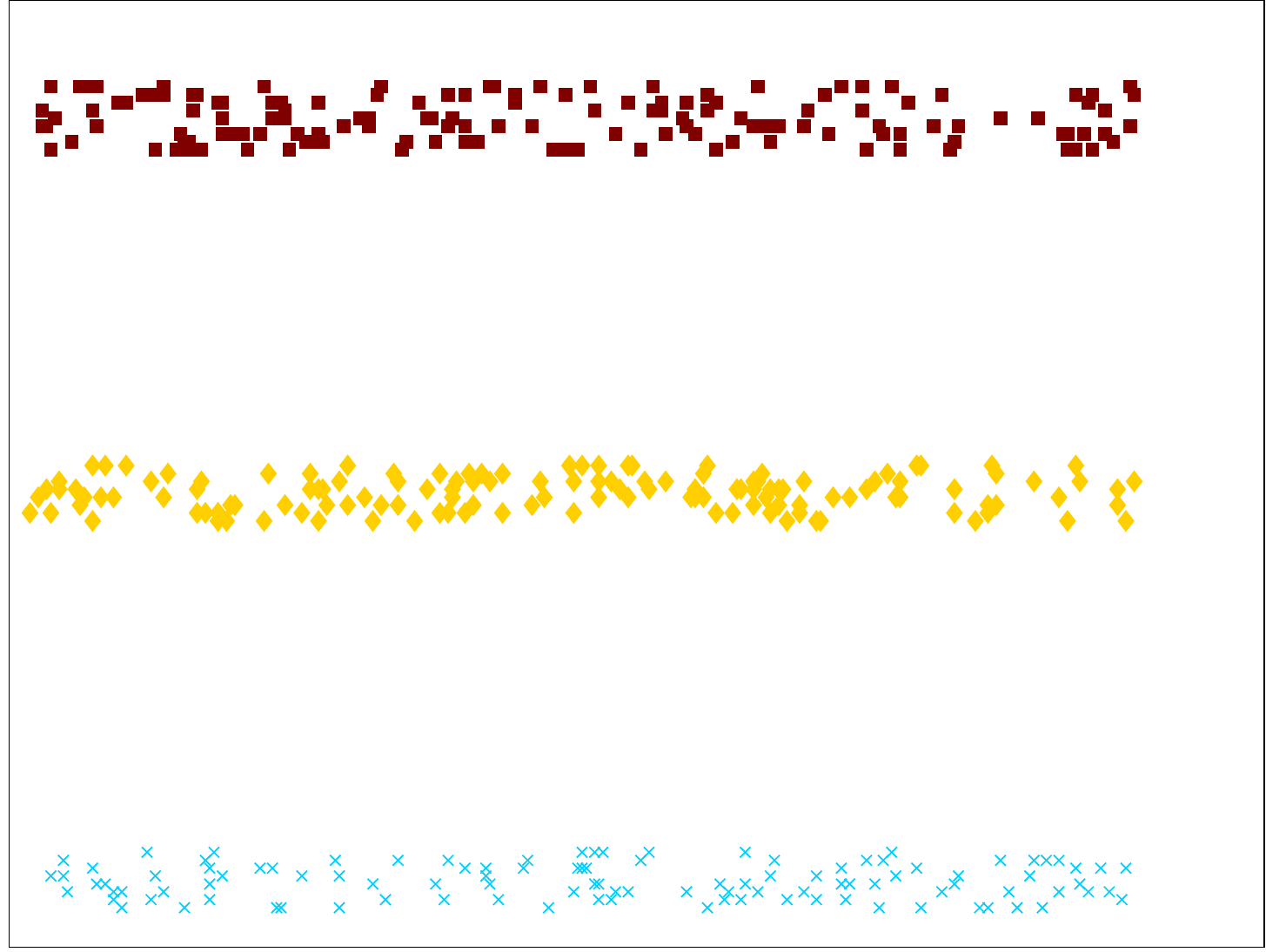}}
	\end{minipage}
	\hfill
	\begin{minipage}{0.19\linewidth}
		\centerline{\includegraphics[width=1\textwidth]{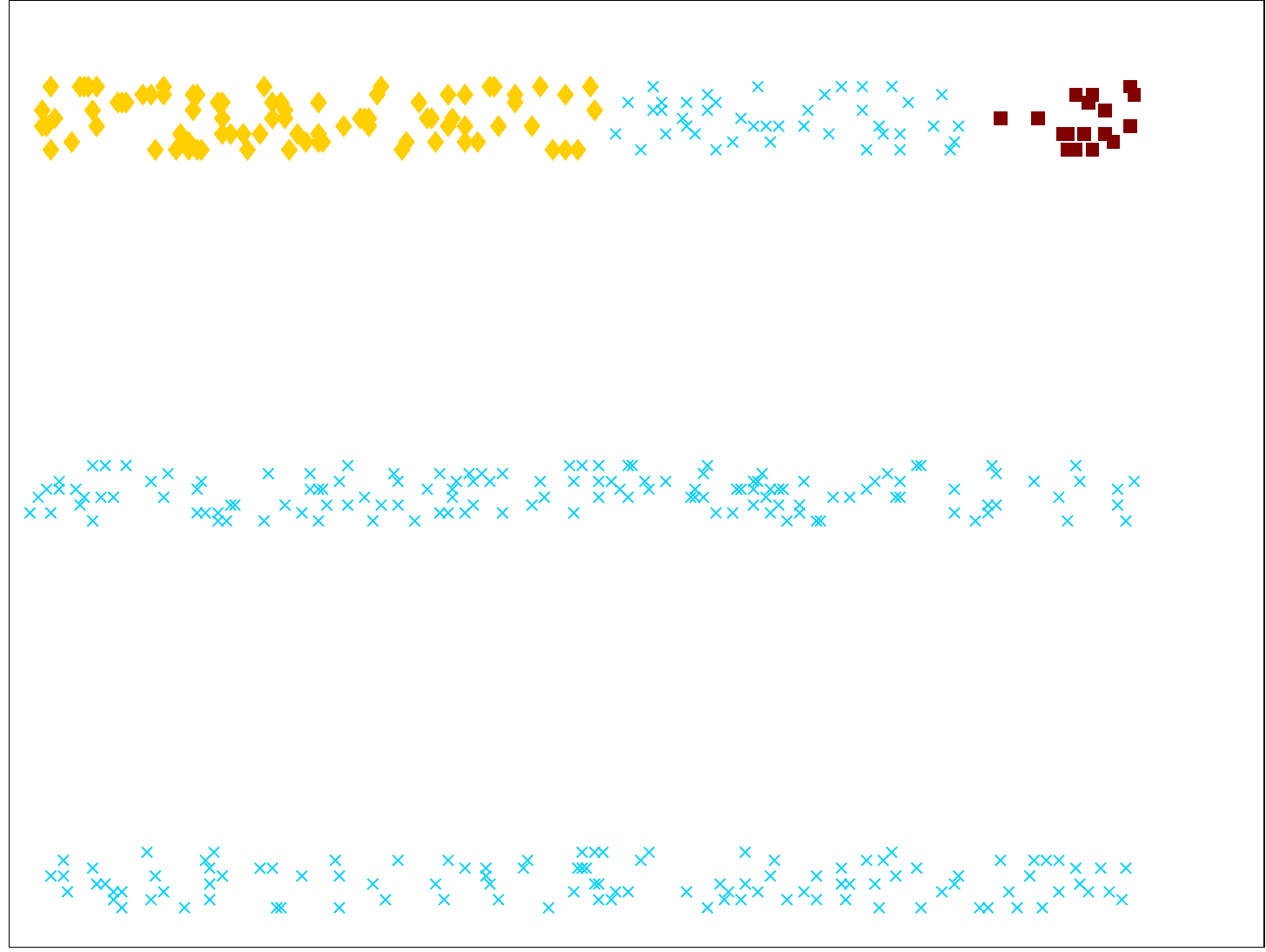}}
	\end{minipage}
	\hfill
	\begin{minipage}{0.19\linewidth}
		\centerline{\includegraphics[width=1\textwidth]{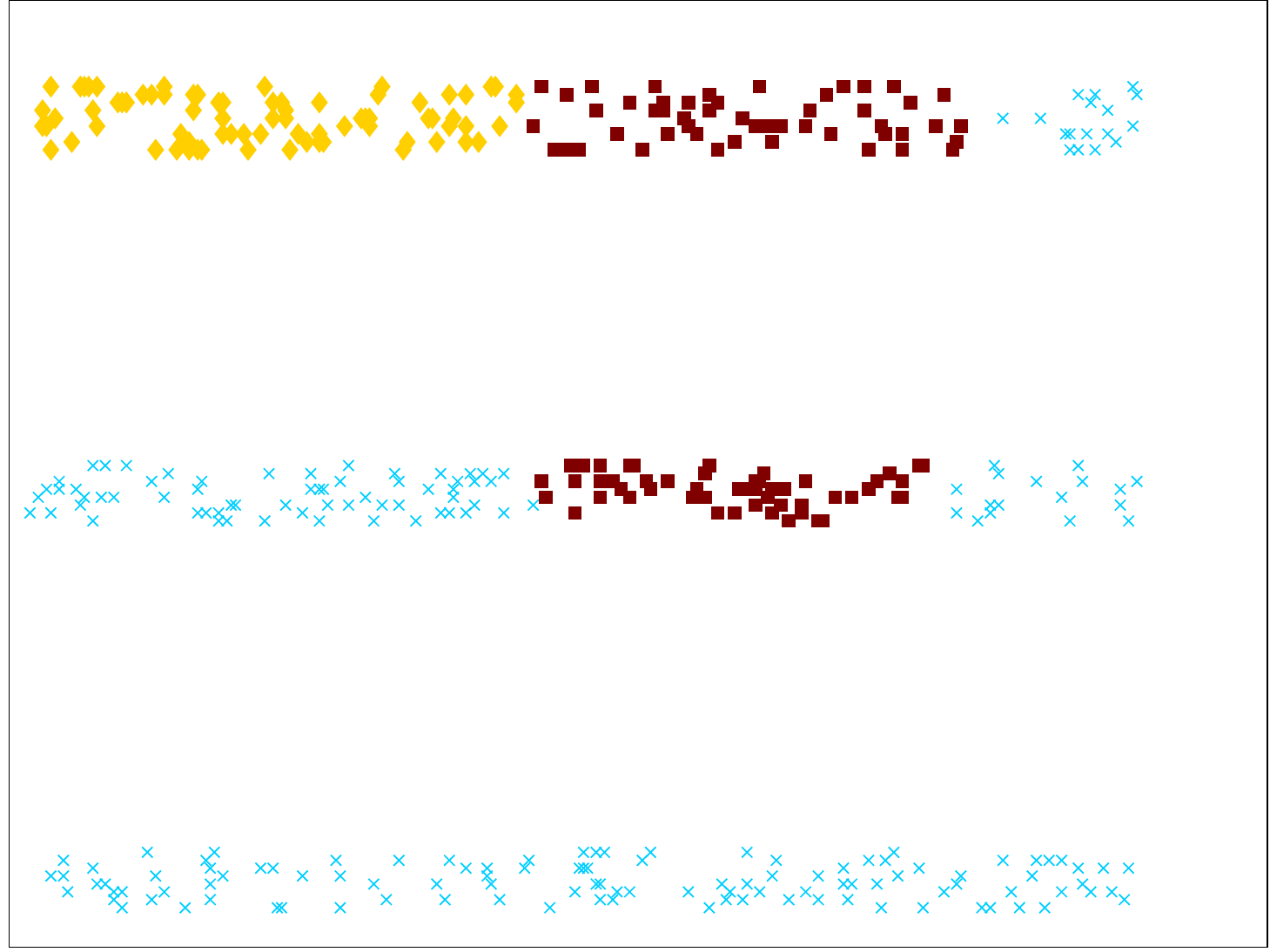}}
	\end{minipage}
	\hfill
	\begin{minipage}{0.19\linewidth}
		\centerline{\includegraphics[width=1\textwidth]{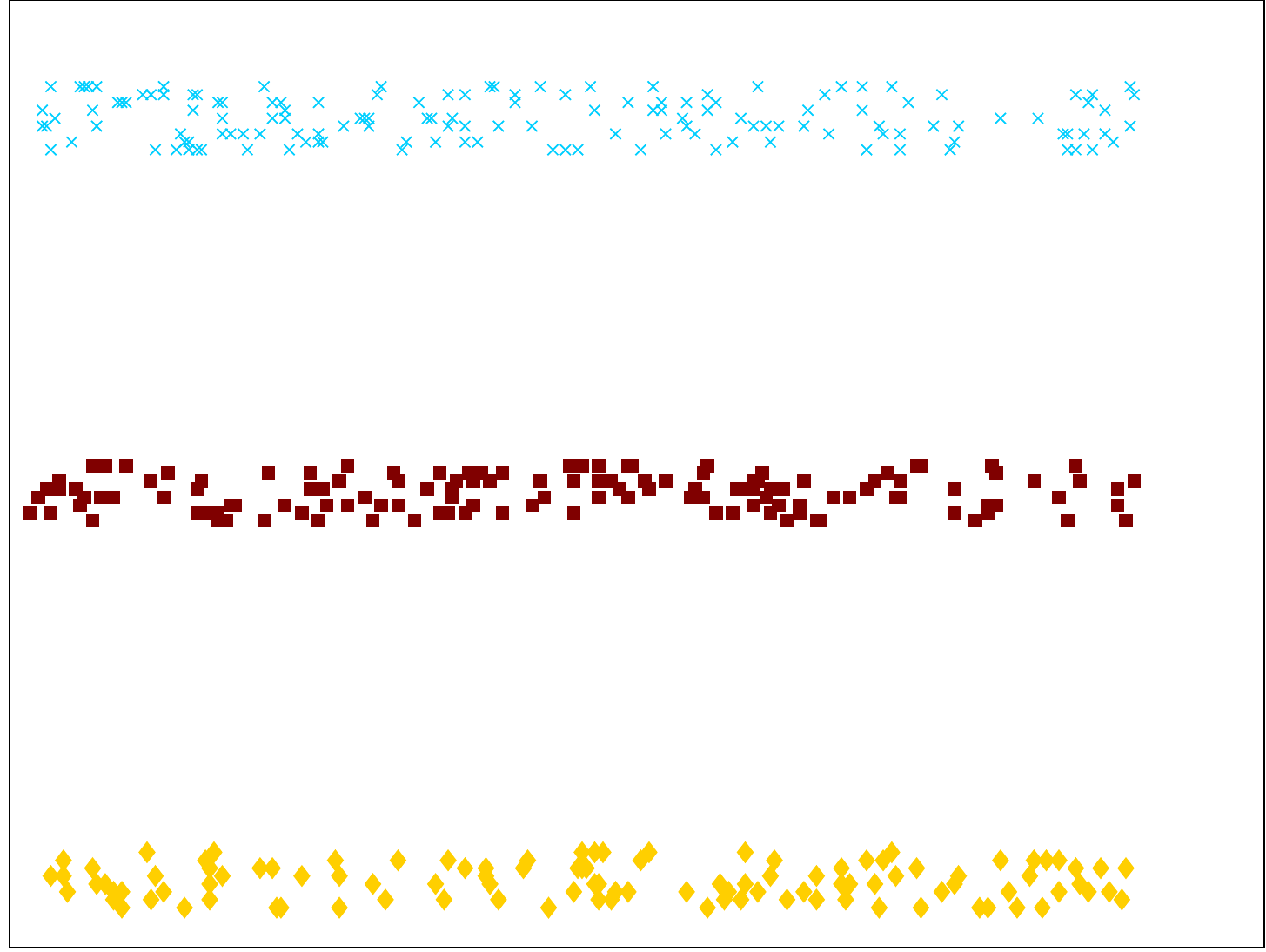}}
	\end{minipage}
	\hfill
	\begin{minipage}{0.19\linewidth}
		\centerline{\includegraphics[width=1\textwidth]{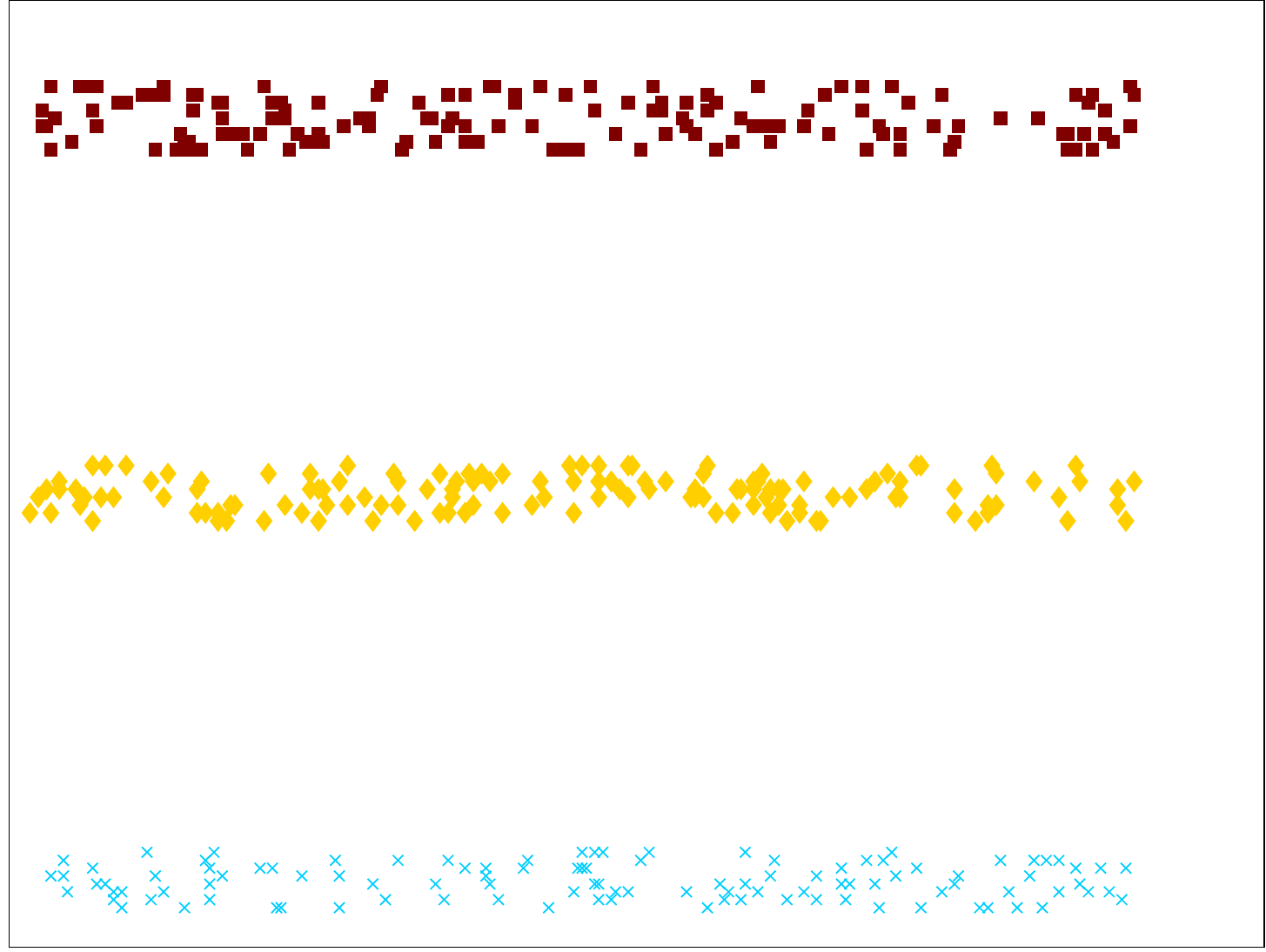}}
	\end{minipage}
	\vfill
	\begin{minipage}{0\linewidth}
		\rightline{B}
	\end{minipage}
	\hfill
	\begin{minipage}{0.19\linewidth}
		\centerline{\includegraphics[width=1\textwidth]{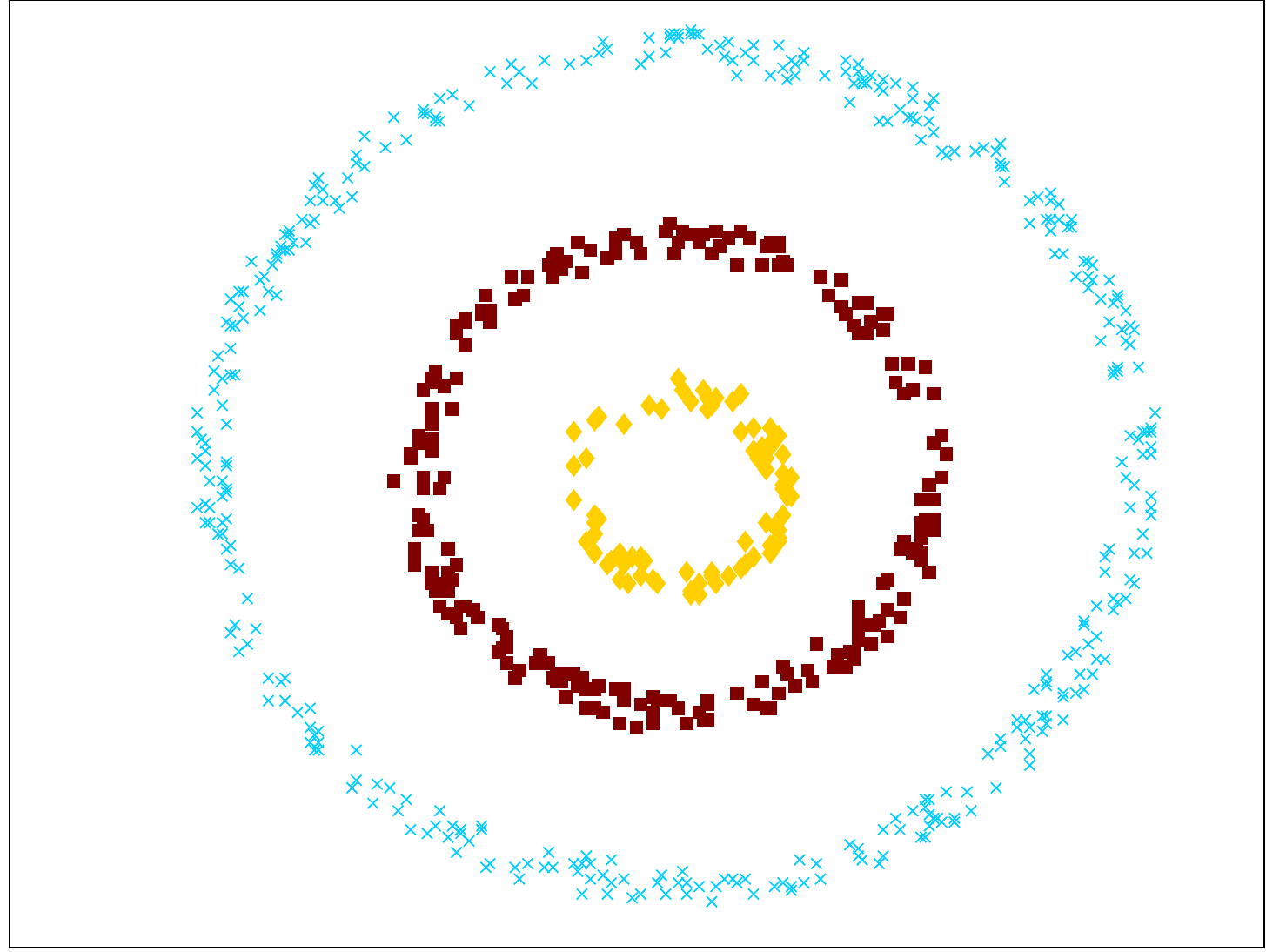}}
	\end{minipage}
	\hfill
	\begin{minipage}{0.19\linewidth}
		\centerline{\includegraphics[width=1\textwidth]{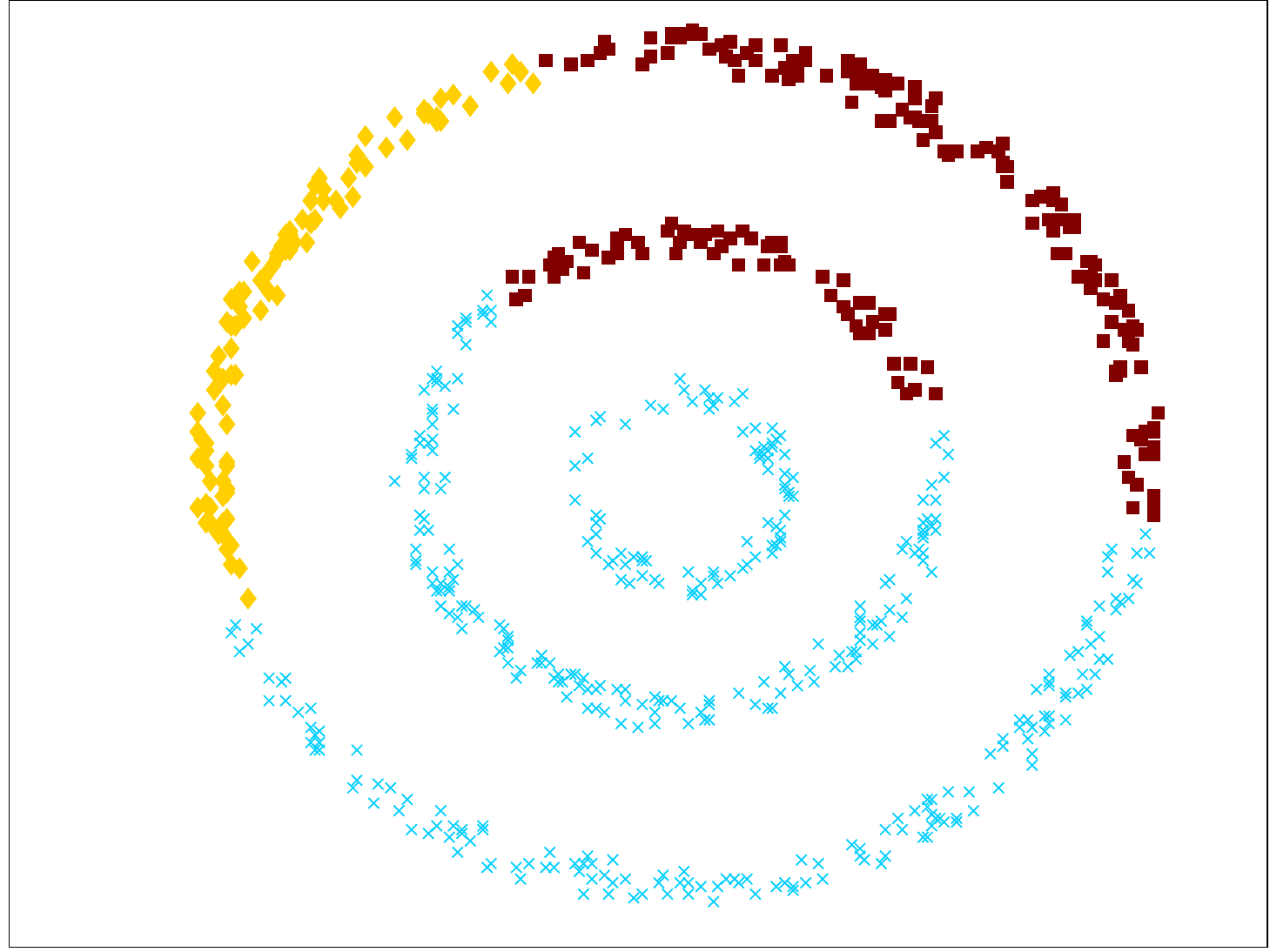}}
	\end{minipage}
	\hfill
	\begin{minipage}{0.19\linewidth}
		\centerline{\includegraphics[width=1\textwidth]{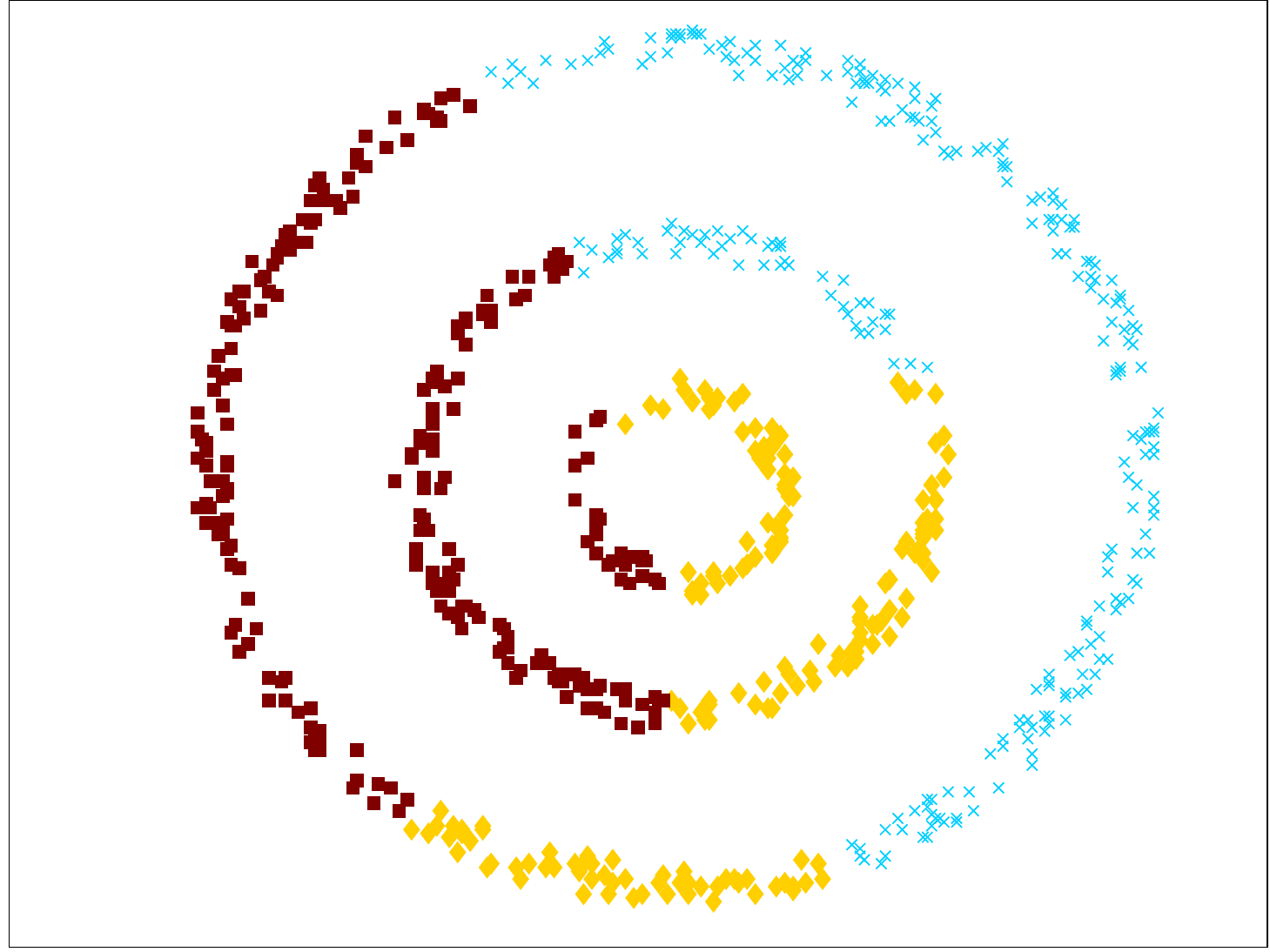}}
	\end{minipage}
	\hfill
	\begin{minipage}{0.19\linewidth}
		\centerline{\includegraphics[width=1\textwidth]{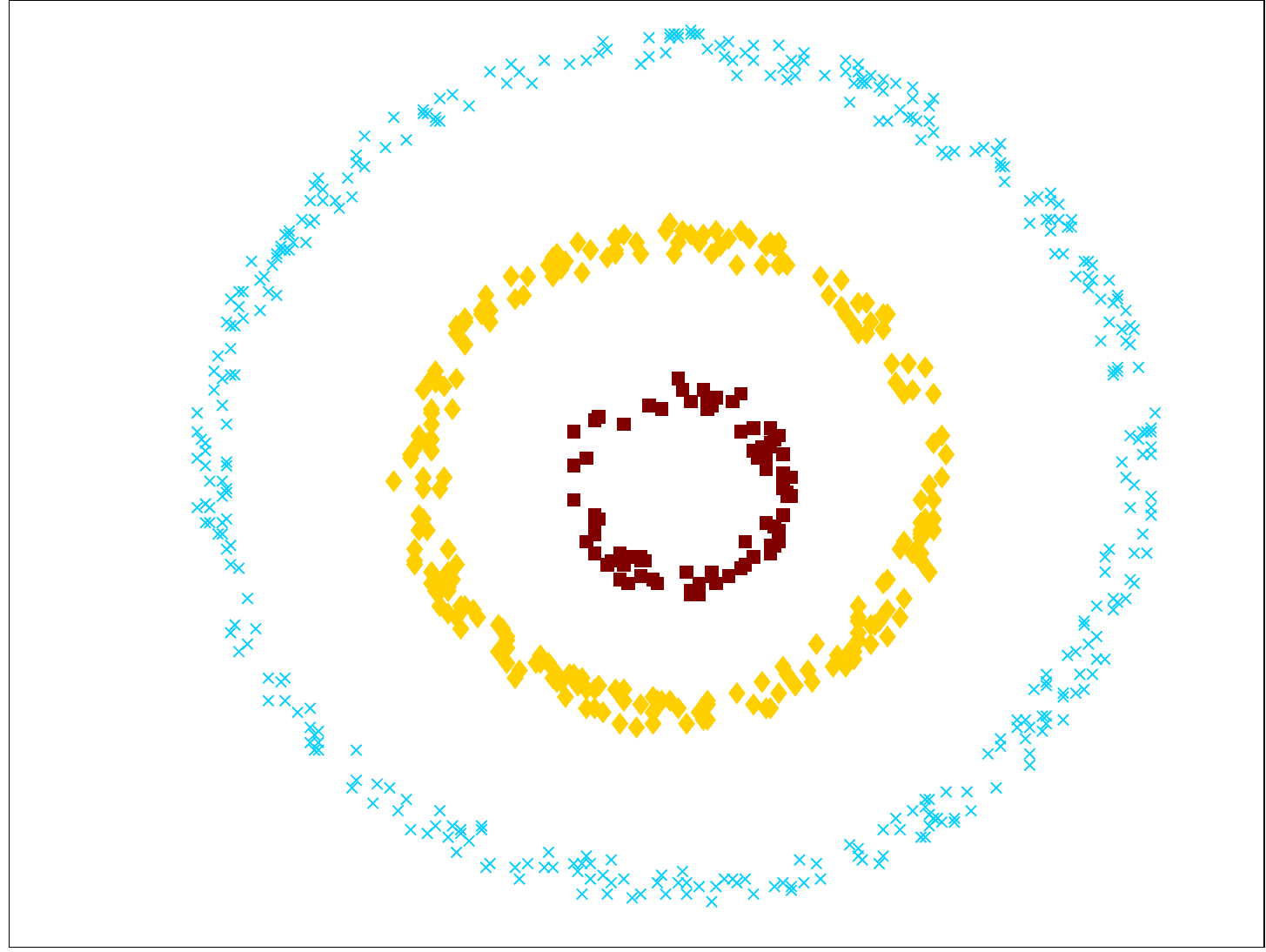}}
	\end{minipage}
	\hfill
	\begin{minipage}{0.19\linewidth}
		\centerline{\includegraphics[width=1\textwidth]{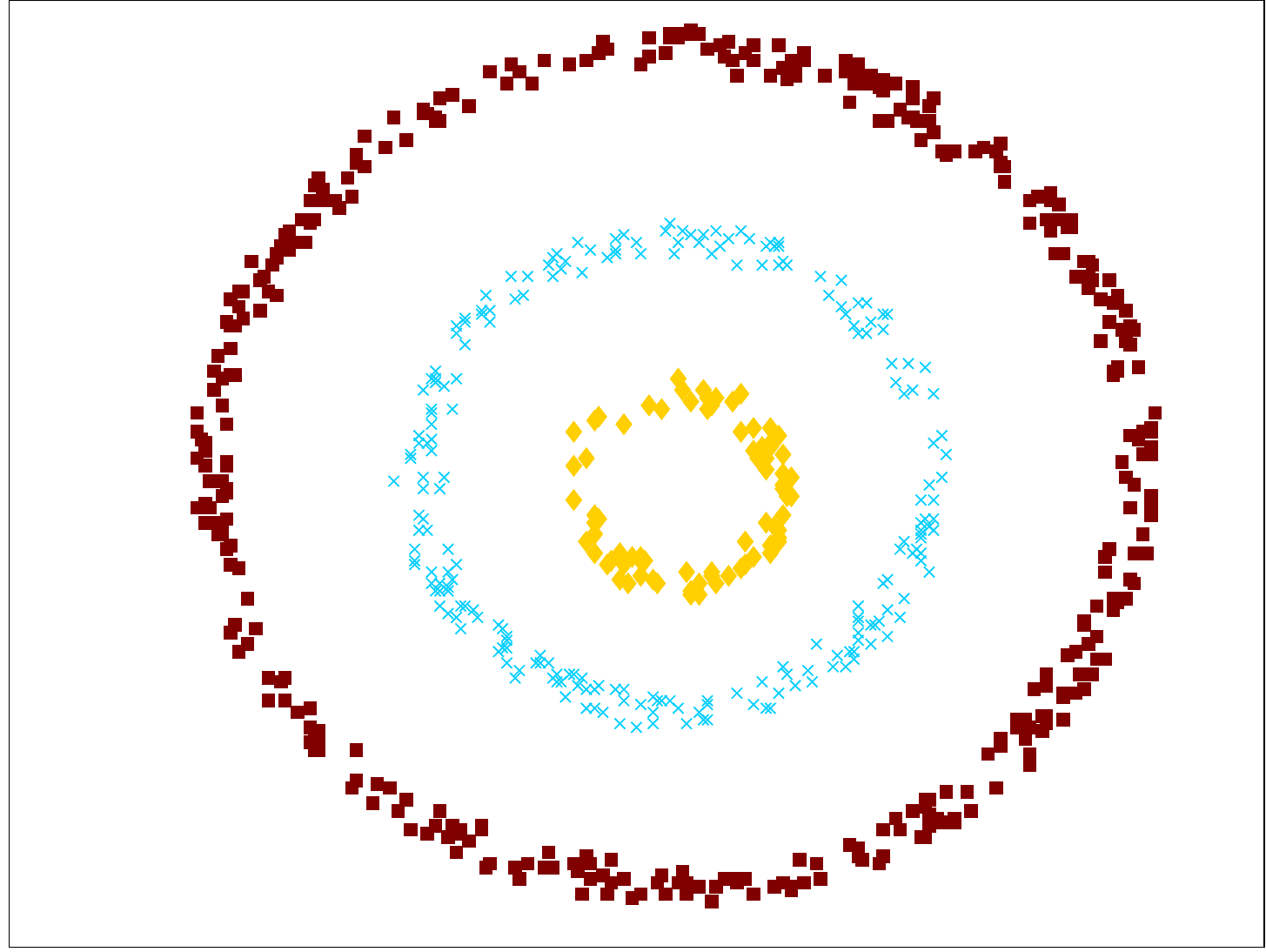}}
	\end{minipage}
	\vfill
	\begin{minipage}{0\linewidth}
		\rightline{C}
	\end{minipage}
	\hfill
	\begin{minipage}{0.19\linewidth}
		\centerline{\includegraphics[width=1\textwidth]{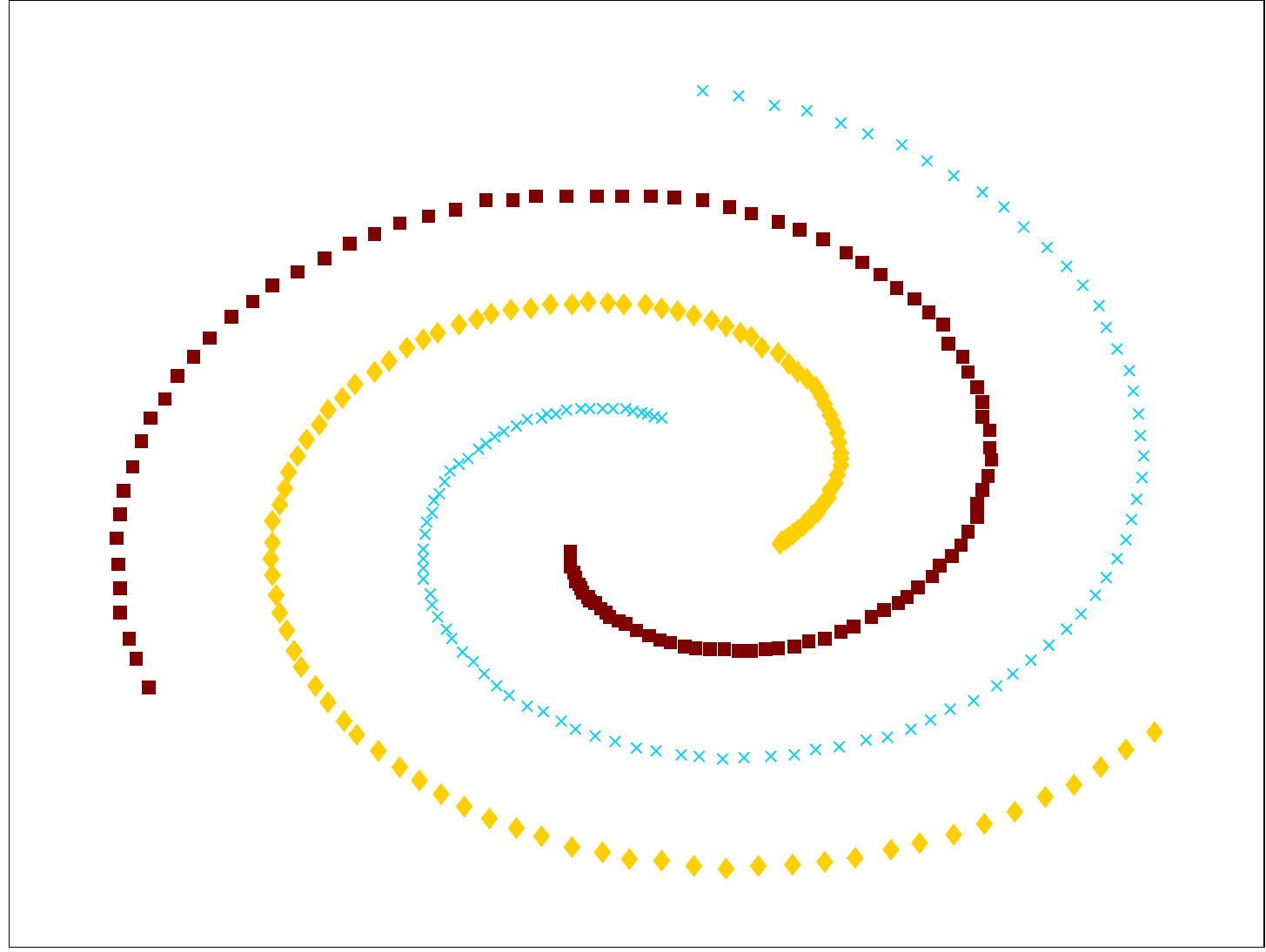}}
	\end{minipage}
	\hfill
	\begin{minipage}{0.19\linewidth}
		\centerline{\includegraphics[width=1\textwidth]{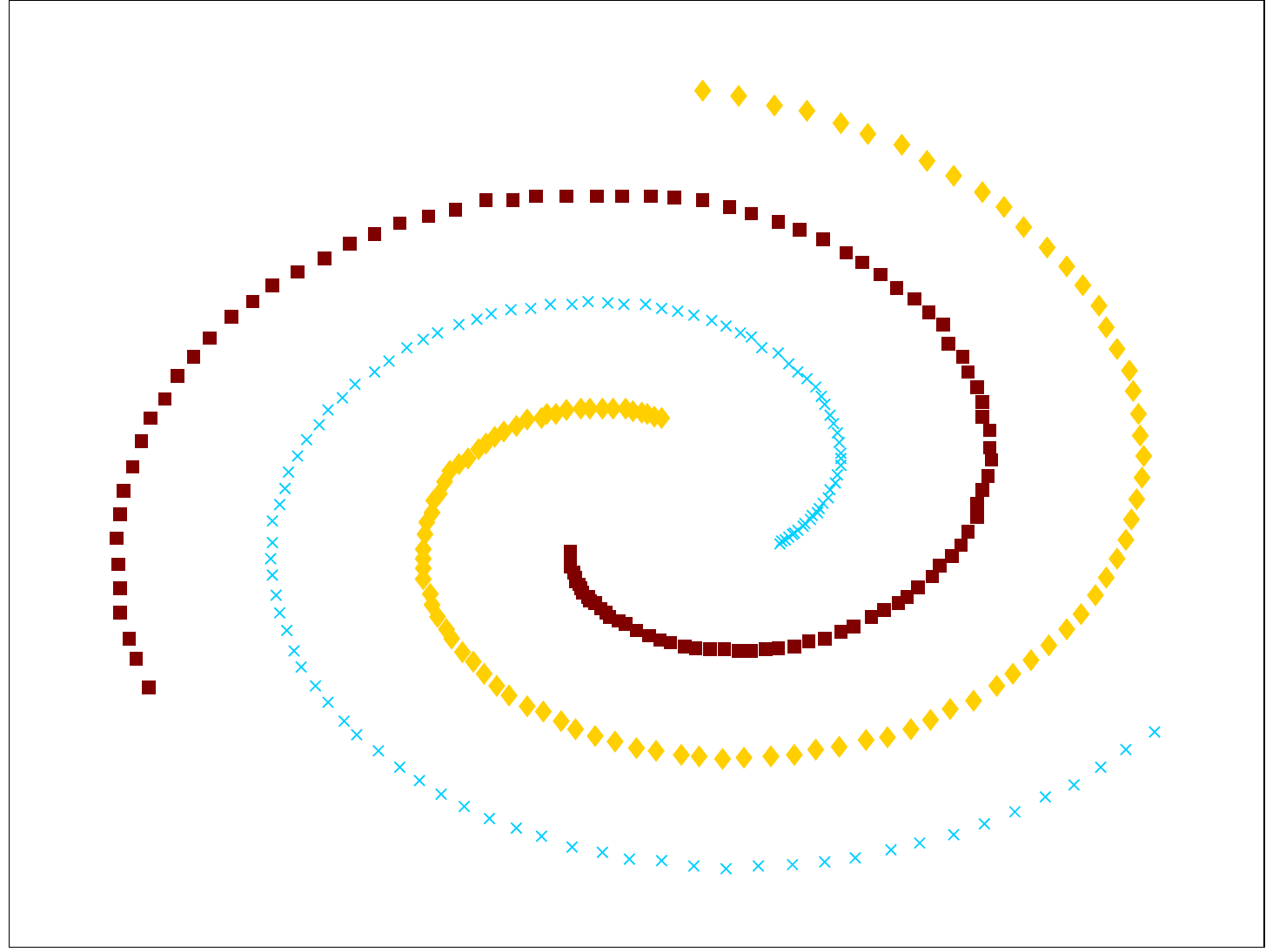}}
	\end{minipage}
	\hfill
	\begin{minipage}{0.19\linewidth}
		\centerline{\includegraphics[width=1\textwidth]{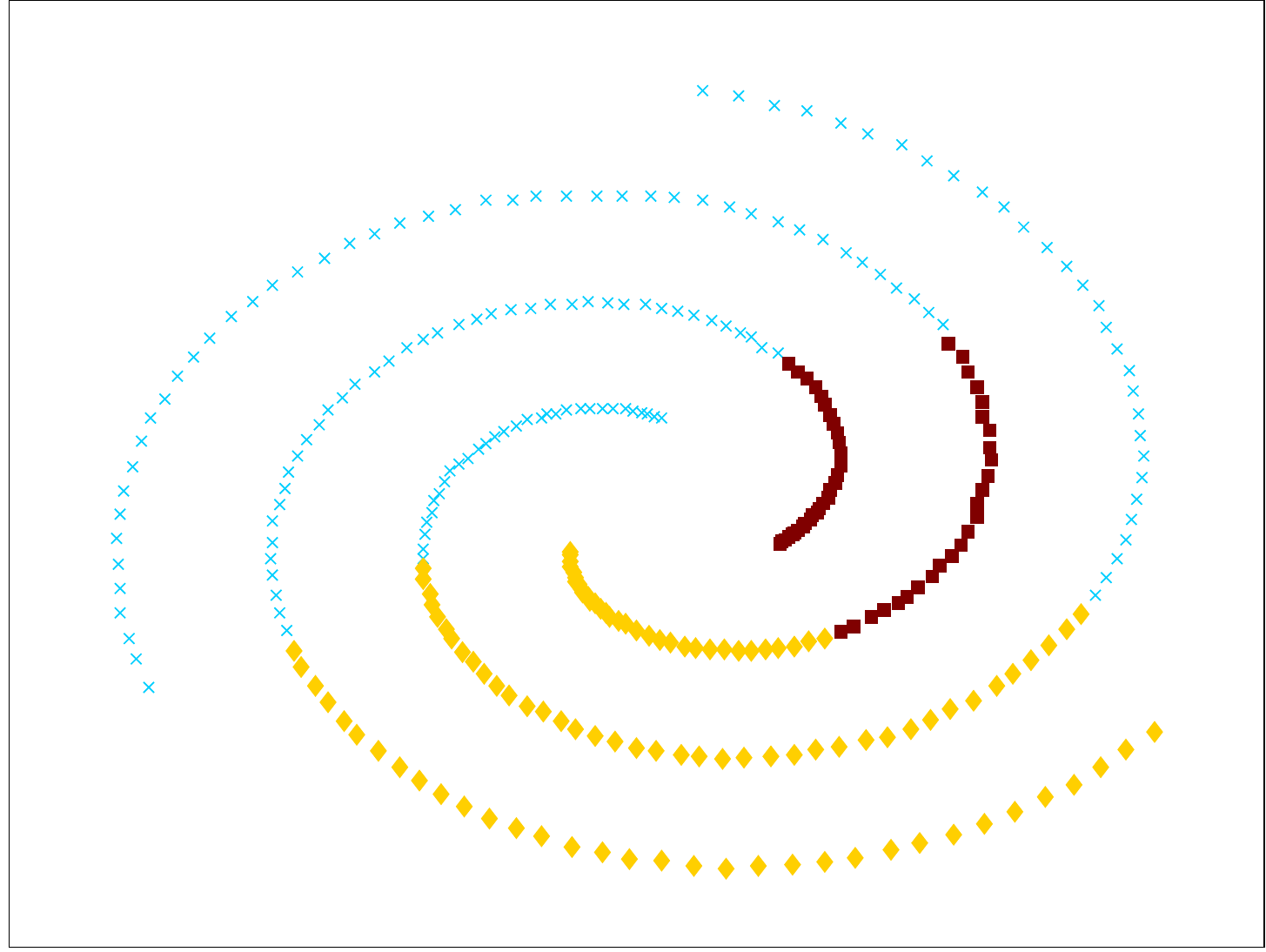}}
	\end{minipage}
	\hfill
	\begin{minipage}{0.19\linewidth}
		\centerline{\includegraphics[width=1\textwidth]{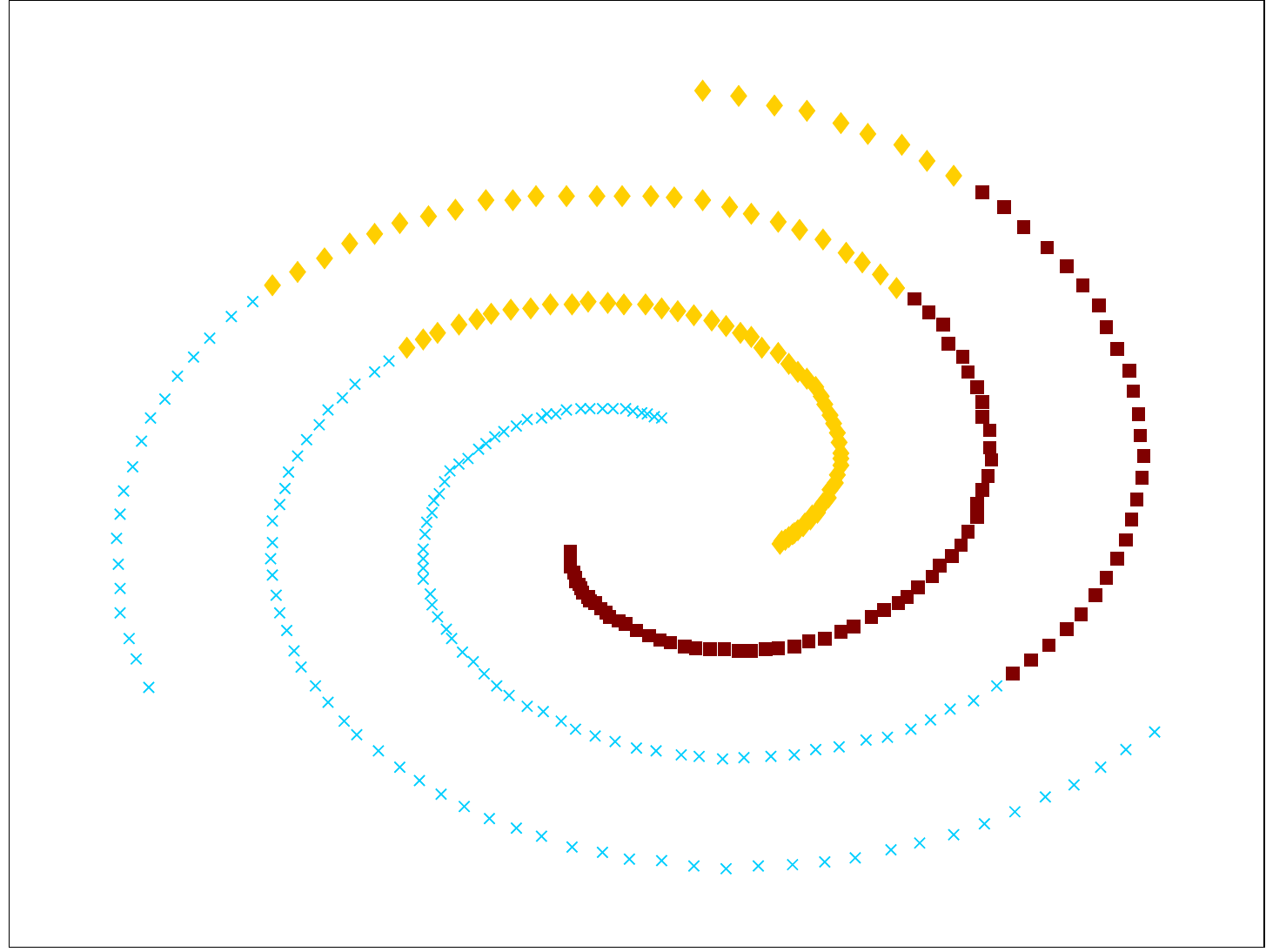}}
	\end{minipage}
	\hfill
	\begin{minipage}{0.19\linewidth}
		\centerline{\includegraphics[width=1\textwidth]{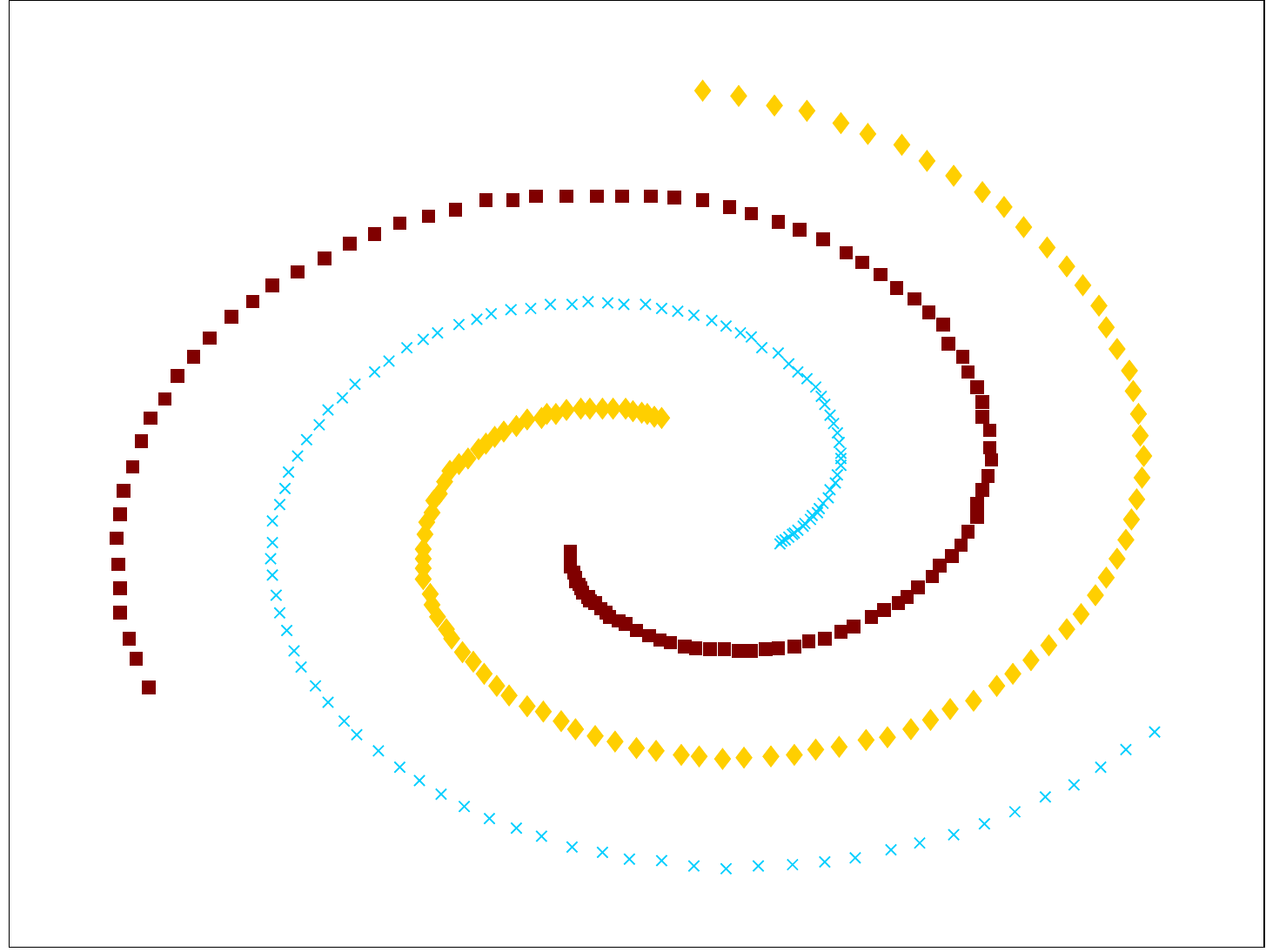}}
	\end{minipage}
	\vfill
	\begin{minipage}{0\linewidth}
		\rightline{D}
	\end{minipage}
	\hfill
	\begin{minipage}{0.19\linewidth}
		\centerline{\includegraphics[width=1\textwidth]{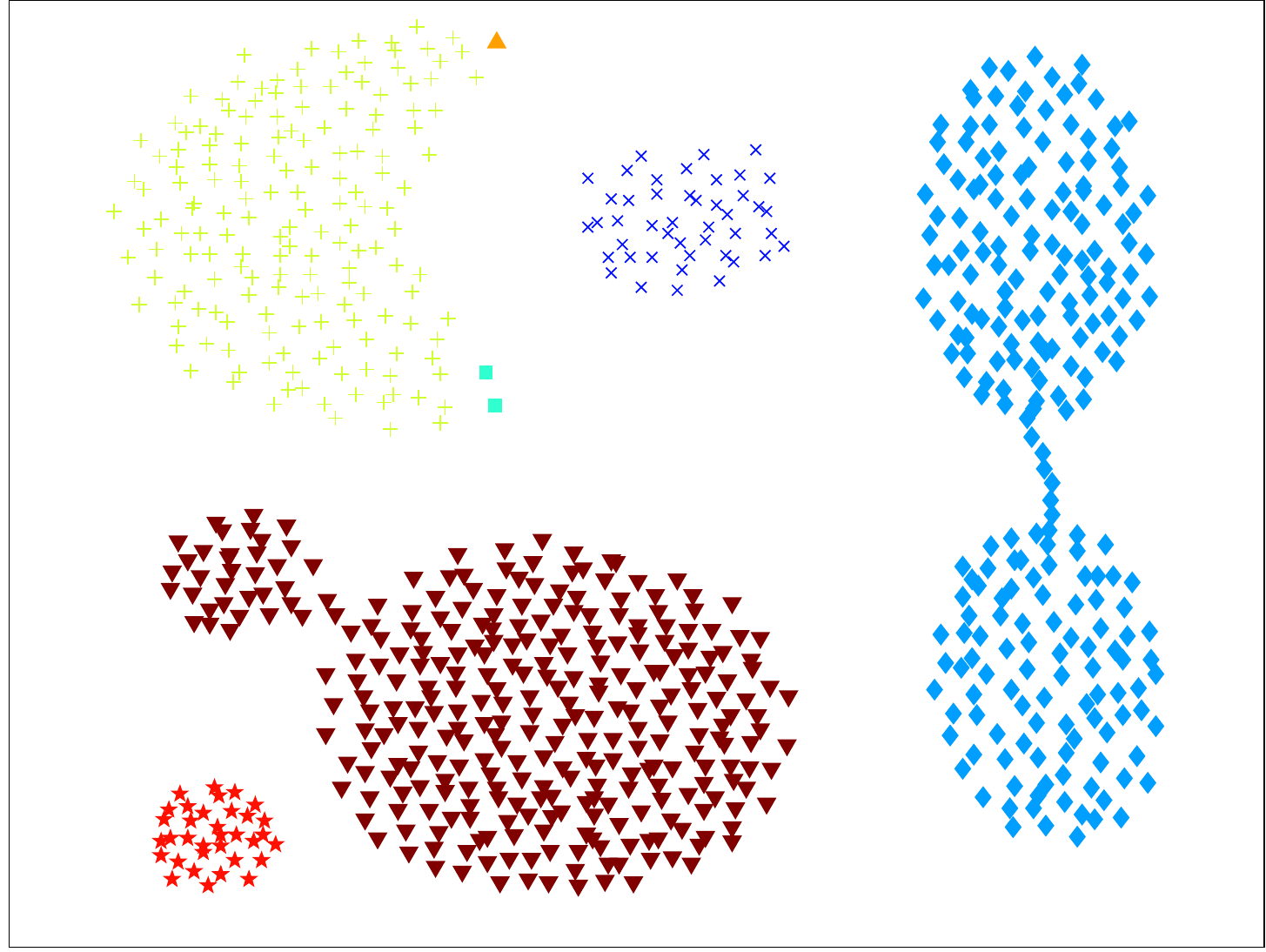}}
	\end{minipage}
	\hfill
	\begin{minipage}{0.19\linewidth}
		\centerline{\includegraphics[width=1\textwidth]{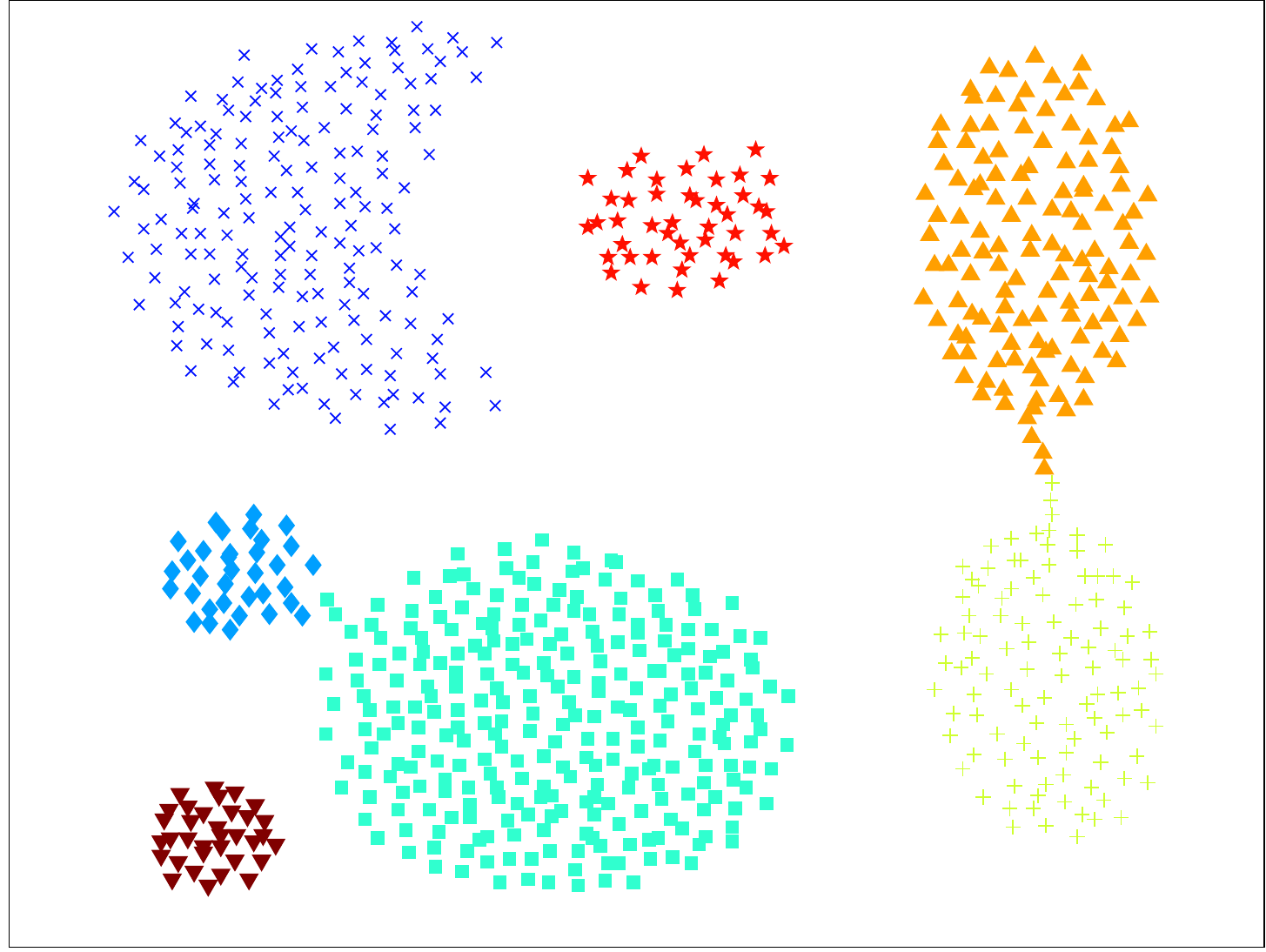}}
	\end{minipage}
	\hfill
	\begin{minipage}{0.19\linewidth}
		\centerline{\includegraphics[width=1\textwidth]{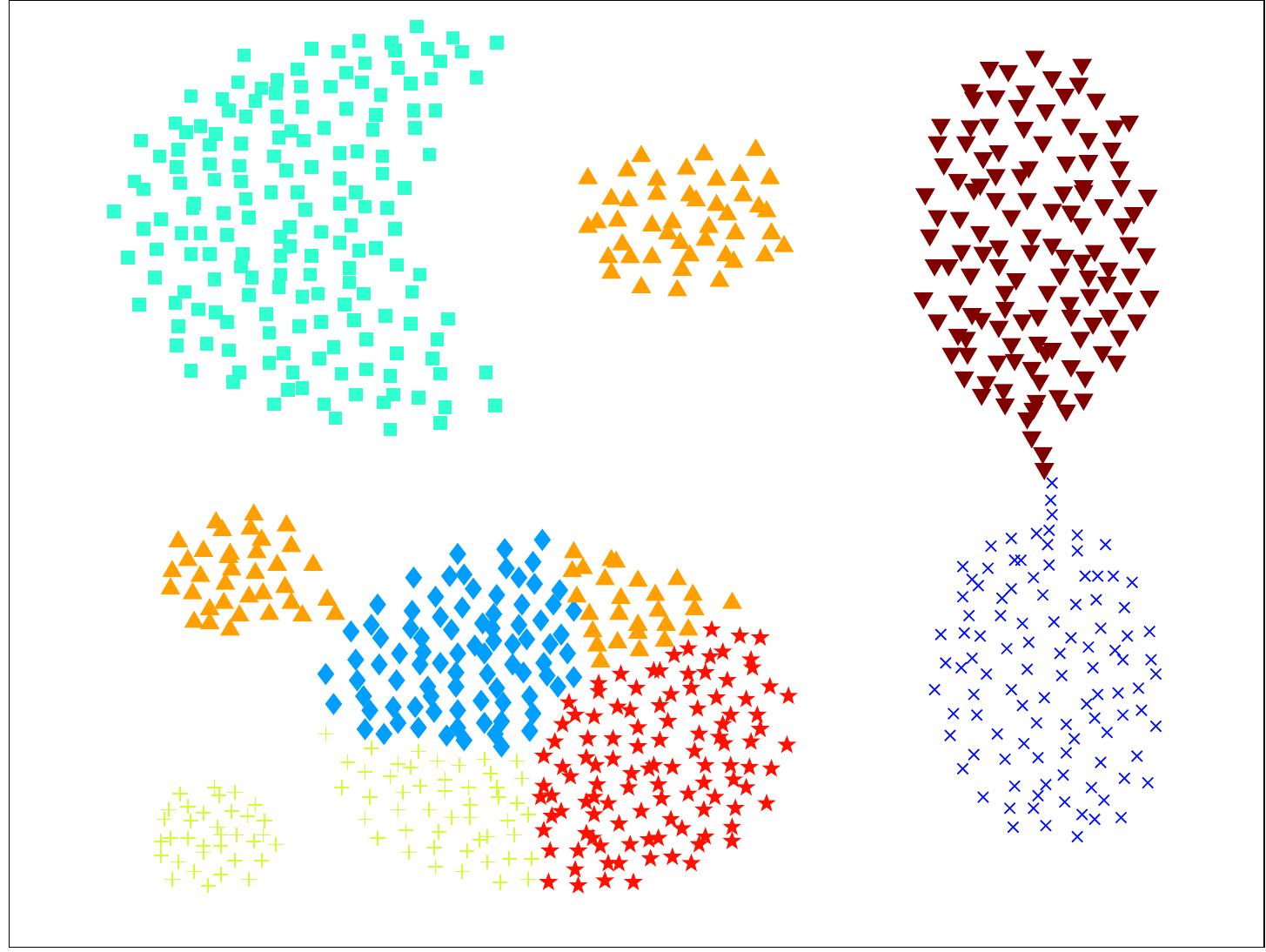}}
	\end{minipage}
	\hfill
	\begin{minipage}{0.19\linewidth}
		\centerline{\includegraphics[width=1\textwidth]{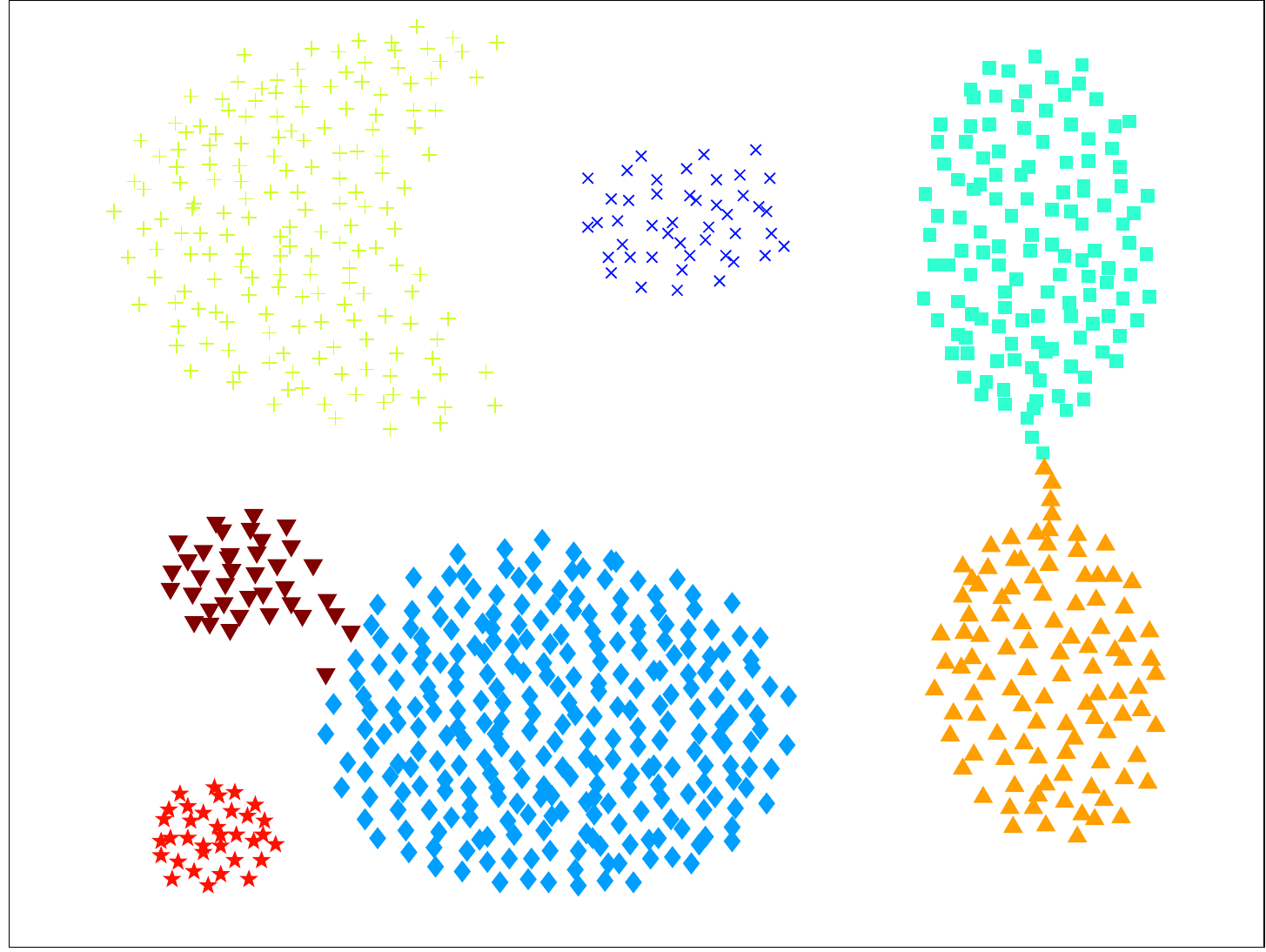}}
	\end{minipage}
	\hfill
	\begin{minipage}{0.19\linewidth}
		\centerline{\includegraphics[width=1\textwidth]{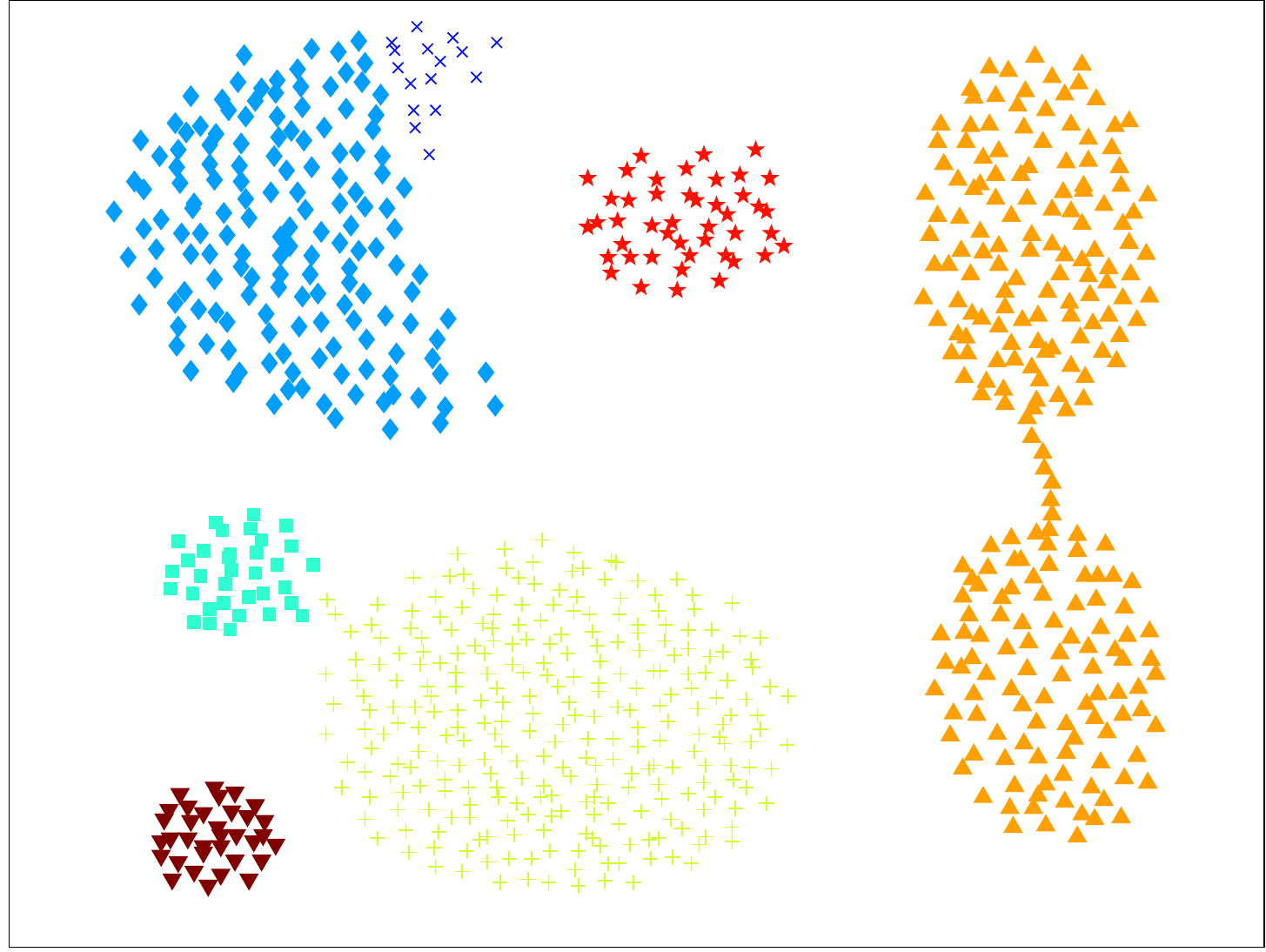}}
	\end{minipage}
	\vfill
	\begin{minipage}{0\linewidth}
		\rightline{E}
	\end{minipage}
	\hfill
	\begin{minipage}{0.19\linewidth}
		\centerline{\includegraphics[width=1\textwidth]{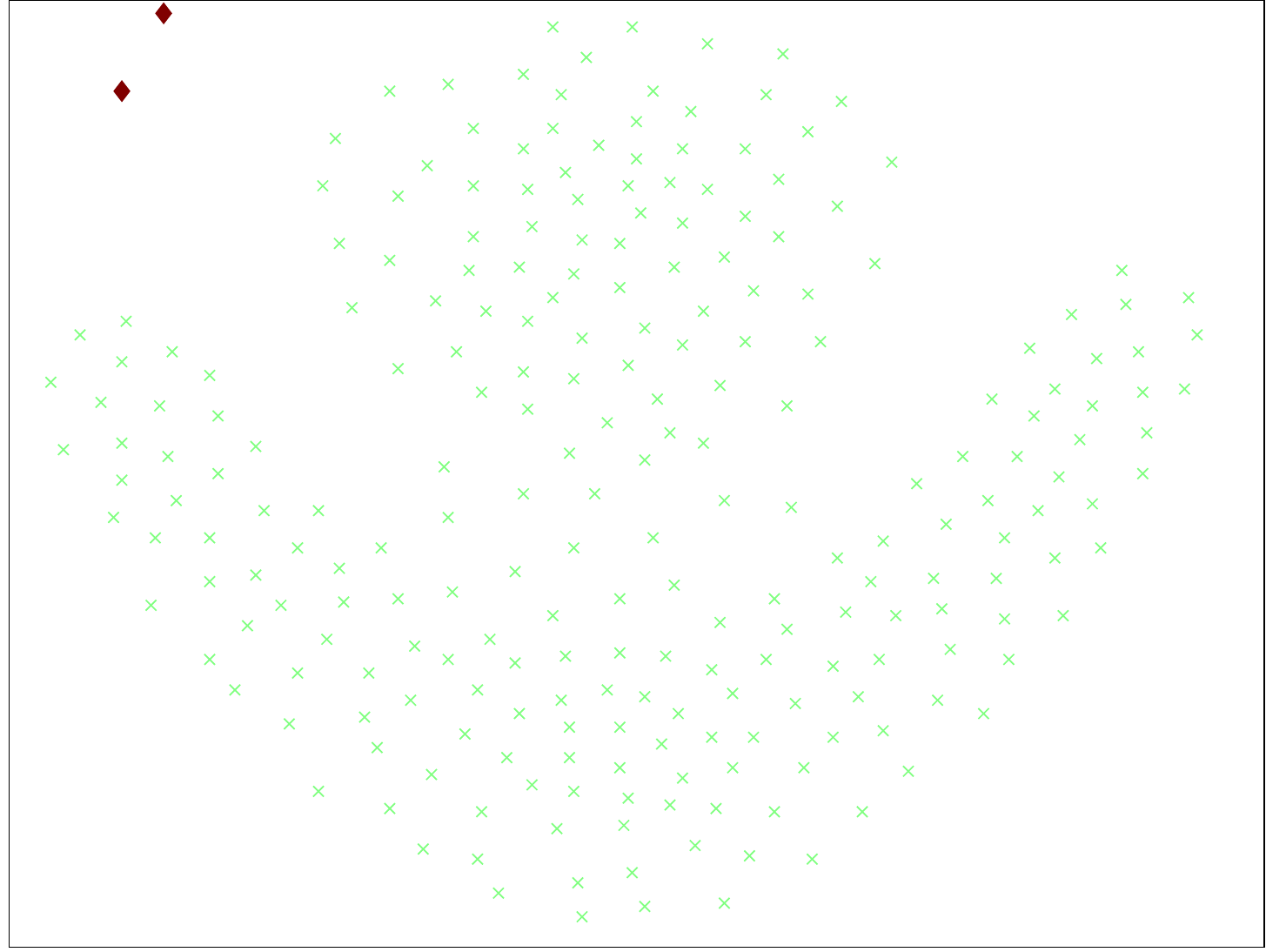}}
	\end{minipage}
	\hfill
	\begin{minipage}{0.19\linewidth}
		\centerline{\includegraphics[width=1\textwidth]{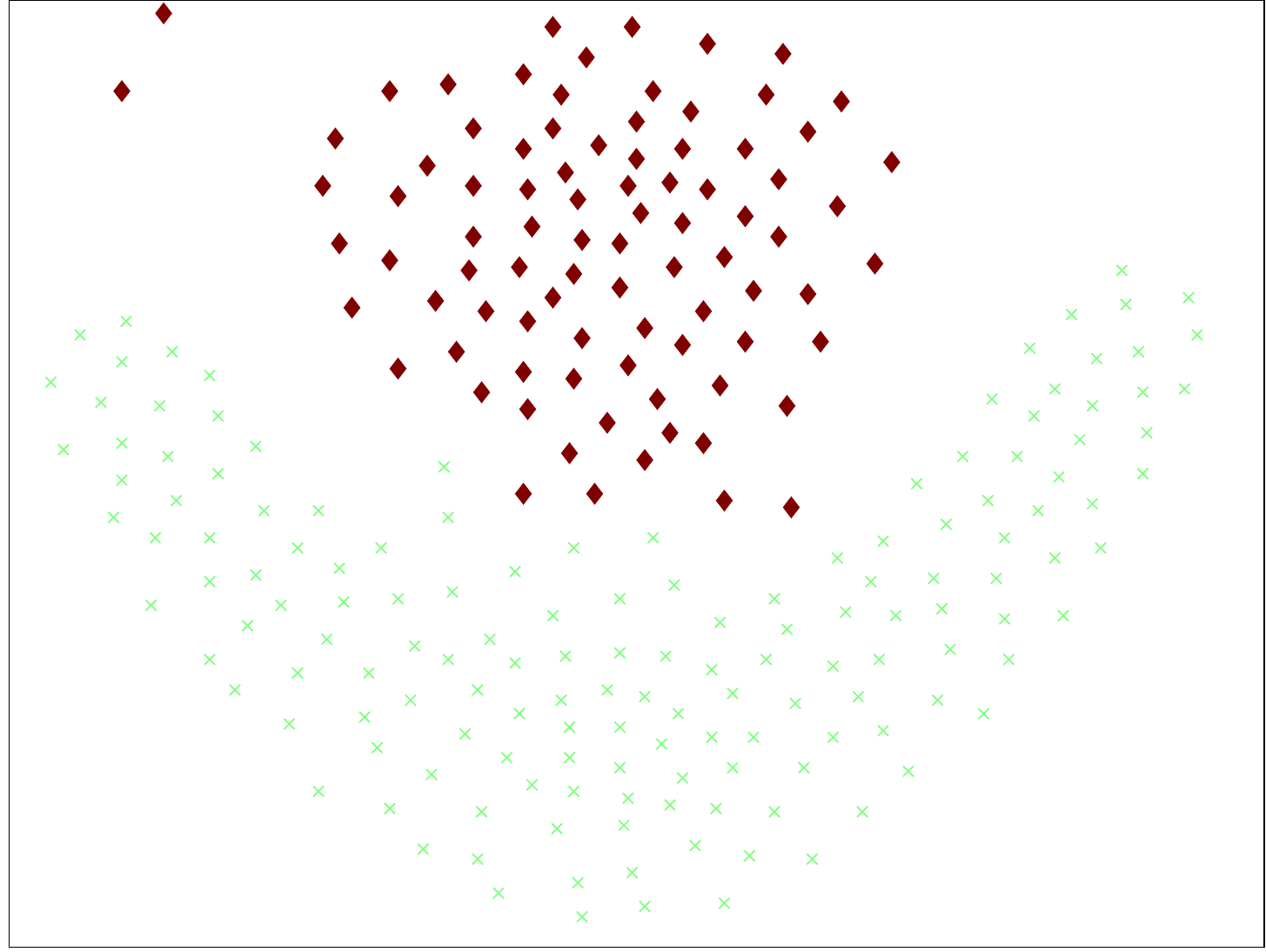}}
	\end{minipage}
	\hfill
	\begin{minipage}{0.19\linewidth}
		\centerline{\includegraphics[width=1\textwidth]{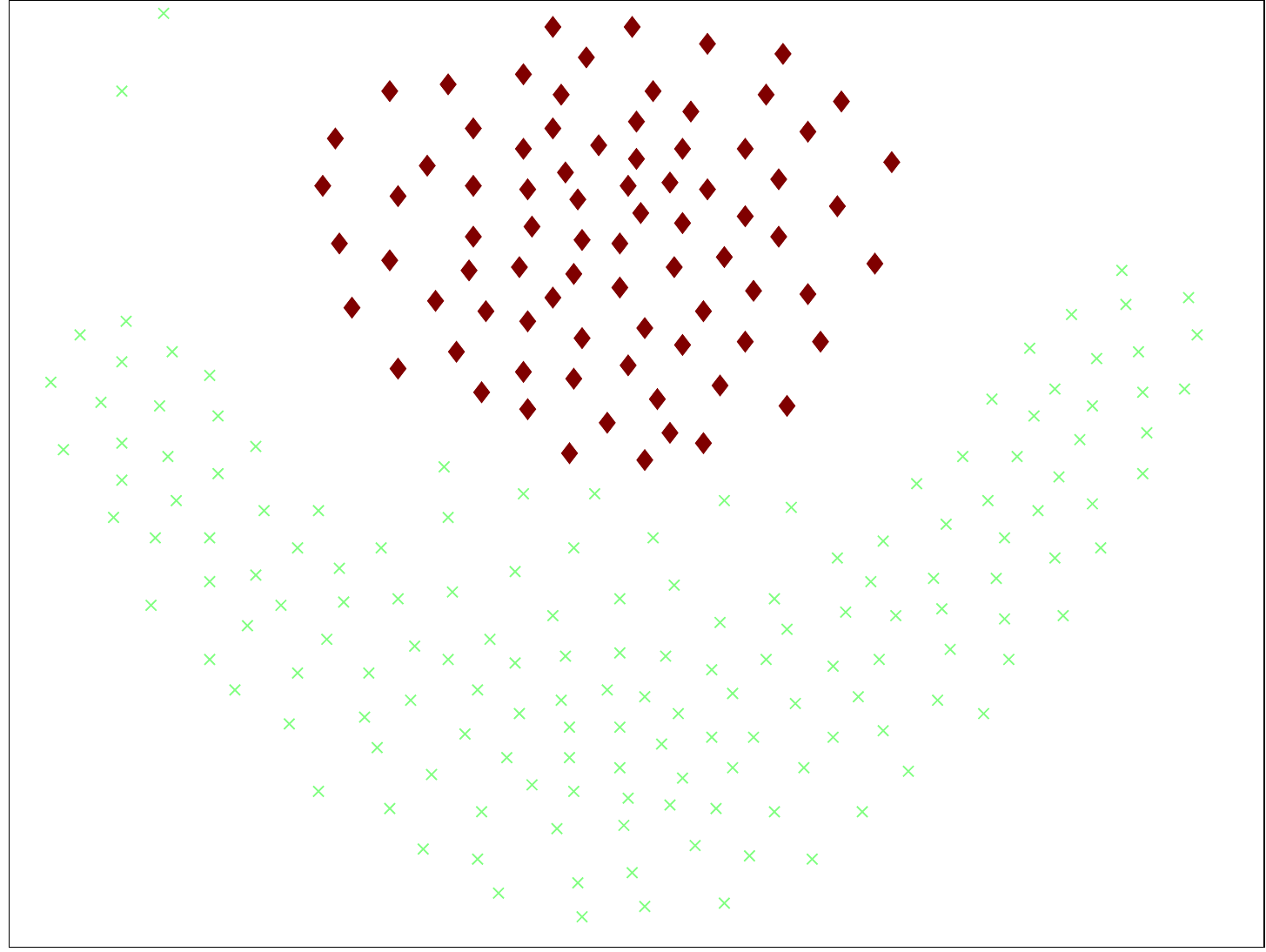}}
	\end{minipage}
	\hfill
	\begin{minipage}{0.19\linewidth}
		\centerline{\includegraphics[width=1\textwidth]{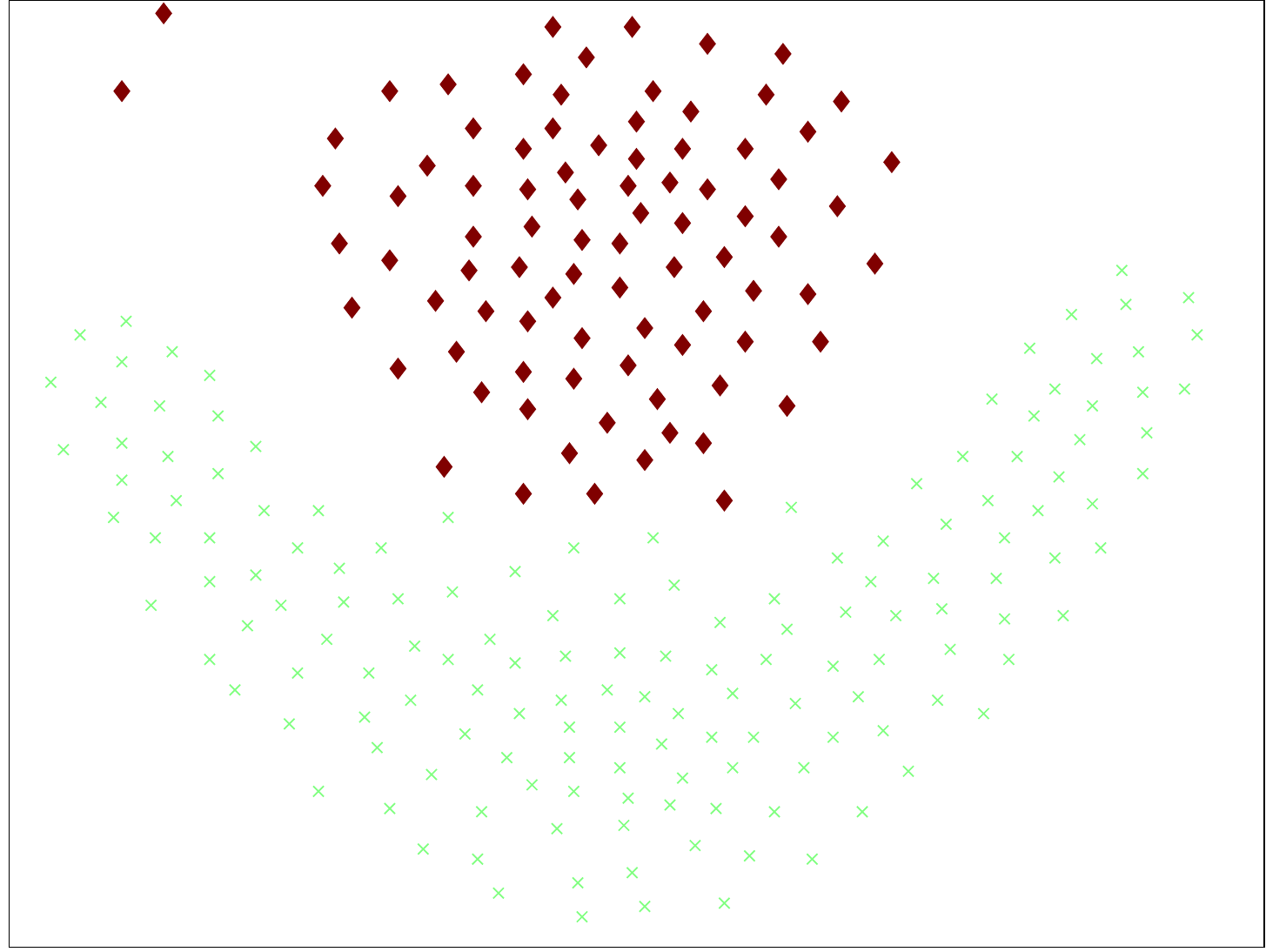}}
	\end{minipage}
	\hfill
	\begin{minipage}{0.19\linewidth}
		\centerline{\includegraphics[width=1\textwidth]{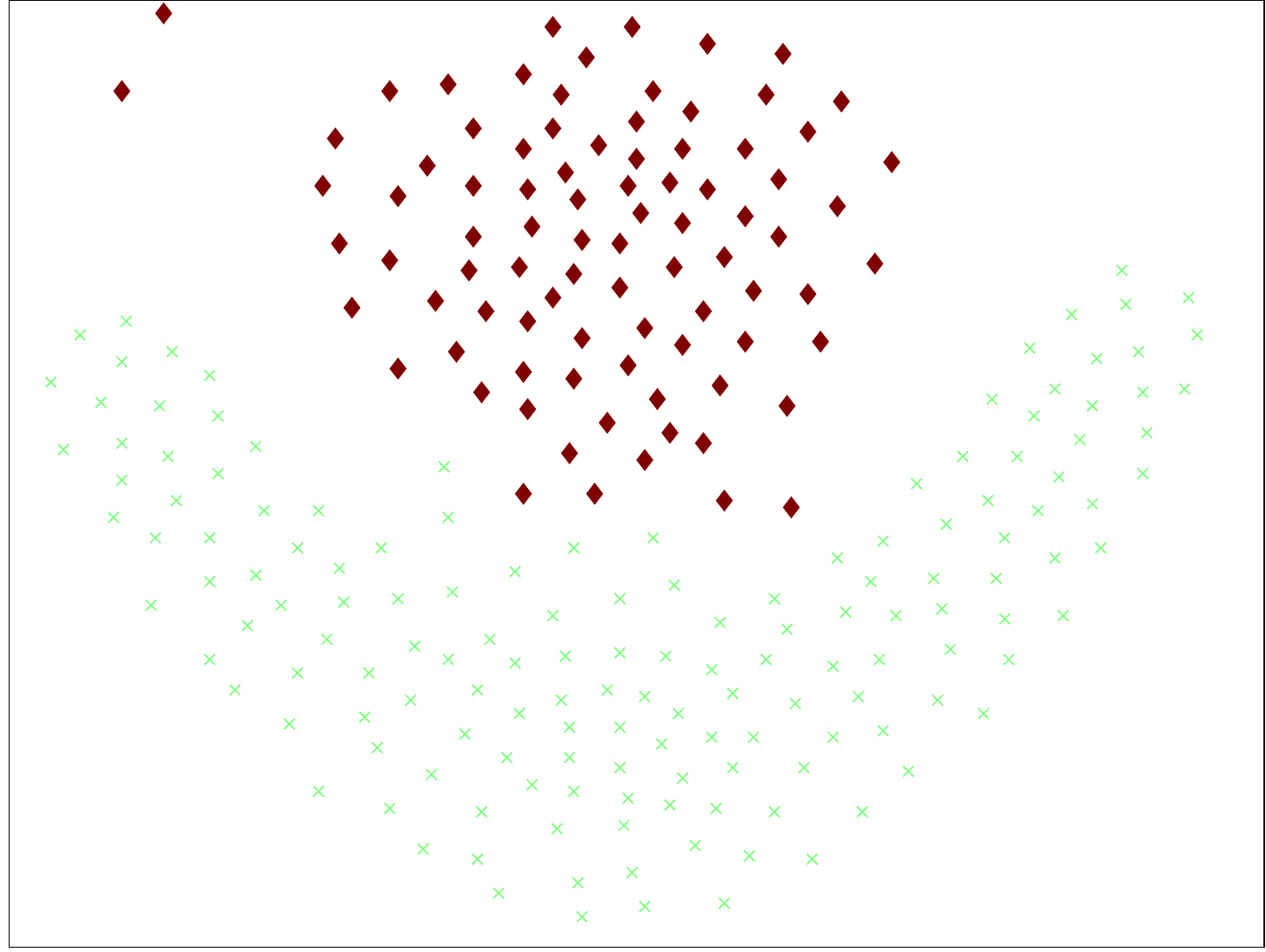}}
	\end{minipage}
	\vfill
	\begin{minipage}{0\linewidth}
		\rightline{F}
	\end{minipage}
	\hfill
	\begin{minipage}{0.19\linewidth}
		\centerline{\includegraphics[width=1\textwidth]{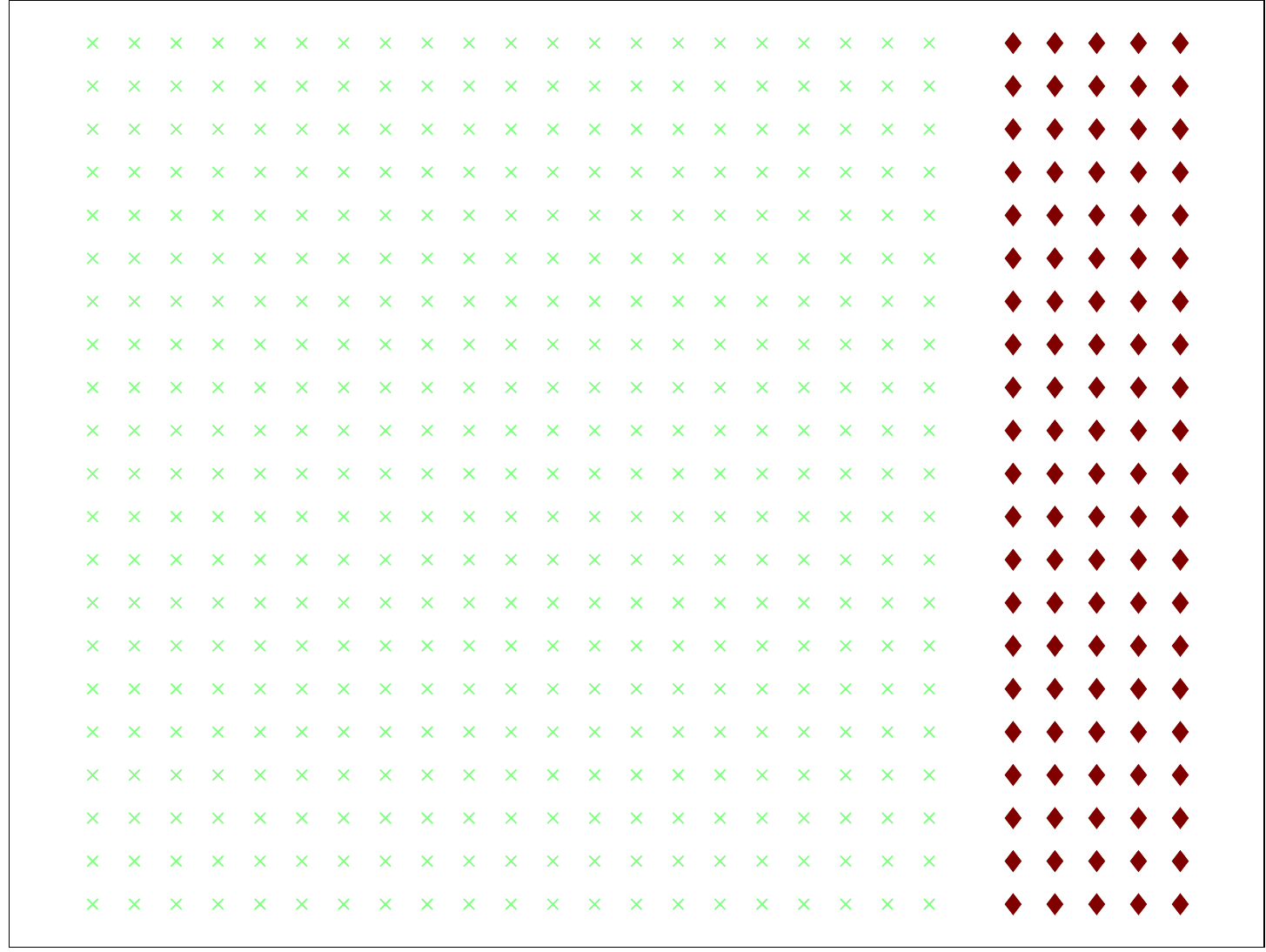}}
	\end{minipage}
	\hfill
	\begin{minipage}{0.19\linewidth}
		\centerline{\includegraphics[width=1\textwidth]{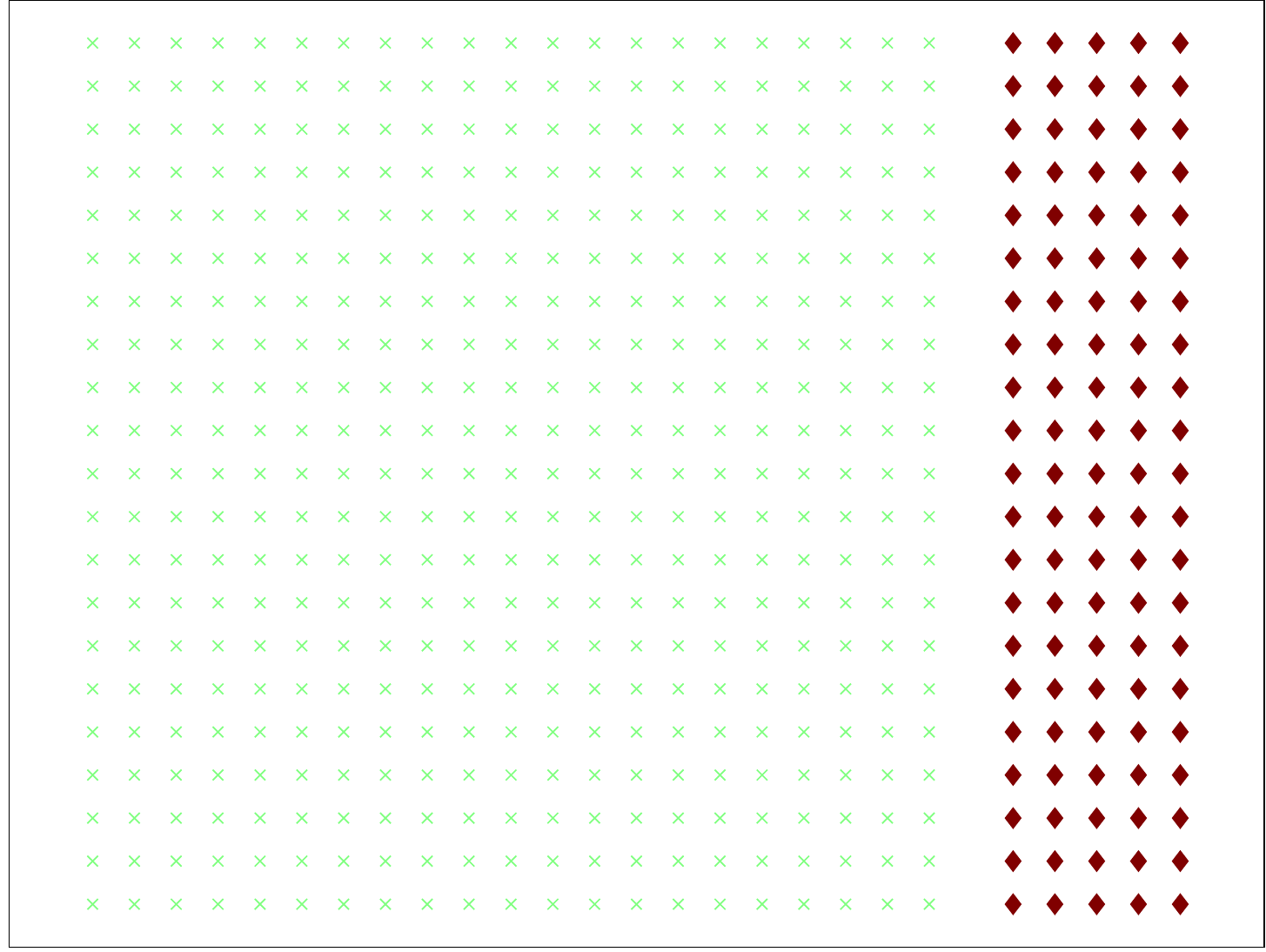}}
	\end{minipage}
	\hfill
	\begin{minipage}{0.19\linewidth}
		\centerline{\includegraphics[width=1\textwidth]{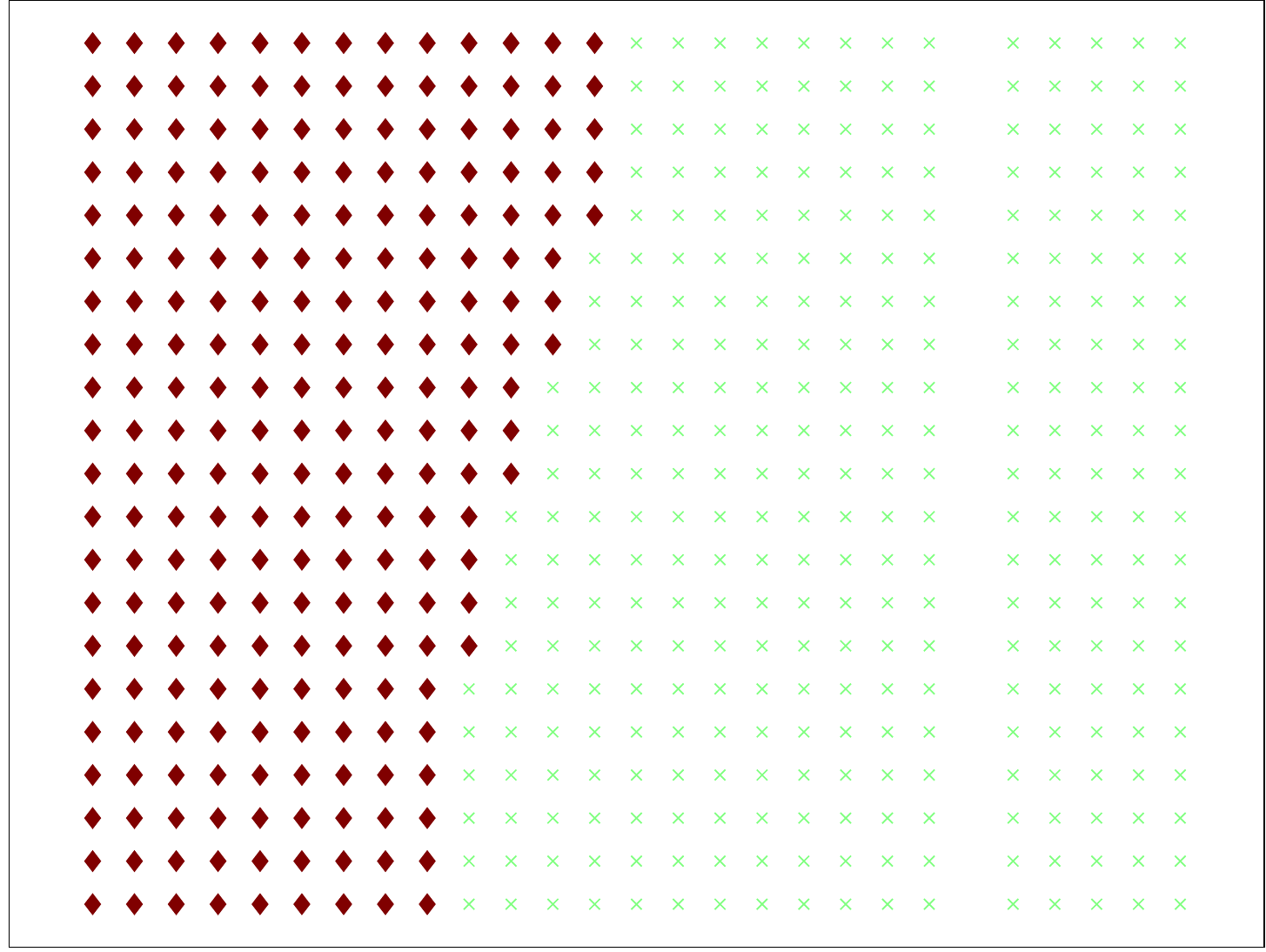}}
	\end{minipage}
	\hfill
	\begin{minipage}{0.19\linewidth}
		\centerline{\includegraphics[width=1\textwidth]{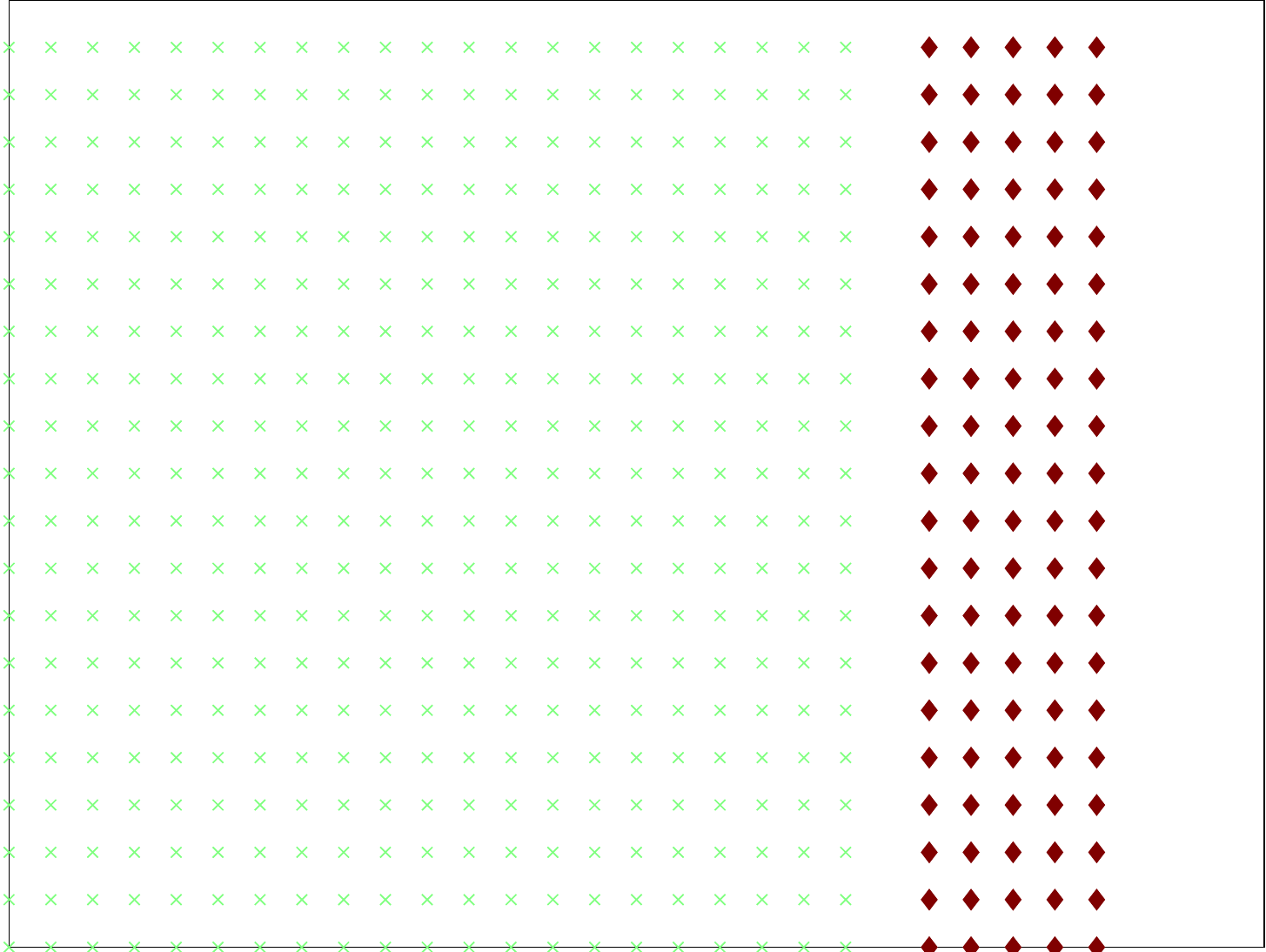}}
	\end{minipage}
	\hfill
	\begin{minipage}{0.19\linewidth}
		\centerline{\includegraphics[width=1\textwidth]{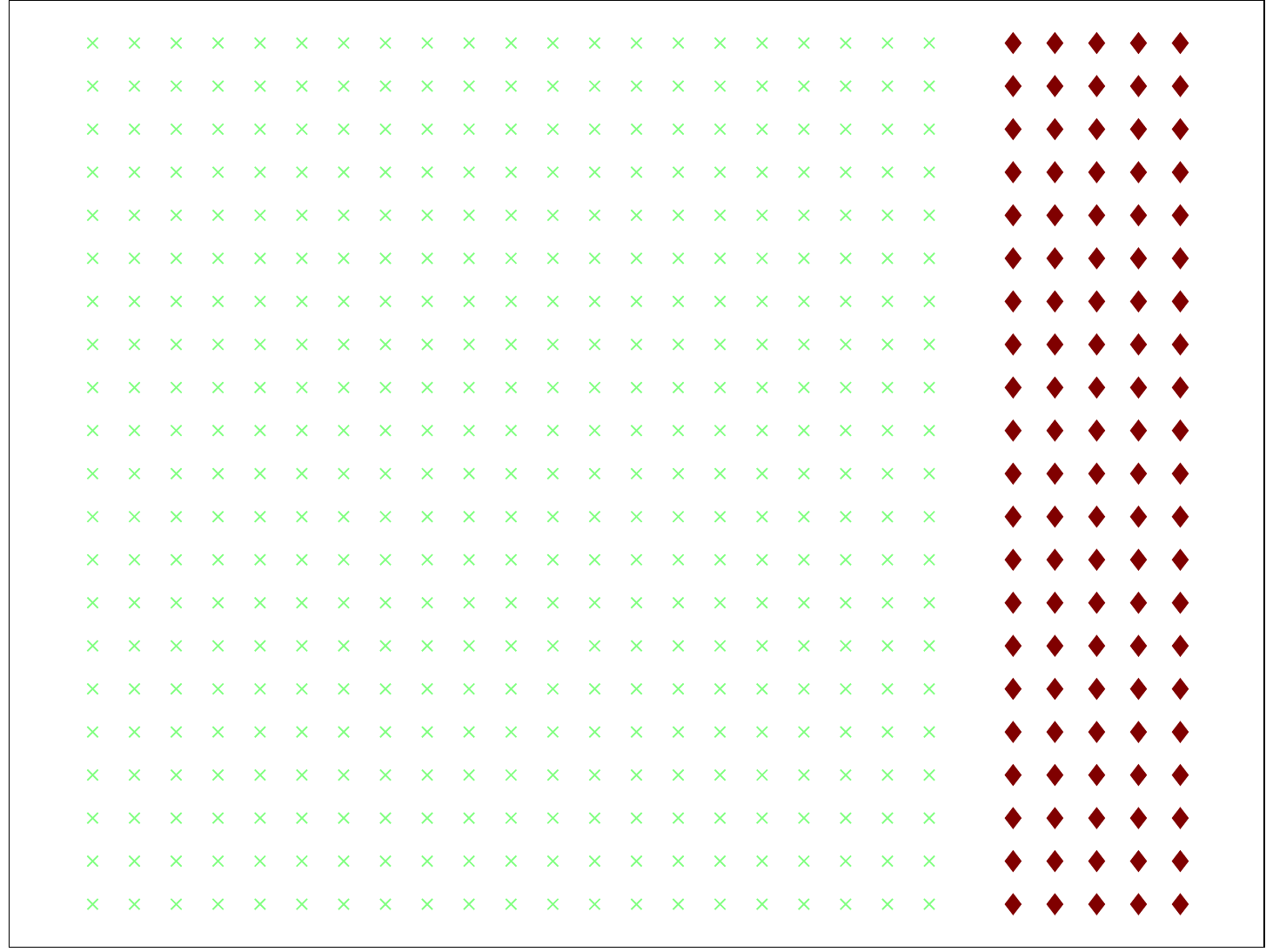}}
	\end{minipage}
	\vfill
	\begin{minipage}{0\linewidth}
		\rightline{G}
	\end{minipage}
	\hfill
	\begin{minipage}{0.19\linewidth}
		\centerline{\includegraphics[width=1\textwidth]{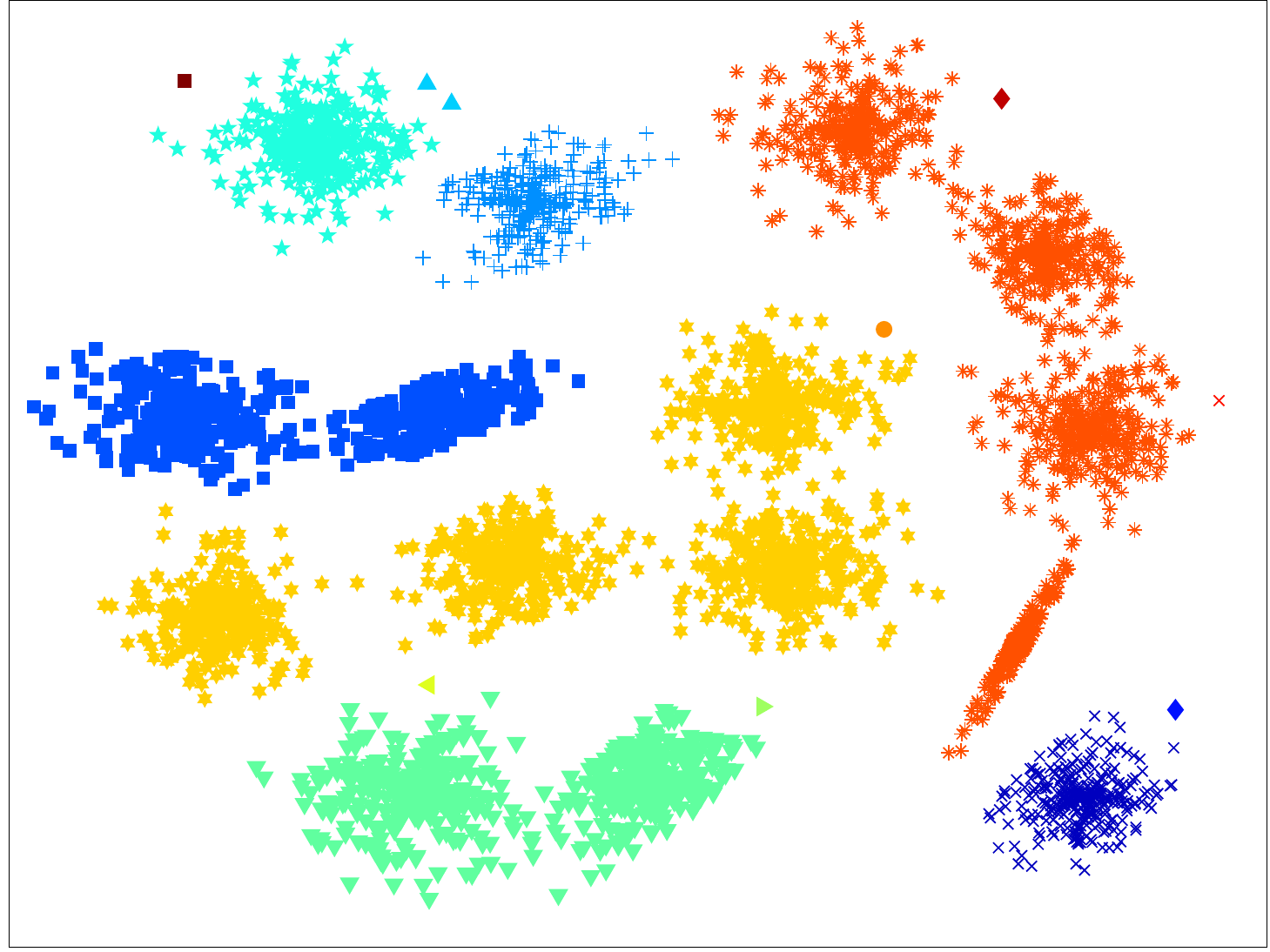}}
	\end{minipage}
	\hfill
	\begin{minipage}{0.19\linewidth}
		\centerline{\includegraphics[width=1\textwidth]{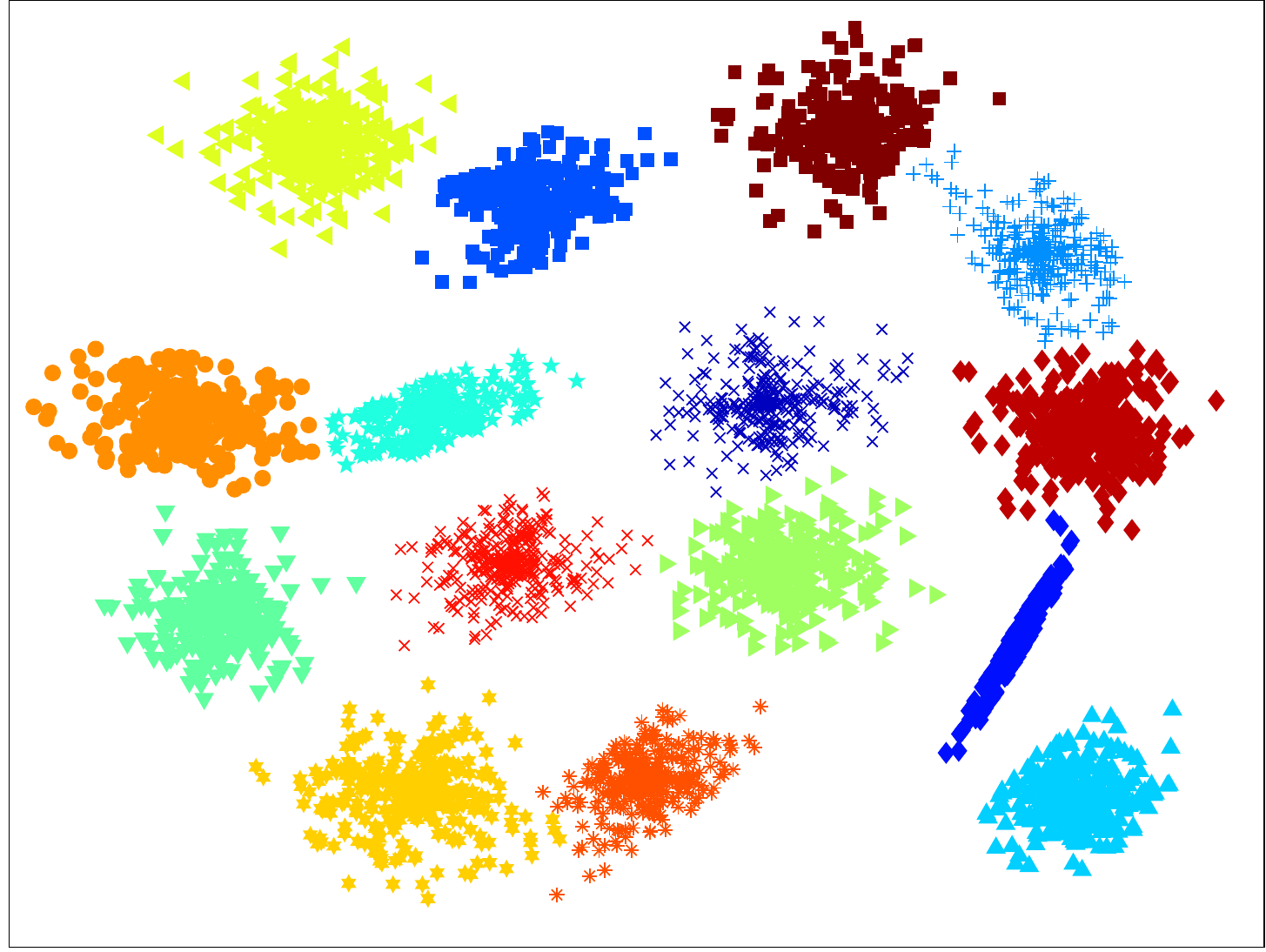}}
	\end{minipage}
	\hfill
	\begin{minipage}{0.19\linewidth}
		\centerline{\includegraphics[width=1\textwidth]{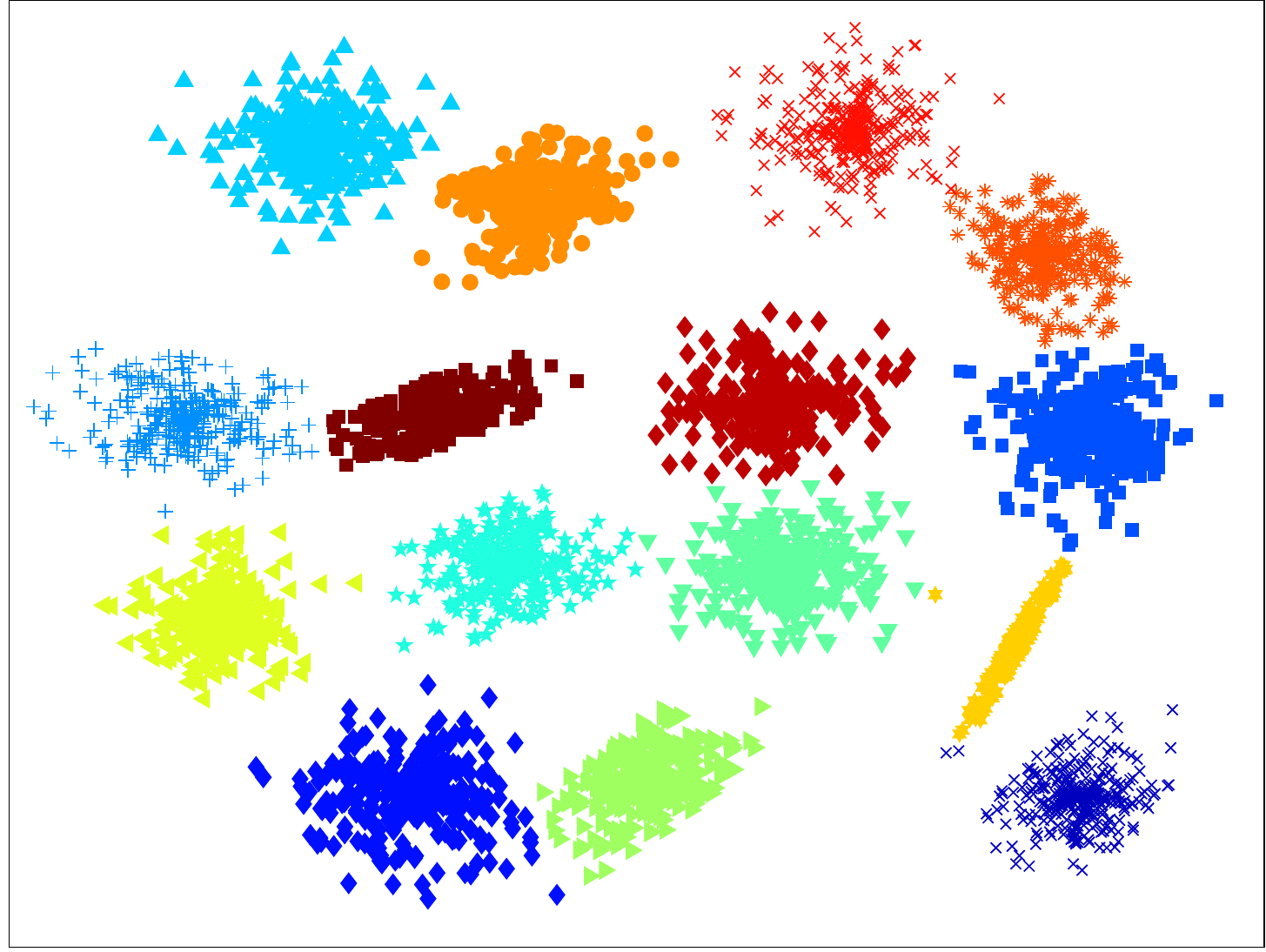}}
	\end{minipage}
	\hfill
	\begin{minipage}{0.19\linewidth}
		\centerline{\includegraphics[width=1\textwidth]{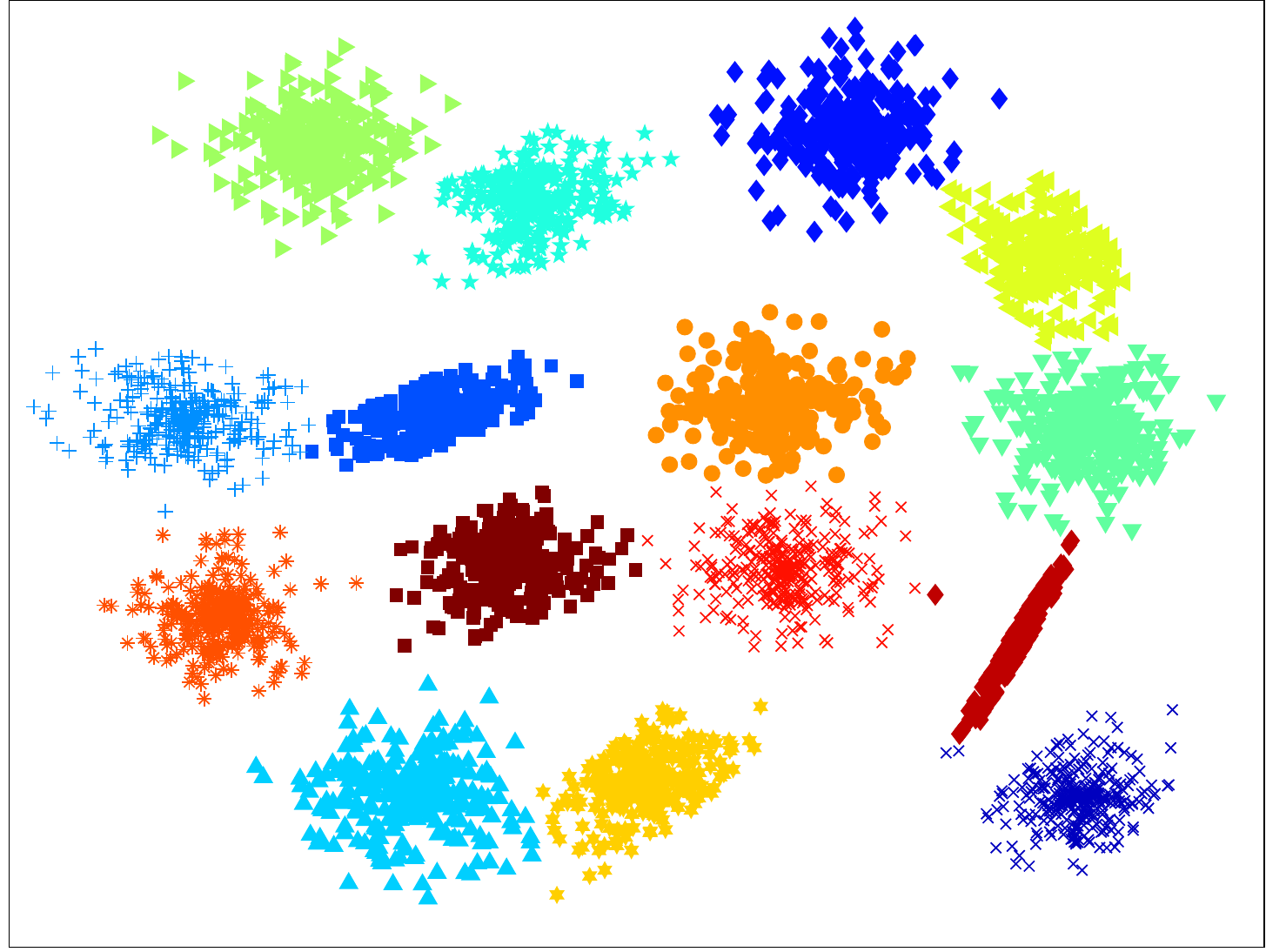}}
	\end{minipage}
	\hfill
	\begin{minipage}{0.19\linewidth}
		\centerline{\includegraphics[width=1\textwidth]{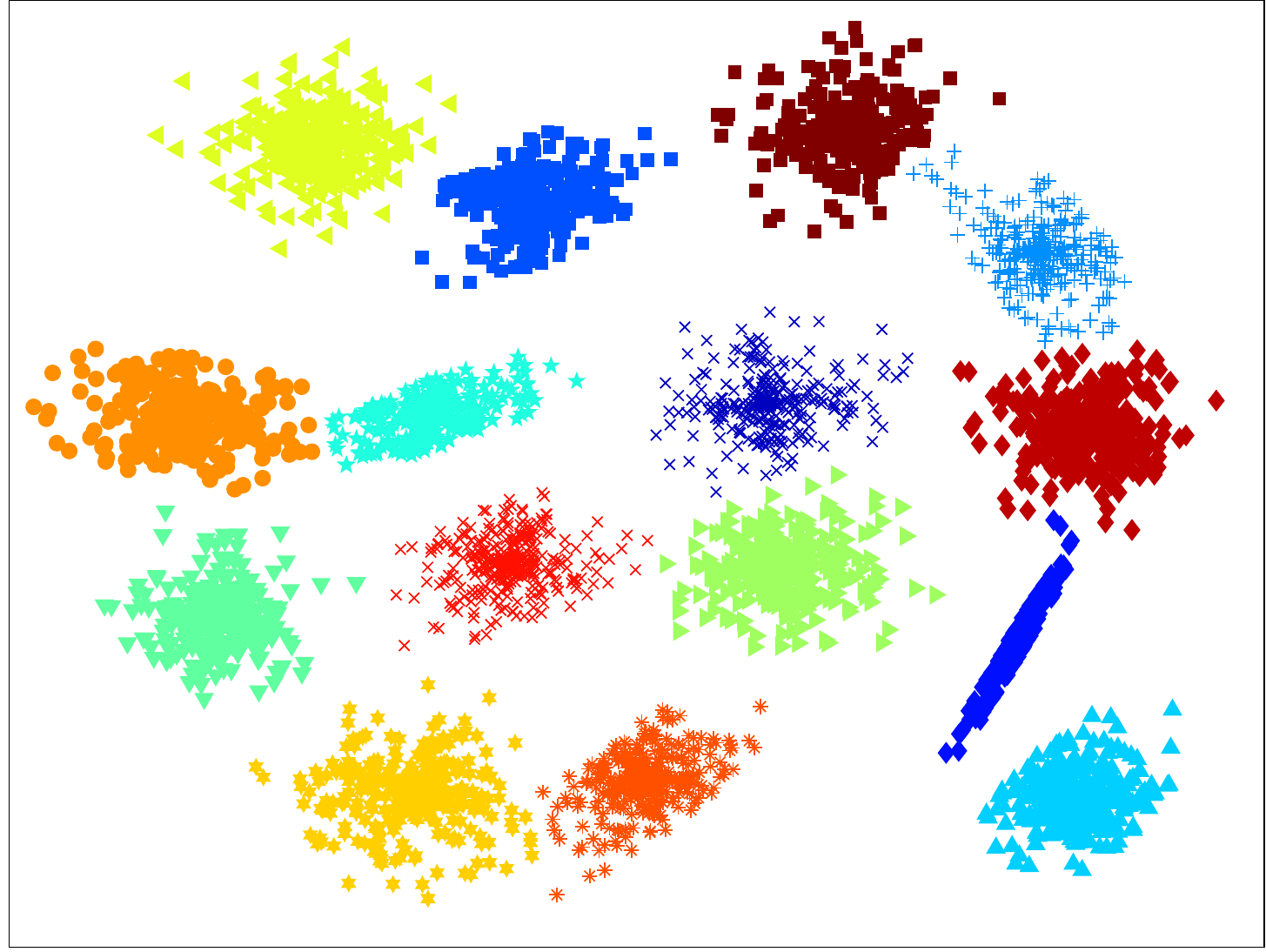}}
	\end{minipage}
	\vfill
	\begin{minipage}{0\linewidth}
		\rightline{H}
	\end{minipage}
	\hfill
	\begin{minipage}{0.19\linewidth}
		\centerline{\includegraphics[width=1\textwidth]{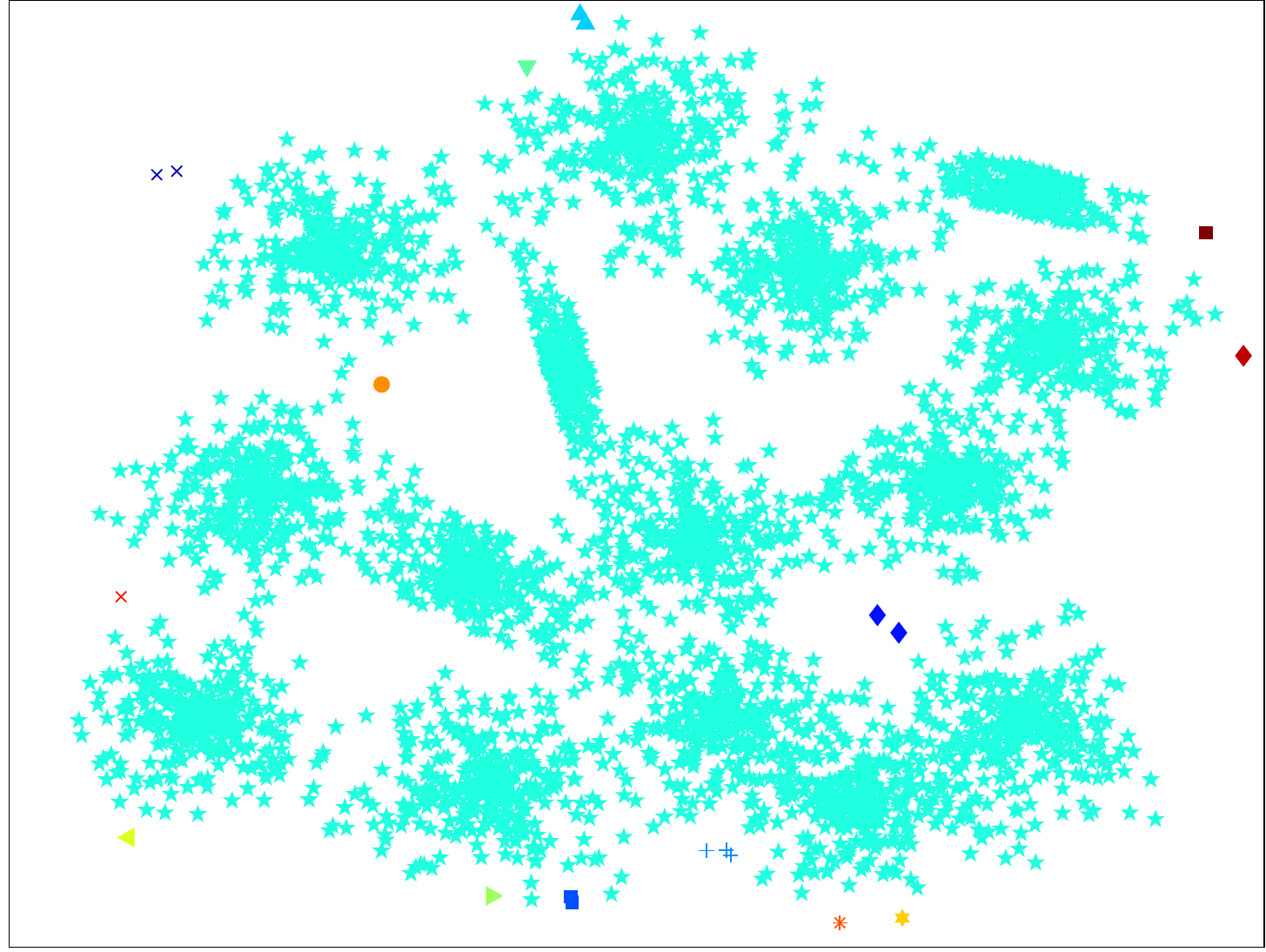}}
		\centerline{SLC}
	\end{minipage}
	\begin{minipage}{0.19\linewidth}
		\centerline{\includegraphics[width=1\textwidth]{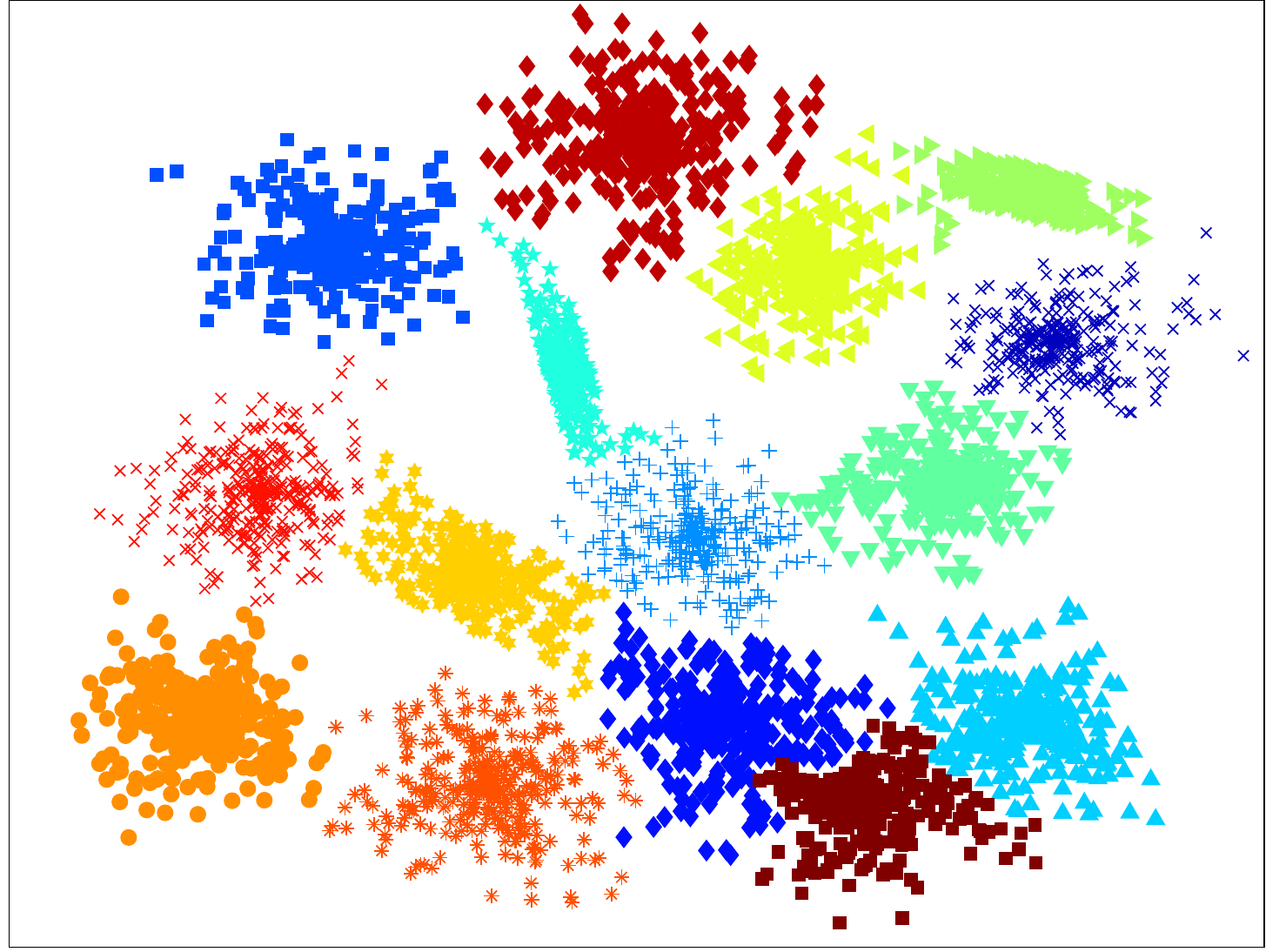}}
		\centerline{FDPC}
	\end{minipage}
	\hfill
	\begin{minipage}{0.19\linewidth}
		\centerline{\includegraphics[width=1\textwidth]{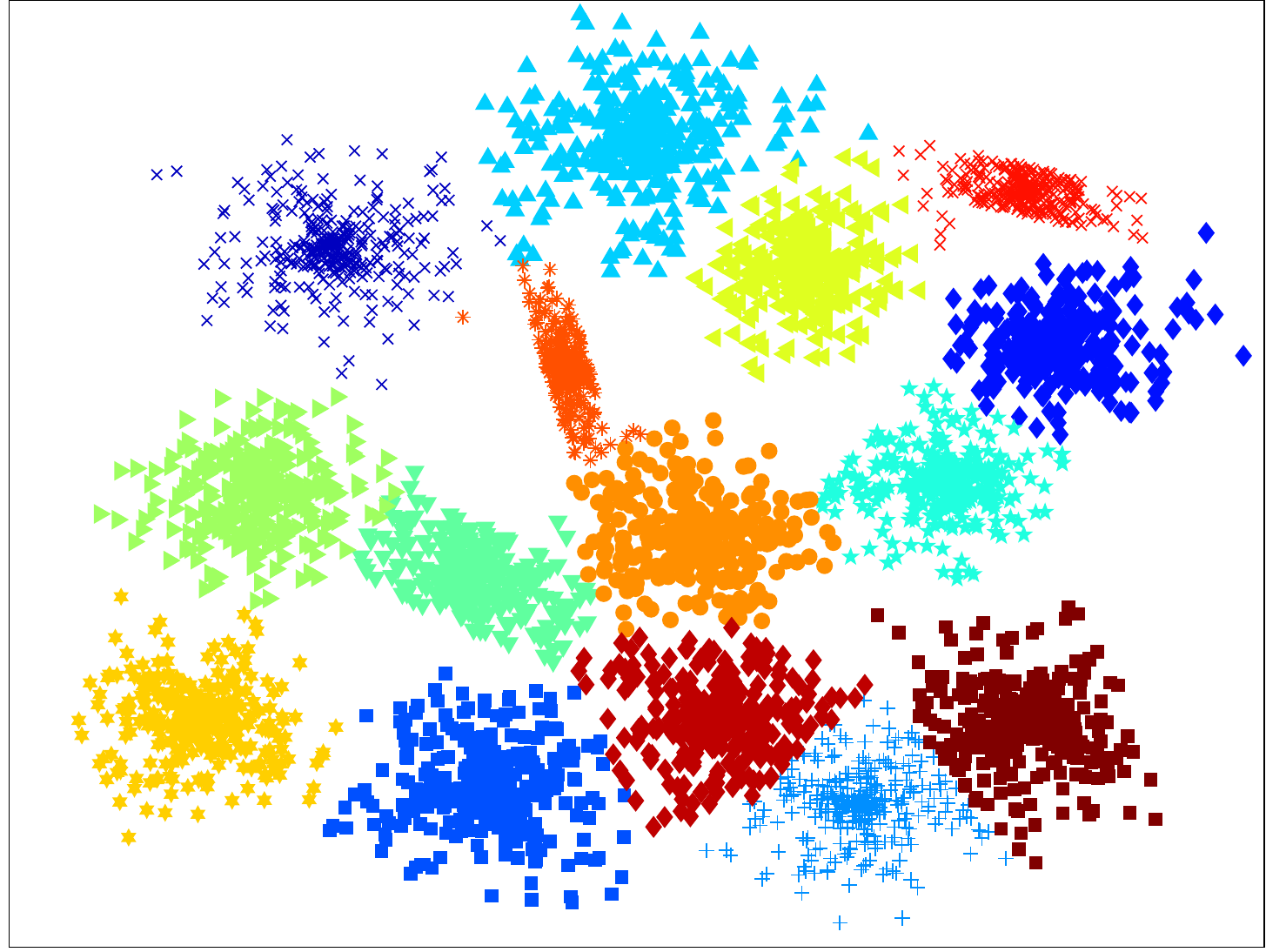}}
		\centerline{kernel KM}
	\end{minipage}
	\hfill
	\begin{minipage}{0.19\linewidth}
		\centerline{\includegraphics[width=1\textwidth]{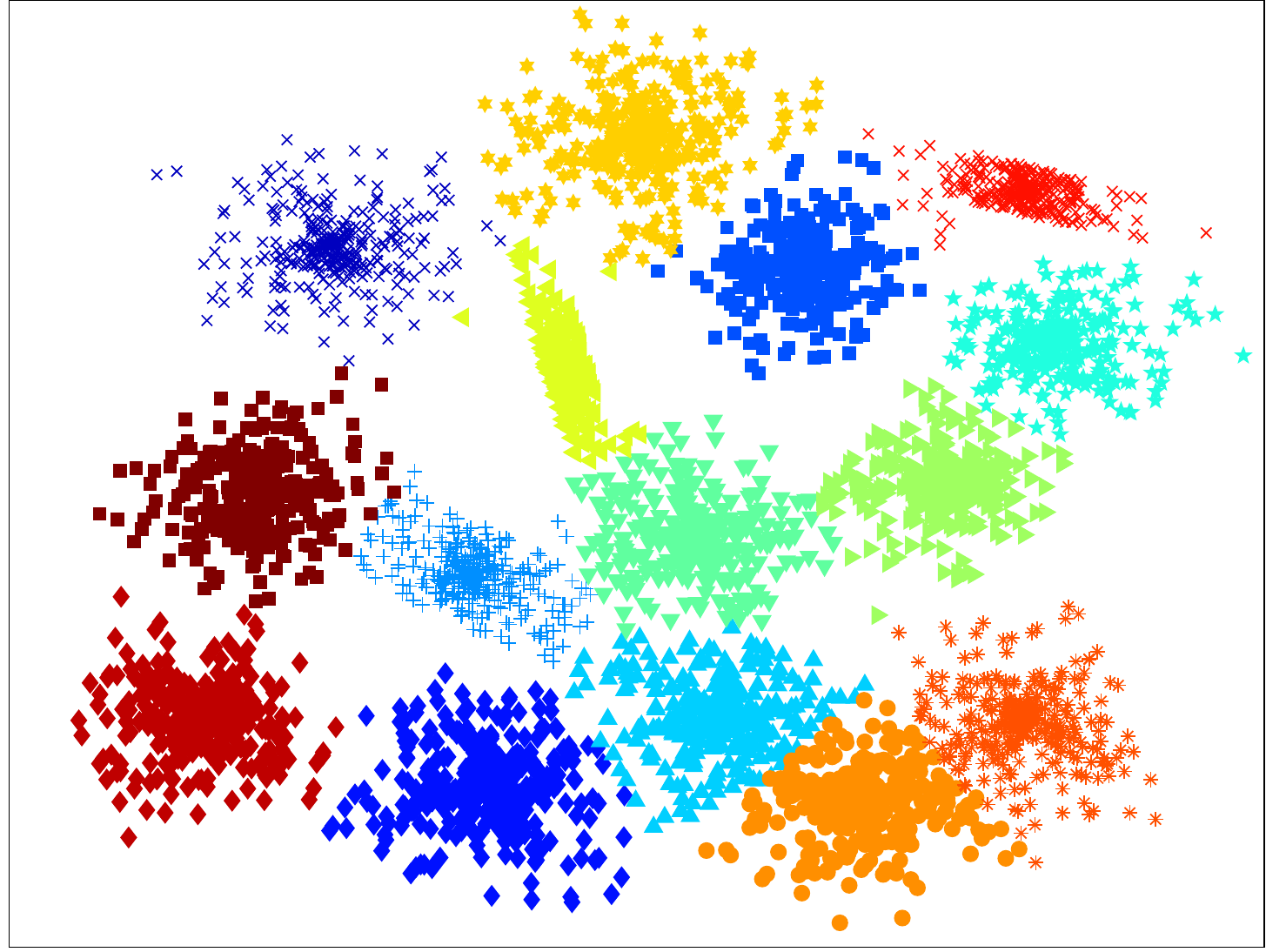}}
		\centerline{NCUT}
	\end{minipage}
	\hfill
	\begin{minipage}{0.19\linewidth}
		\centerline{\includegraphics[width=1\textwidth]{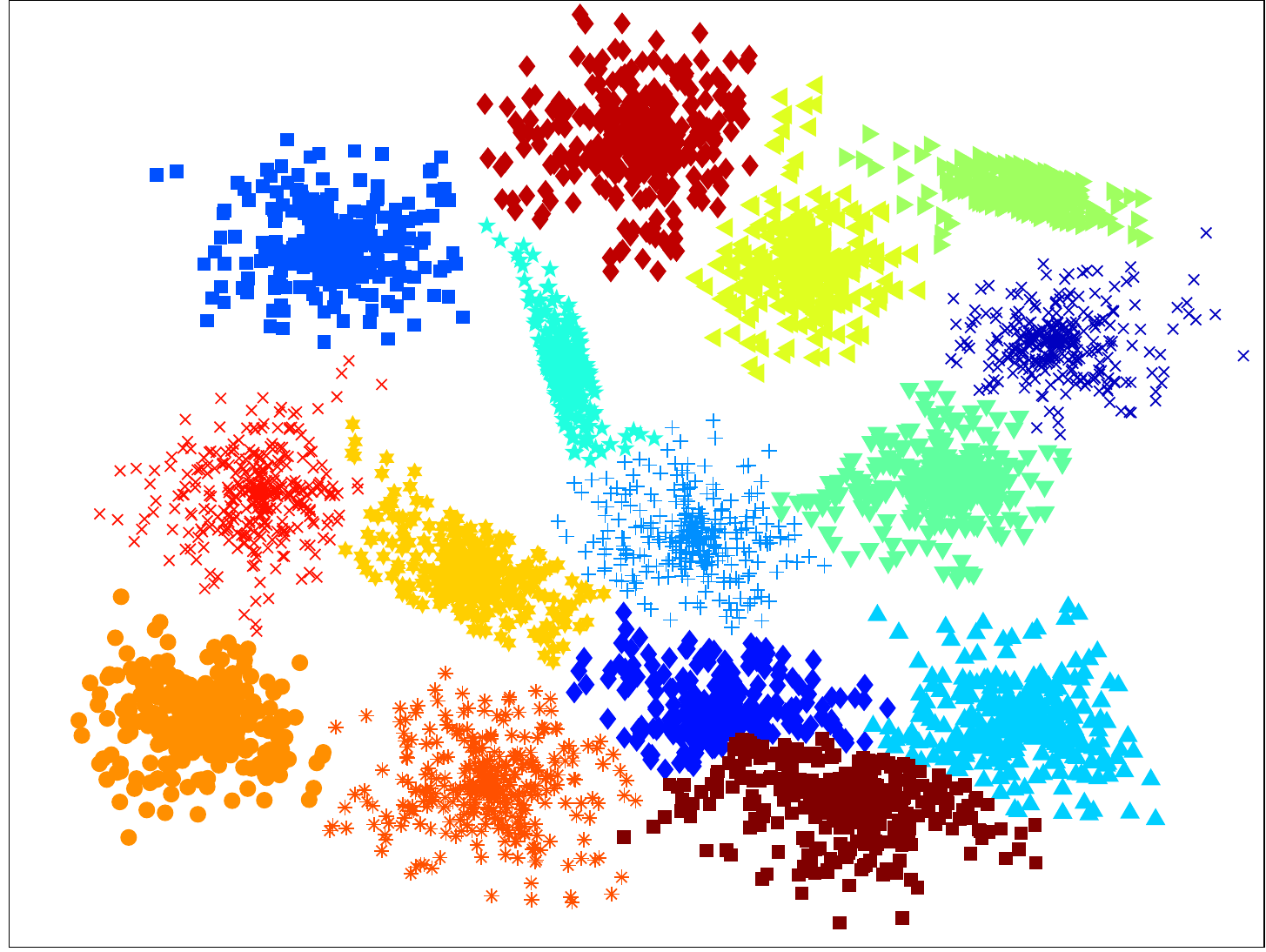}}
		\centerline{GOPC}
	\end{minipage}	
\caption{Experimental results on synthetic datasets; (A) $DS1$; (B) $DS2$; (C) $Spiral$; (D) $Aggregation$; (E) $Flame$; (F) $Lygd$; (G) $S1$; (H) $S2$.}
\label{synthetic1}
\end{figure*}

\begin{figure*}
	\begin{minipage}{0\linewidth}
		\rightline{I}
	\end{minipage}
	\hfill
	\begin{minipage}{0.19\linewidth}
		\centerline{\includegraphics[width=1\textwidth]{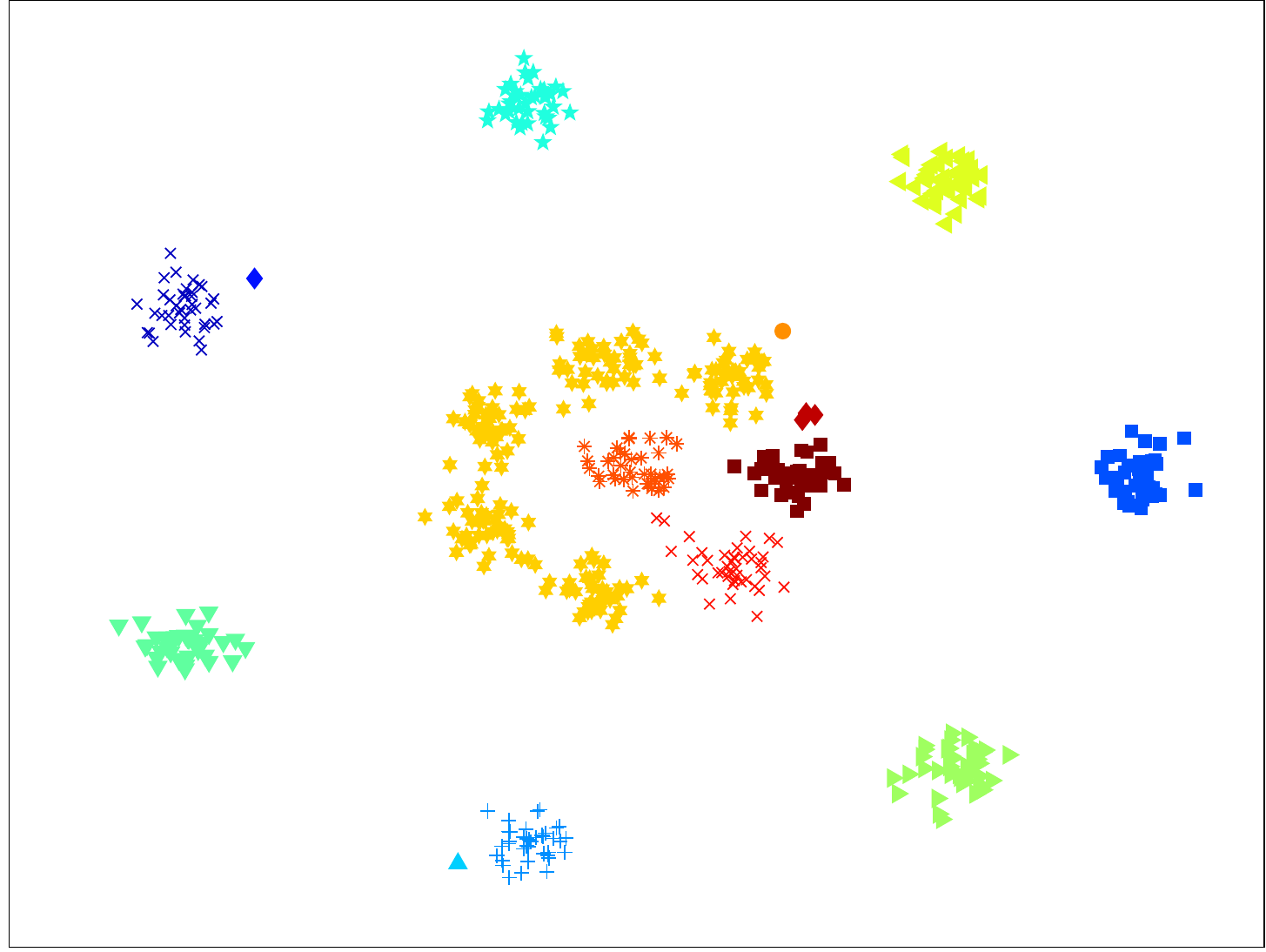}}
	\end{minipage}
	\hfill
	\begin{minipage}{0.19\linewidth}
		\centerline{\includegraphics[width=1\textwidth]{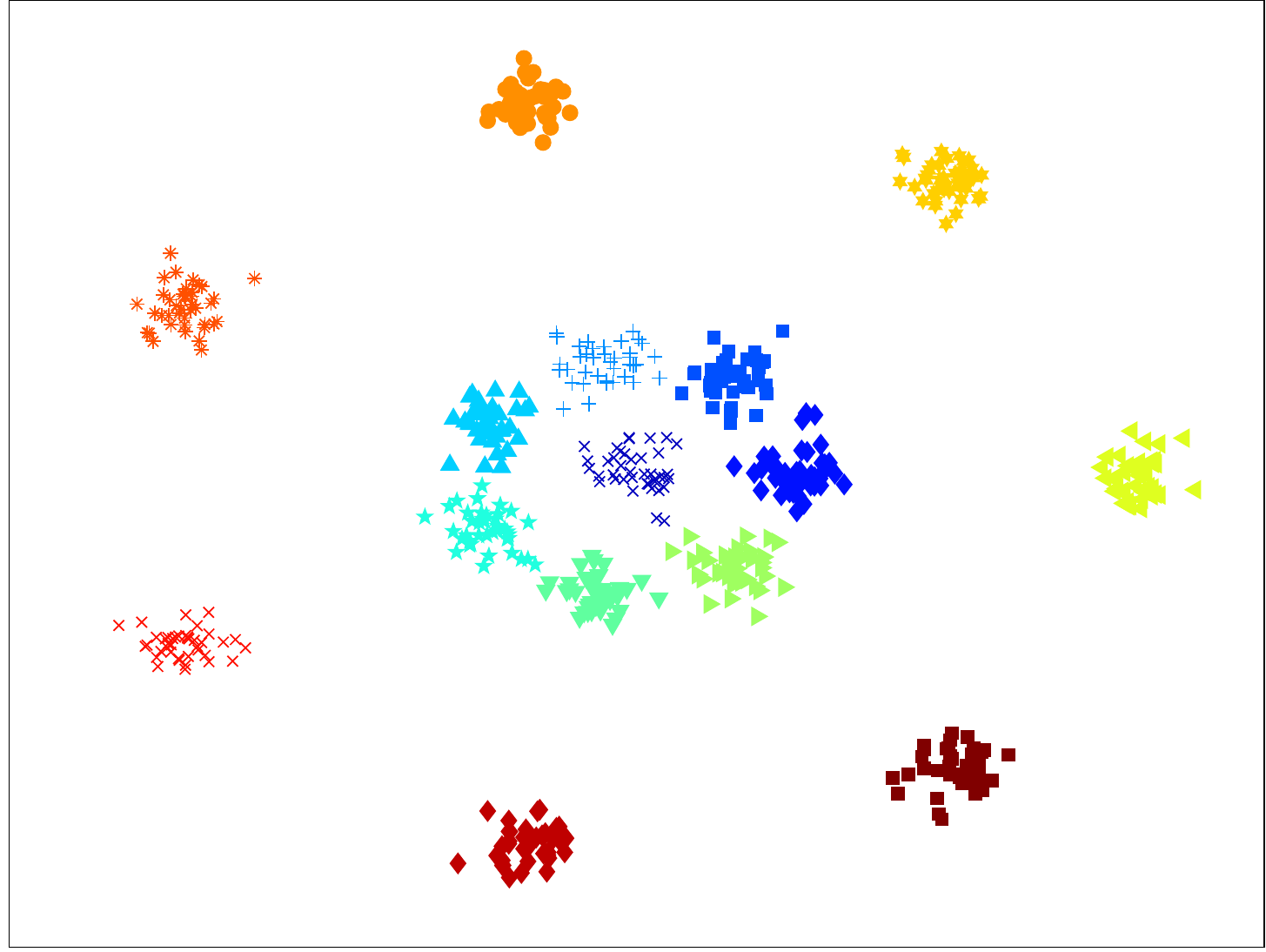}}
	\end{minipage}
	\hfill
	\begin{minipage}{0.19\linewidth}
		\centerline{\includegraphics[width=1\textwidth]{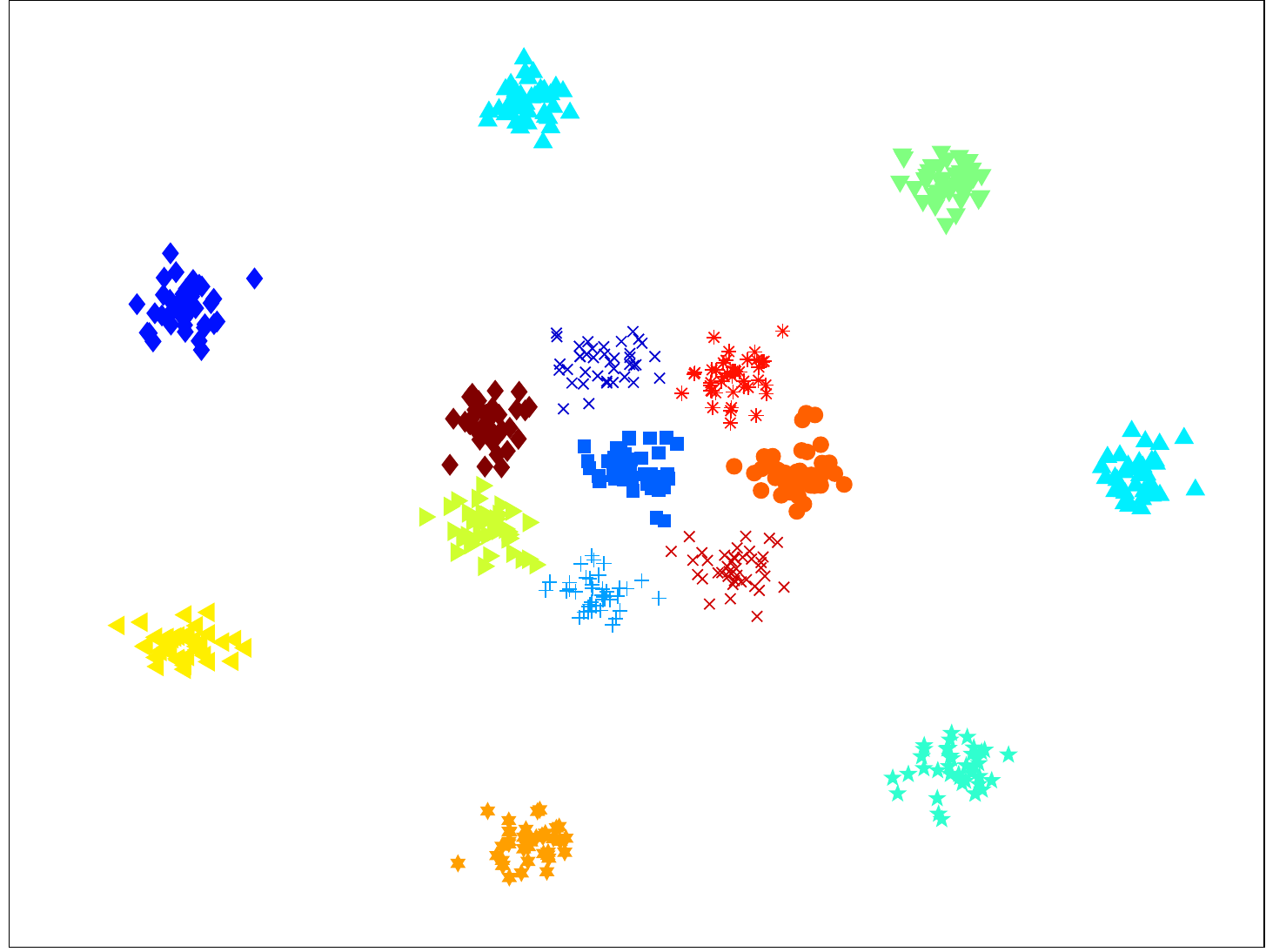}}
	\end{minipage}
	\hfill
	\begin{minipage}{0.19\linewidth}
		\centerline{\includegraphics[width=1\textwidth]{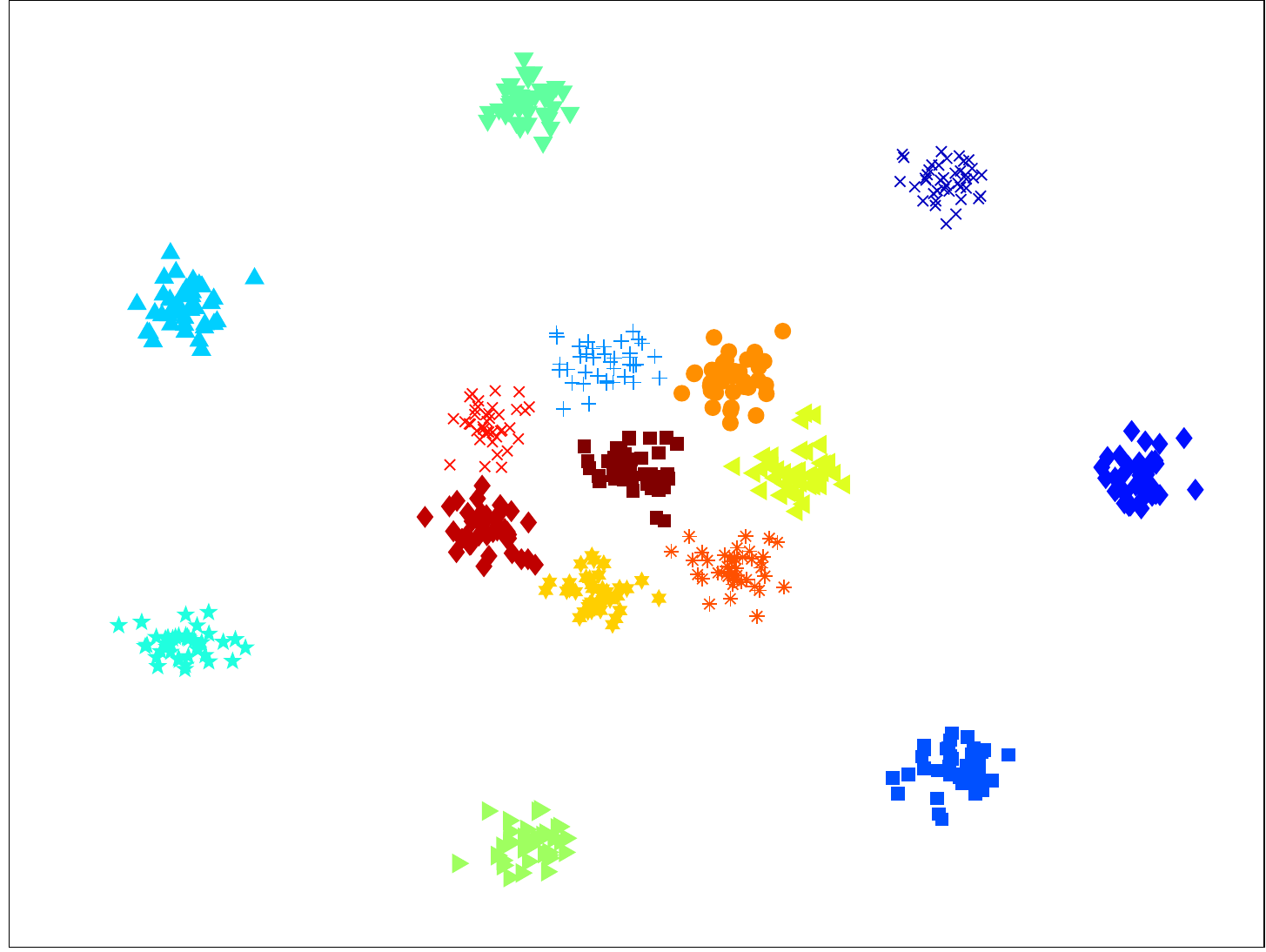}}
	\end{minipage}
	\hfill
	\begin{minipage}{0.19\linewidth}
		\centerline{\includegraphics[width=1\textwidth]{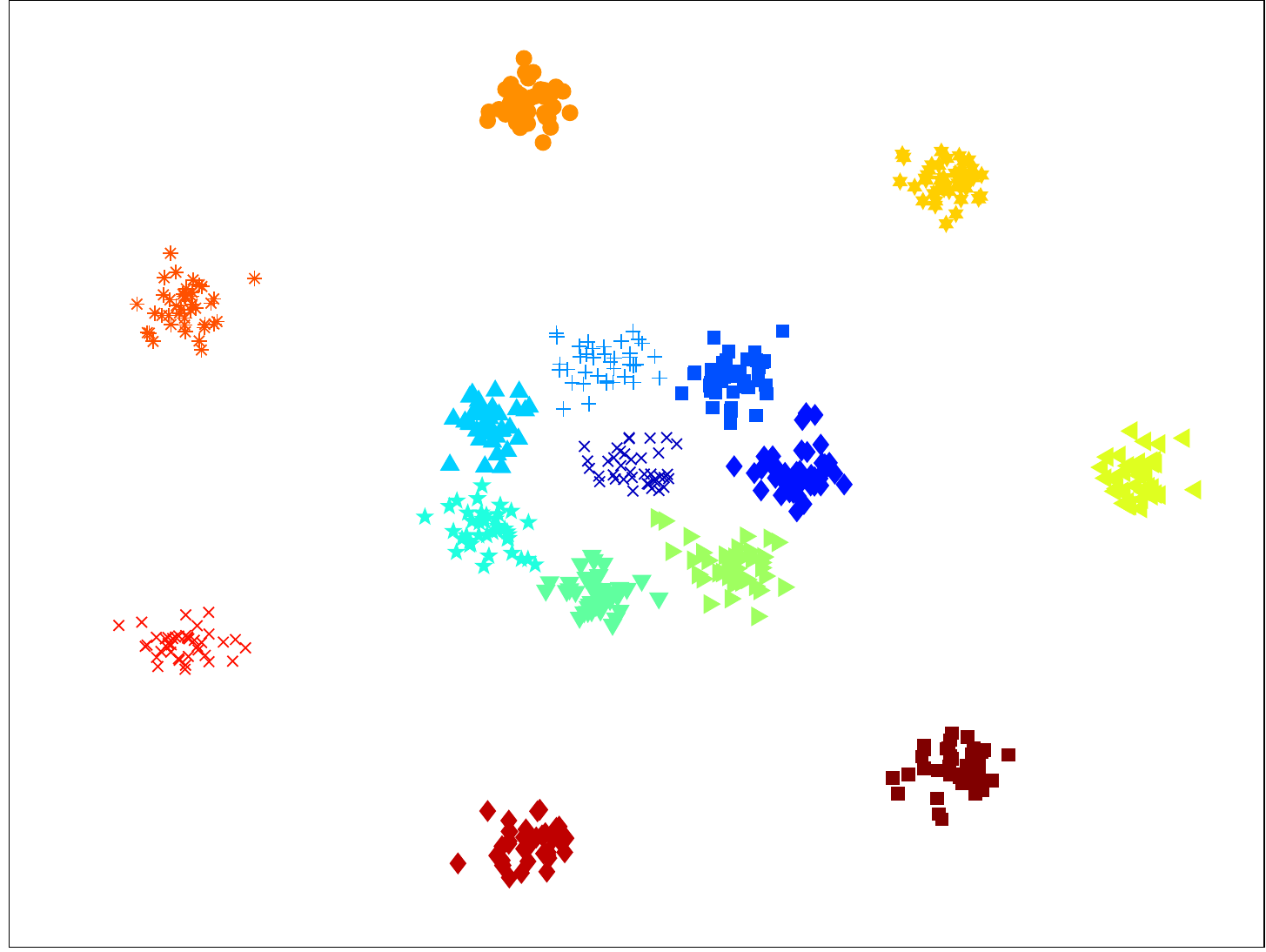}}
	\end{minipage}
	\vfill
	\begin{minipage}{0\linewidth}
		\rightline{J}
	\end{minipage}
	\hfill
	\begin{minipage}{0.19\linewidth}
		\centerline{\includegraphics[width=1\textwidth]{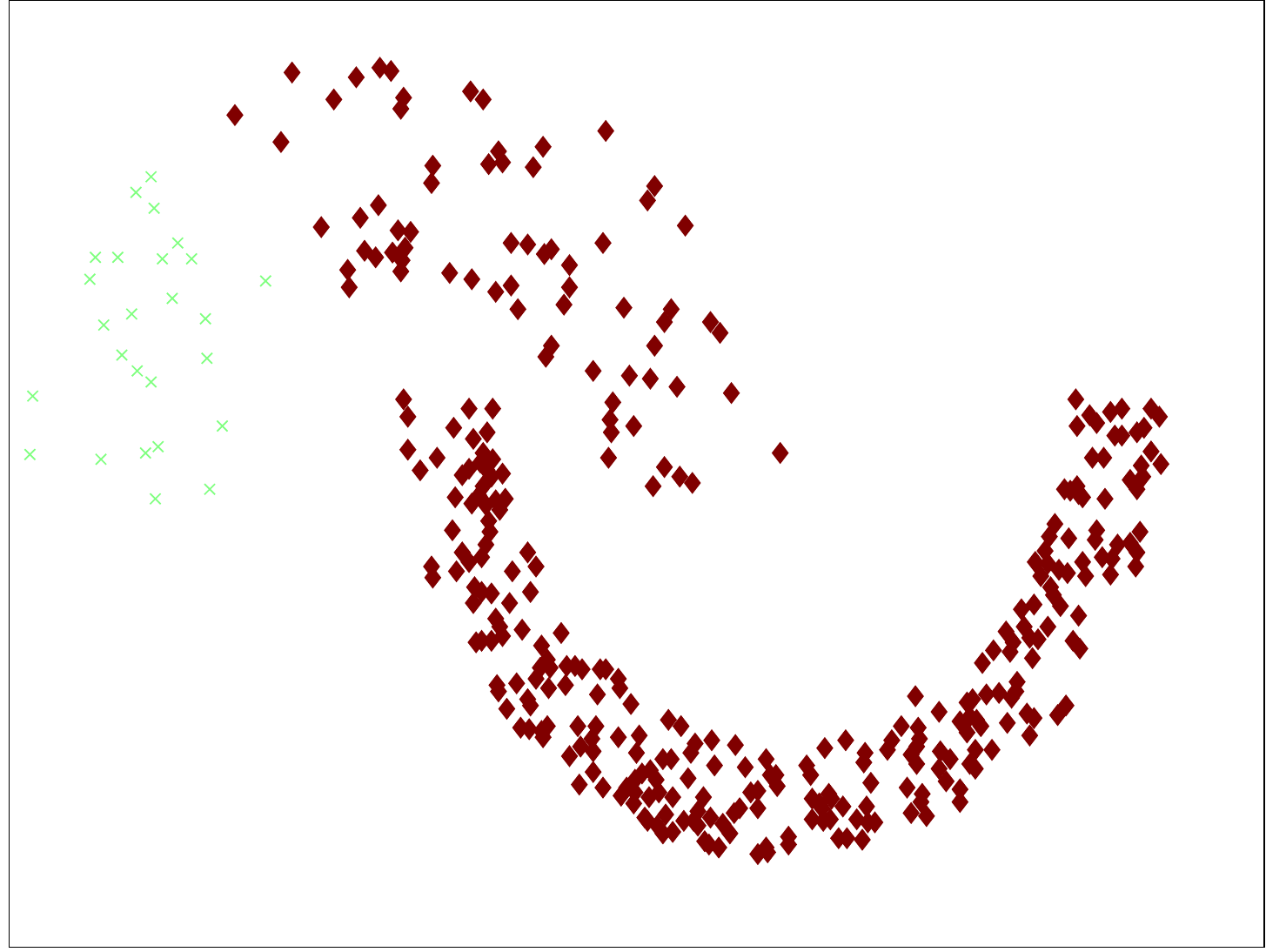}}
	\end{minipage}
	\hfill
	\begin{minipage}{0.19\linewidth}
		\centerline{\includegraphics[width=1\textwidth]{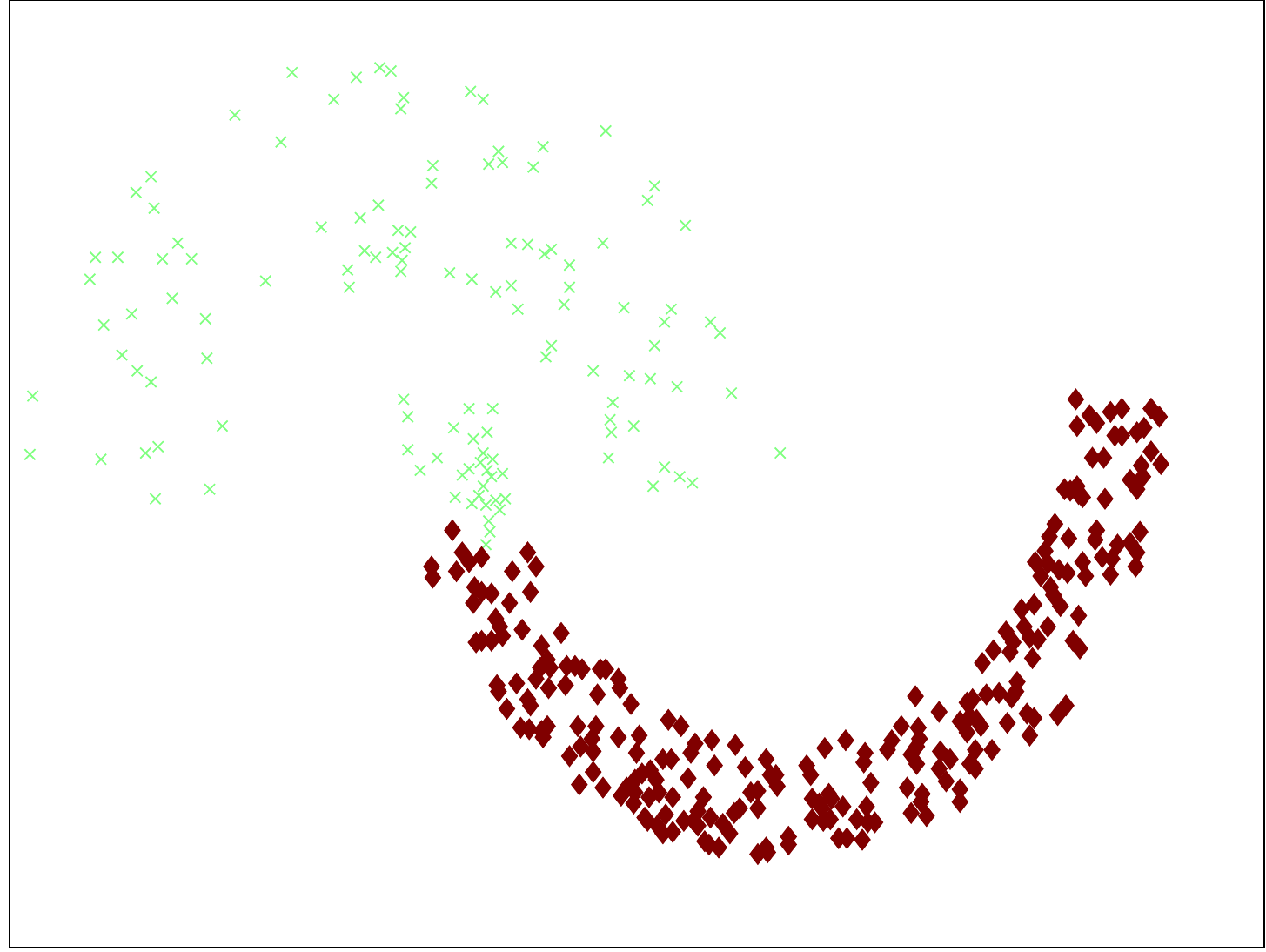}}
	\end{minipage}
	\hfill
	\begin{minipage}{0.19\linewidth}
		\centerline{\includegraphics[width=1\textwidth]{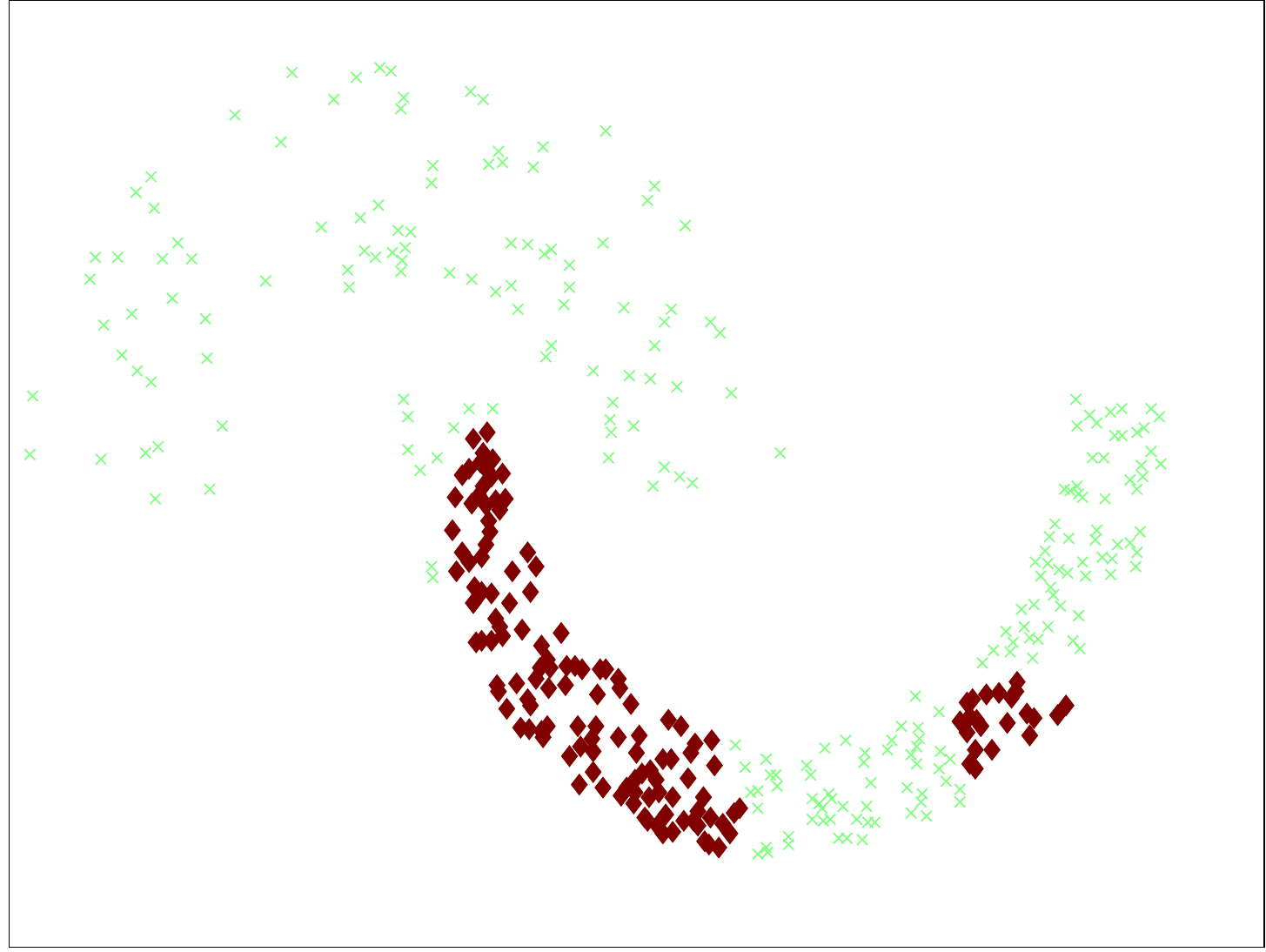}}
	\end{minipage}
	\hfill
	\begin{minipage}{0.19\linewidth}
		\centerline{\includegraphics[width=1\textwidth]{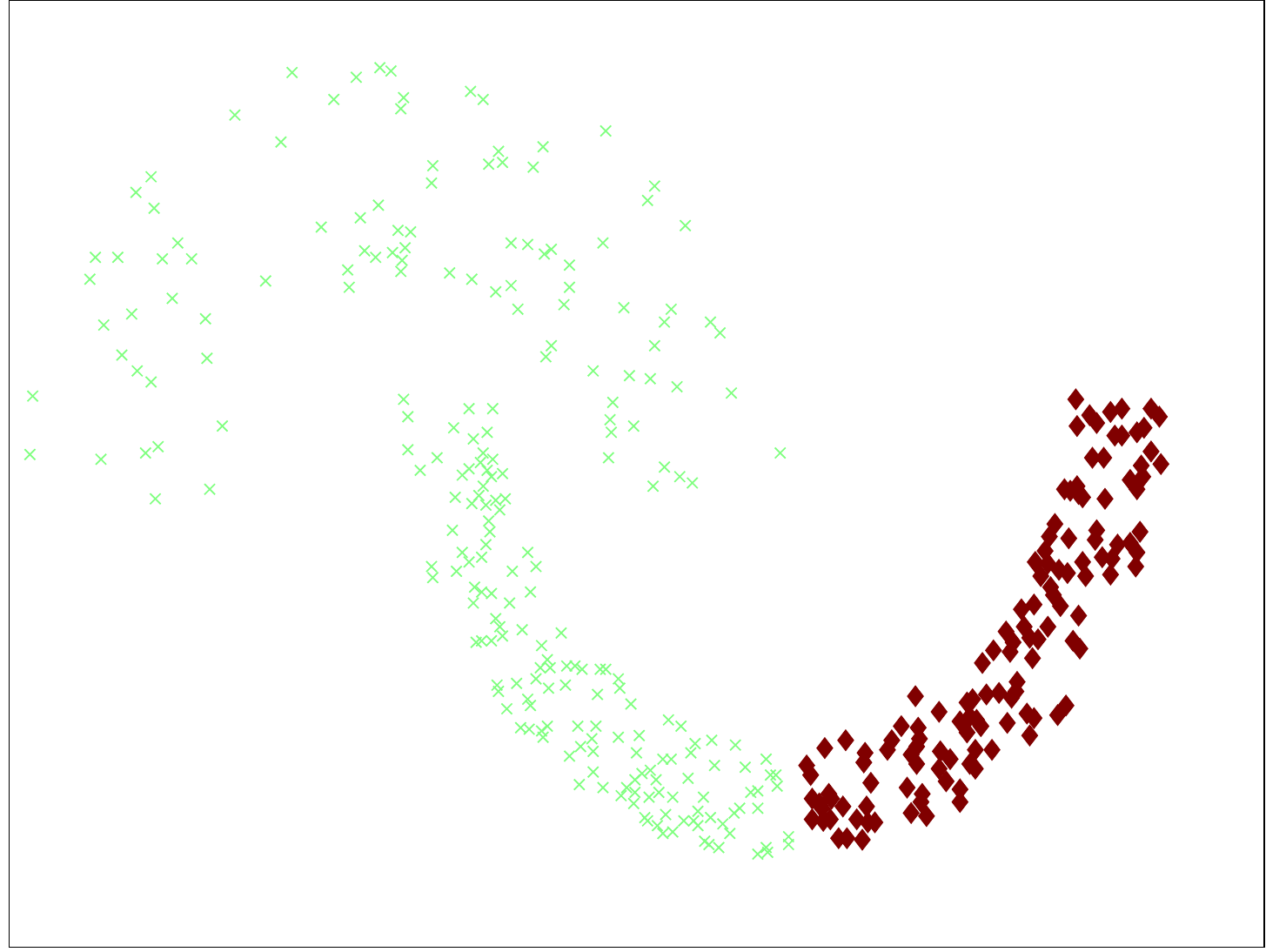}}
	\end{minipage}
	\hfill
	\begin{minipage}{0.19\linewidth}
		\centerline{\includegraphics[width=1\textwidth]{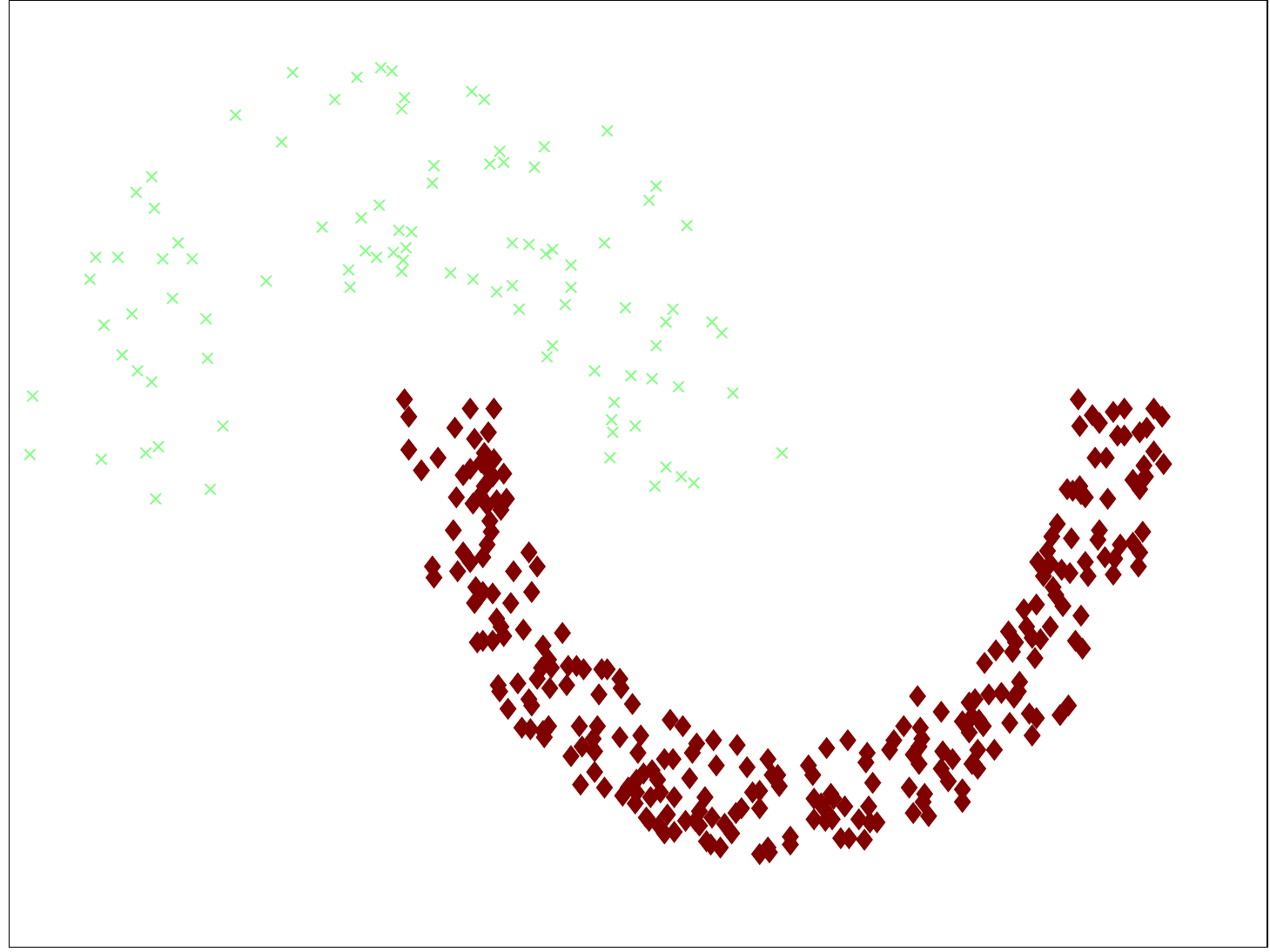}}
	\end{minipage}
	\vfill
	\begin{minipage}{0\linewidth}
		\rightline{K}
	\end{minipage}
	\hfill
	\begin{minipage}{0.19\linewidth}
		\centerline{\includegraphics[width=1\textwidth]{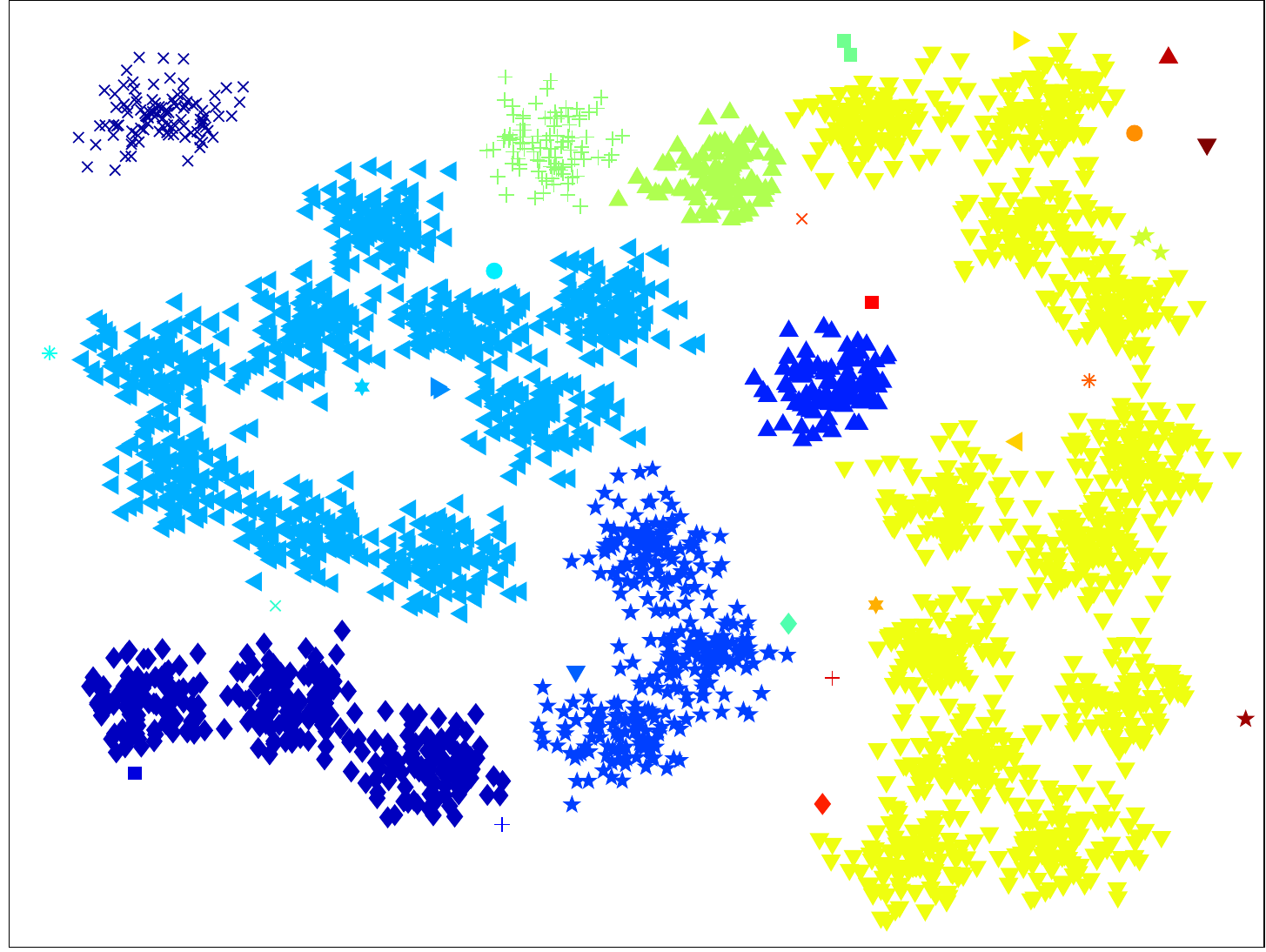}}
	\end{minipage}
	\hfill
	\begin{minipage}{0.19\linewidth}
		\centerline{\includegraphics[width=1\textwidth]{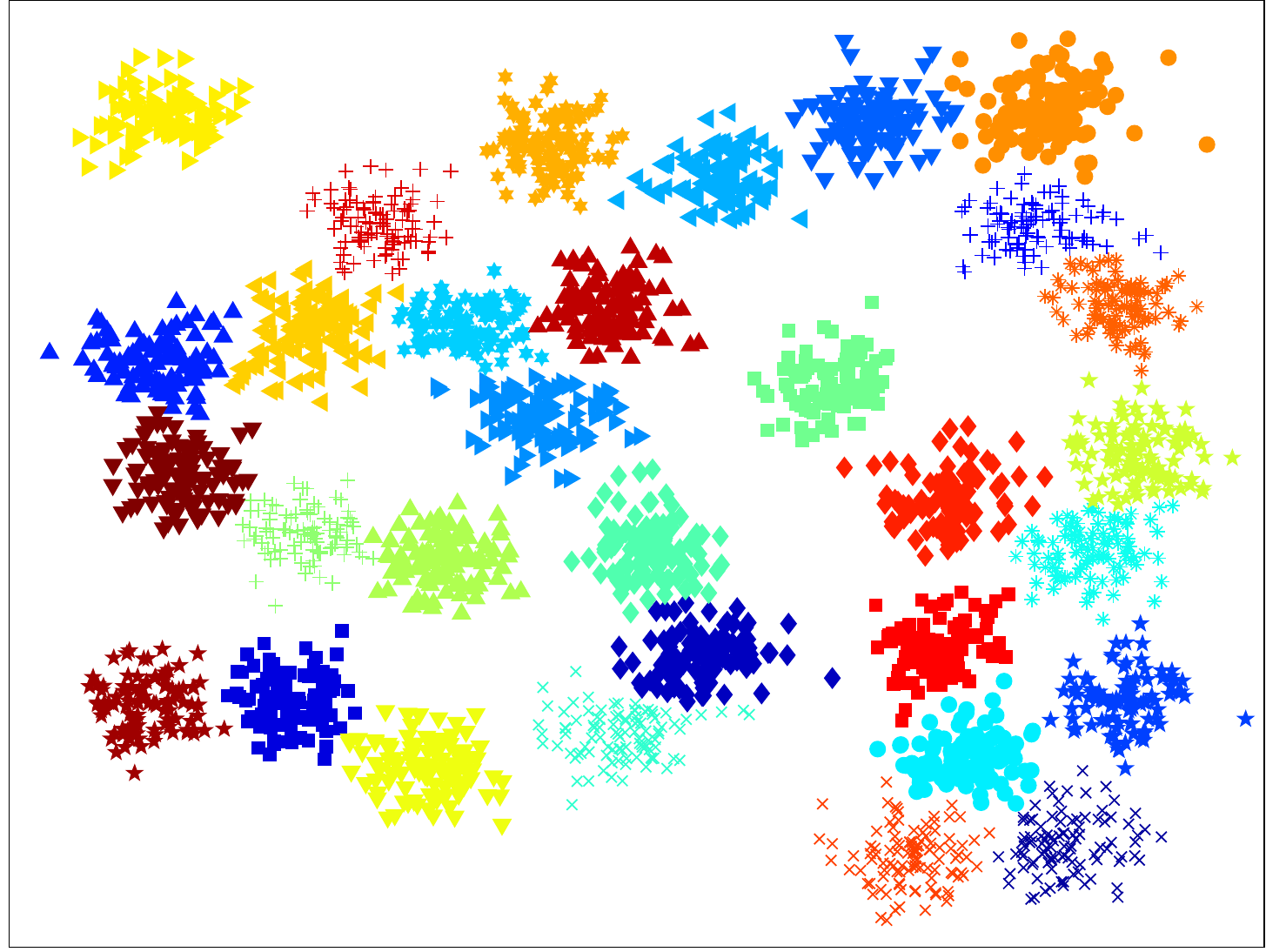}}
	\end{minipage}
	\hfill
	\begin{minipage}{0.19\linewidth}
		\centerline{\includegraphics[width=1\textwidth]{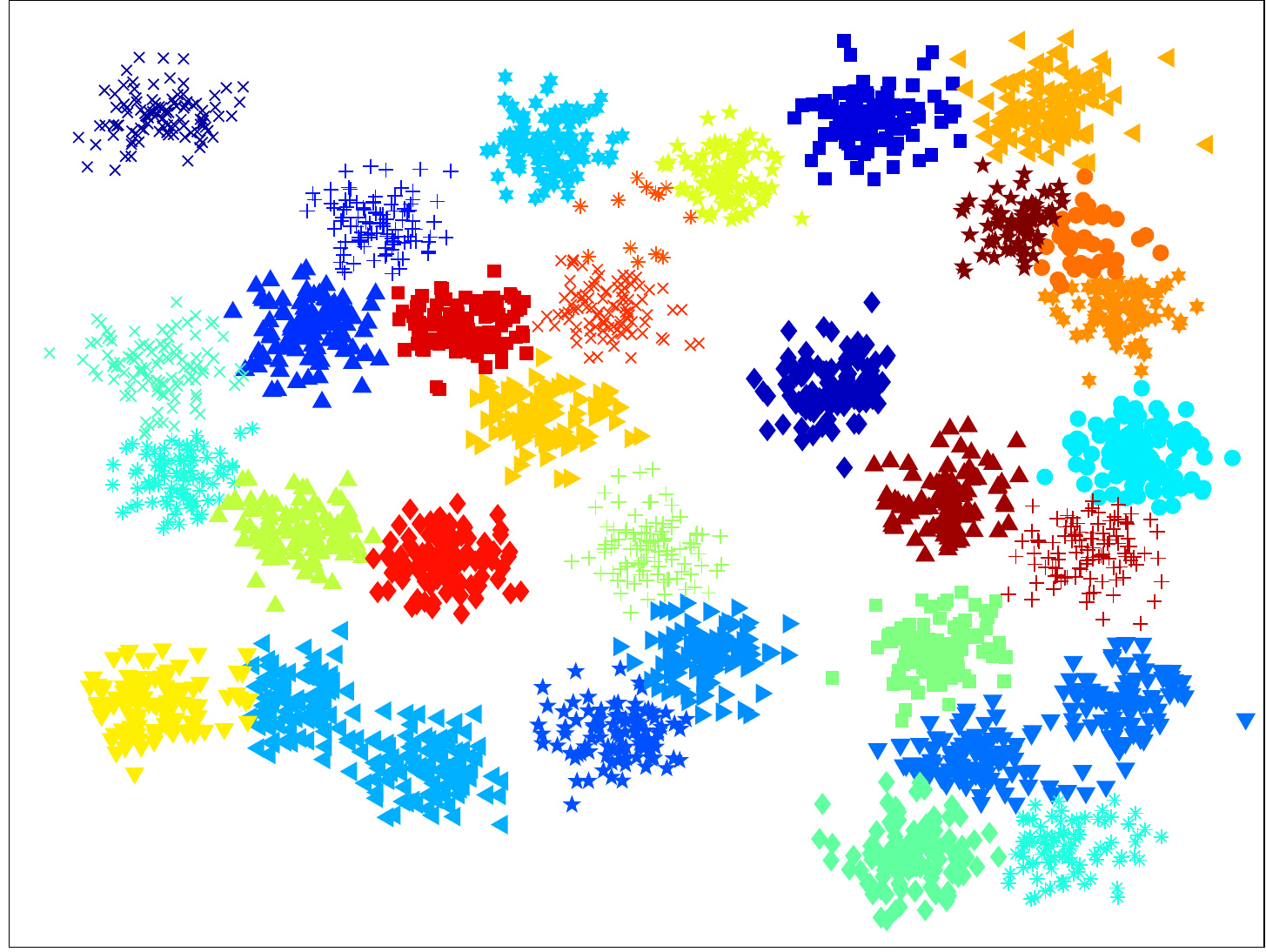}}
	\end{minipage}
	\hfill
	\begin{minipage}{0.19\linewidth}
		\centerline{\includegraphics[width=1\textwidth]{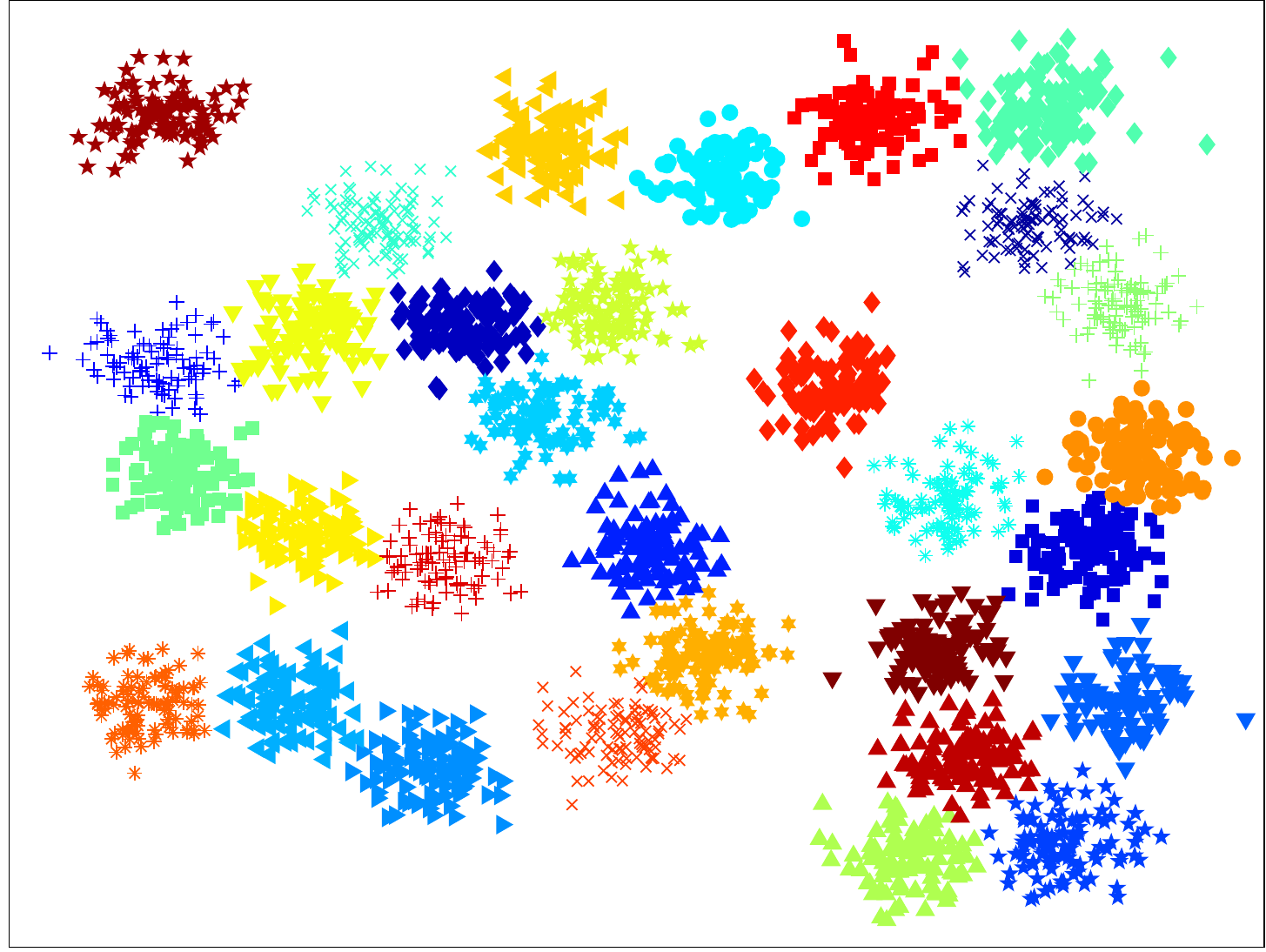}}
	\end{minipage}
	\hfill
	\begin{minipage}{0.19\linewidth}
		\centerline{\includegraphics[width=1\textwidth]{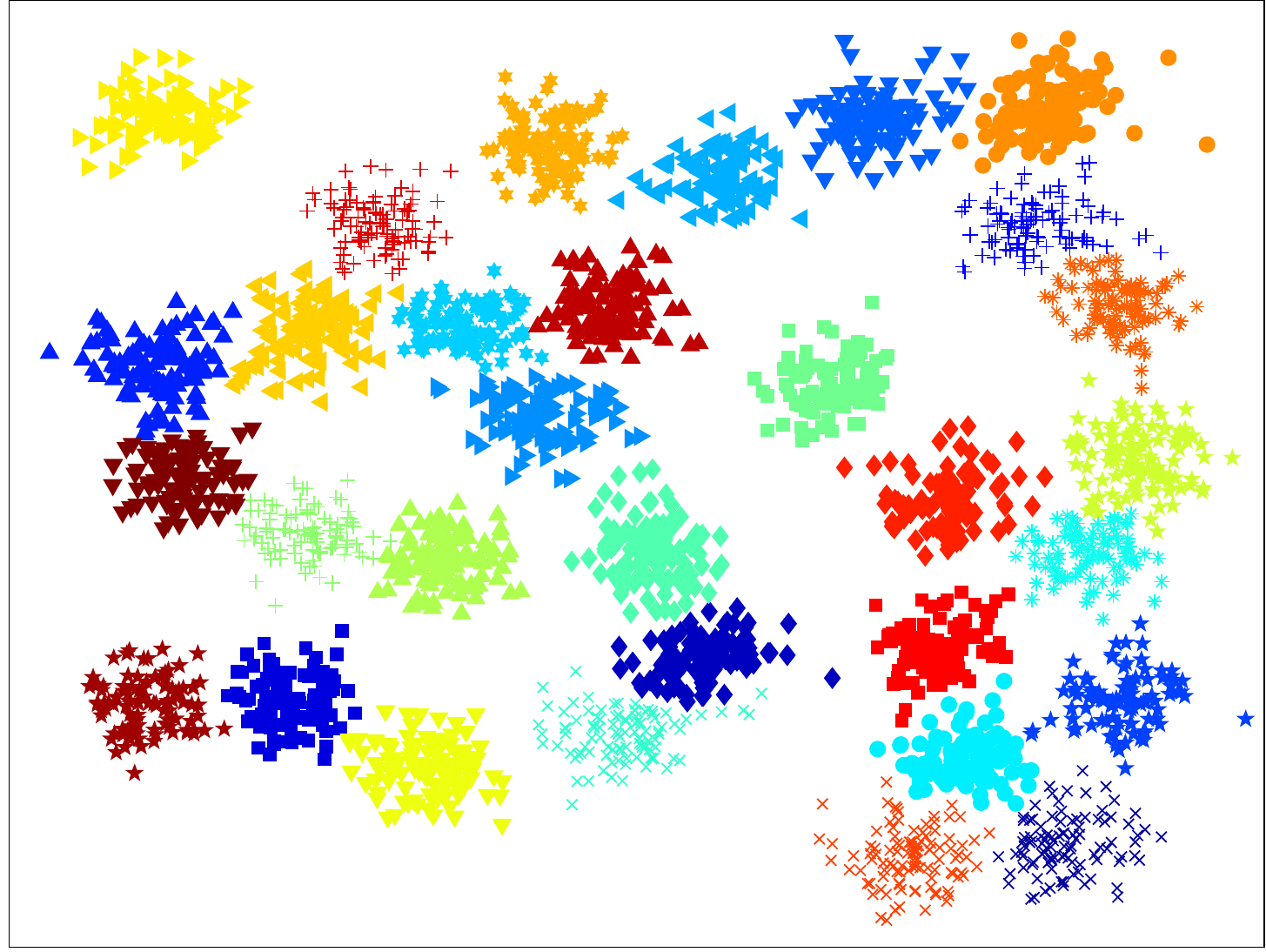}}
	\end{minipage}
	\vfill
\begin{minipage}{0\linewidth}
	\rightline{L}
\end{minipage}
\hfill
\begin{minipage}{0.19\linewidth}
	\centerline{\includegraphics[width=1\textwidth]{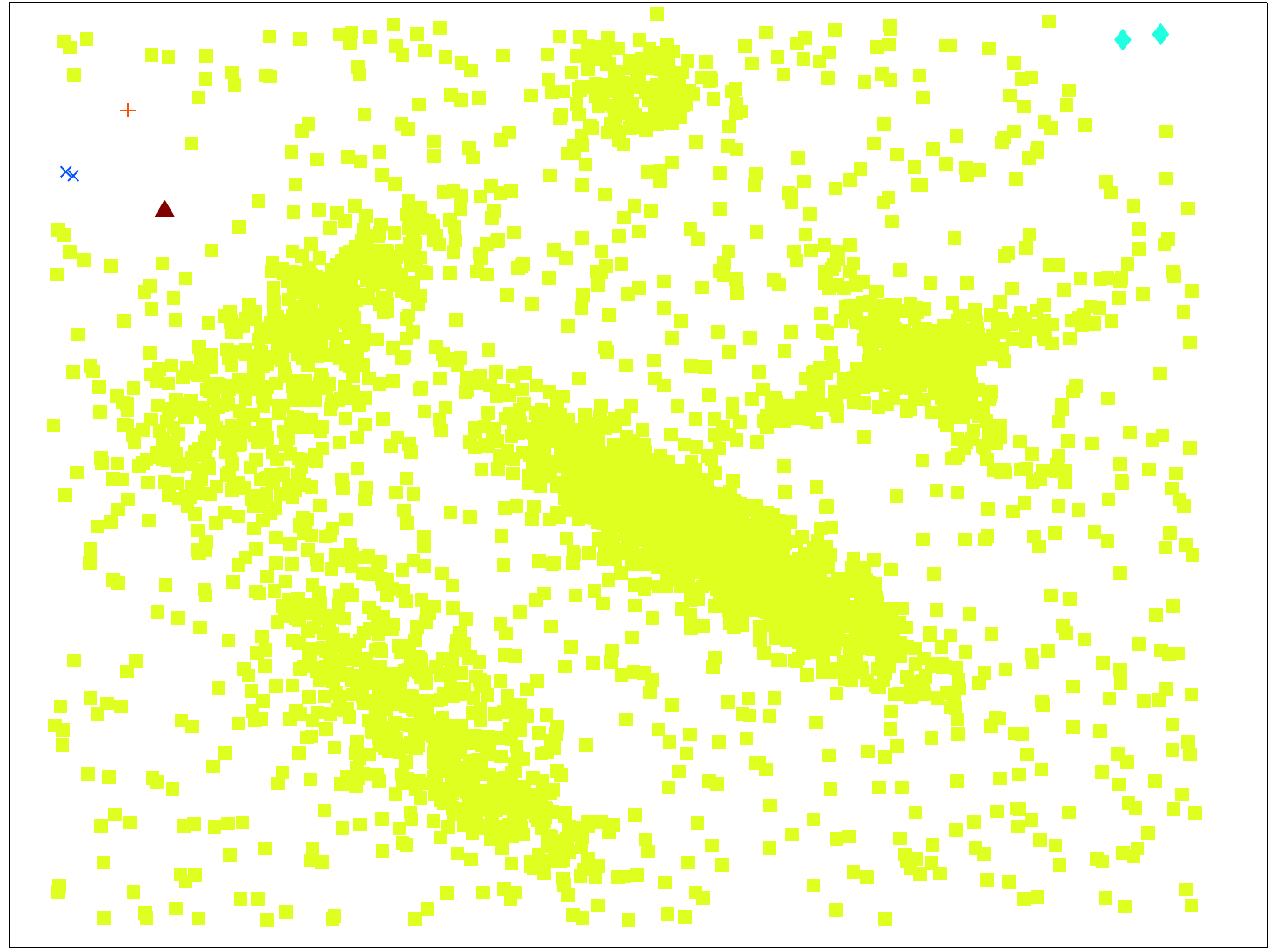}}
\end{minipage}
\hfill
\begin{minipage}{0.19\linewidth}
	\centerline{\includegraphics[width=1\textwidth]{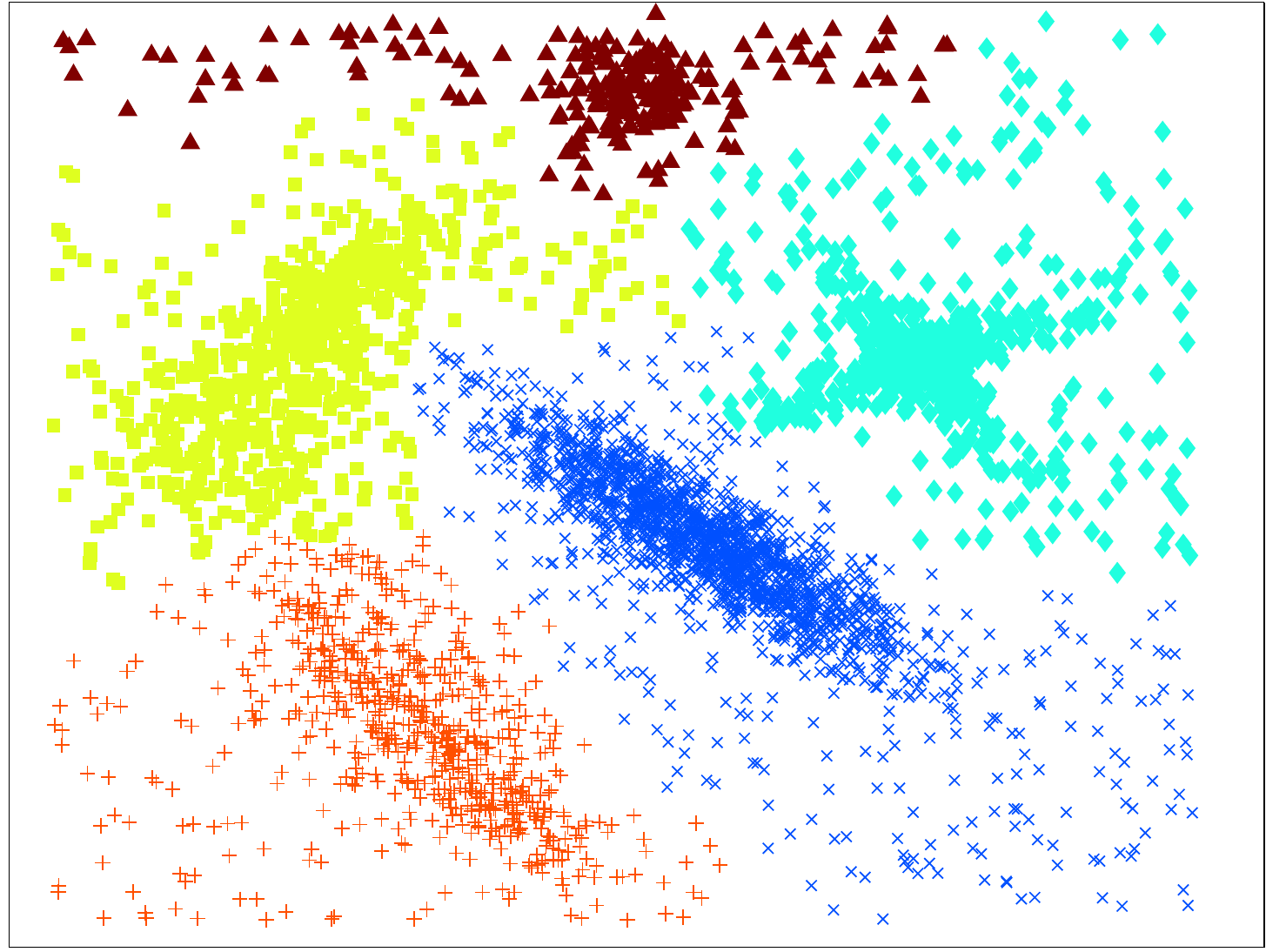}}
\end{minipage}
\hfill
\begin{minipage}{0.19\linewidth}
	\centerline{\includegraphics[width=1\textwidth]{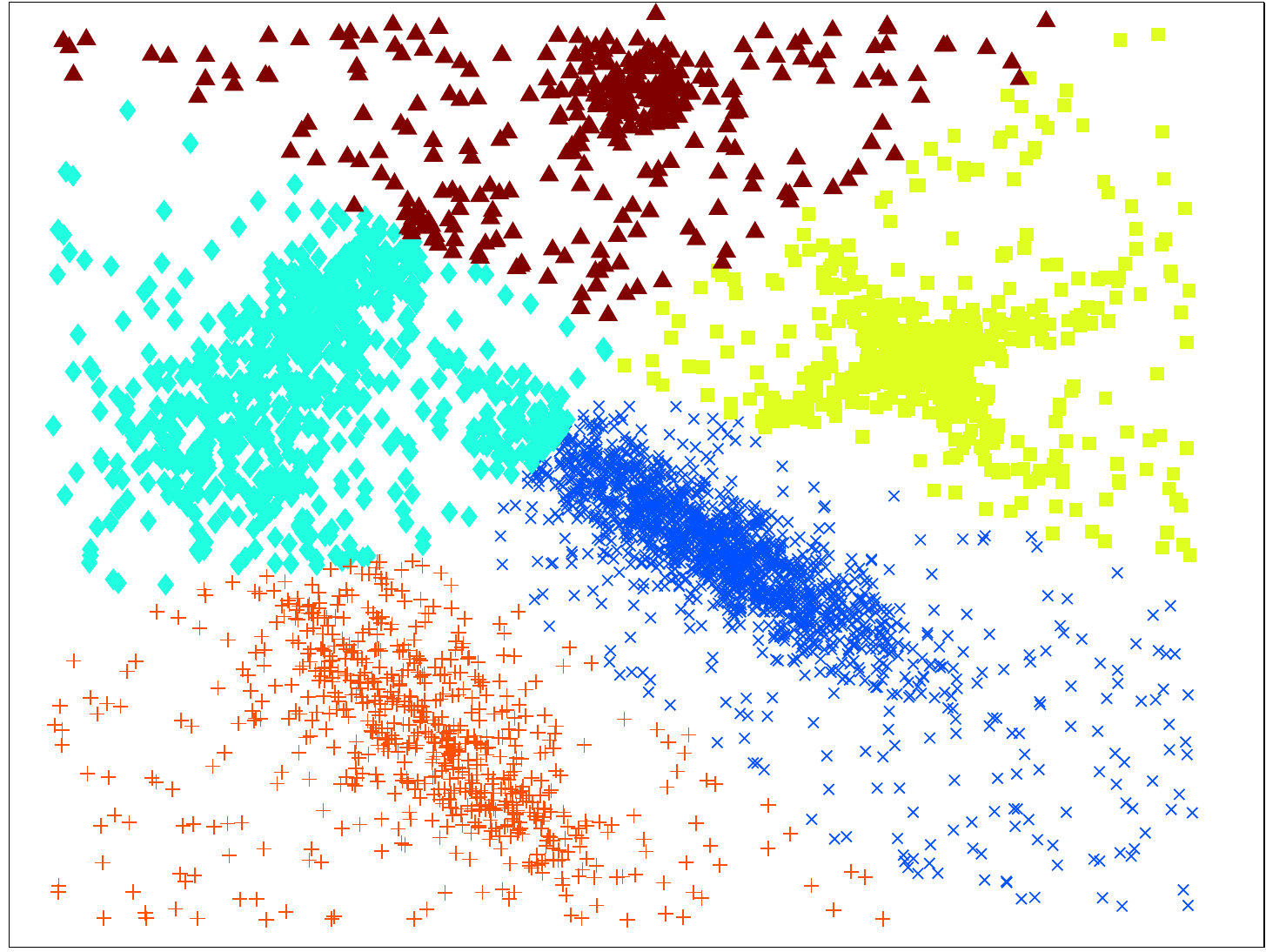}}
\end{minipage}
\hfill
\begin{minipage}{0.19\linewidth}
	\centerline{\includegraphics[width=1\textwidth]{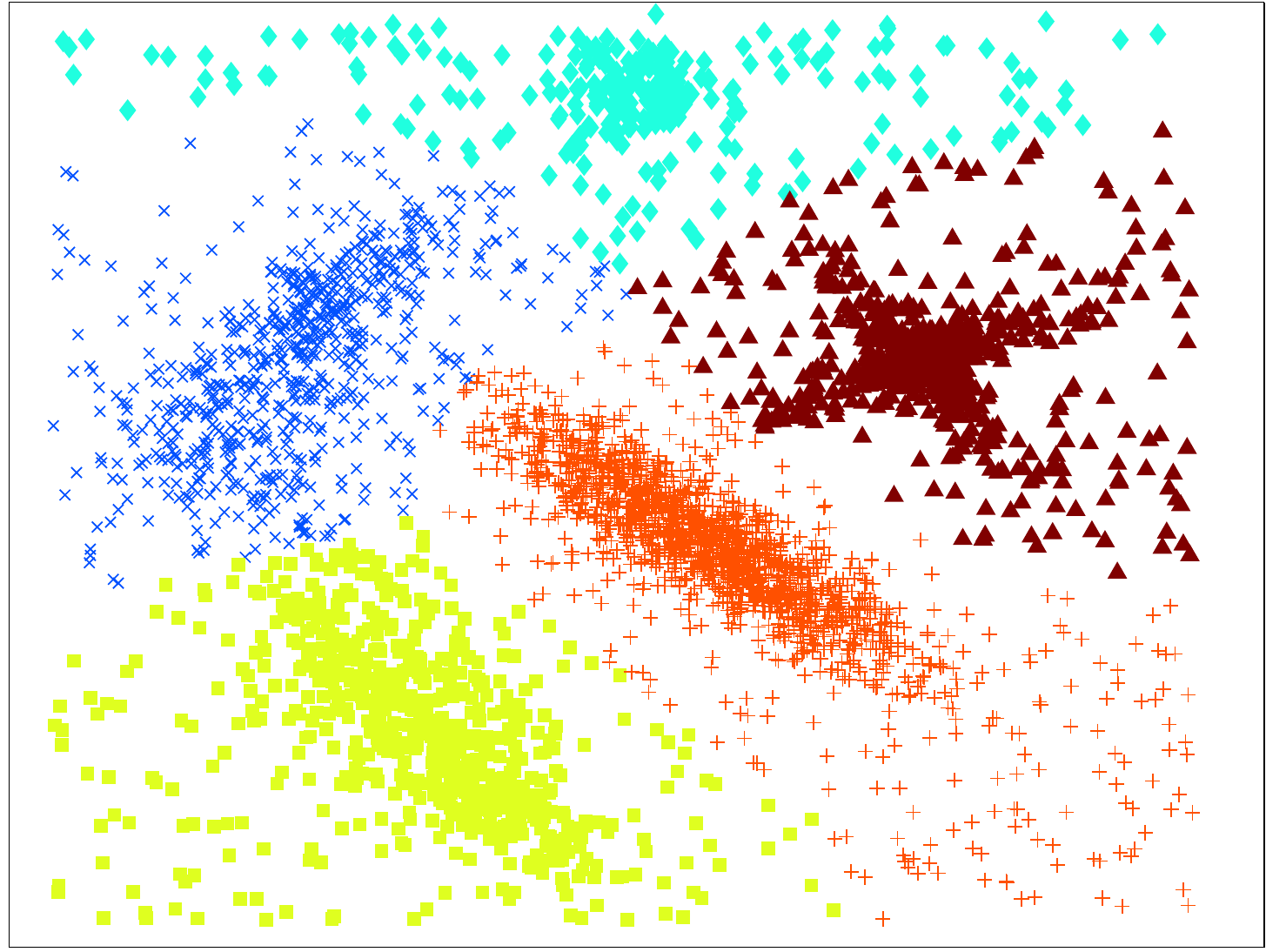}}
\end{minipage}
\hfill
\begin{minipage}{0.19\linewidth}
	\centerline{\includegraphics[width=1\textwidth]{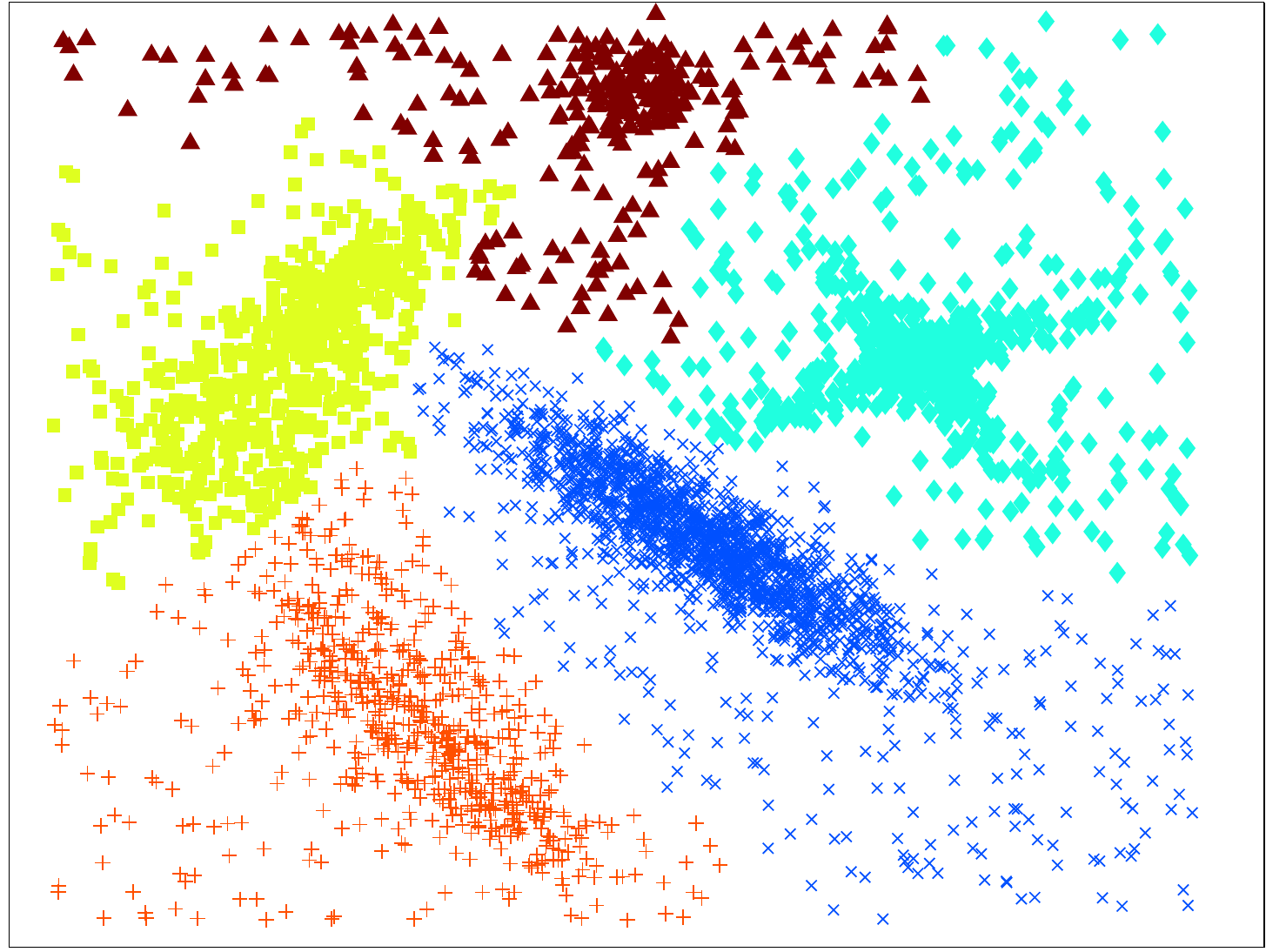}}
\end{minipage}
\vfill
\begin{minipage}{0\linewidth}
	\rightline{M}
\end{minipage}
\hfill
\begin{minipage}{0.19\linewidth}
	\centerline{\includegraphics[width=1\textwidth]{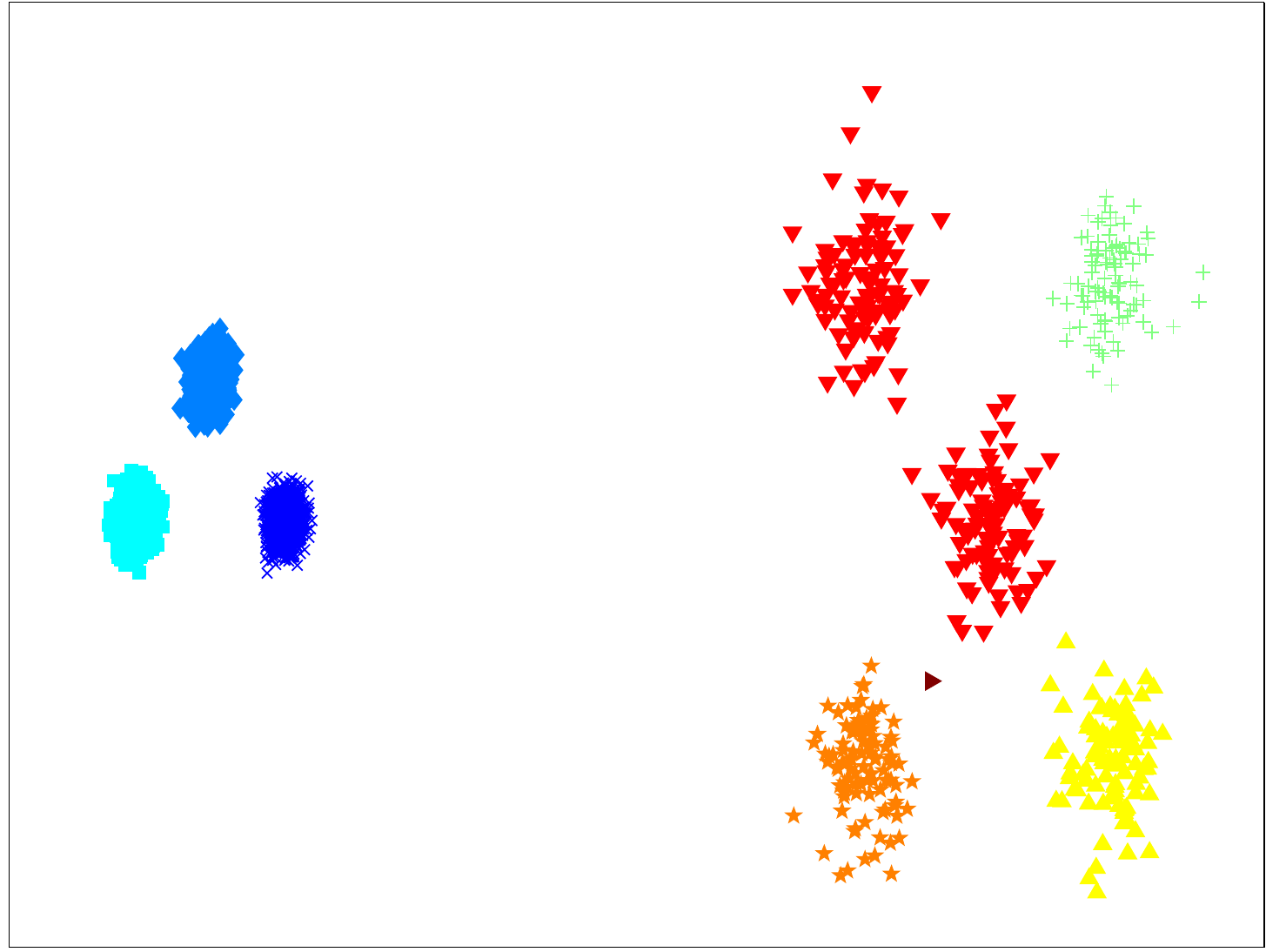}}
\end{minipage}
\hfill
\begin{minipage}{0.19\linewidth}
	\centerline{\includegraphics[width=1\textwidth]{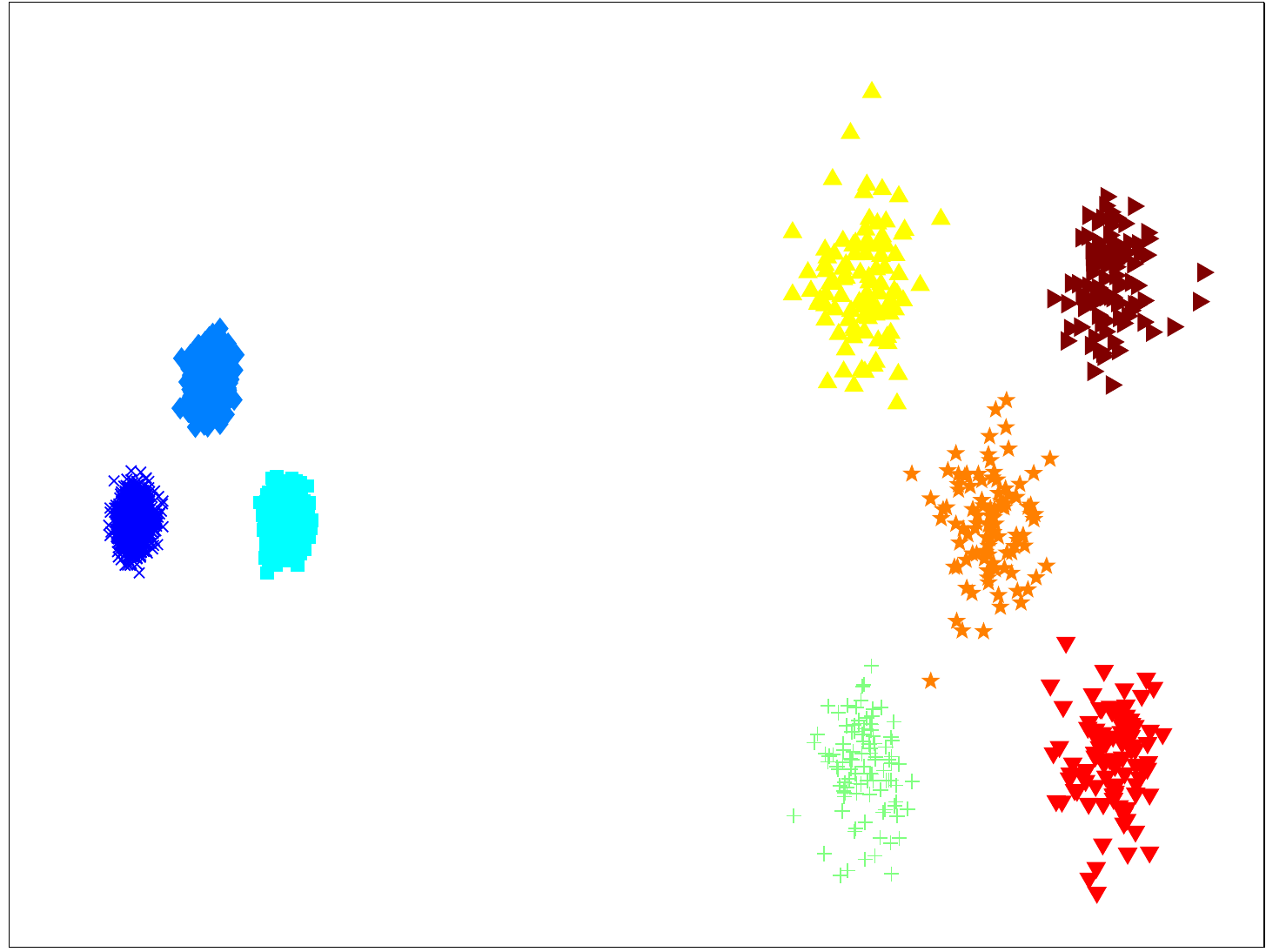}}
\end{minipage}
\hfill
\begin{minipage}{0.19\linewidth}
	\centerline{\includegraphics[width=1\textwidth]{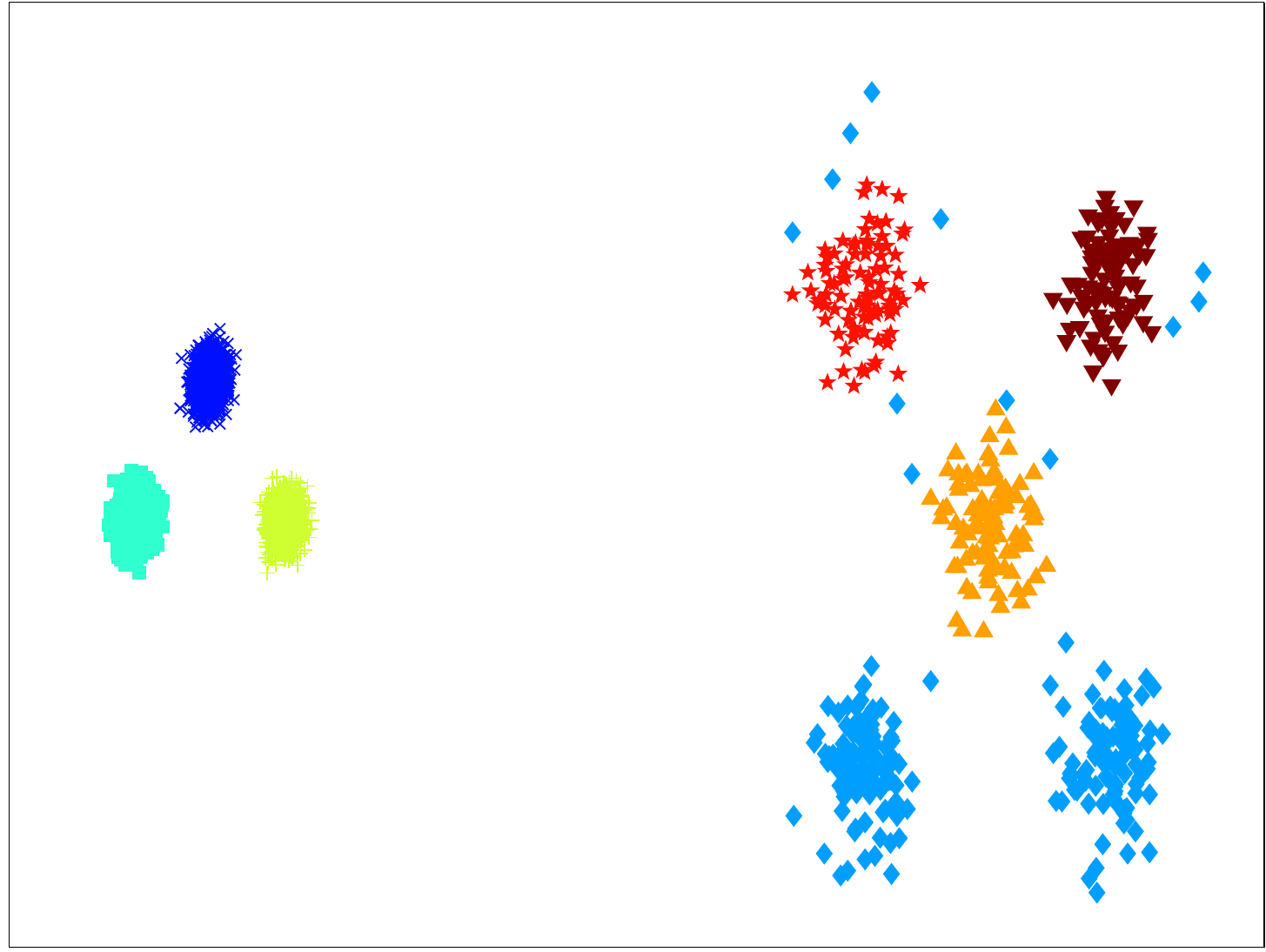}}
\end{minipage}
\hfill
\begin{minipage}{0.19\linewidth}
	\centerline{\includegraphics[width=1\textwidth]{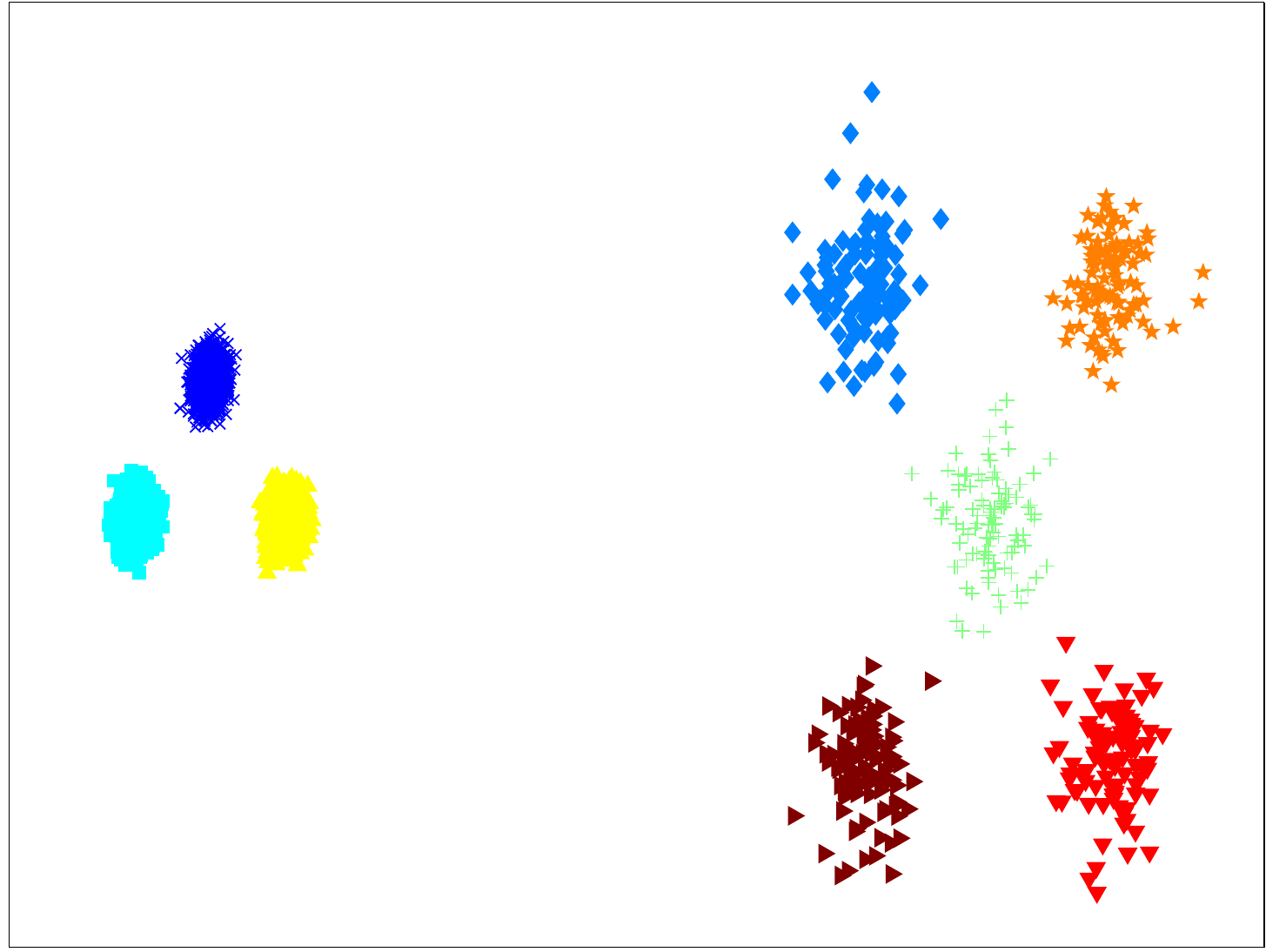}}
\end{minipage}
\hfill
\begin{minipage}{0.19\linewidth}
	\centerline{\includegraphics[width=1\textwidth]{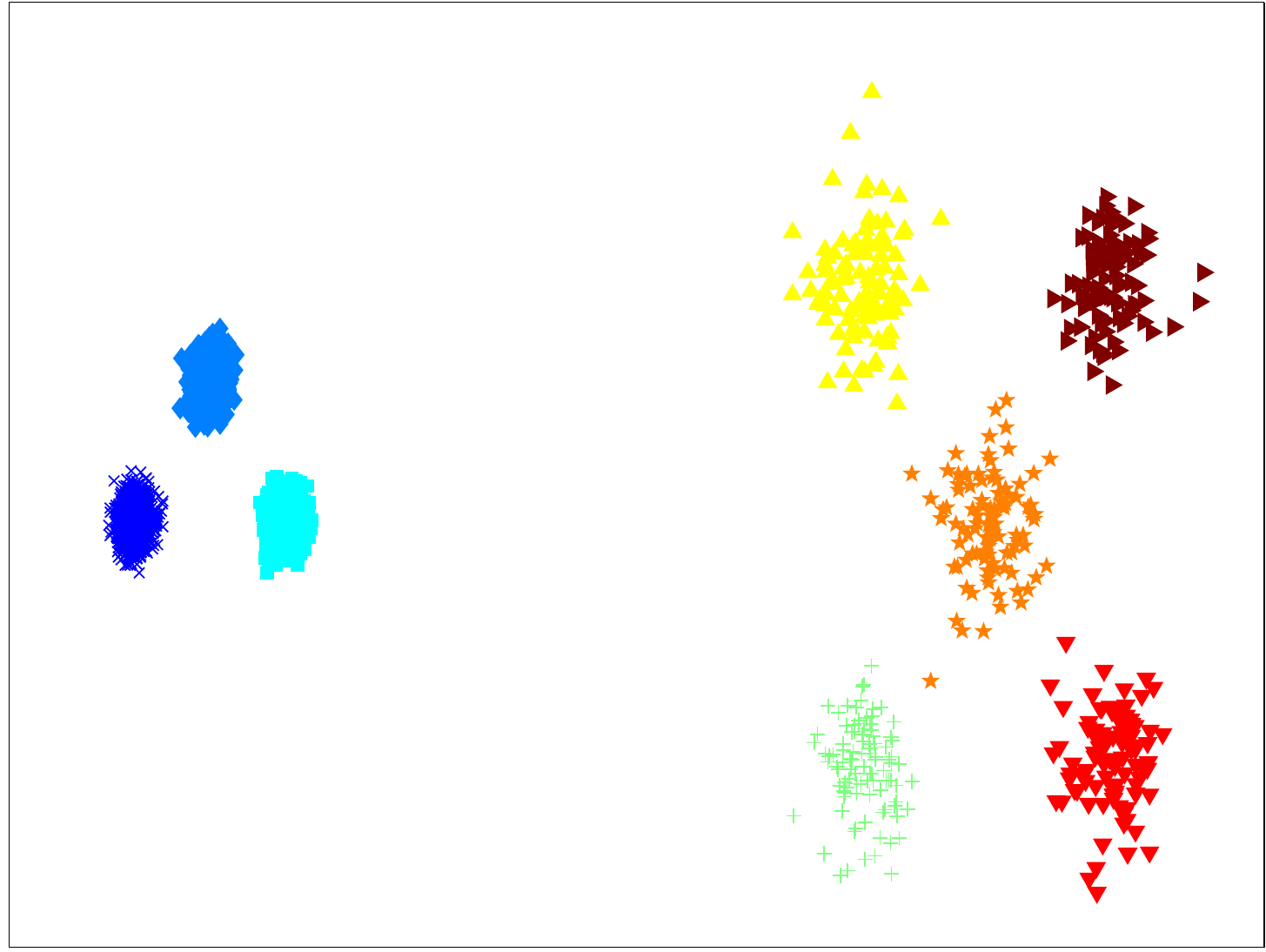}}
\end{minipage}
\vfill
\begin{minipage}{0\linewidth}
	\rightline{N}
\end{minipage}
\hfill
\begin{minipage}{0.19\linewidth}
	\centerline{\includegraphics[width=1\textwidth]{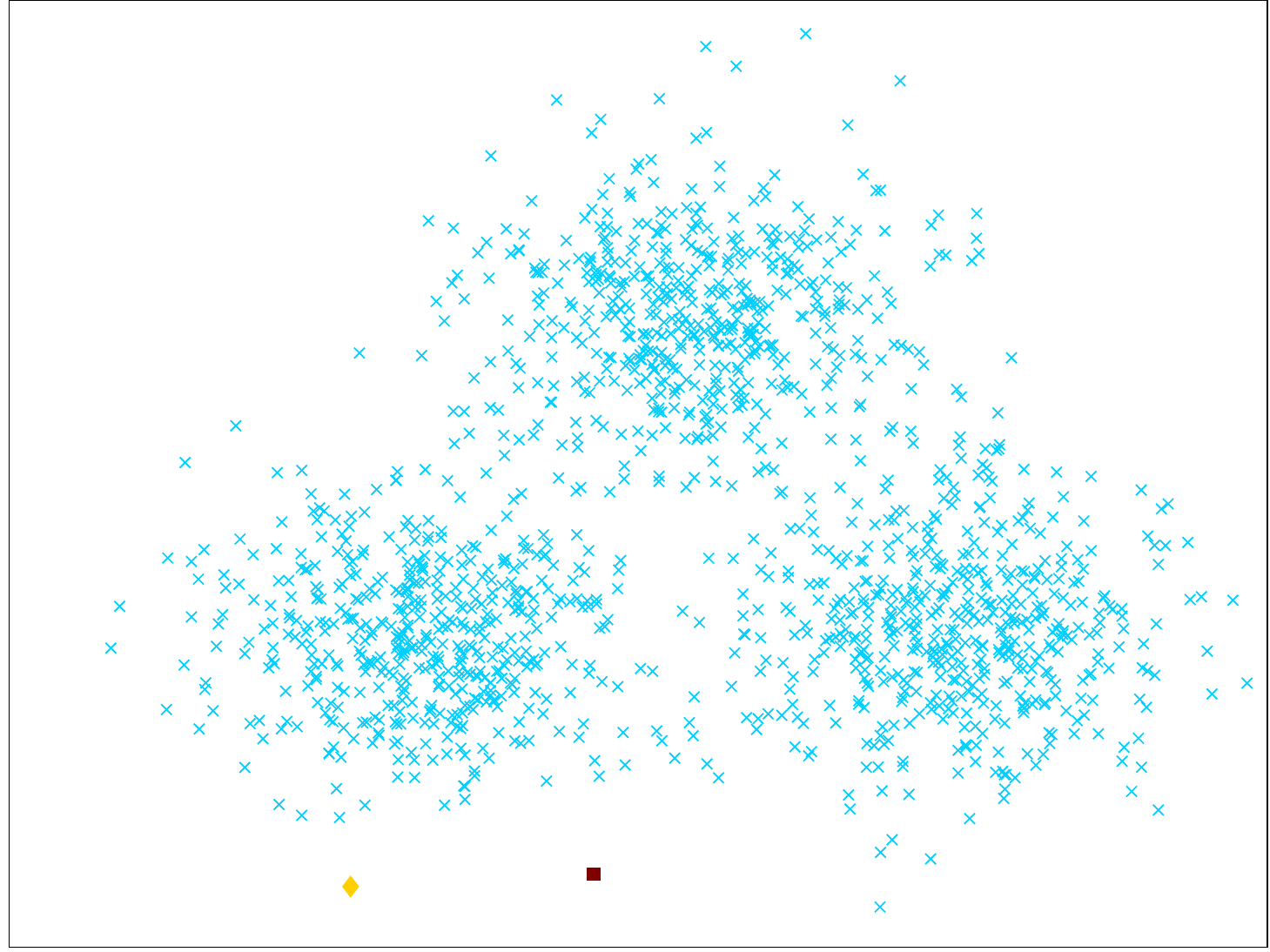}}
\end{minipage}
\hfill
\begin{minipage}{0.19\linewidth}
	\centerline{\includegraphics[width=1\textwidth]{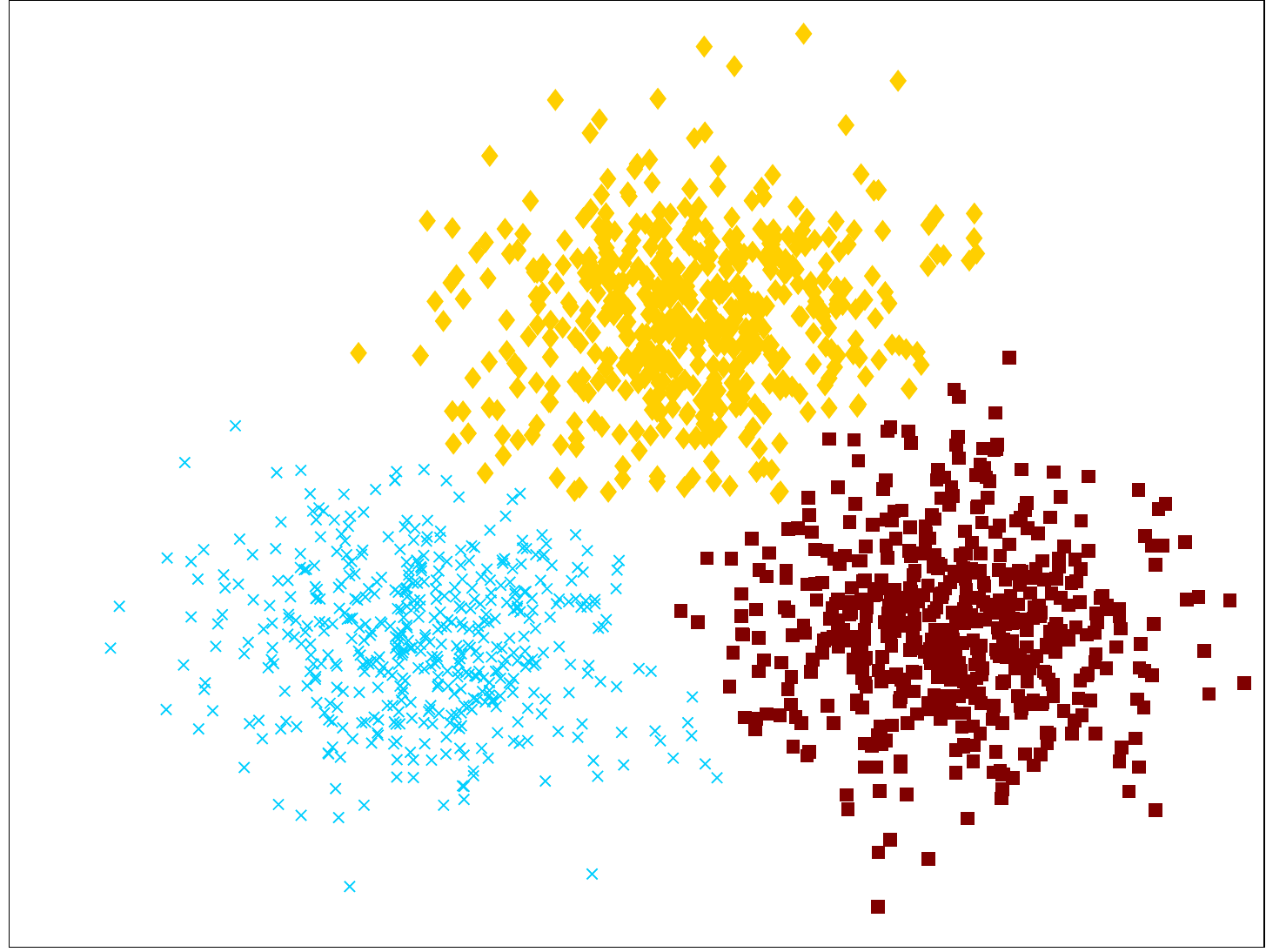}}
\end{minipage}
\hfill
\begin{minipage}{0.19\linewidth}
	\centerline{\includegraphics[width=1\textwidth]{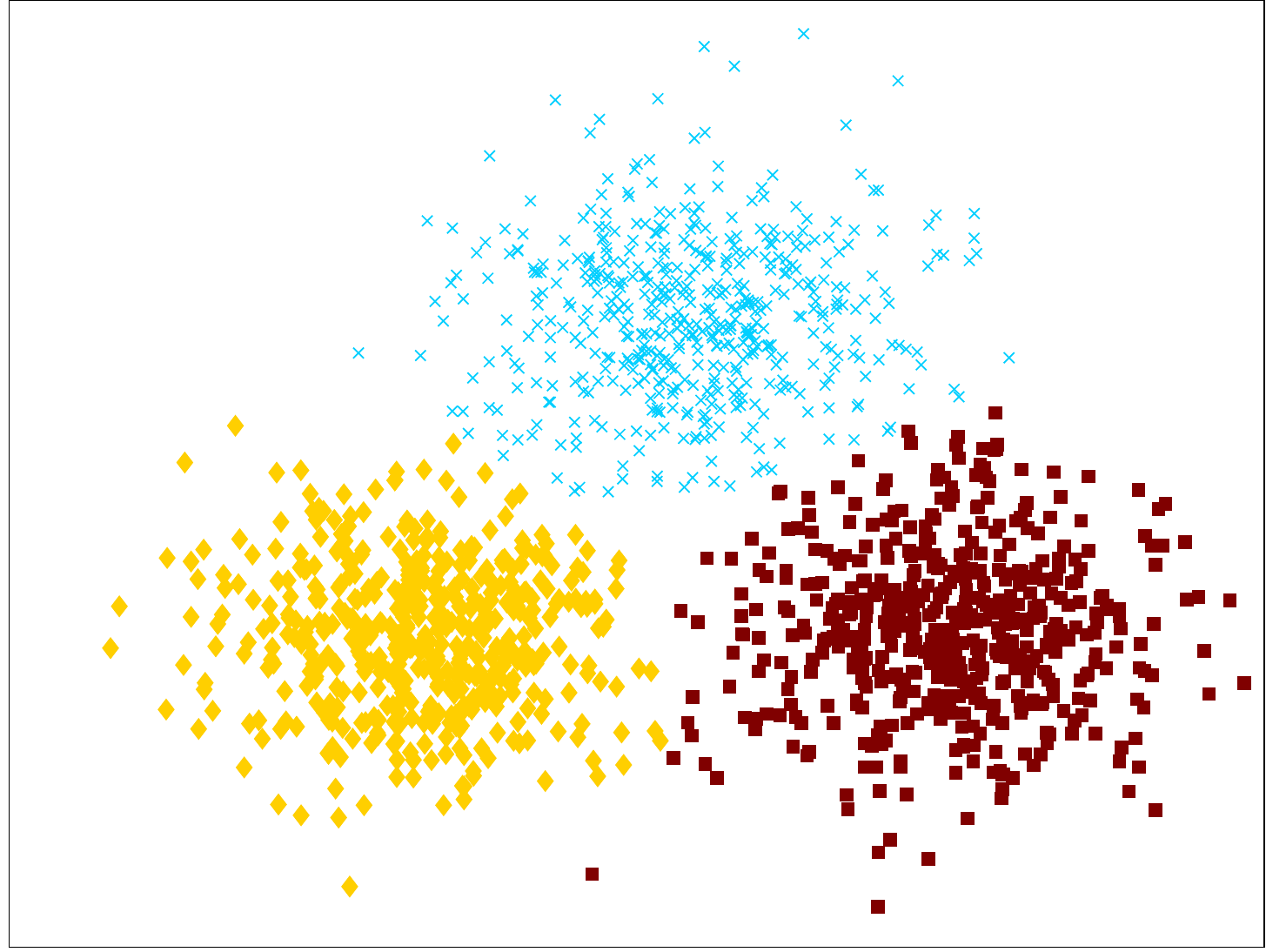}}
\end{minipage}
\hfill
\begin{minipage}{0.19\linewidth}
	\centerline{\includegraphics[width=1\textwidth]{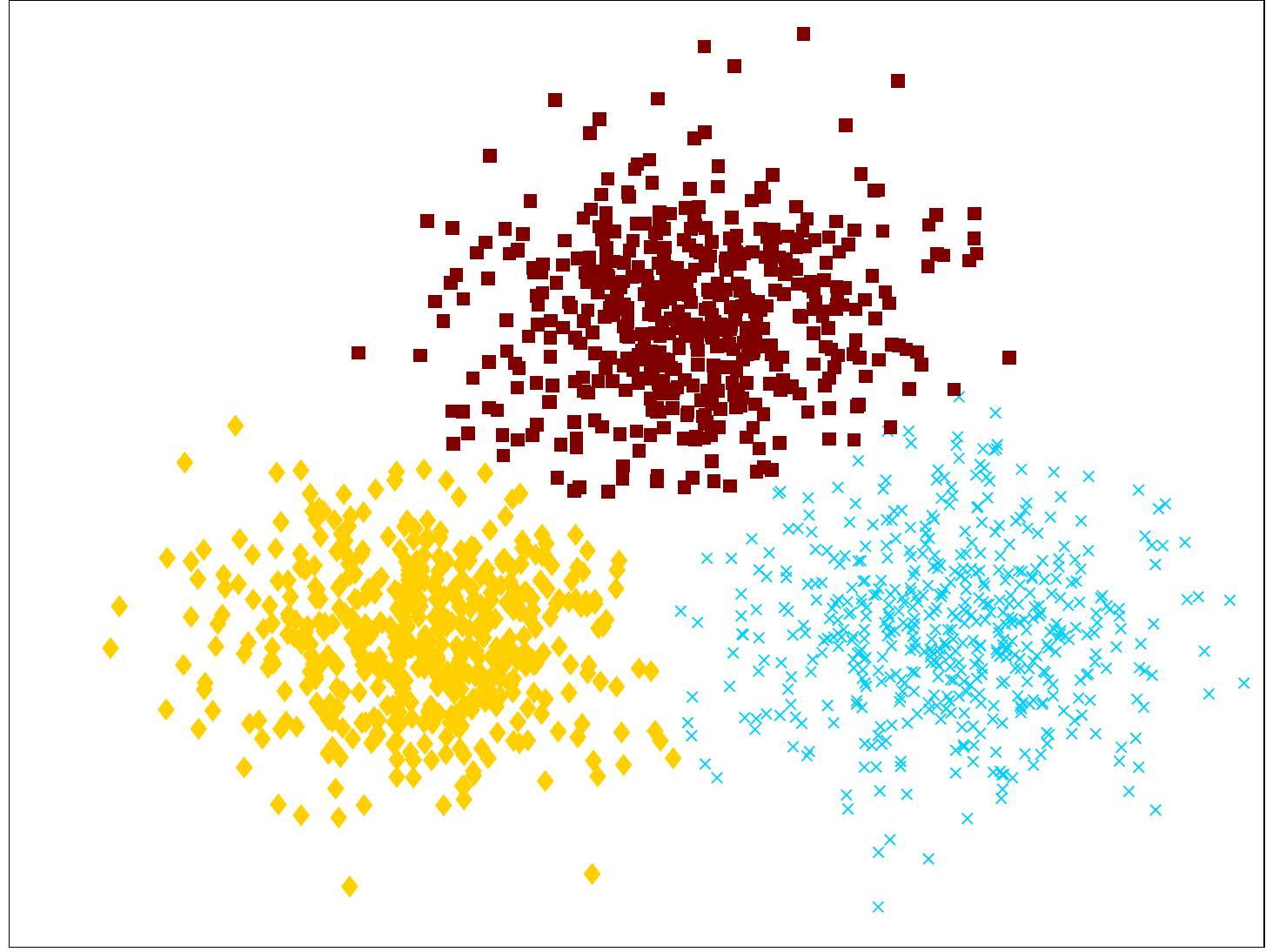}}
\end{minipage}
\hfill
\begin{minipage}{0.19\linewidth}
	\centerline{\includegraphics[width=1\textwidth]{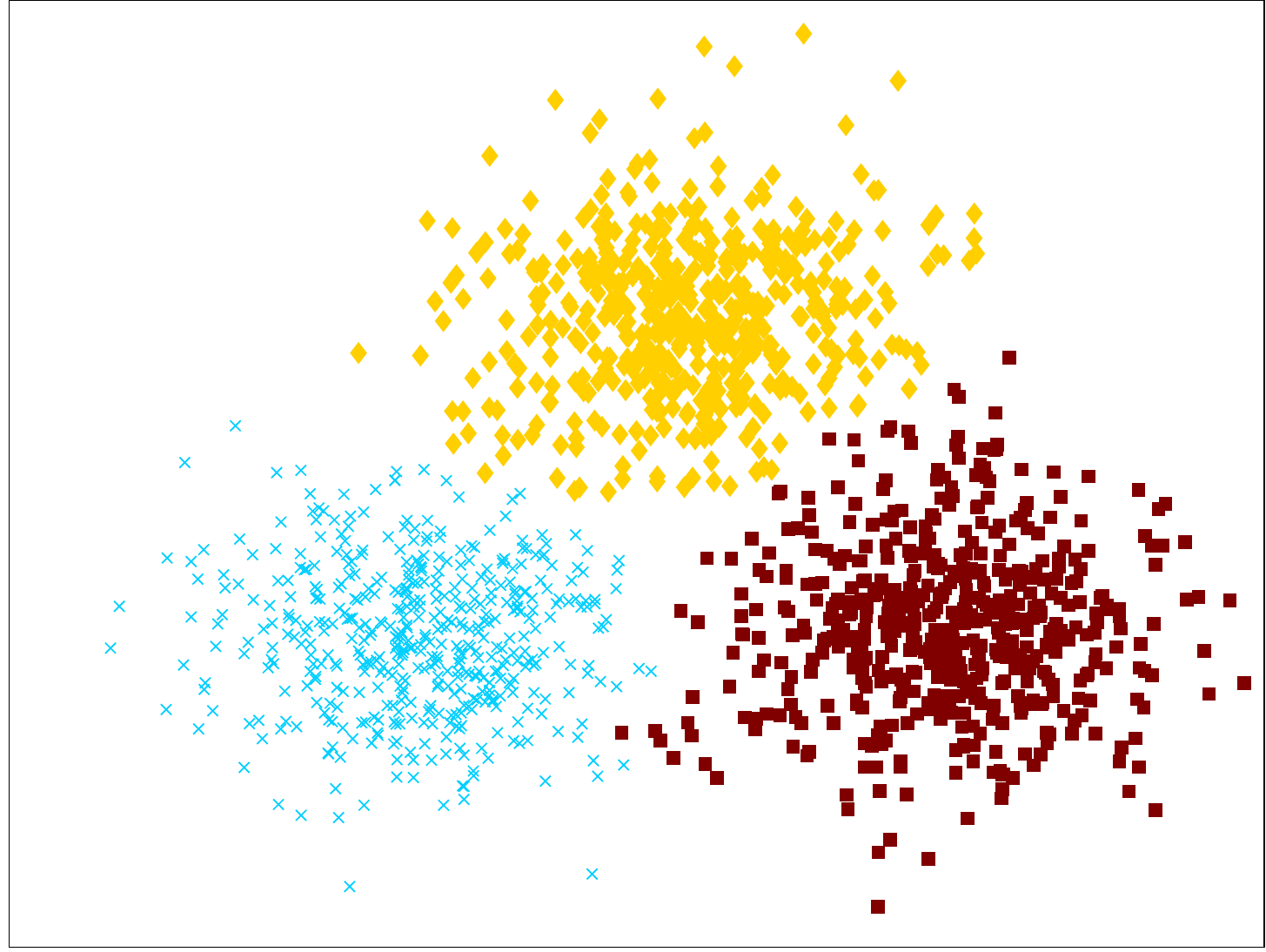}}
\end{minipage}
\vfill
\begin{minipage}{0\linewidth}
	\rightline{O}
\end{minipage}
\hfill
\begin{minipage}{0.19\linewidth}
	\centerline{\includegraphics[width=1\textwidth]{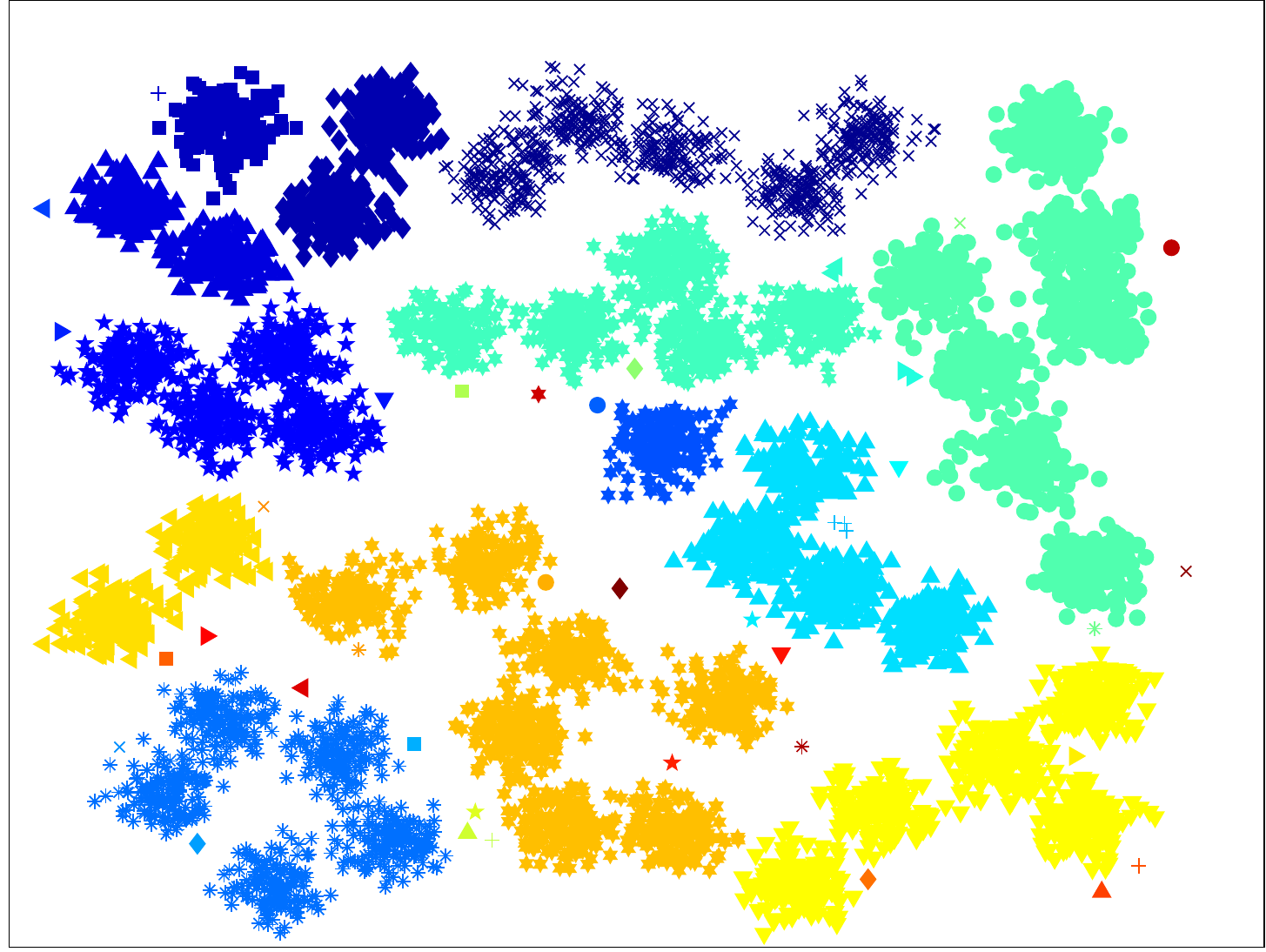}}
	\centerline{SLC}
\end{minipage}
\hfill
\begin{minipage}{0.19\linewidth}
	\centerline{\includegraphics[width=1\textwidth]{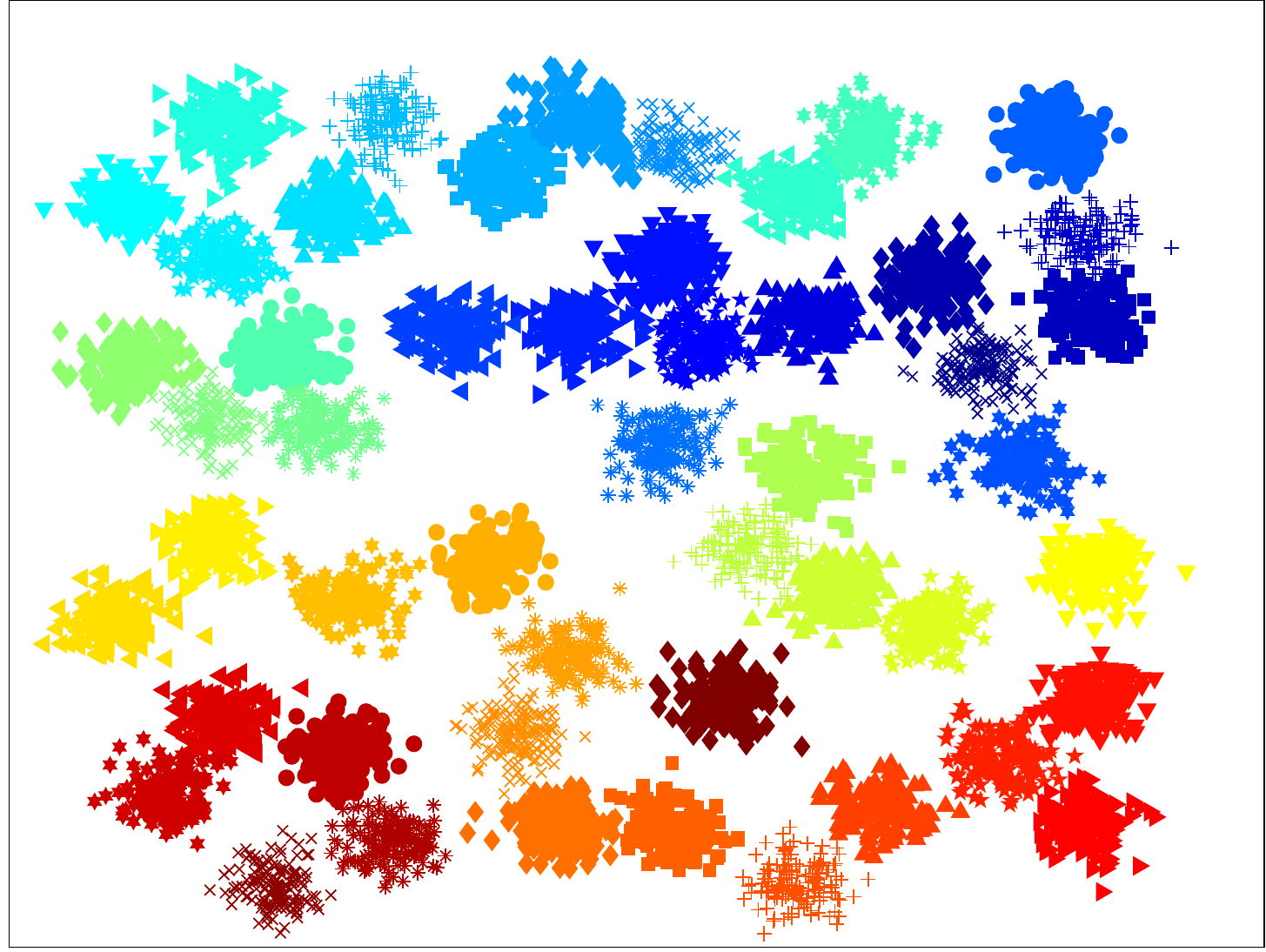}}
	\centerline{FDPC}
\end{minipage}
\hfill
\begin{minipage}{0.19\linewidth}
	\centerline{\includegraphics[width=1\textwidth]{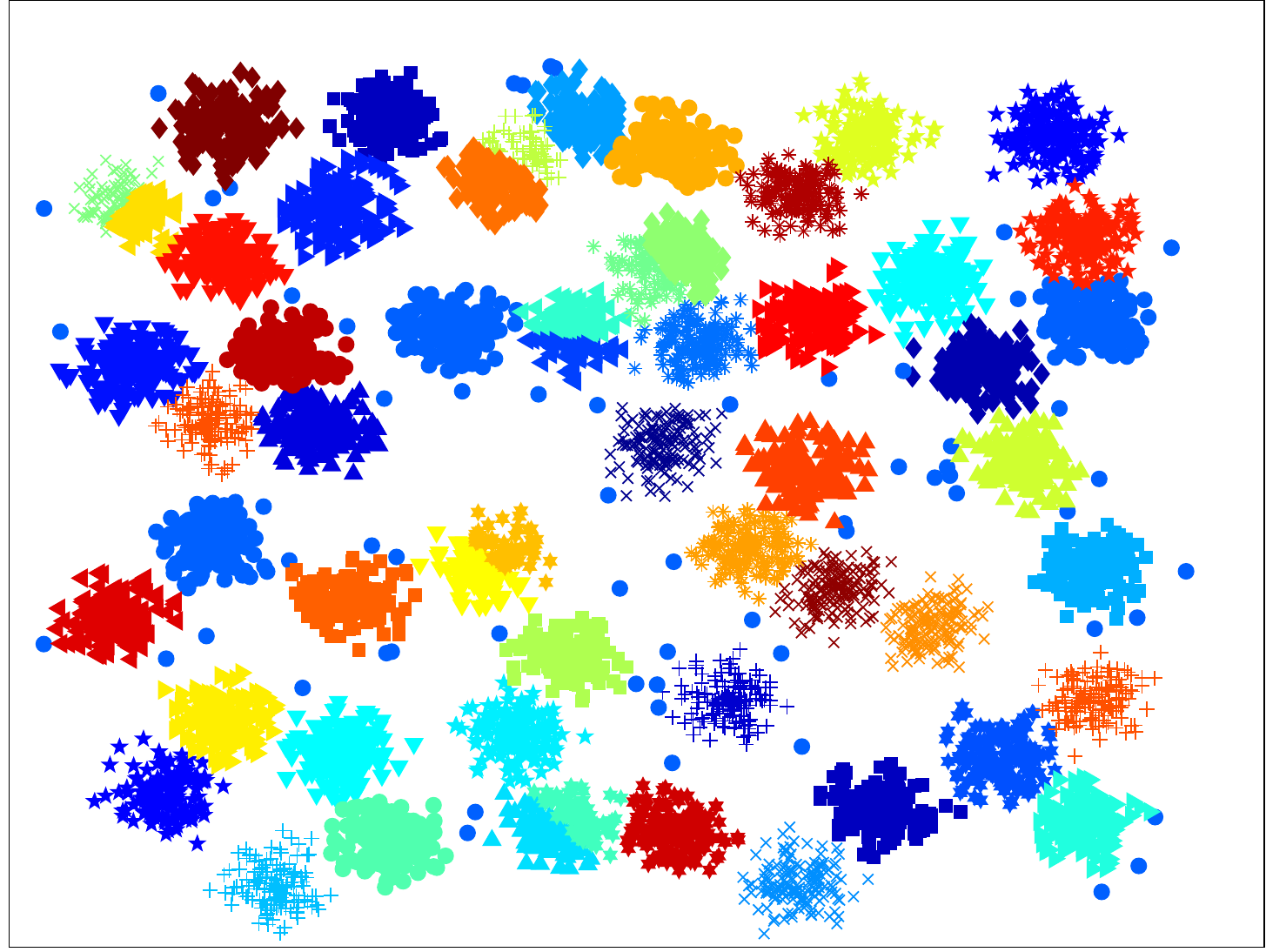}}
	\centerline{Kernel KM}
\end{minipage}
\hfill
\begin{minipage}{0.19\linewidth}
	\centerline{\includegraphics[width=1\textwidth]{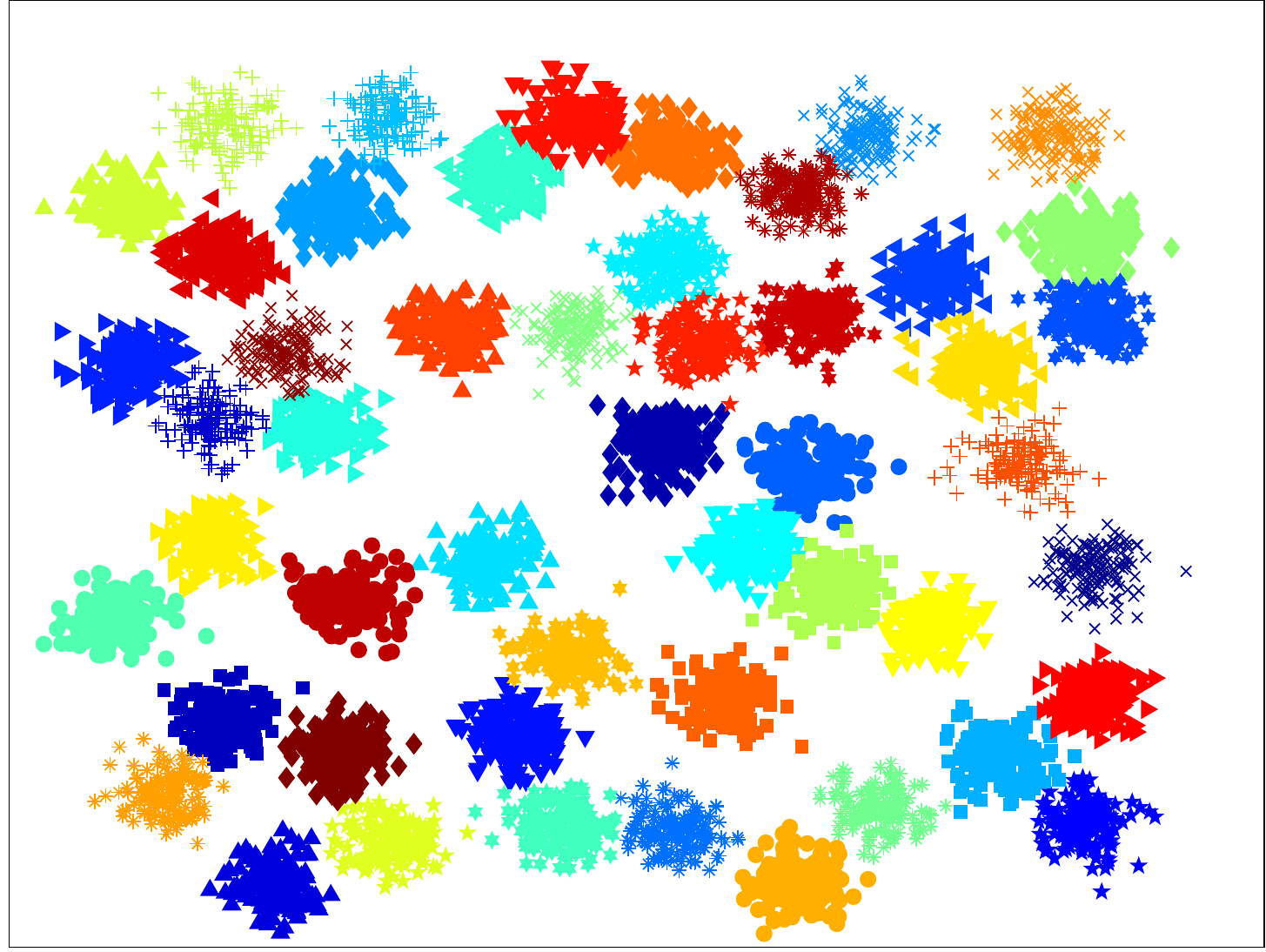}}
	\centerline{NCUT}
\end{minipage}
\hfill
\begin{minipage}{0.19\linewidth}
	\centerline{\includegraphics[width=1\textwidth]{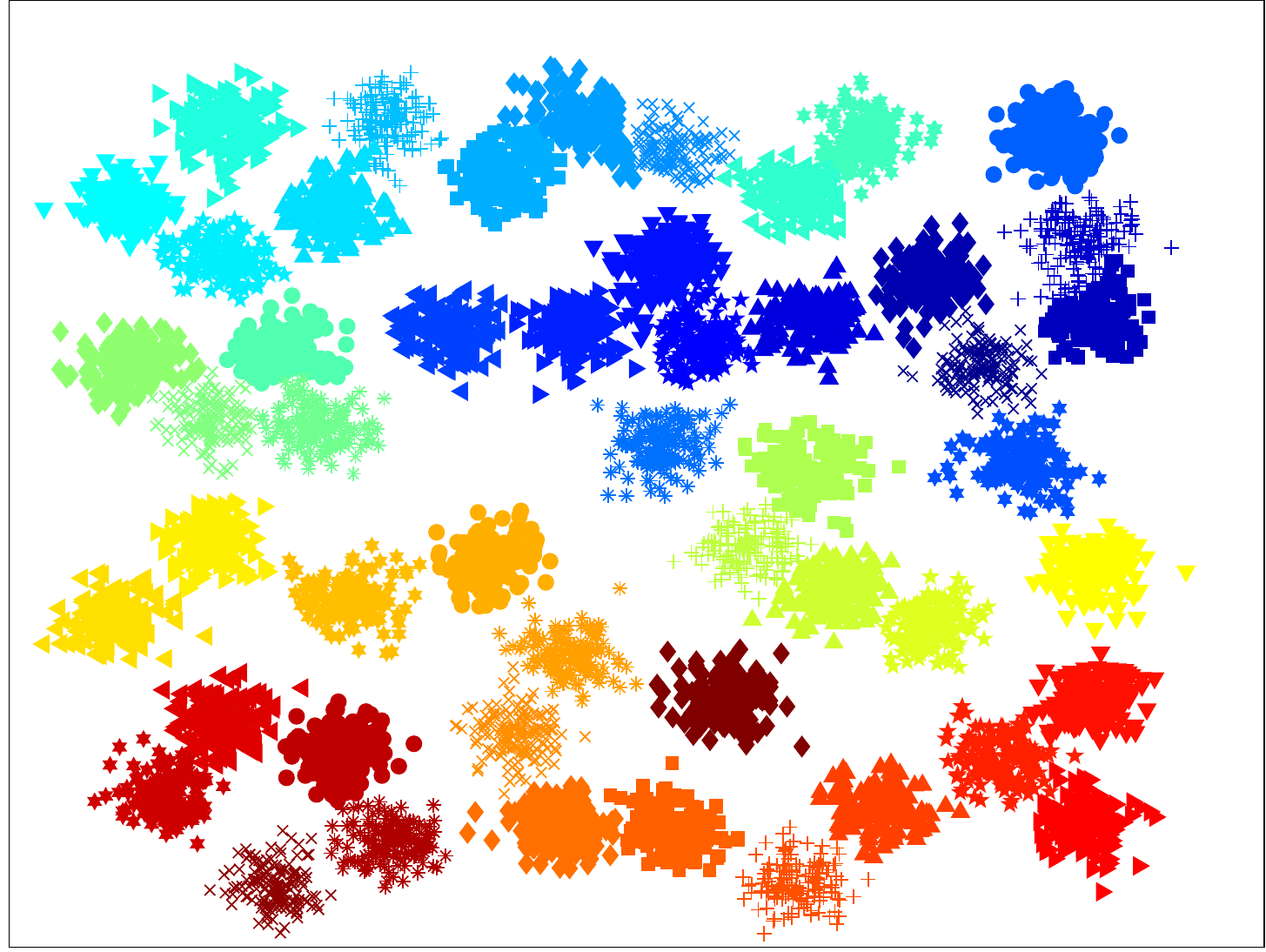}}
	\centerline{GOPC}
\end{minipage}
\caption{Experimental results on synthetic datasets; (I) $R15$; (J) $Jain$; (K) $D31$; (L) $PanelB$; (M) $Unbalance$; (N) $Lyga$; (O) $A3$.}
\label{synthetic3}
\end{figure*}

The experimental results shown in \cref{synthetic1,synthetic3} indicate that the GOPC algorithm achieves the best performance, followed by the FDPC. The GOPC algorithm yields satisfactory results on all datasets described above except for $Aggregation$. The reason is that there are a series of tight nodes in $Aggregation$ that bridge the two clusters together, making it difficult to manage this kind of dataset based only on dissimilarity. FDPC also performs well on most datasets except for the randomly distributed datasets, such as $DS1$ and $DS2$, which can be attributed to the fact that FDPC can only detect clusters that have an obvious center~\cite{Chen2016Effectively}. Although NCUT works perfectly on most datasets, it gives very poor results on datasets $Spiral$ and $Jain$. According to Refs.~\cite{Dhillon2007Weighted,Tzortzis2008The}, Kernel KM has a close relation to spectral clustering and can manage data that is not linearly separable. However, it is difficult to find a proper $\sigma$ for datasets with arbitrary shaped clusters in our experiments. The poor clustering results are due mainly to the fact that using the Gaussian kernel as the corresponding transformation is intractable for some datasets. SLC performs the worst of all the algorithms. Like the MST-based clustering algorithms, it has two drawbacks: first, a few objects far from other objects define a separate cluster; second, two connected clusters are taken as two parts of one cluster. In summary, the GOPC algorithm outperforms the compared clustering algorithms. Although the GOPC algorithm could not separate the connected clusters like $Aggregation$, it could separate clusters as long as a narrow gap existed between them in the experiments above. 

\subsection{Experiments on real-world datasets}
\label{subsec3.2}
The experimental results on synthetic datasets depict that the GOPC algorithm can recognize all kinds of clusters regardless of their shapes, sizes, or densities. Thus, we would like to test the GOPC algorithm on real-world clustering tasks. For quantitative performance metrics, we employed three popular external indexes, namely the Rand index ($RI$), the Adjusted Rand index ($ARI$) and Normalized mutual information ($NMI$). These three metrics, where the value of 1 denotes that the clustering result is perfect, measure how perfectly the clustering results match with the ground truth.

\subsubsection{Mnist dataset}
\begin{figure*} 
	\centering
	\subfloat[Examples of $Mnist$ dataset.]{\includegraphics[width=0.33\textwidth]{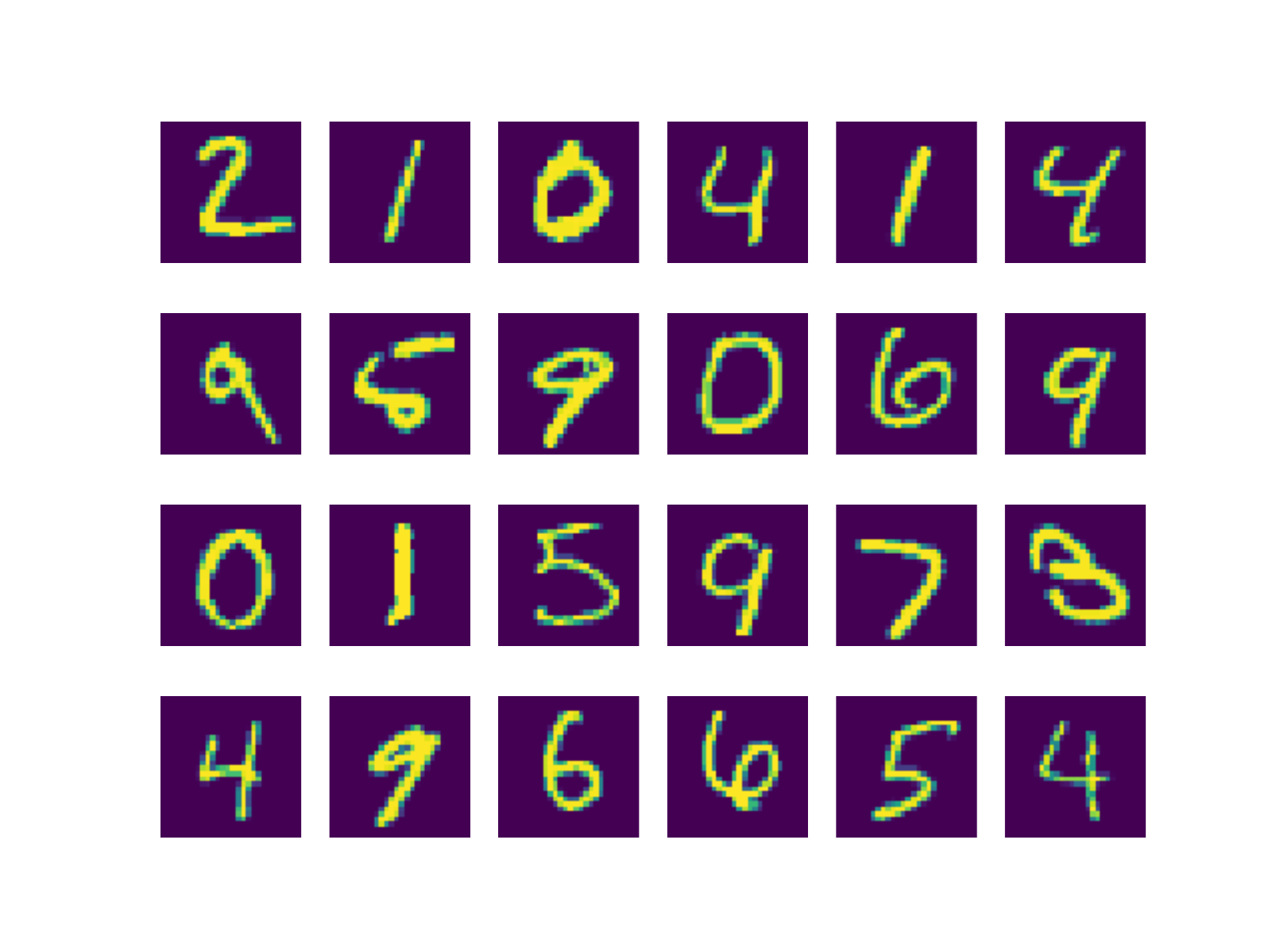}}     
	\subfloat[Three-dimensional features of test set.]{\includegraphics[width=0.33\textwidth]{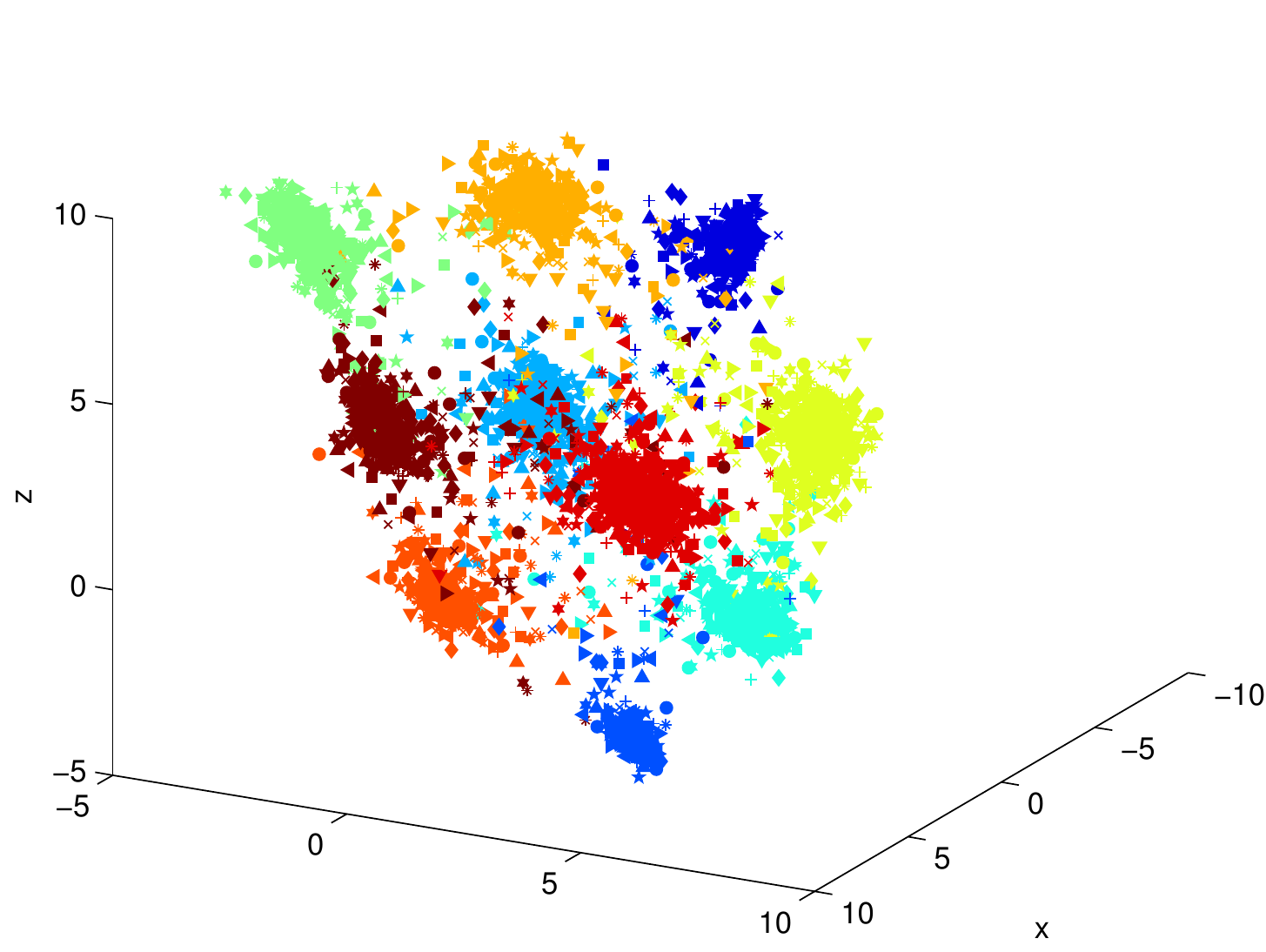}}
	\subfloat[Quantitative comparison.]{\includegraphics[width=0.33\textwidth]{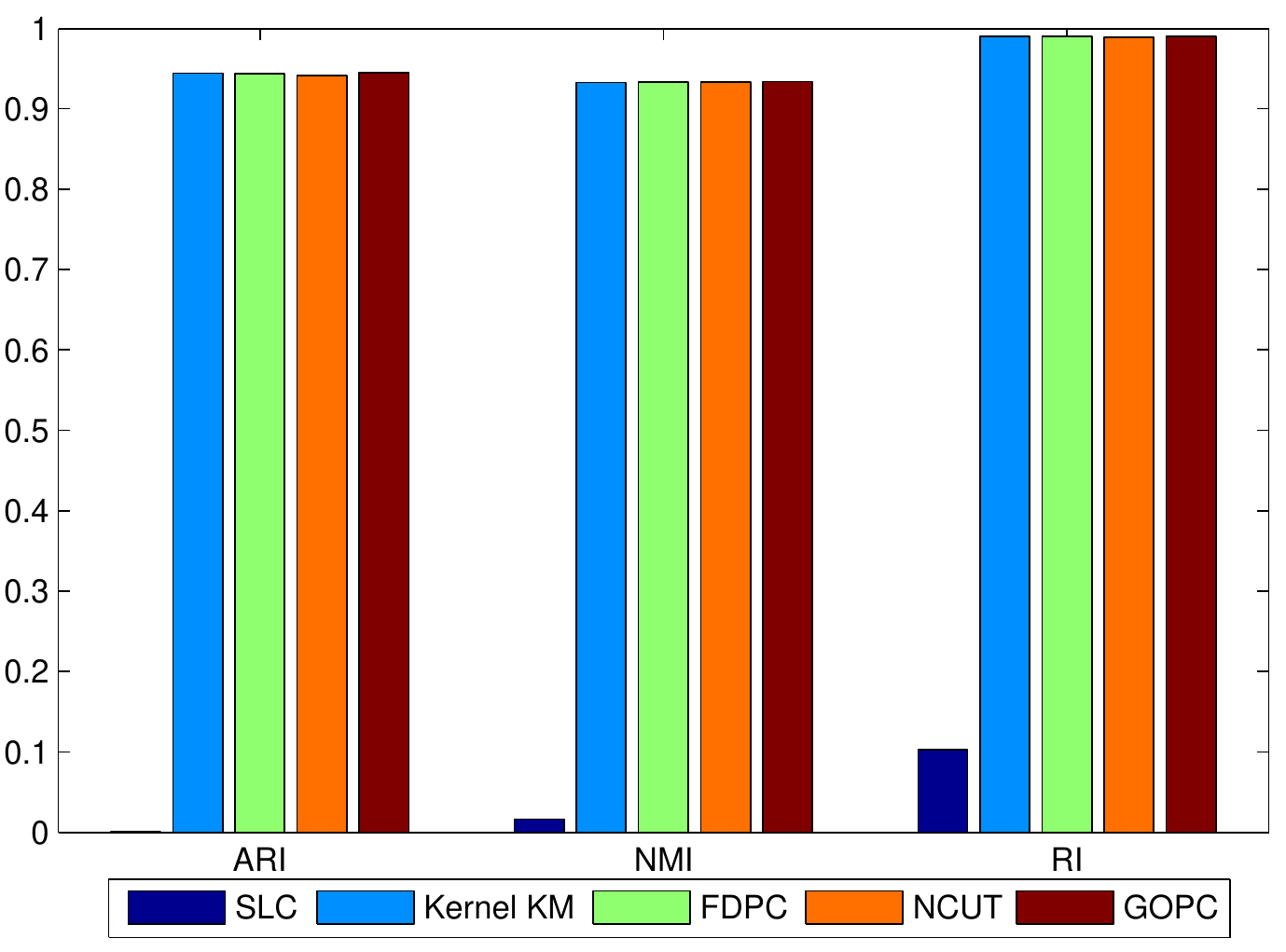}} 
	\caption{$Mnist$ dataset.}
	\label{Mnist}
\end{figure*}
$Mnist$\footnote{http://yann.lecun.com/exdb/mnist/} is a well-known handwritten digit dataset for data mining. It has a training set of 60,000 examples, and a test set of 10,000 examples. The digits in the database have been size-normalized and centered in fix-sized images. \Cref{Mnist}(a) displays some sample images from this database. We trained the Siamese network~\cite{Chopra2005Learning} on the training set, and then embed an image with 784 pixels into a point with three-dimensional features for the test set, as shown in \cref{Mnist}(b). 

We tested the GOPC algorithm on the test set with three dimensional features. Four well-known clustering algorithms---SLC, kernel KM, FDPC, and NCUT---were used for baseline comparison. The clustering results are summarized in \cref{Mnist}(c), where the three groups of bars represent the results in terms of $ARI$, $NMI$, and $RI$ in order. It is evident that all the algorithms could obtain satisfactory results, except for SLC, which defines a few objects as a cluster. The GOPC algorithm is slightly better than the other algorithms. Based on the Siamese network trained on the training set, objects with different numbers in the test set are well separated and form ten spherical clusters. This kind of dataset is not too difficult for most clustering algorithms.

\subsubsection{Olivetti face dataset}

\begin{figure*}[h] 
	\centering
	\subfloat[Examples of $Olivetti\; face$ dataset.]{\includegraphics[width=0.32\textwidth]{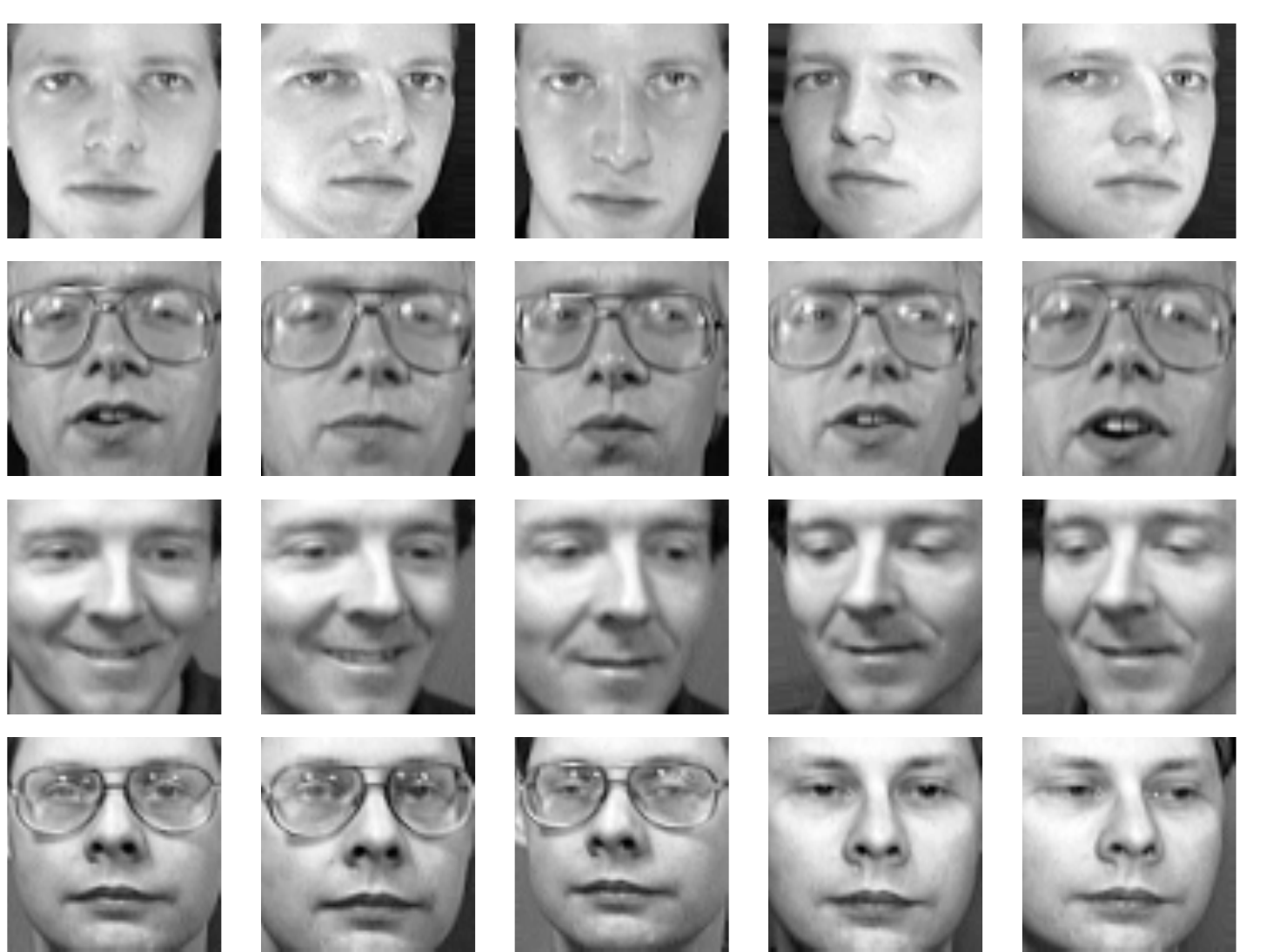}}     
	\subfloat[Decision graph of FDPC.]{\includegraphics[width=0.33\textwidth]{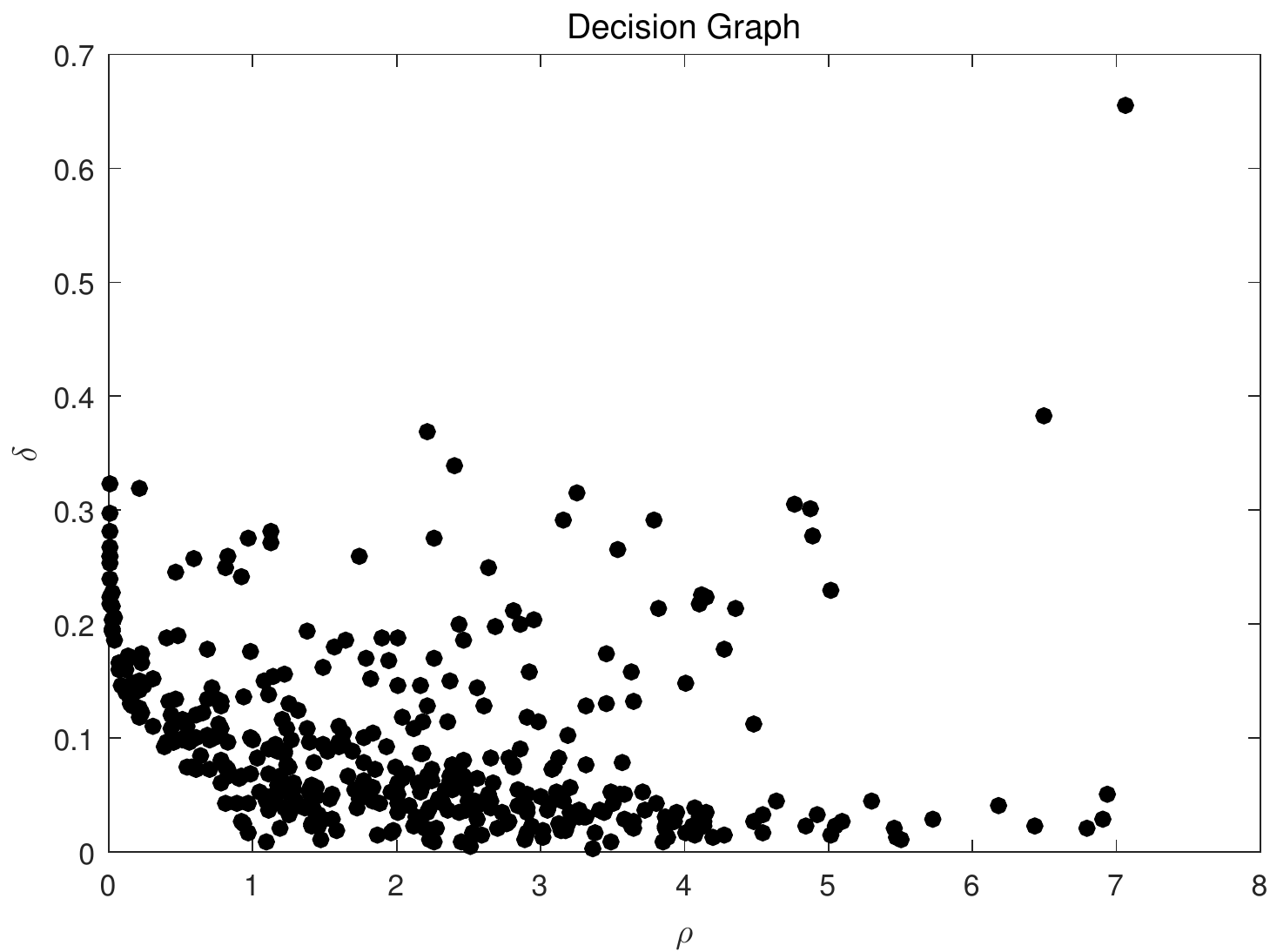}} 
	\subfloat[Quantitative comparison.]{\includegraphics[width=0.33\textwidth]{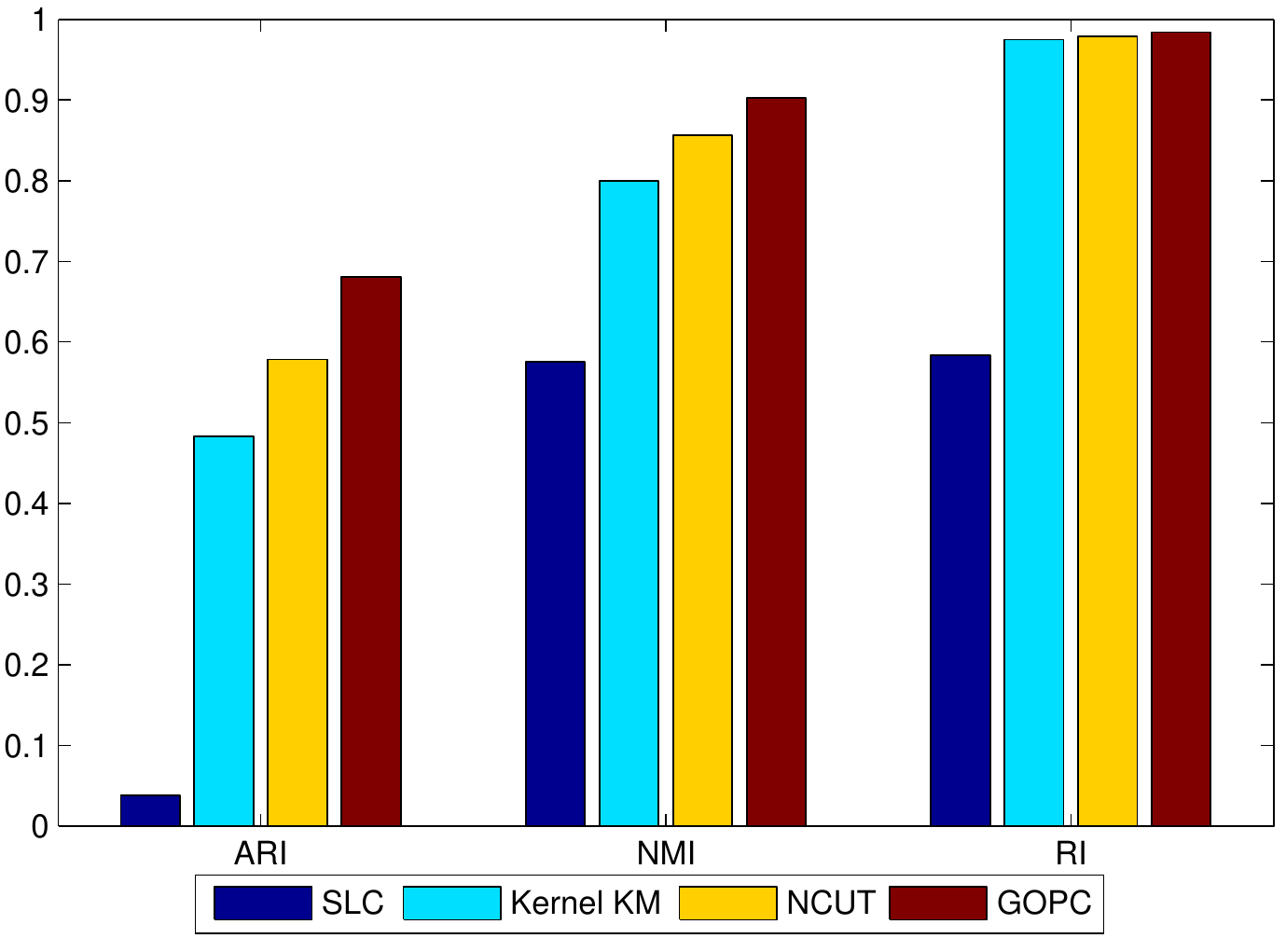}}
	\caption{$Olivetti\; face$ dataset.}
	\label{face}
\end{figure*}

The $Olivetti\; face$\footnote{http://www.cl.cam.ac.uk/research/dtg/attarchive/facedatabase.html} dataset consists of 400 gray-level images from 40 persons. The images were taken at different times, under different lighting conditions, and show different facial expressions and facial details. The size of each image is $64\times 64$ pixels. Some examples of the face images are shown in \cref{face}(a). The similarity between two images is calculated by complex wavelet structural similarity \cite{Sampat2009Complex}. There is a little difference in the similarity between the images. Thus, clustering them is quite challenging.

The clustering results using different methods are shown in \cref{face}(c). Here, we do not give the result implemented by FDPC because it is difficult to obtain the right number of clusters according to the decision graph in \cref{face}(b). The experimental results show that the proposed algorithm is obviously superior to the other three algorithms. 

\subsubsection{Other datasets}
We conducted further experiments on six real-world datasets, including $Iris$, $Pendigits$, $Mice$ $Protein$, $YaleB$, $COIL20$, and $COIL100$. $Iris$ contains 3 clusters of 50 instances, where each cluster refers to a type of $Iris$ plant. $Pendigits$ is a collection of handwritten digits (0-9) from 44 writers. $Mice$ $Protein$ is a dataset that consists of the expression levels of 77 proteins measured in the cerebral cortex of eight classes of control and trisomic mice. $YaleB$ consists of 38 subjects, each of which is represented by 64 face images acquired under different illumination conditions. $COIL20$ consists of 1440 gray-scale image samples of 20 objects, such as ducks and car models. $COIL100$ consists of 7200 images of 100 objects. For the latter three datasets, which consist of images, we used the deep neural network architecture proposed in Ref.~\cite{ji2017deep} to form an affinity matrix for each dataset, then applied the clustering algorithms mentioned above using this affinity matrix to obtain the clustering results.
\begin{table}[h]
	\centering
	\caption{Quantitative comparison of clustering results generated by the Kernel KM, NCut, FDPC and GOPC algorithms on six real-world datasets.}
	\label{UCI}
	\begin{tabular}{|c | c | c| c| c|}
		\hline
		\multicolumn{2}{|c|}{methods } & RI & ARI & NMI \\
		\hline
		\multirow{5}{*}{$Iris$} & Kernel KM & 0.8859 & 0,7434 & 0,7660 \\
		& NCut & 0.8797 & 0.7302 & 0.7582  \\
		& FDPC & 0.8923 & 0.7592 & 0.8058 \\
		& GOPC & \textbf{0.9495} & \textbf{0.8858} & \textbf{0.8705} \\
		\hline
		\multirow{5}{*}{$Pendigits$} & Kernel KM & 0.9325 & 0.6396 & 0.7464 \\
		& NCut & 0.9013 & 0.5138 & 0.6837 \\
		& FDPC & 0.9322 & 0.6457 & 0.7761  \\
		& GOPC & \textbf{0.9371} & \textbf{0.6813} & \textbf{0.8146} \\
		\hline
		\multirow{5}{*}{$Mice$ $Protein$} & Kernel KM & 0.8066 & 0.1671 & 0.3115  \\
		& NCut & 0.8207 & 0.2128 & 0.3345 \\
		& FDPC & 0.8224 & 0.3007 & 0.5074 \\
		& GOPC & \textbf{0.8533} & \textbf{0.4154} & \textbf{0.6131} \\
		\hline
		\multirow{5}{*}{$YaleB$} & Kernel KM & 0.9591 & 0.2428 & 0.5743 \\
		& NCut & \textbf{0.9778} & \textbf{0.6015} & \textbf{0.8149} \\
		& FDPC & NA & NA & NA \\
		& GOPC & 0.9571 & 0.3766 & 0.7343 \\
		\hline
		\multirow{5}{*}{$COIL20$} & Kernel KM & 0.9179 & 0.1919 & 0.4569 \\
		& NCut & 0.9475 & 0.5290 & 0.7543 \\
		& FDPC & NA & NA & NA \\
		& GOPC & \textbf{0.9614} & \textbf{0.6718} & \textbf{0.9114} \\
		\hline
		\multirow{5}{*}{$COIL100$} & Kernel KM & 0.8838 & 0.0544 & 0.6298 \\
		& NCut & 0.9773 & 0.2881 & 0.6795 \\
		& FDPC & NA & NA & NA \\
		& GOPC & \textbf{0.9878} & \textbf{0.5089} & \textbf{0.8360} \\
		\hline
	\end{tabular}

\footnotesize{$^{\rm 1}$NA, not applicable.}   
\end{table}

Results from six real-world datasets are reported in \Cref{UCI}. For each dataset, the best results are highlighted in bold. SLC is sensitive to noise, thus it was no longer used as the baseline for the rest of the experiments. The FDPC algorithm is not applicable to $YaleB$, $COIL20$, and $COIL100$ because its source code is designed for a distance matrix, not an affinity matrix. \Cref{UCI} shows that the GOPC algorithm achieves the highest accuracy on five of the six datasets, which implies that the GOPC algorithm can handle most datasets in real-world applications.

\subsection{Comparison of running time}

If we ignore the filtering step, the GOPC algorithm's time complexity is $O(k\times n^2)$, where $k$ is the number of clusters and $n$ is the number of data points. Compared with NCUT ($O(n^3)$ in general), the proposed algorithm is quite acceptable. More importantly, with the help of $nn(x)$, the running time of the GOPC algorithm is faster than $O(k\times n^2)$. In this subsection, we show only the running time of the GOPC algorithm because the running environments of the algorithms mentioned in \Cref{sec3} are different. For instance, FDPC is implemented in MATLAB while SLC is implemented in Python.

\begin{figure*}[h]
	\centering
	\subfloat[Running time for different values of $k$.]{\includegraphics[width=0.4\textwidth]{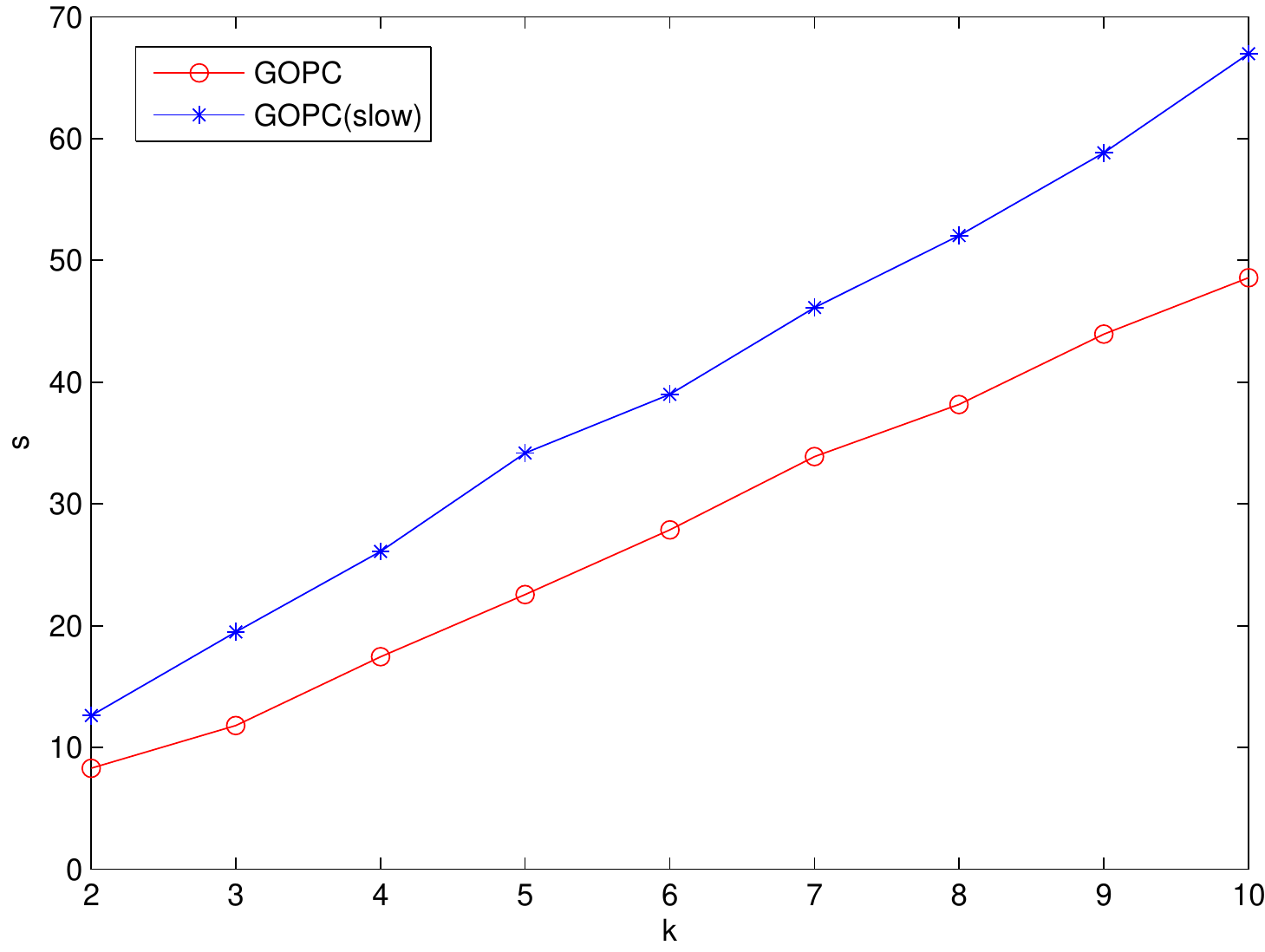}}     
	\subfloat[Running time for different sizes.]{\includegraphics[width=0.4\textwidth]{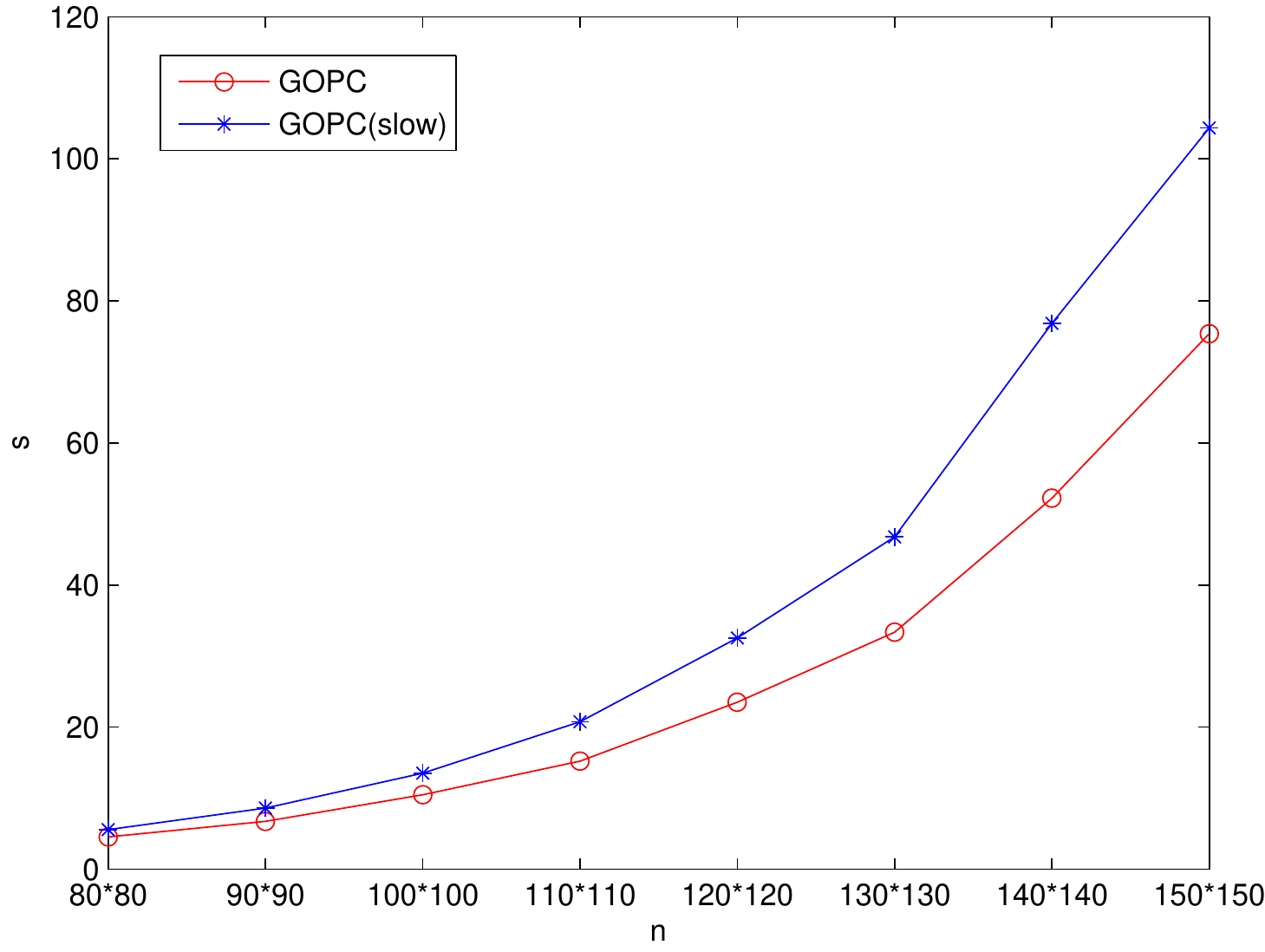}}
	\caption{Comparison of running time.}
	\label{time}
\end{figure*}

We carry out experiments on an image with $1000\times 1000$ pixels to examine the running time of GOPC. The algorithm that does not use $nn(x)$ to filter out information is recorded as the GOPC(slow) algorithm. We conducted two experiments to investigate the running time of these algorithms. In the first experiment, we changed $k$ from 2 to 10 and resized the image to $130\times 130$ to observe the running time trend of the algorithms. In the second experiment, with the number of clusters $k$ fixed at 7, we resized the image to different sizes and ran the algorithms on the subsets. Experimental results are shown in \cref{time}.

From these results, we reached the following conclusions: (1) The running time becomes longer as $k$ or $n$ increases. (2) The GOPC algorithm is superior to the GOPC(slow) algorithm. This means that $nn(x)$ plays an important role in the proposed algorithms and it accelerates the algorithms' operation. It also shows that the running time of GOPC is faster than $O(k\times n^2)$. 

\section{Image Segmentation}
\label{is}

We applied the GOPC algorithm to the segmentation evaluation database\footnote{http://www.wisdom.weizmann.ac.il/~vision/MorossLab/} and used NCUT, which has been widely used in image segmentation, for baseline comparison. We resized each image to $130\times 130$ pixels and constructed an affinity matrix based on the intervening contours method with the source code taken from Shi's website\footnote{http://www.cis.upenn.edu/$\sim$jshi/software/}. 

Because the length was limited, we provided nine representative examples to demonstrate that the GOPC algorithm could successfully be used in this domain. The experimental results are given in \cref{image segmentation}. The optimal number of image segments was selected by trial and error.

From the experimental results, we can see that both algorithms yield satisfactory results because the regions corresponding to objects or object parts are clearly separate from each other. The GOPC algorithm is superior to the NCUT algorithm for the following reasons: First, the GOPC algorithm has higher accuracy. For example, the wheels of carriage in the eighth picture can be detected by the GOPC algorithm. Second, the GOPC algorithm can identify small clusters. The cross sign on the building in the second picture and the scrubby tree in the sixth picture are separated as a single cluster by the GOPC algorithm in their respective segmentations. Lastly, the GOPC algorithm needs fewer clusters to separate objects and background. The GOPC algorithm can separate the helicopter from the background by setting $k=2$, while the NCUT algorithm needs at least $k=4$. In summary, the GOPC algorithm can be successfully applied in this domain and has better results than the NCUT algorithm.
\begin{figure}
	\centering
	\subfloat[Input images]{
		\begin{minipage}{.32\linewidth}
			\includegraphics[width=\textwidth]{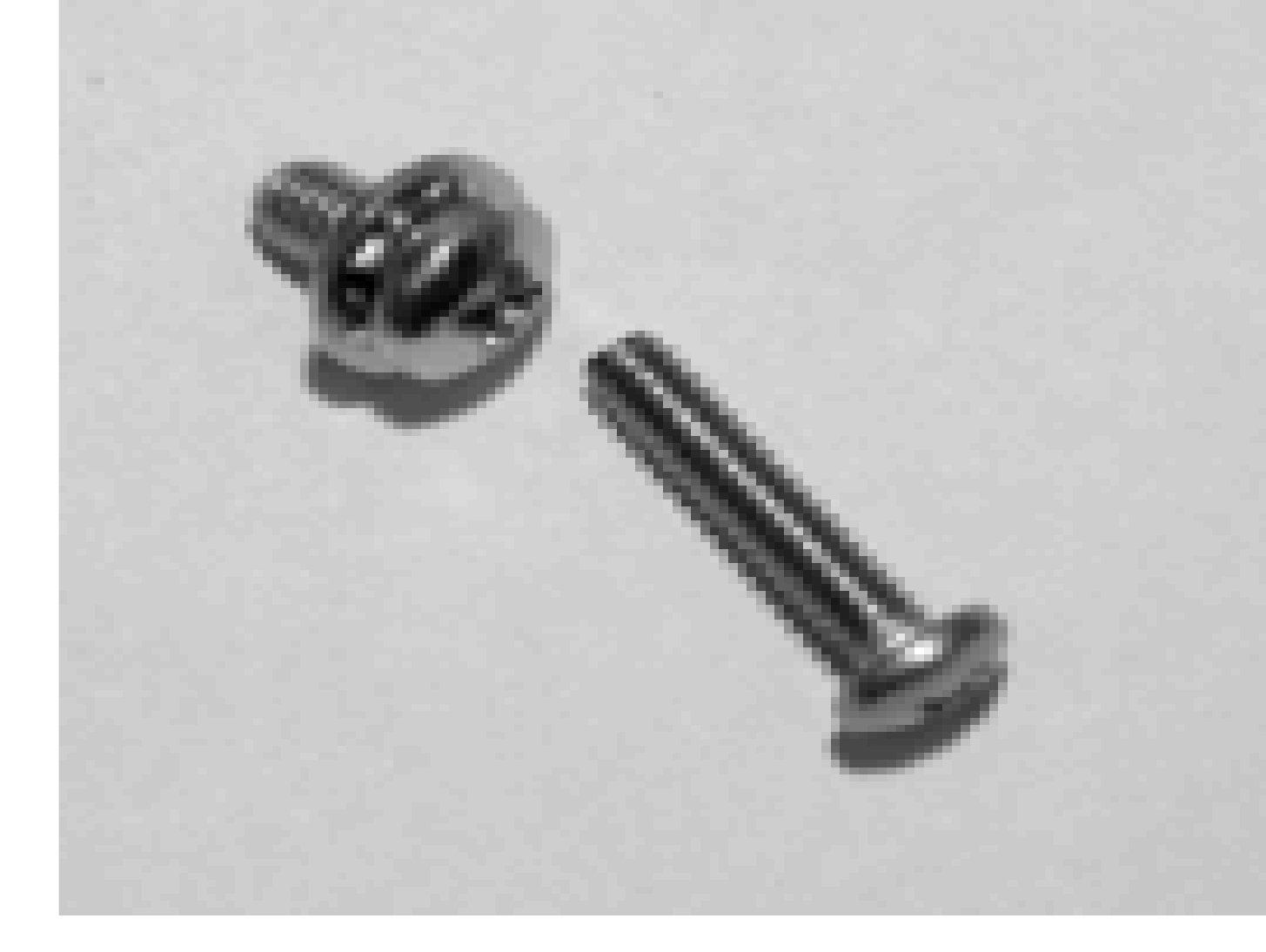} \\ \vfill
			\includegraphics[width=\textwidth]{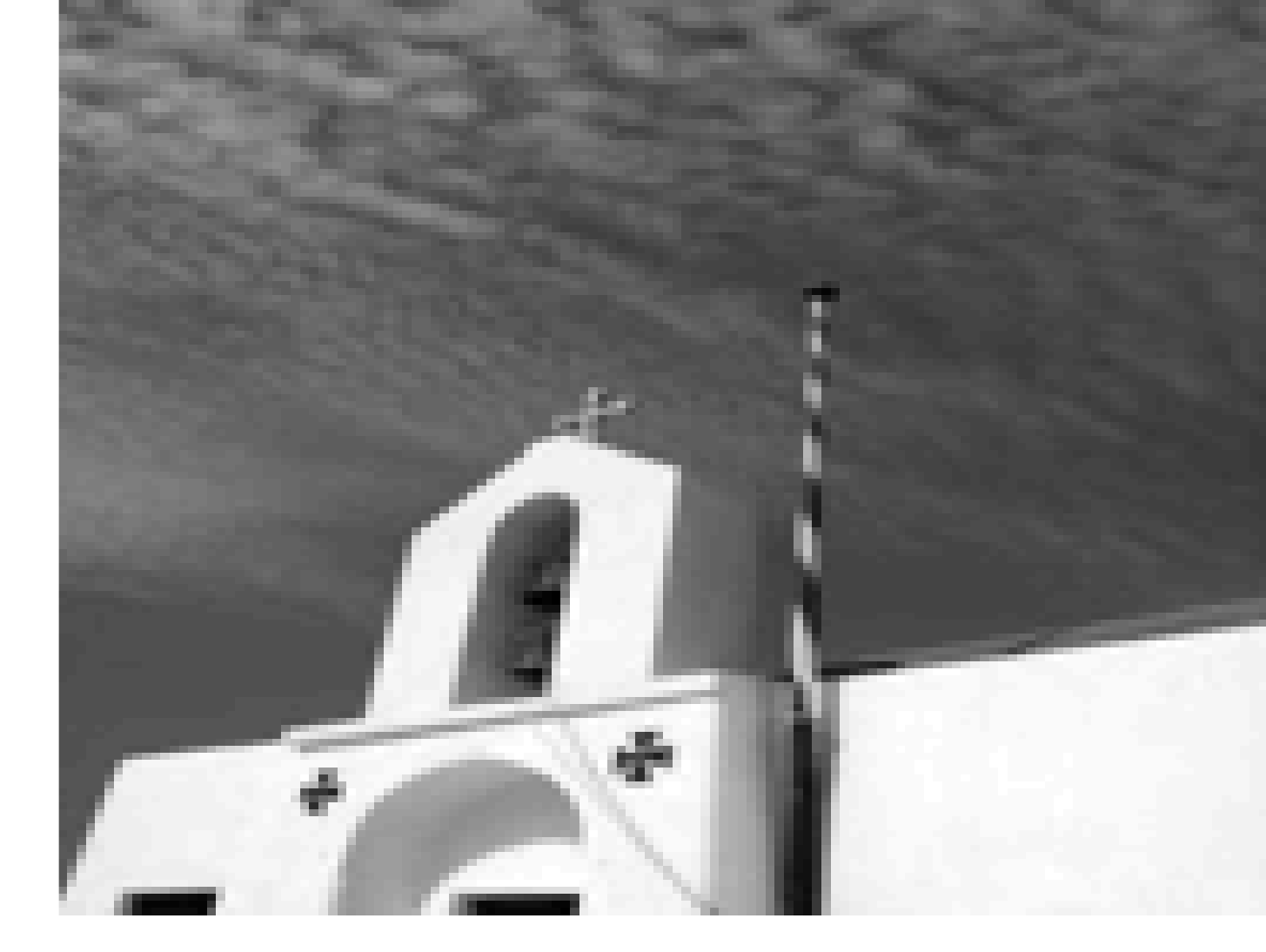}  \\ \vfill
			\includegraphics[width=\textwidth]{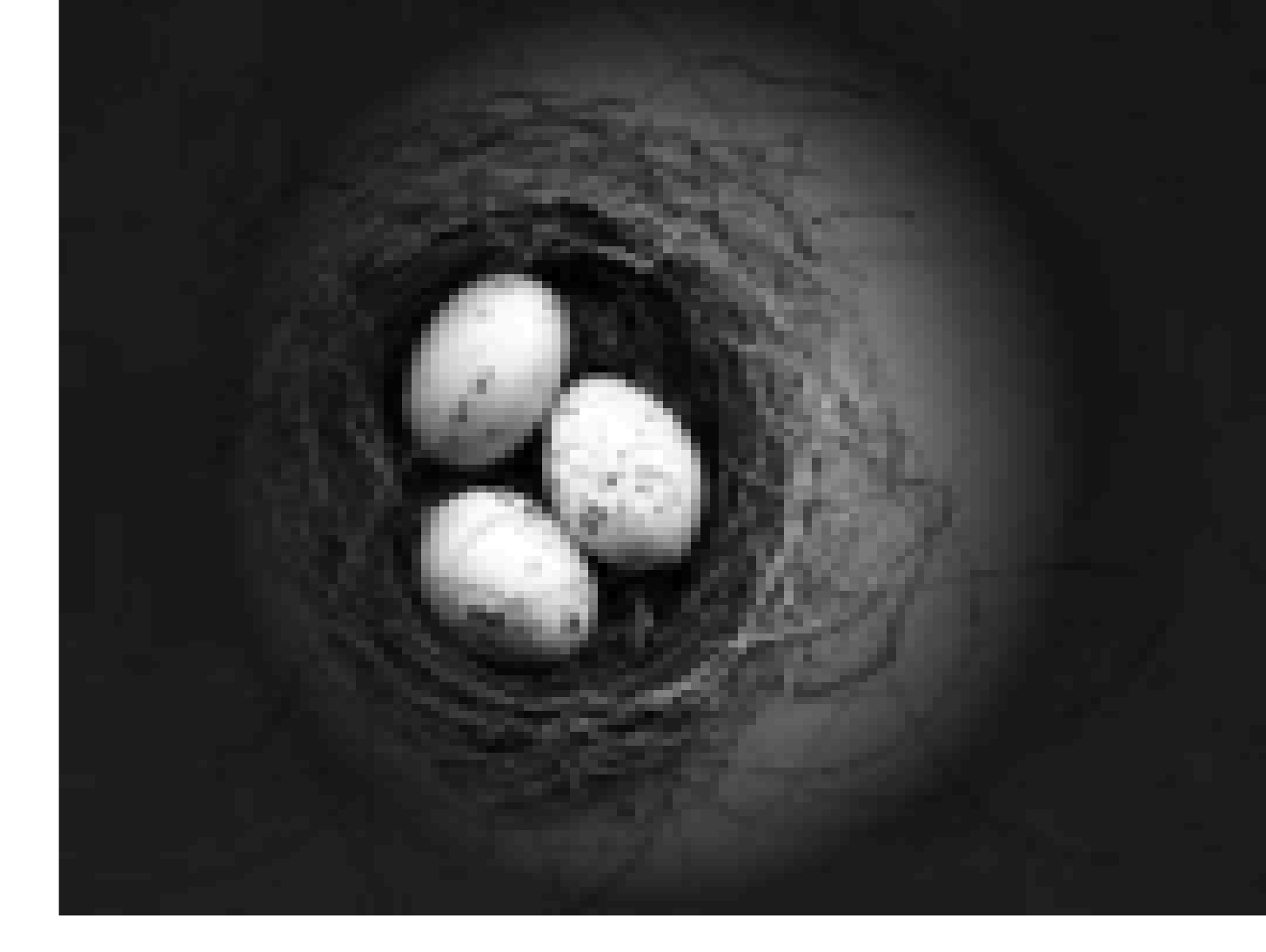}  \\ \vfill
			\includegraphics[width=\textwidth]{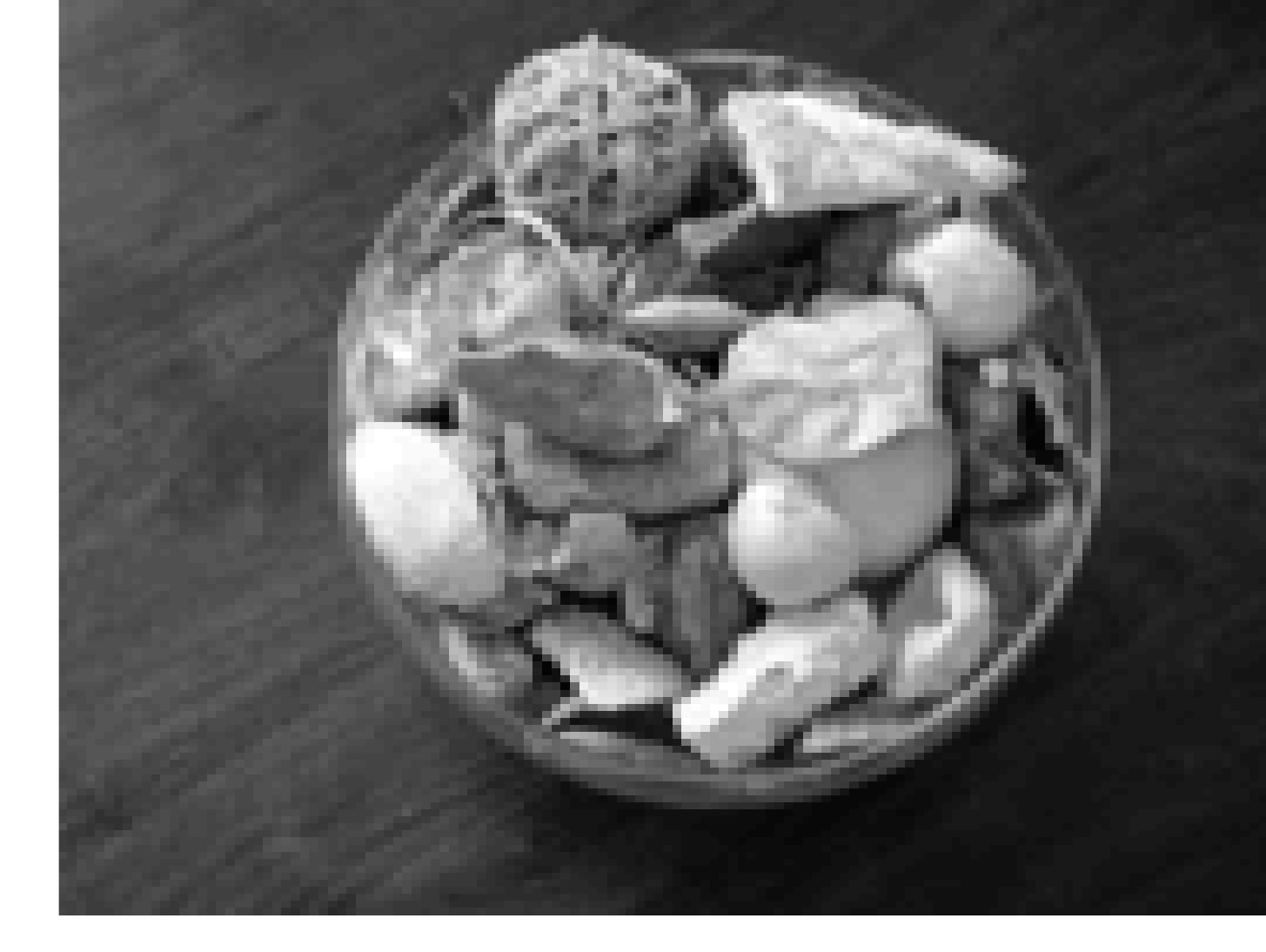} \\ \vfill
			\includegraphics[width=\textwidth]{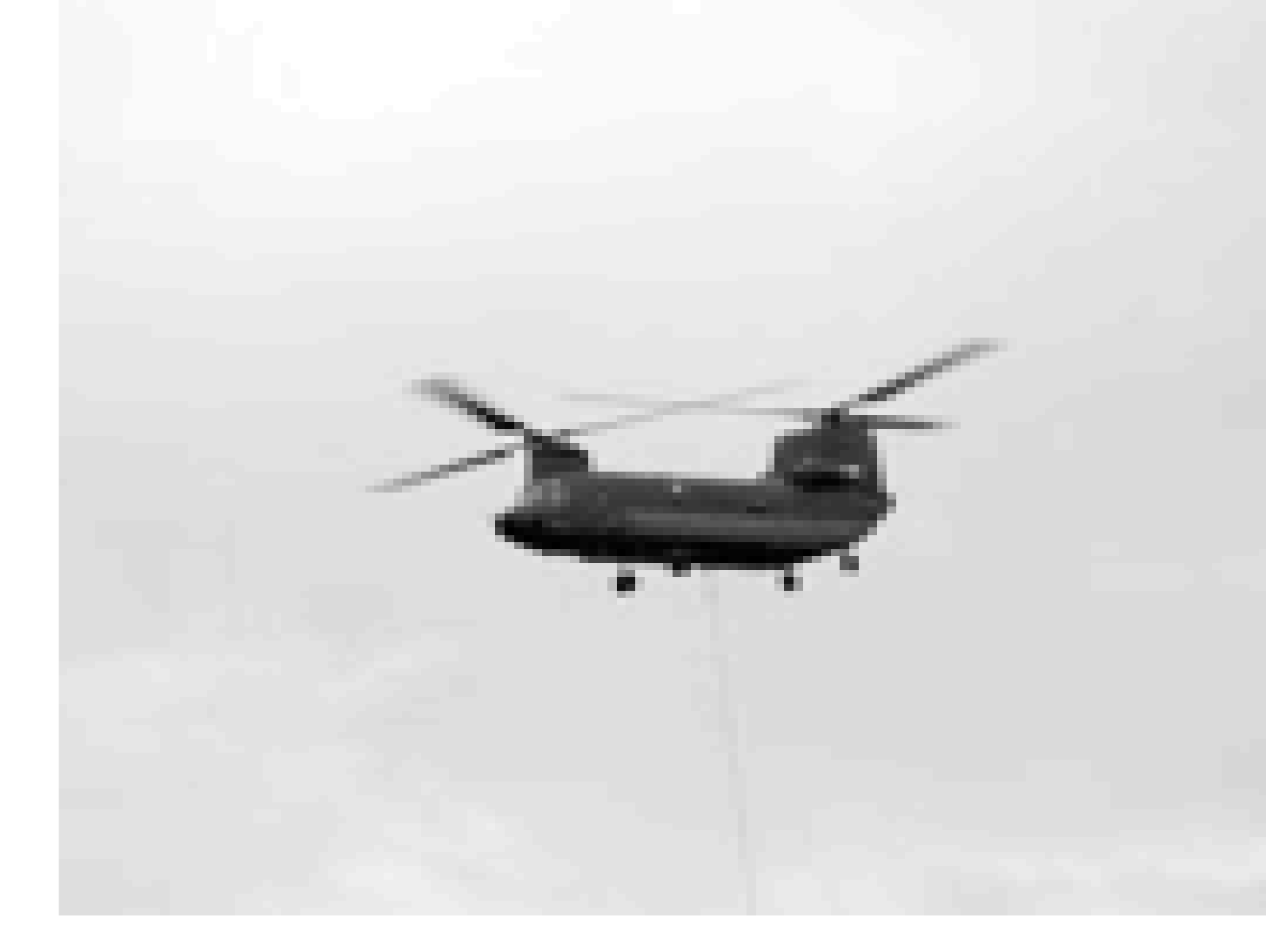}  \\ \vfill
			\includegraphics[width=\textwidth]{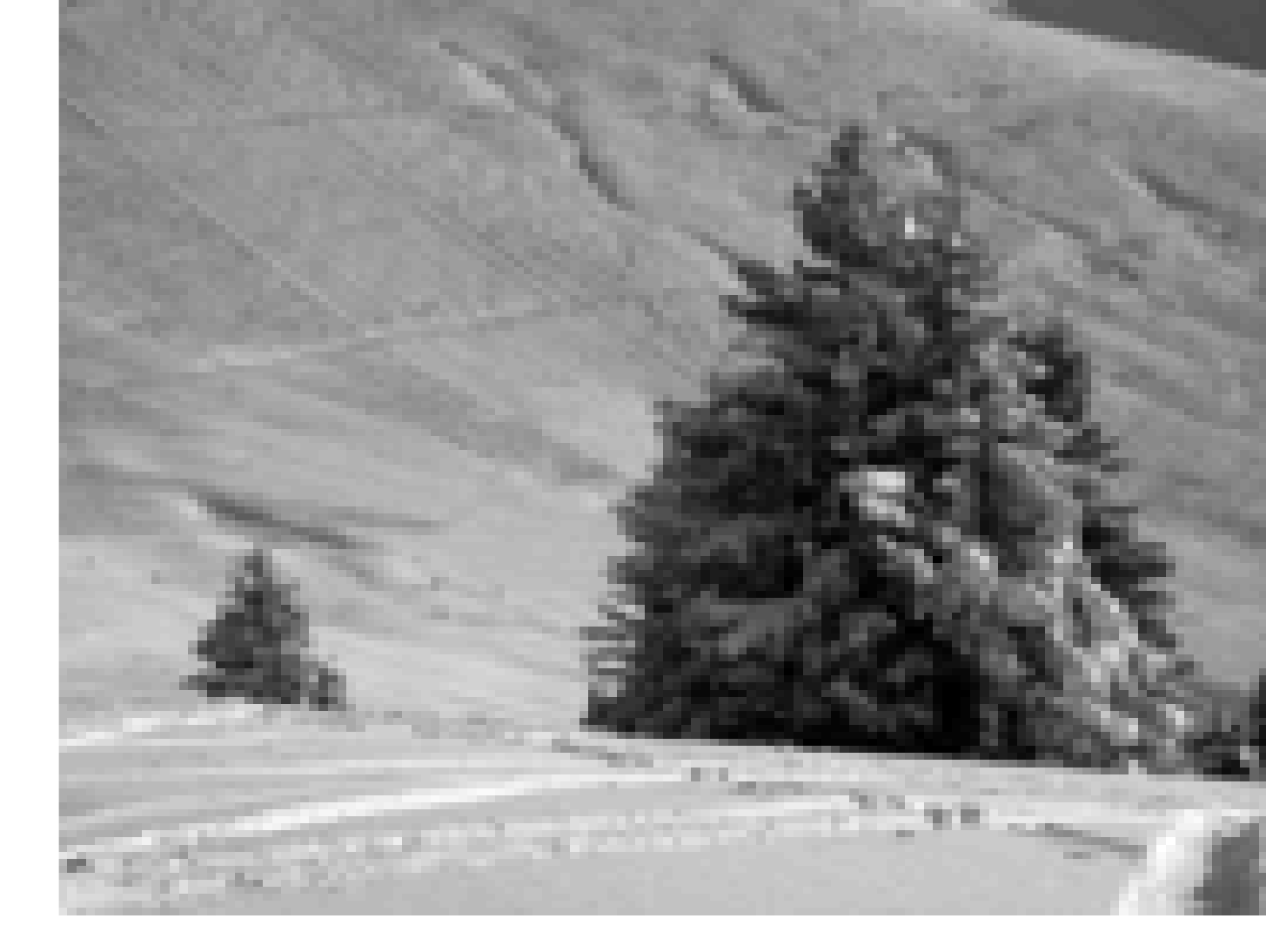}  \\ \vfill
			\includegraphics[width=\textwidth]{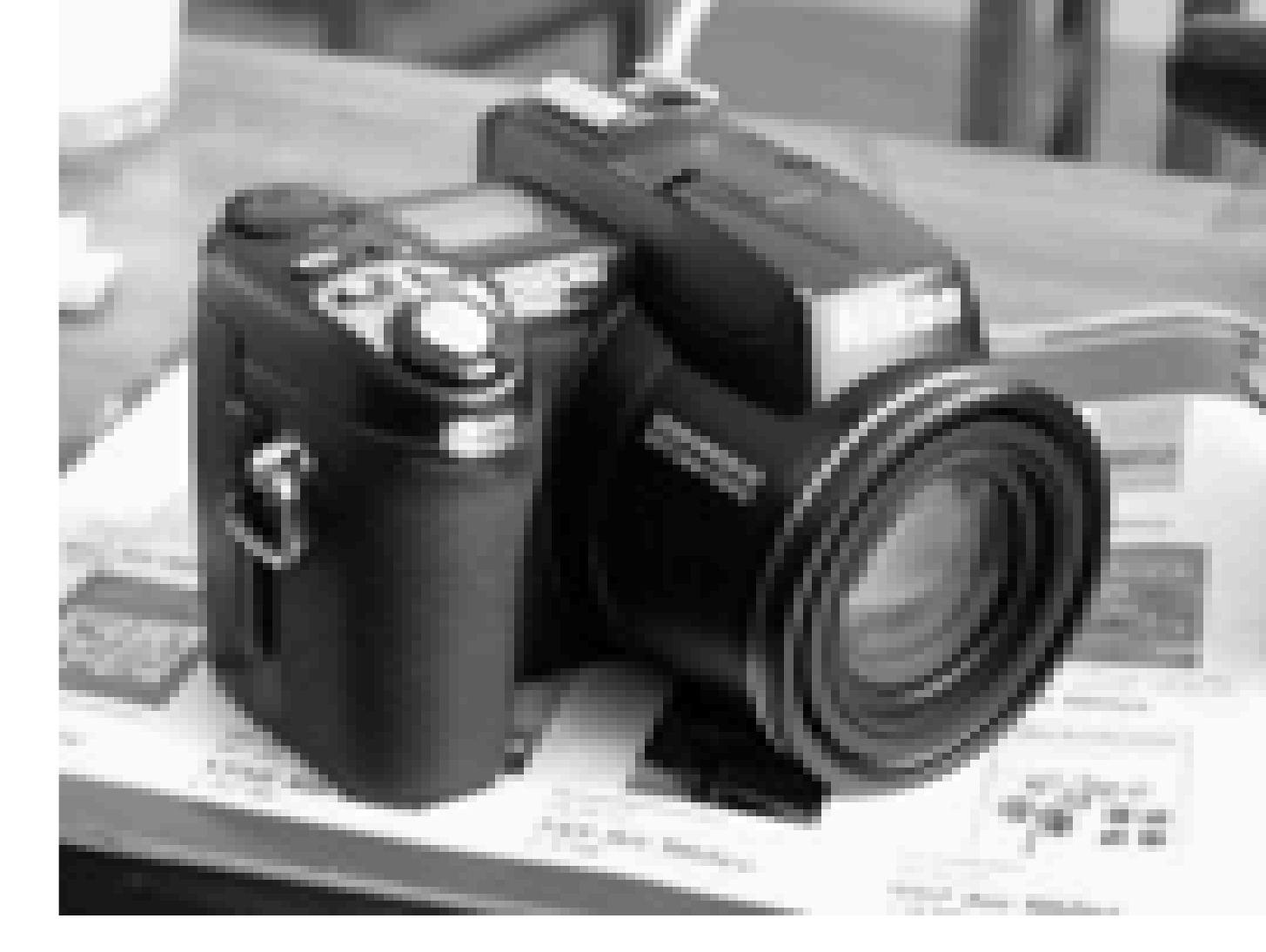}  \\ \vfill
			\includegraphics[width=\textwidth]{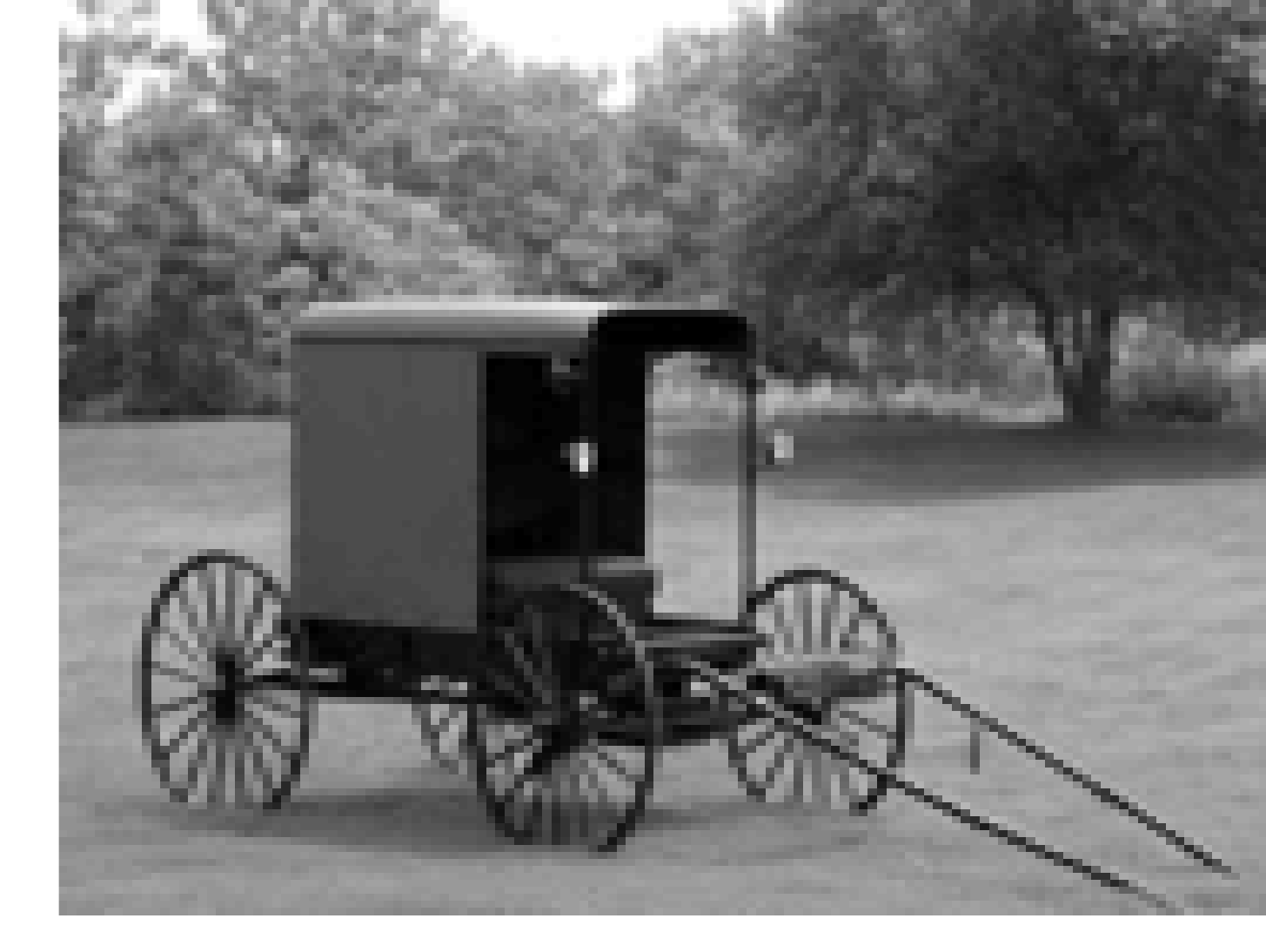}  \\ \vfill
			\includegraphics[width=\textwidth]{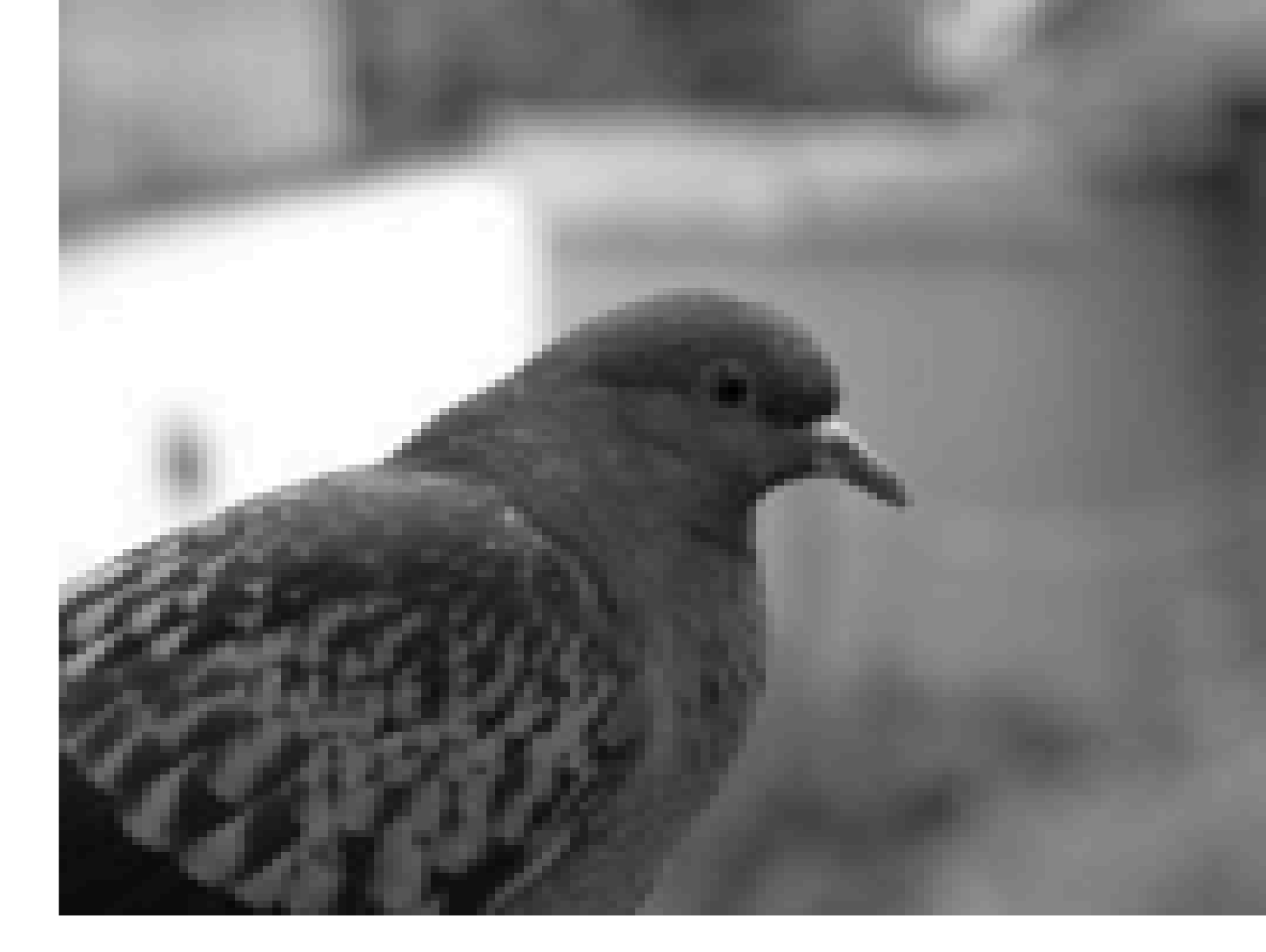}
		\end{minipage}
	}
	\subfloat[GOPC algorithm]{
		\begin{minipage}{.32\linewidth}
			\includegraphics[width=\textwidth]{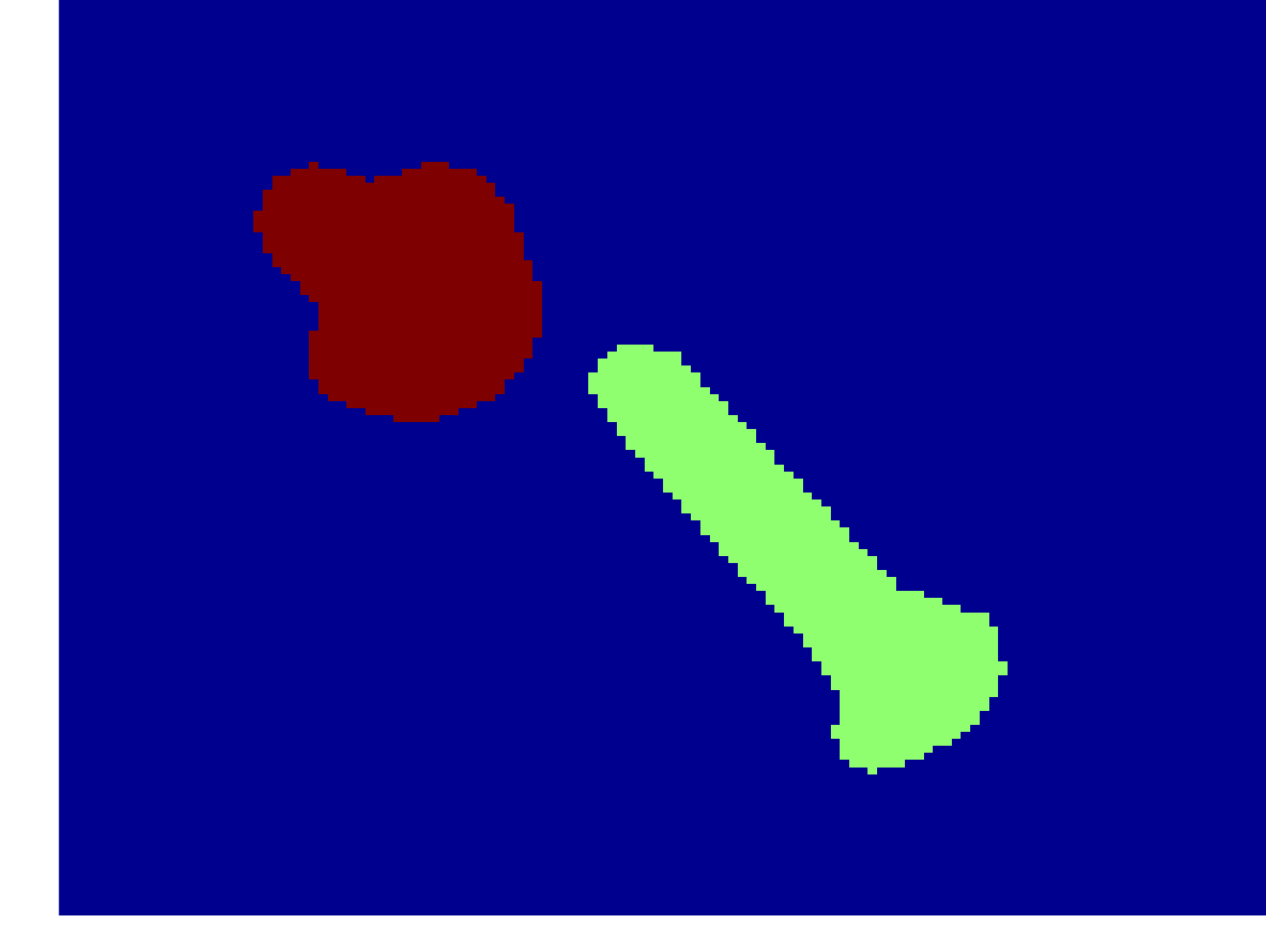} \\ \vfill
			\includegraphics[width=\textwidth]{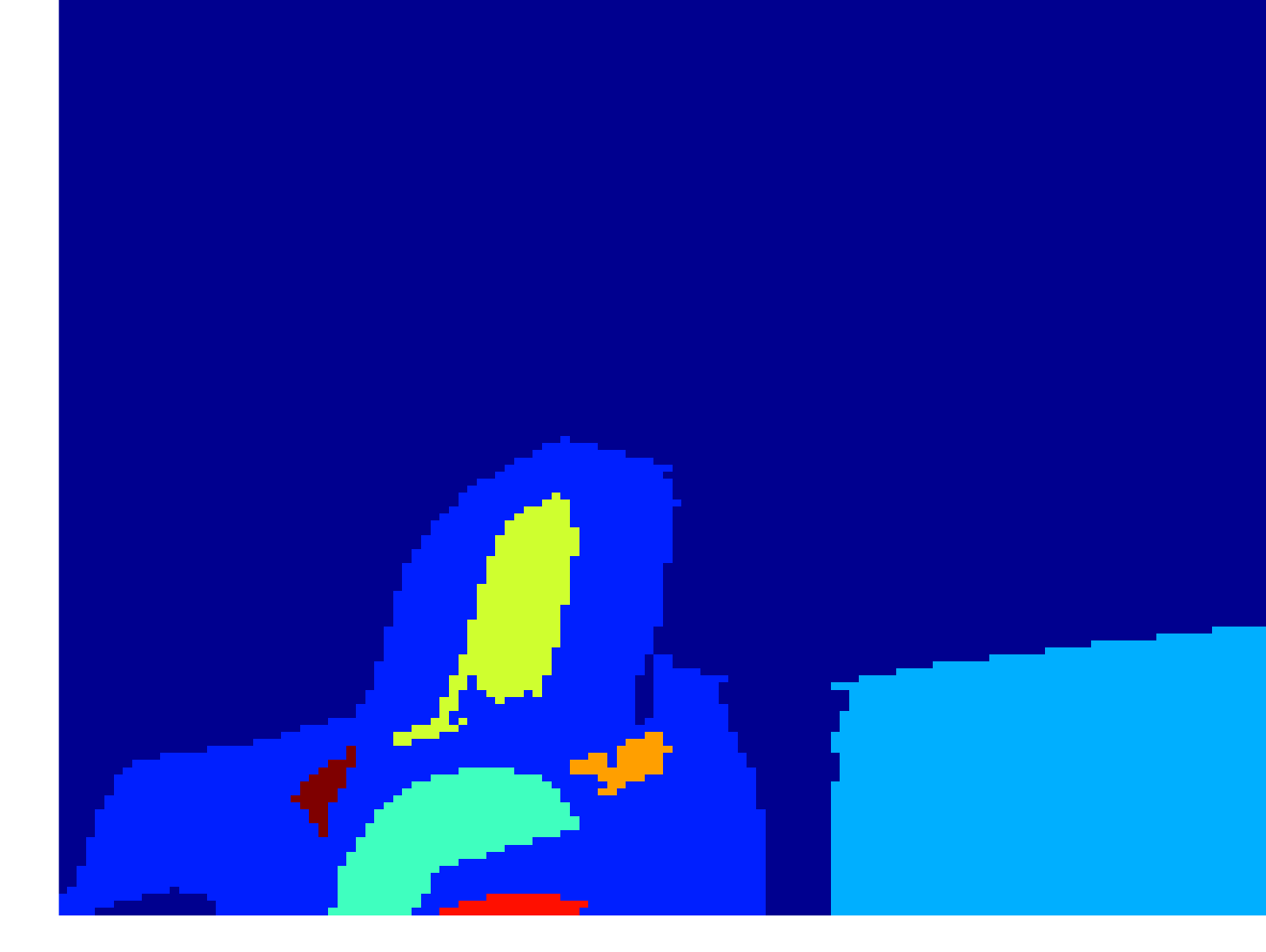} \\ \vfill
			\includegraphics[width=\textwidth]{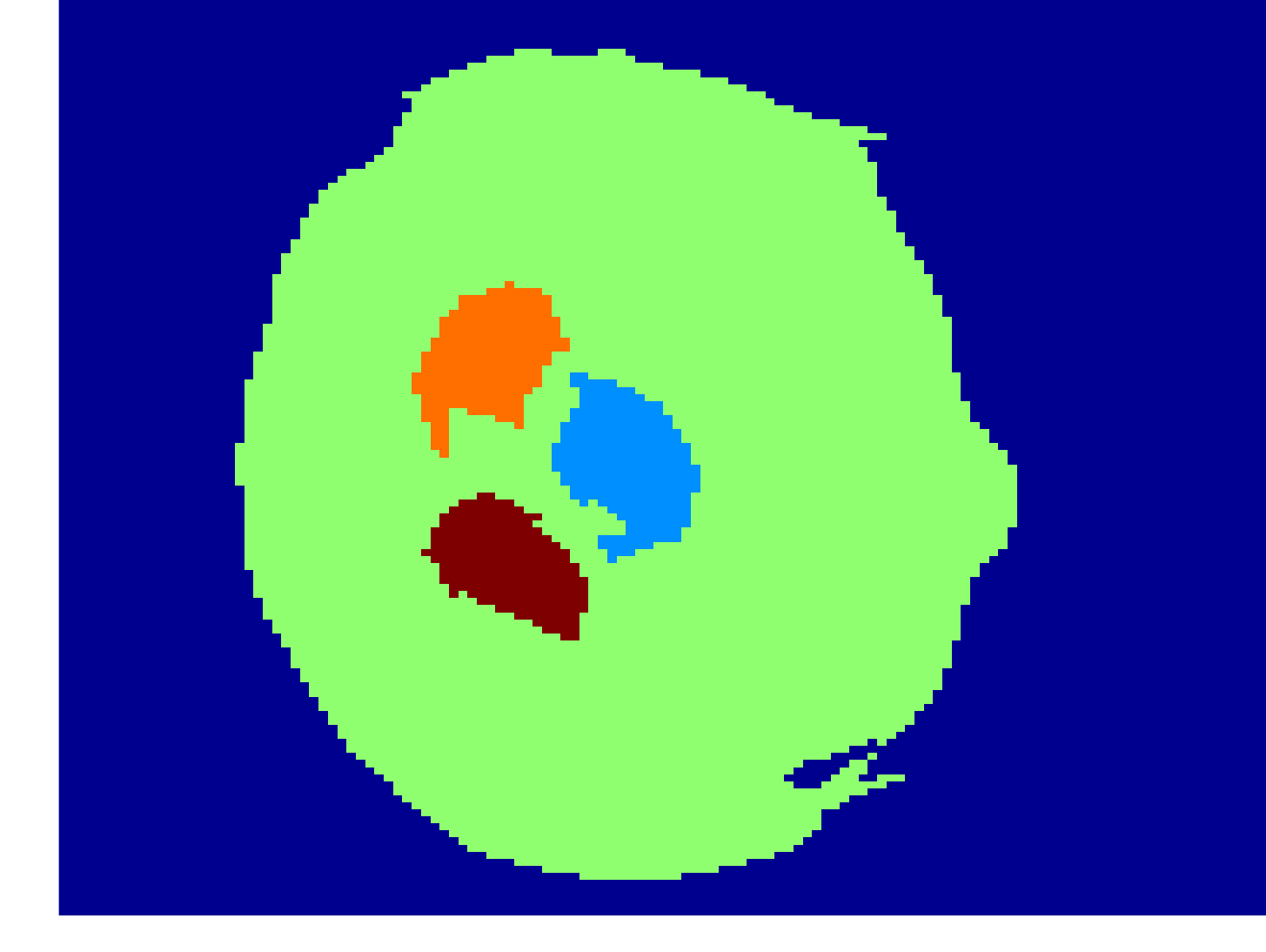} \\ \vfill
			\includegraphics[width=\textwidth]{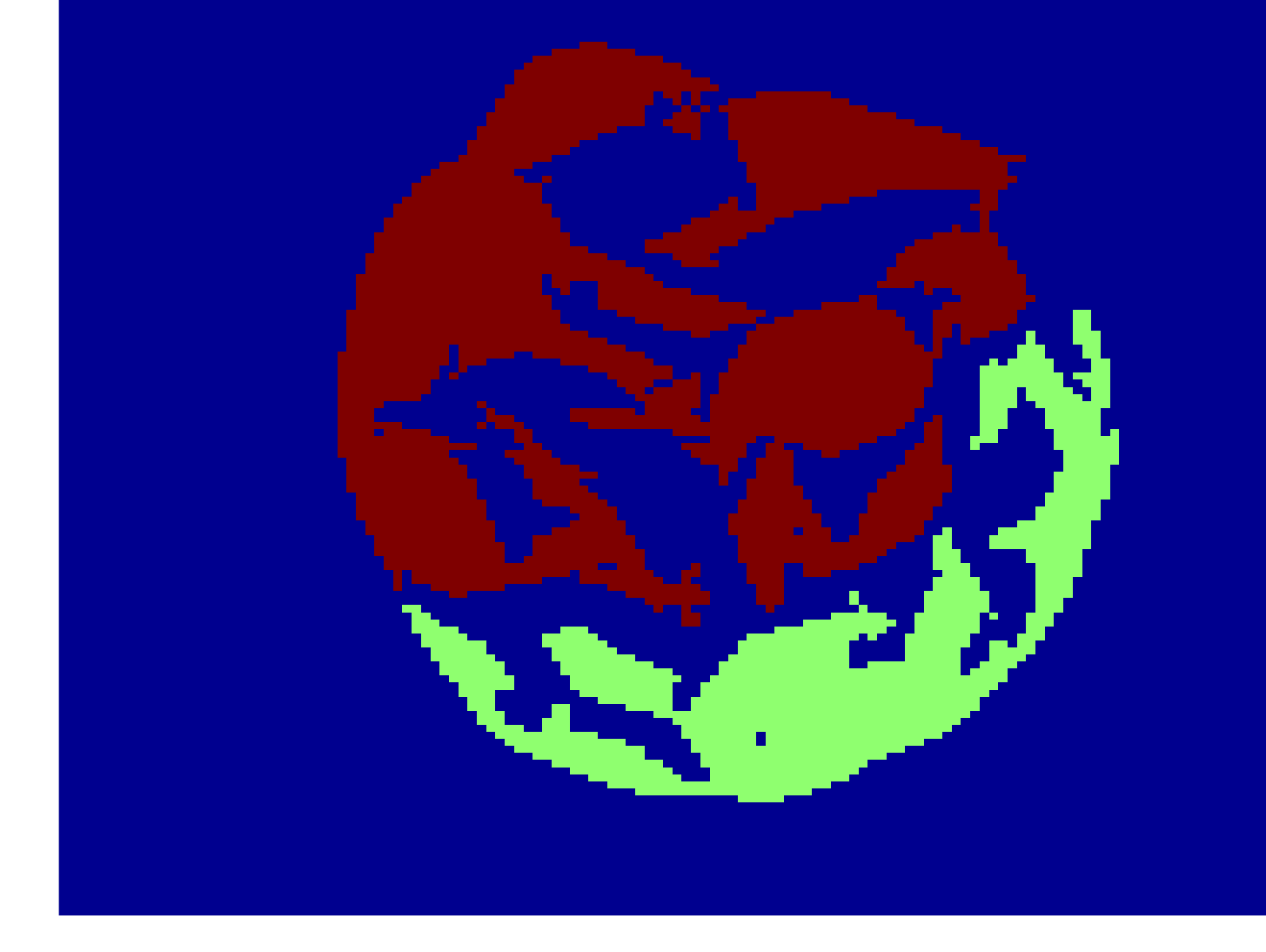} \\ \vfill
			\includegraphics[width=\textwidth]{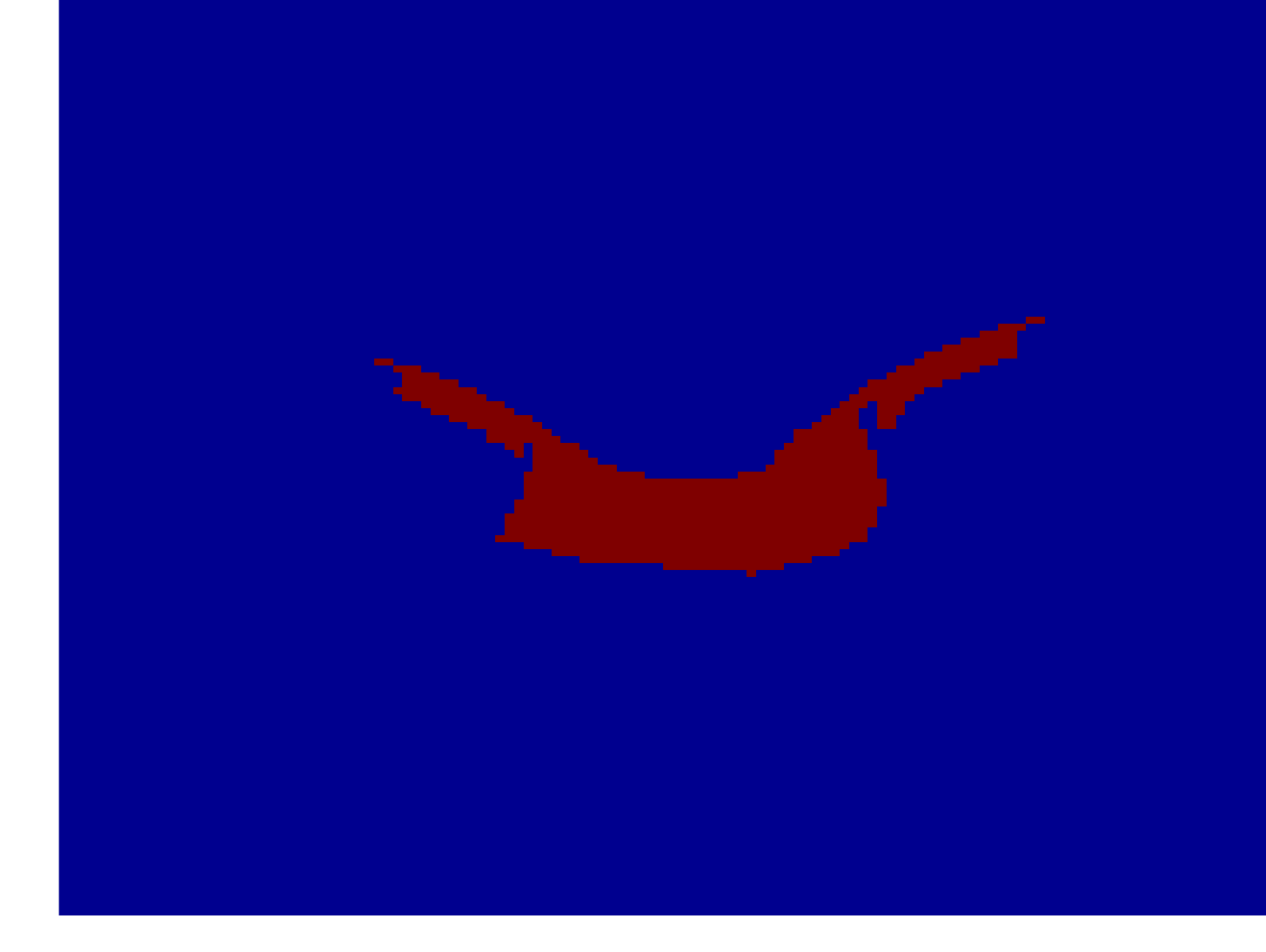} \\ \vfill
			\includegraphics[width=\textwidth]{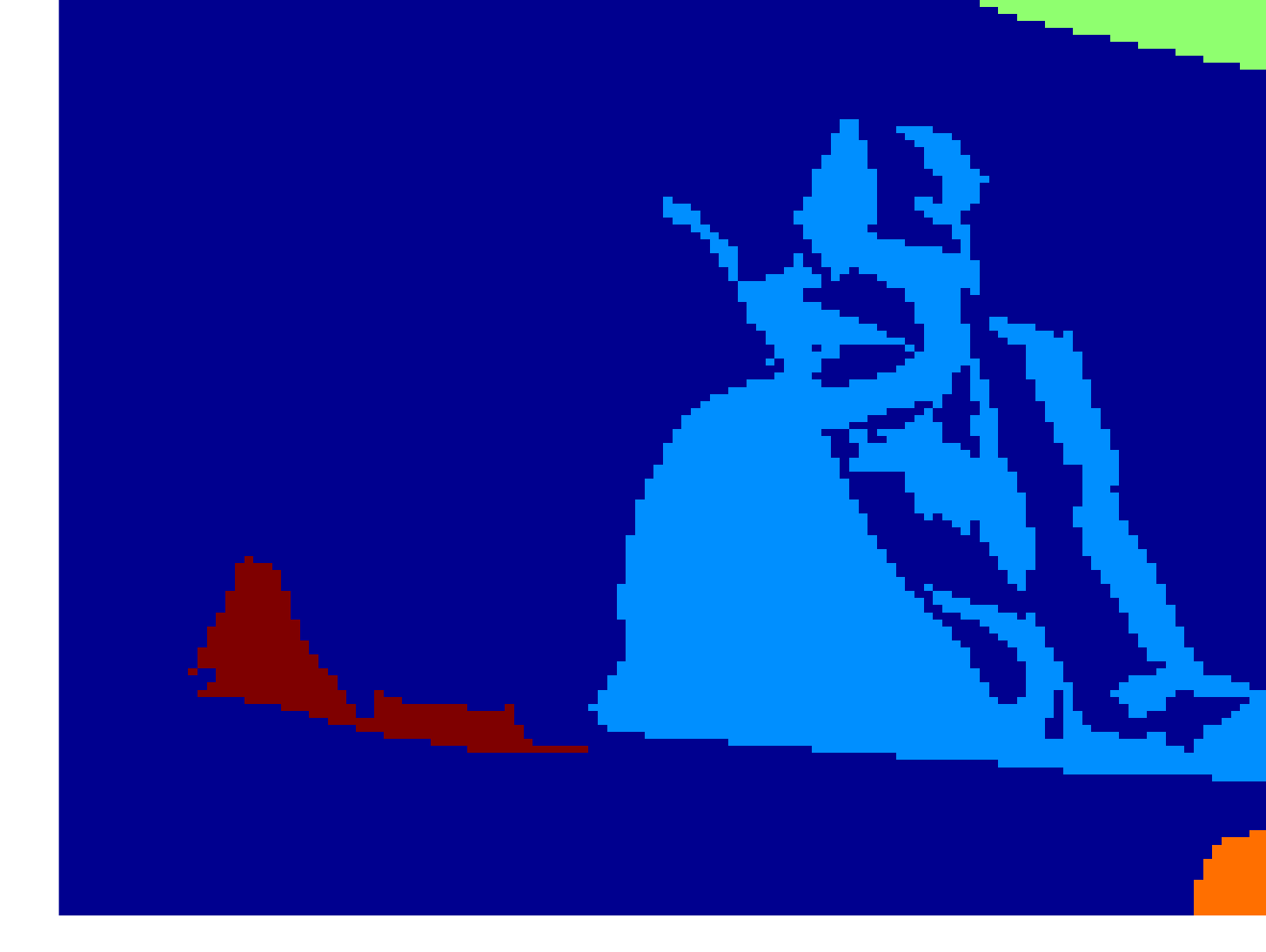} \\ \vfill
			\includegraphics[width=\textwidth]{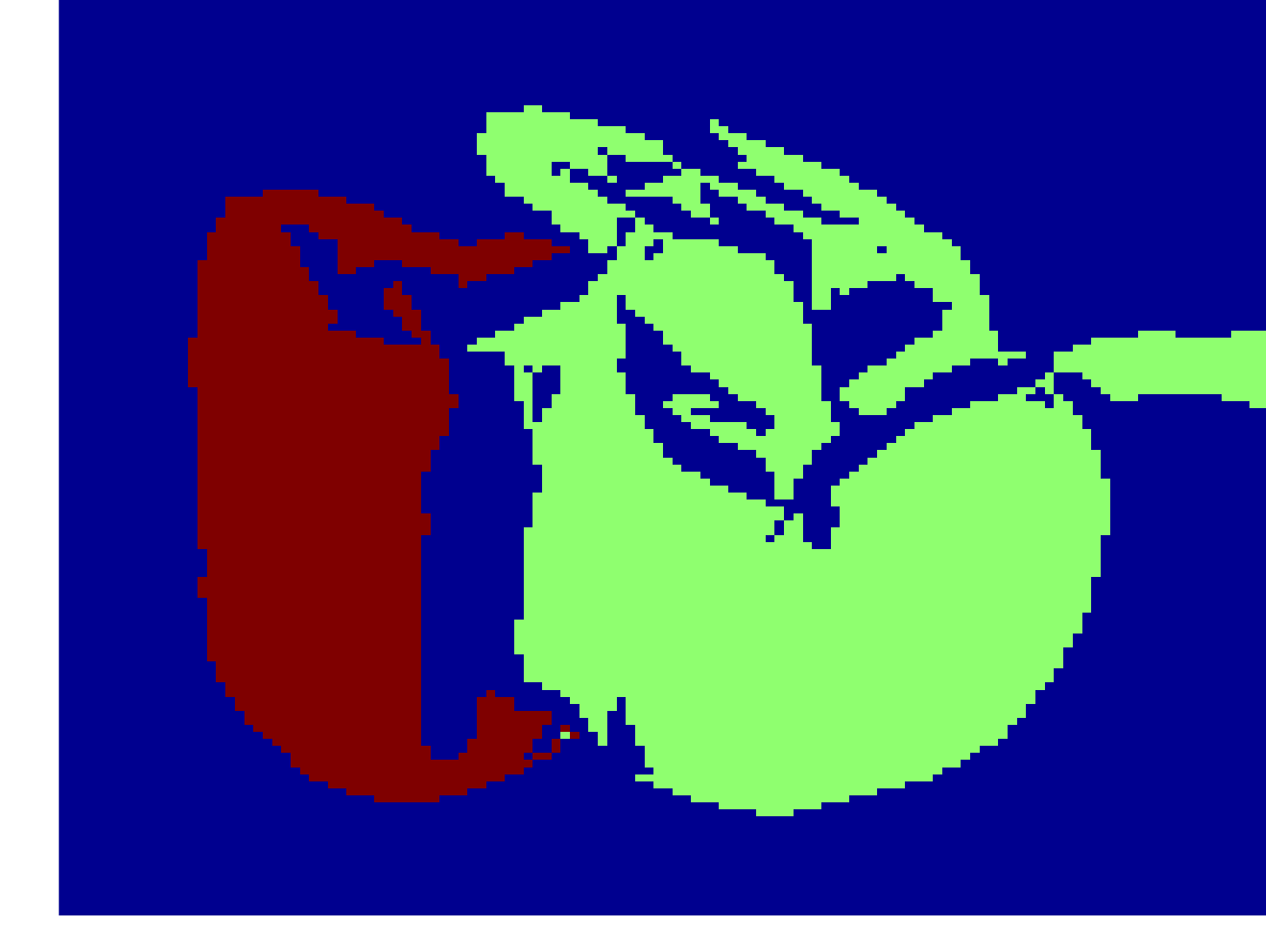} \\ \vfill
			\includegraphics[width=\textwidth]{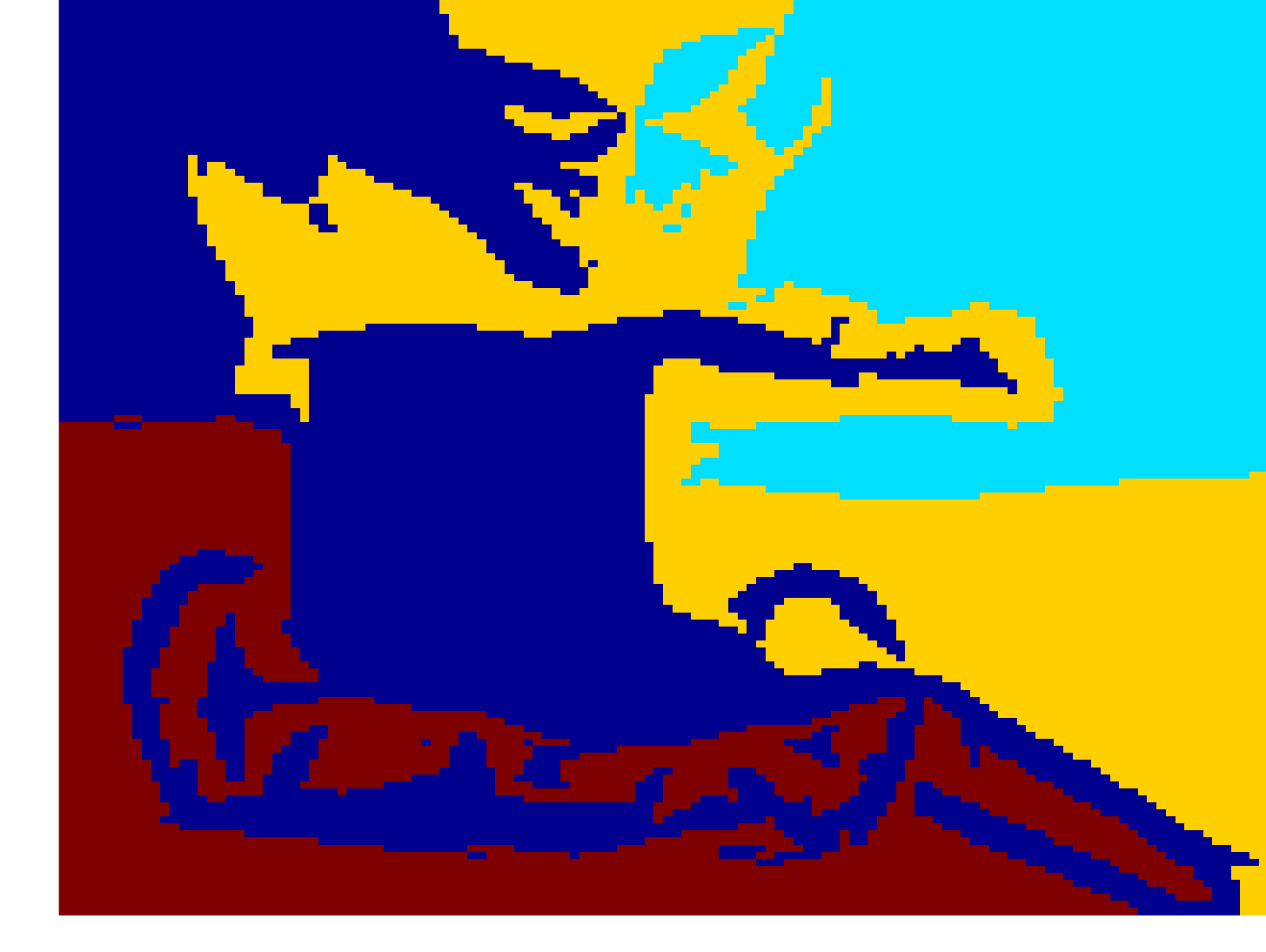} \\ \vfill
			\includegraphics[width=\textwidth]{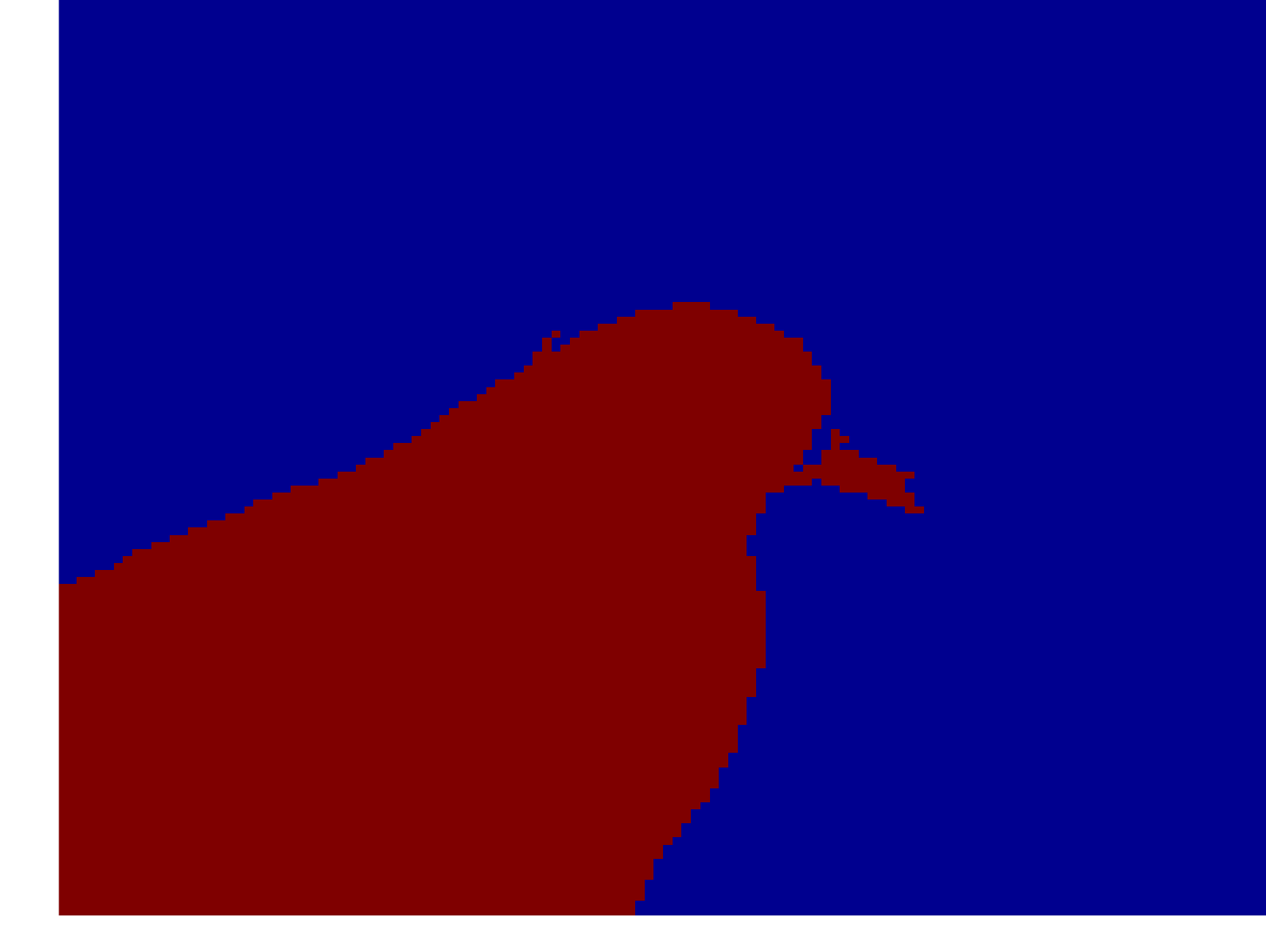}
		\end{minipage}
	}
	\subfloat[NCUT algorithm]{
		\begin{minipage}{.32\linewidth}
			\includegraphics[width=\textwidth]{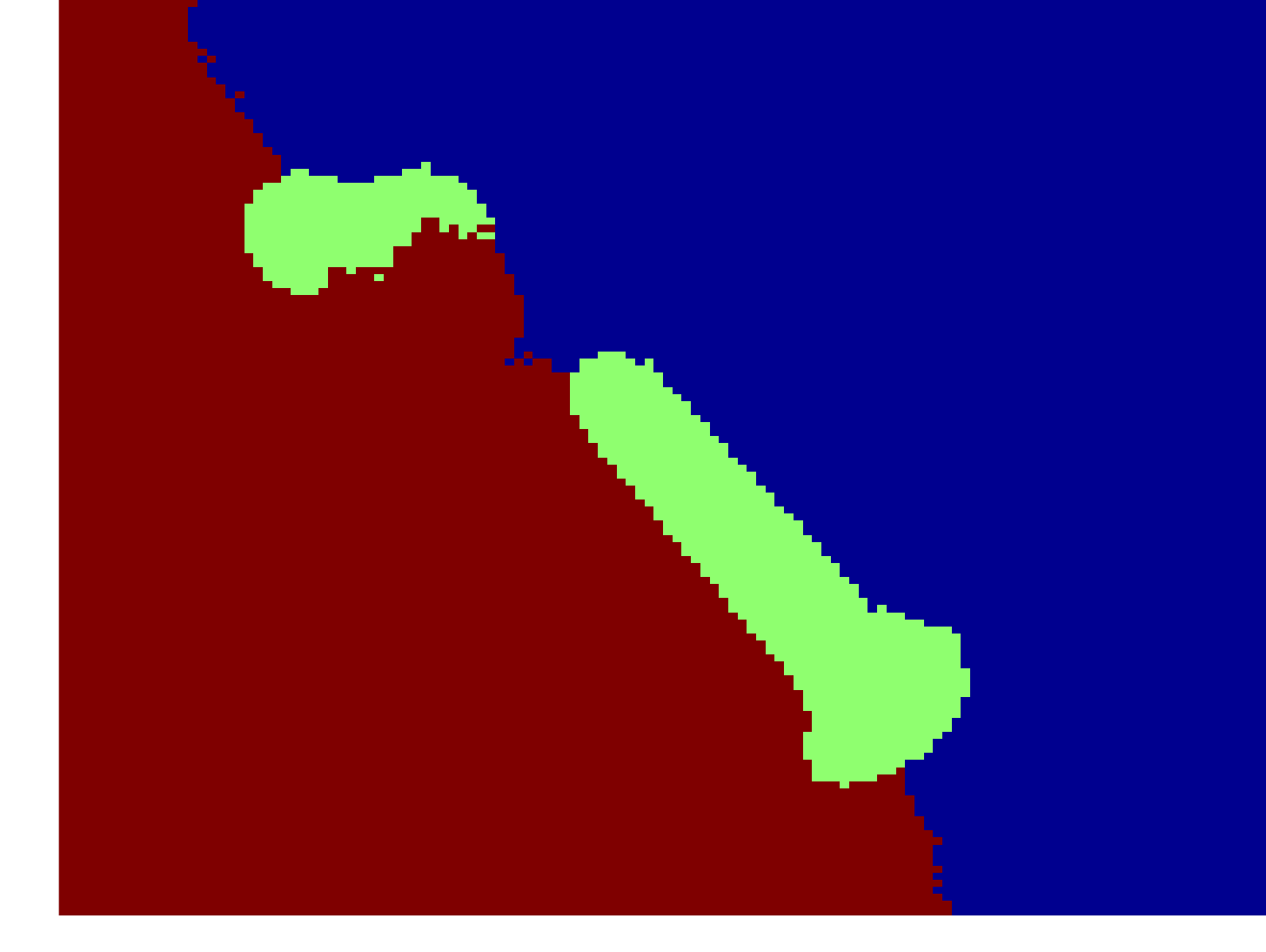} \\ \vfill
			\includegraphics[width=\textwidth]{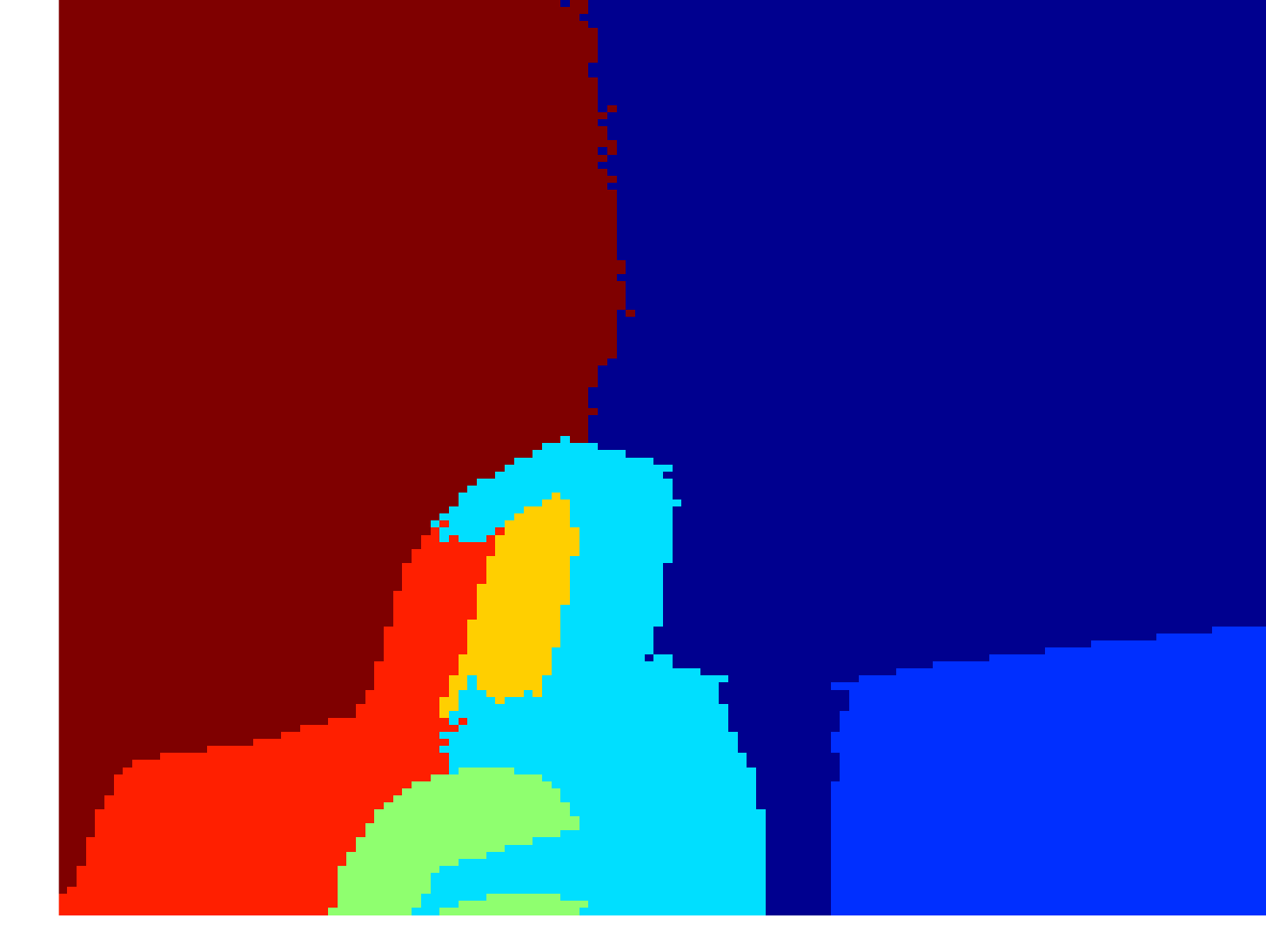} \\ \vfill
			\includegraphics[width=\textwidth]{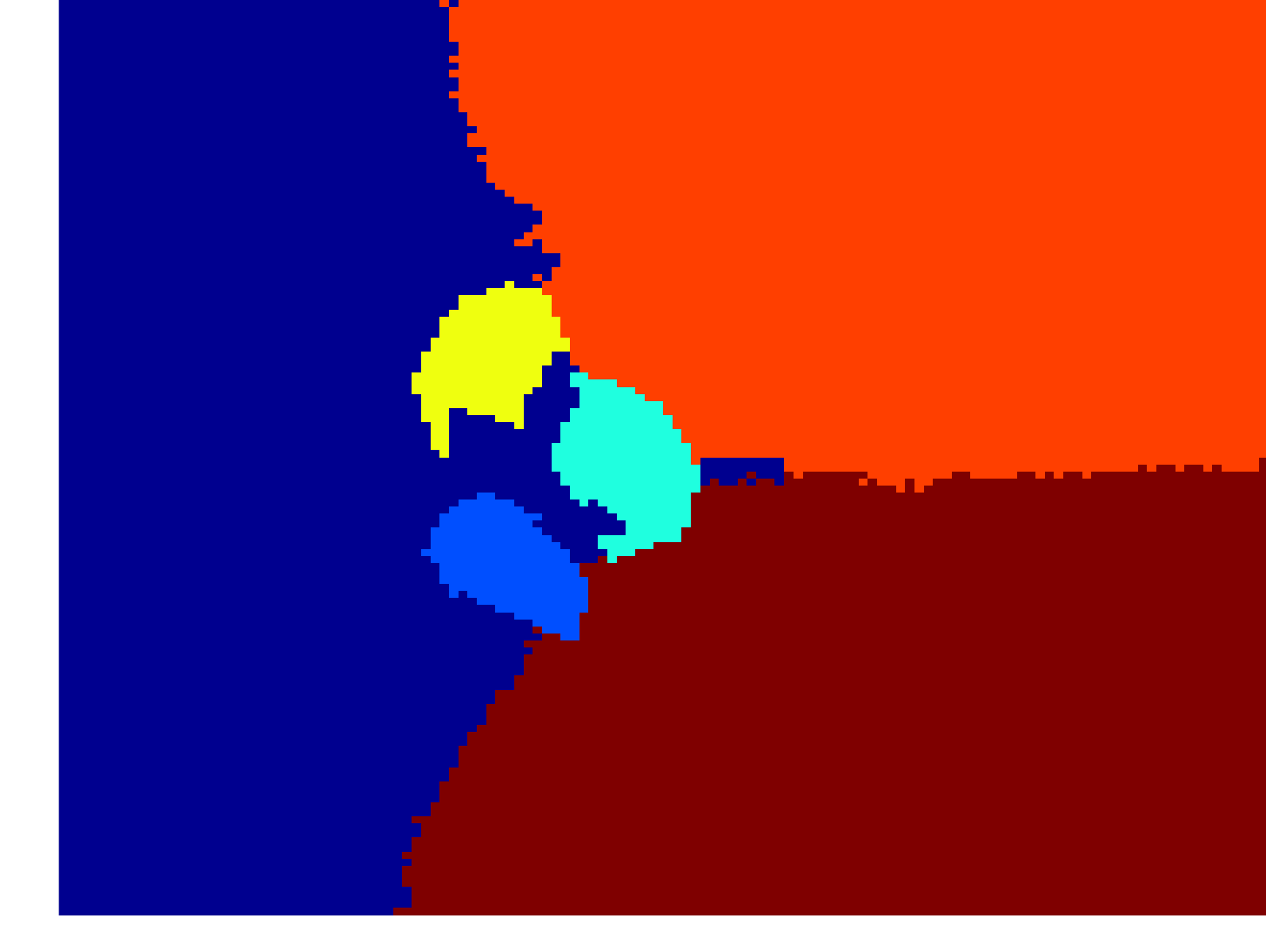} \\ \vfill
			\includegraphics[width=\textwidth]{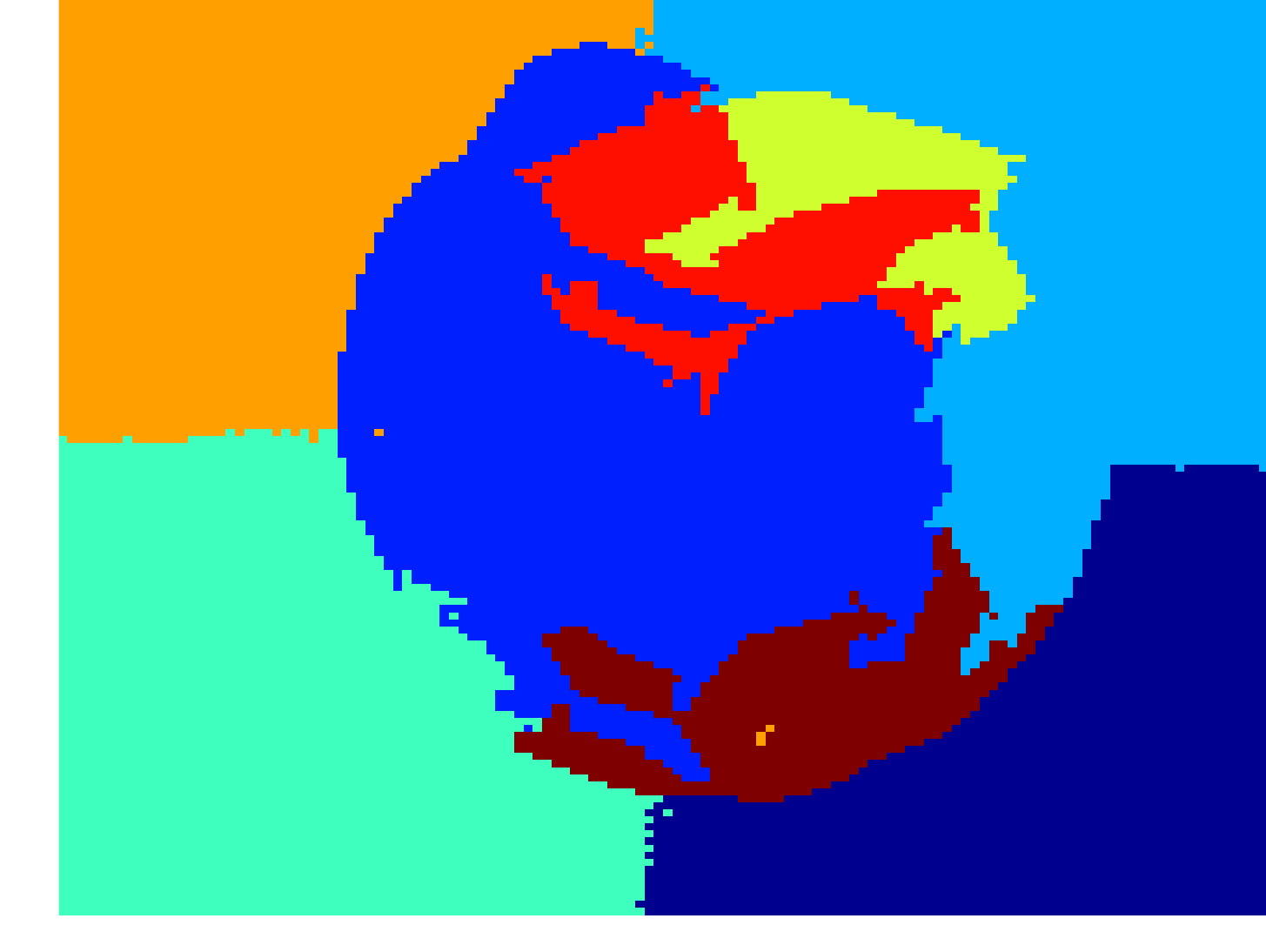} \\ \vfill
			\includegraphics[width=\textwidth]{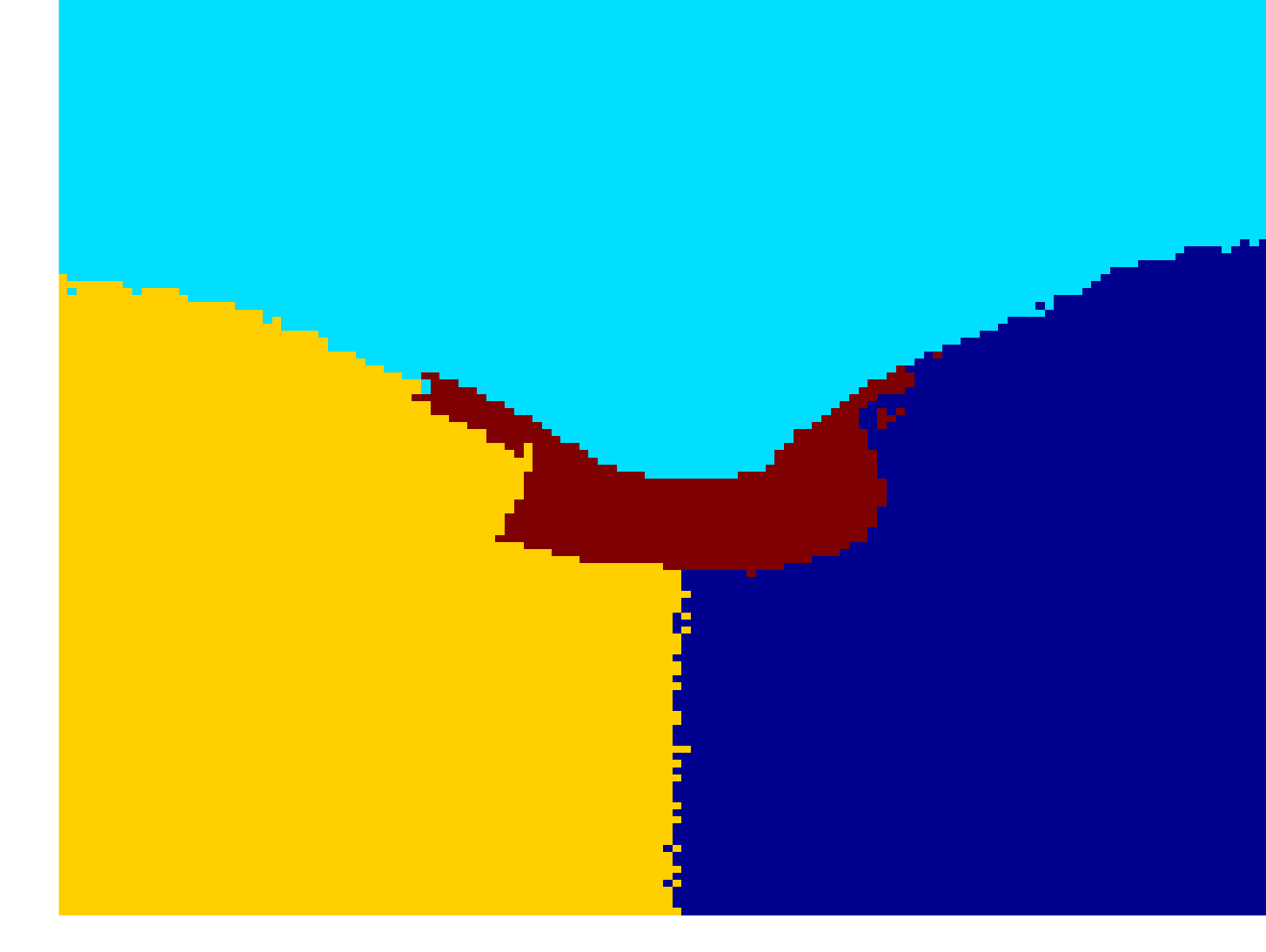} \\ \vfill
			\includegraphics[width=\textwidth]{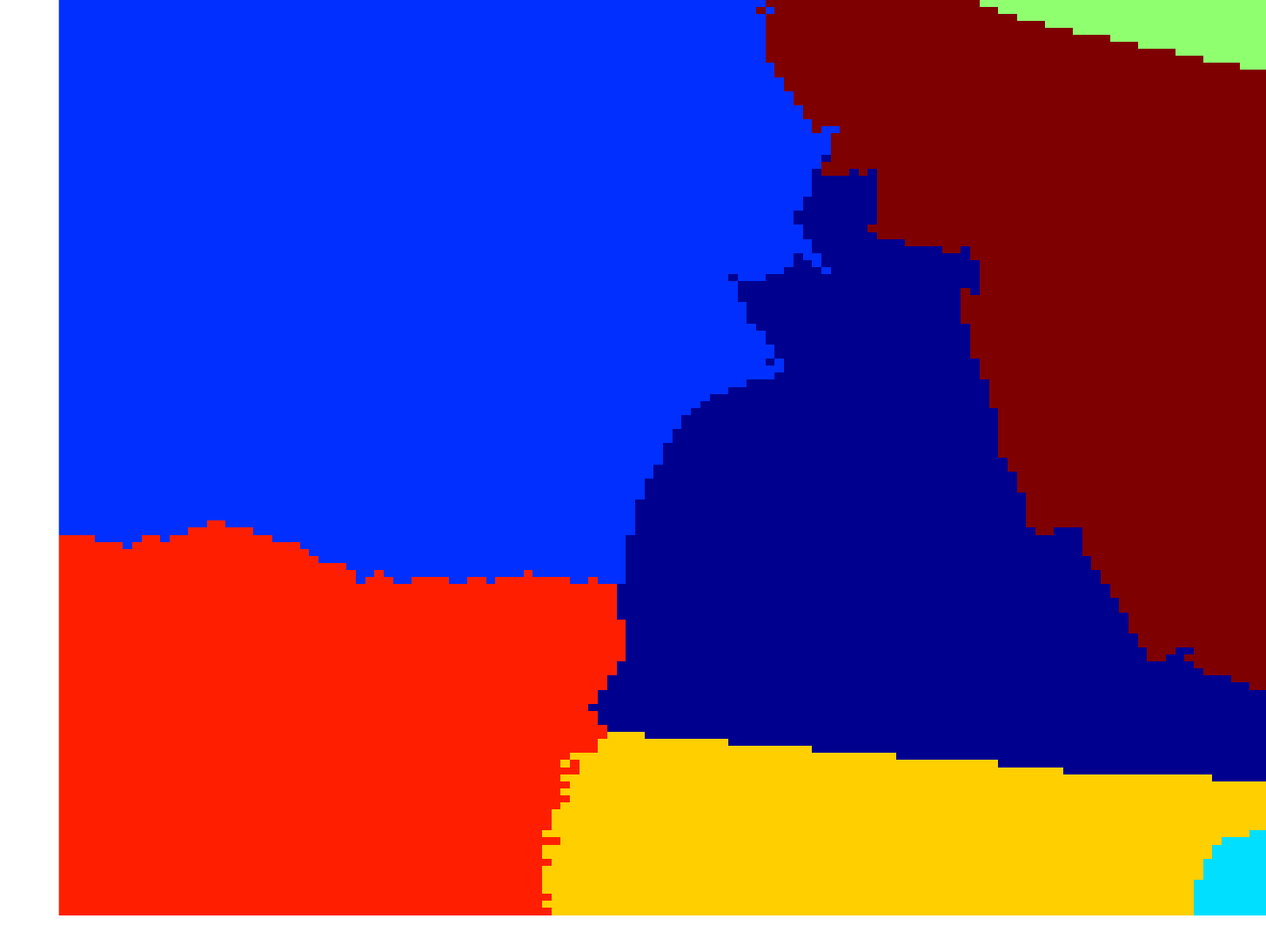} \\ \vfill
			\includegraphics[width=\textwidth]{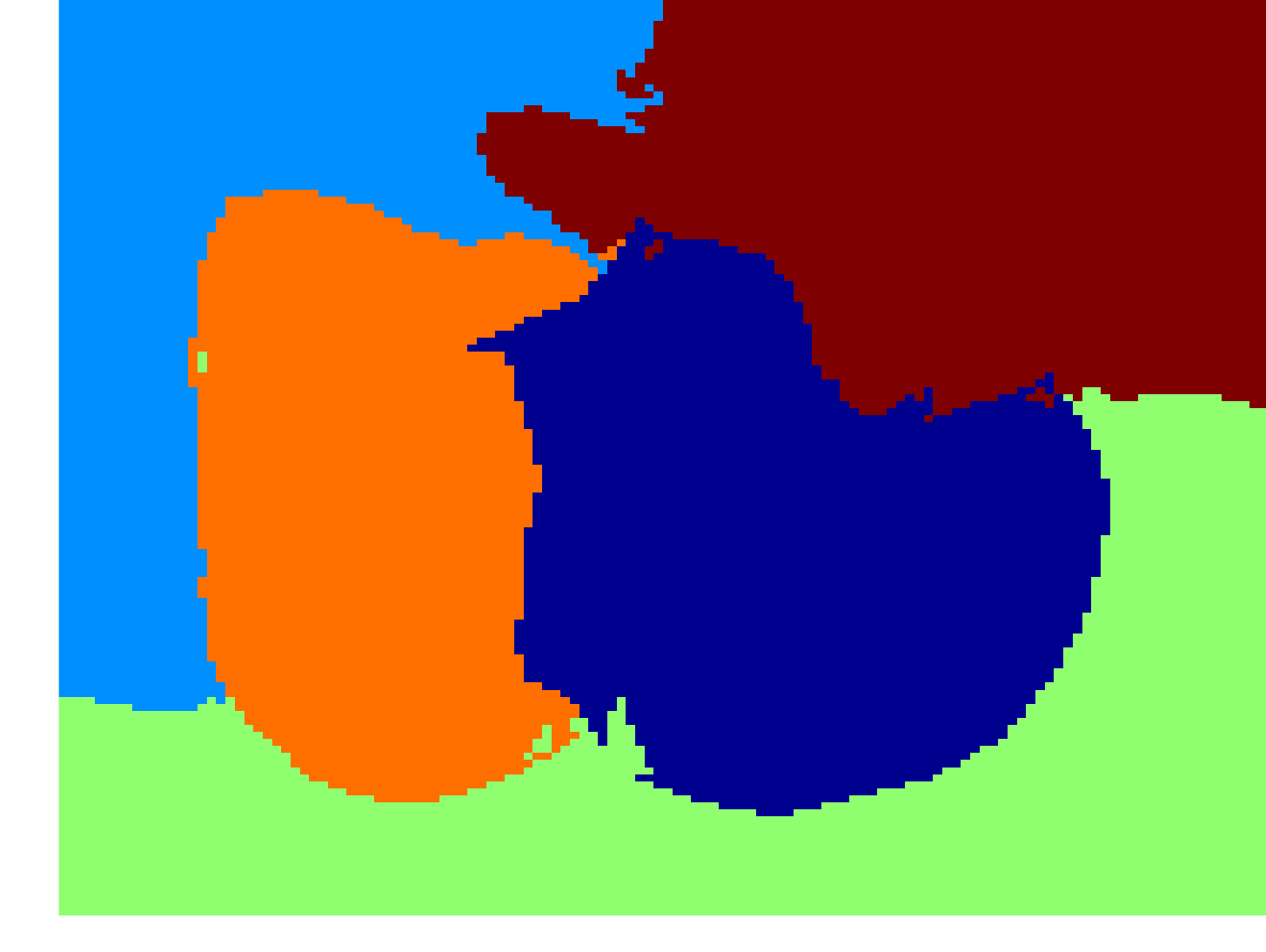} \\ \vfill
			\includegraphics[width=\textwidth]{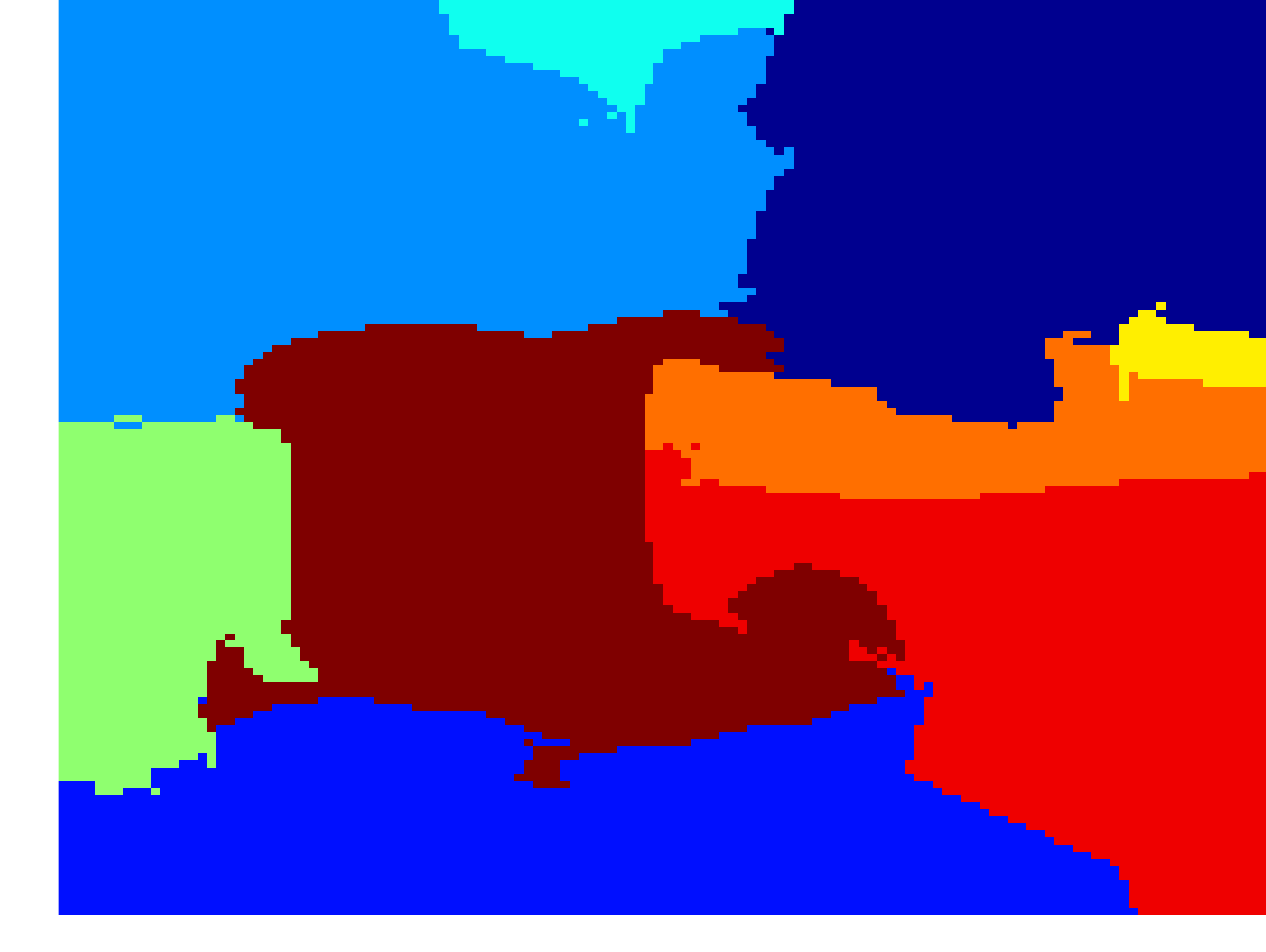} \\ \vfill
			\includegraphics[width=\textwidth]{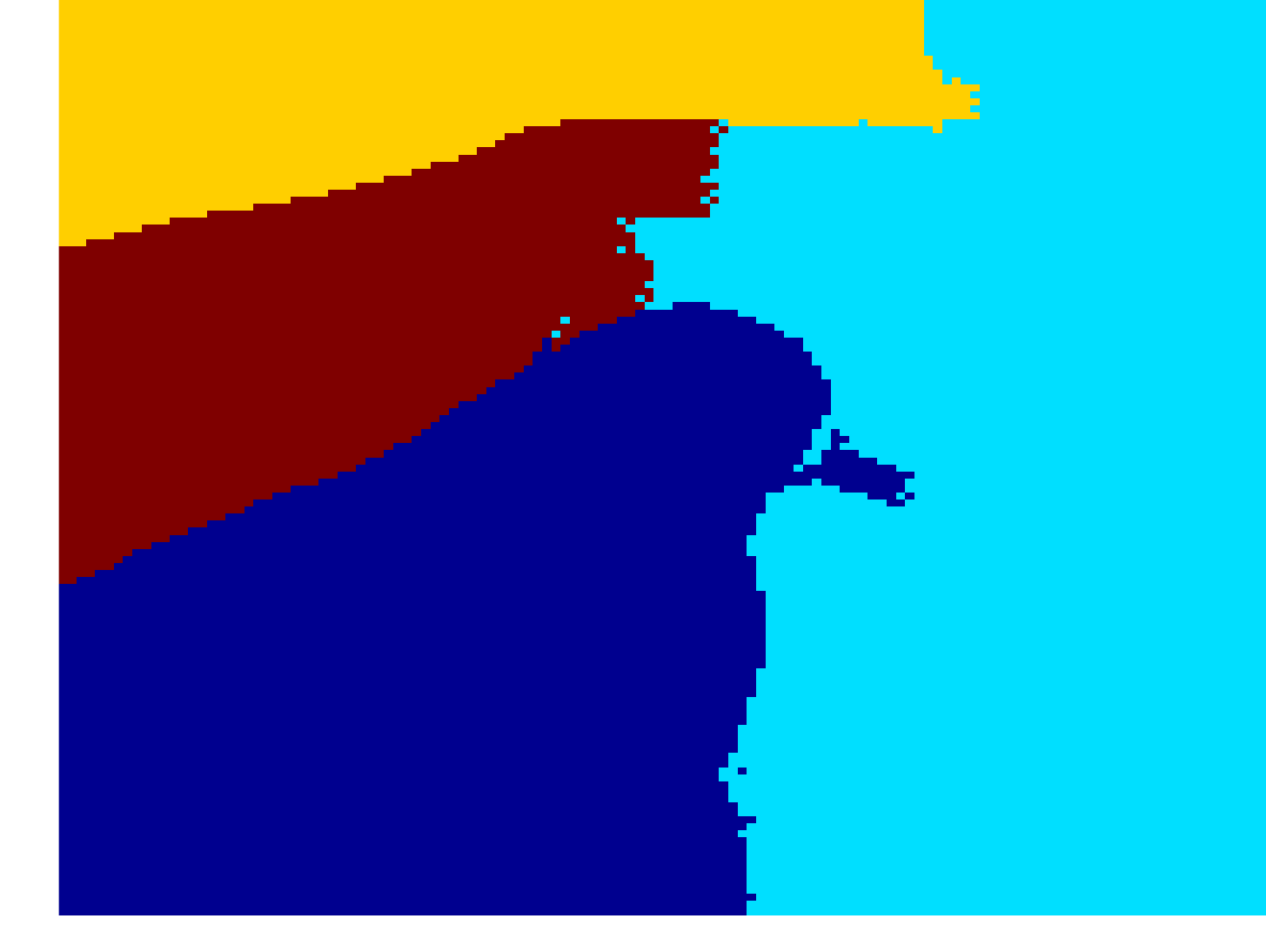}
		\end{minipage}
	}
	\caption{Image segmentation evaluation.}
	\label{image segmentation}
\end{figure}

\section{Conclusion}
\label{sec4}
As a fundamental procedure for data analysis, clustering has been widely used across the sciences. Existing clustering algorithms have limited effectiveness at high dimensions and often require tuning parameters for different domains and datasets~\cite{Shah9814}. Many real-world problems, such as images, videos, text, DNA microarray data, etc., are actually collections of high-dimensional data. Usually, such high-dimensional data lie close to low-dimensional structures corresponding to several classes or categories to which the data belong~\cite{6482137}. Because data points in a subspace are often distributed arbitrarily and not around a centroid, standard clustering methods, like k-means, are no longer generally applicable.

To solve the above problem, we proposed a GOPC algorithm based on the minimax distance. The minimax distance is a special case of ultrametrics in which any three points can form an isosceles triangle. Based on this property, we find that medoids have special constraints for the minimax distance; see \Cref{P1}, \Cref{P2}, and \Cref{P3} for details. Our main contribution in this paper is the proposal of a method that gives a global optimal solution for the objective function defined in \cref{criterion}. Under the condition $nn(x)\in M$, many unqualified candidates were filtered out, which made the entire clustering process much faster than $O(k\times n^2)$.

The minimax distance has a close relationship with the MST, but it is more flexible and convenient to use than the MST. Compared with Euclidean distance, the minimax distance can capture the global structure of data. Therefore, the GOPC algorithm is good at identifying arbitrarily shaped clusters. In addition, the GOPC algorithm needs only one parameter, i.e., the number of clusters, which arises intuitively with the help of a decision graph. The advantages mentioned above make the GOPC algorithm a good candidate for a general clustering algorithm. However, the GOPC algorithm cannot separate connected clusters. Thus, a general and efficient robust minimax distance is our main research area in future work.

\ifCLASSOPTIONcompsoc
  \section*{Acknowledgments}
\else
  \section*{Acknowledgment}
\fi

This work is supported by the National Natural Science Foundation of China (Project No.61906174).

\ifCLASSOPTIONcaptionsoff
  \newpage
\fi



%


\bibliographystyle{./IEEEtran}
\bibliography{sample}

\vfill


\end{document}